\documentclass{article}
\usepackage{setspace}
\usepackage[hyperref,preprint]{jmlr2e}
\usepackage{mathtools}
\usepackage{array}
\usepackage{booktabs}
\usepackage{tabularx}
\usepackage{tikz}
\usetikzlibrary{positioning,calc}

\hypersetup{
	pdftitle={Eigenvalues of the Hessian in Deep Learning: The Origin of Symmetry and Its Breaking},
	pdfauthor={Yossi Arjevani},
	colorlinks=true,
	linkcolor=blue,
	citecolor=blue,
	urlcolor=blue,
	linktoc=all
}
\usepackage{amsmath,bbm}

\usepackage{placeins}
\usepackage{comment}
\usepackage{wrapfig}
\usepackage{pdflscape}
\newcommand{\qed}{\nobreak\hfill$\Box$\par}

\newcommand{\bmu}{\boldsymbol{\mu}}

\newcommand{\std}{{\mathrm{std}}_{O_{W_3}}}
\newcommand{\stdih}{\mathrm{Std}_{D_4}}
\newcommand{\semicircle}{\scalebox{0.53}{$\varnothing$}}

\newcommand{\stdsd}{\fs}
\newcommand{\triv}{\ft}

\newcommand{\dataspace}{Z}
\newcommand{\prompt}{\cP}
\newcommand{\mat}{\mathrm{Mat}}

\newcommand{\dd}{\mu}

\newcommand{{\qgn}}{d^{\mathrm{GN}}}
\newcommand{{\hgn}}{D^{\mathrm{GN}}}

\usepackage{multirow} %
\usepackage{graphicx,amssymb,colordvi,textcomp,latexsym, mathtools, enumitem, breqn} %
\usepackage{makecell}
\usepackage{booktabs}

\newcommand{\relu}{\mathrm{ReLU}}

\newcommand{\htodo}[1]{}

\newtheorem{fact}{Fact}

\def\ddefloop#1{\ifx\ddefloop#1\else\ddef{#1}\expandafter\ddefloop\fi}
\def\ddef#1{\expandafter\def\csname 
bb#1\endcsname{\ensuremath{\mathbb{#1}}}}
\ddefloop ABCDEFGHIJKLMNOPQRSTUVWXYZ\ddefloop
\def\ddef#1{\expandafter\def\csname 
#1\endcsname{\ensuremath{\mathbf{#1}}}}
\ddefloop abcdefghijklmnopqrstuvwxyz\ddefloop
\def\ddefloop#1{\ifx\ddefloop#1\else\ddef{#1}\expandafter\ddefloop\fi}
\def\ddef#1{\expandafter\def\csname 
b#1\endcsname{\ensuremath{\mathbb{#1}}}}
\ddefloop ABCDEFGHIJKLMNOPQRSTUVWXYZ\ddefloop
\def\ddef#1{\expandafter\def\csname 
c#1\endcsname{\ensuremath{\mathcal{#1}}}}
\ddefloop ABCDEFGHIJKLMNOPQRSTUVWXYZ\ddefloop
\def\ddef#1{\expandafter\def\csname 
f#1\endcsname{\ensuremath{\mathfrak{#1}}}}
\ddefloop ABCDEFGHIJKLMNOPQRSTUVWXYZabcdefghjklmnopqrstuvwxyz\ddefloop
\def\ddef#1{\expandafter\def\csname 
h#1\endcsname{\ensuremath{\widehat{#1}}}}
\ddefloop ABCDEFGHIJKLMNOPQRSTUVWXYZ\ddefloop
\def\ddef#1{\expandafter\def\csname 
hc#1\endcsname{\ensuremath{\widehat{\mathcal{#1}}}}}
\ddefloop ABCDEFGHIJKLMNOPQRSTUVWXYZ\ddefloop
\def\ddef#1{\expandafter\def\csname
sf#1\endcsname{\ensuremath{\mathsf{#1}}}}
\ddefloop abcdrs\ddefloop

\newcommand{\GL}{{\rm GL}}

\newcommand{\balpha}{\boldsymbol{\alpha}}
\newcommand{\blambda}{\boldsymbol{\lambda}}

\newcommand{\bepsilon}{\boldsymbol{\epsilon}}

\newcommand{\tr}{\mathrm{Tr}}
\newcommand{\codim}{\mathrm{codim}}

\newcommand{\bones}{\mathbf{1}}
\newcommand{\bzeros}{\mathbf{0}}

\newcommand{\bigO}[1]{\mathcal{O}{\left(#1\right)}}

\newcommand{\tens}{\otimes{}}
\newcommand{\defeq}{\coloneqq}

\DeclarePairedDelimiter{\prn}{(}{)}
\DeclarePairedDelimiter{\nrm}{\|}{\|}
\DeclarePairedDelimiter{\inner}{\langle}{\rangle}
\DeclarePairedDelimiter{\tri}{\langle}{\rangle}

\newcommand{\reals}{\mathbb{R}}

\usepackage{prettyref}
\newcommand{\pref}[1]{\prettyref{#1}}

\newcommand{\savehyperref}[2]{\texorpdfstring{\hyperref[#1]{#2}}{#2}}
\newrefformat{eq}{\savehyperref{#1}{\textup{(\ref*{#1})}}}
\newrefformat{setting}{\savehyperref{#1}{Learning Problem~\ref*{#1}}}
\newrefformat{eqn}{\savehyperref{#1}{Equation~\ref*{#1}}}
\newrefformat{lem}{\savehyperref{#1}{Lemma~\ref*{#1}}}
\newrefformat{obs}{\savehyperref{#1}{Observation~\ref*{#1}}}
\newrefformat{def}{\savehyperref{#1}{Definition~\ref*{#1}}}
\newrefformat{thm}{\savehyperref{#1}{Theorem~\ref*{#1}}}
\newrefformat{corr}{\savehyperref{#1}{Corollary~\ref*{#1}}}
\newrefformat{sec}{\savehyperref{#1}{Section~\ref*{#1}}}
\newrefformat{app}{\savehyperref{#1}{Appendix~\ref*{#1}}}
\newrefformat{ass}{\savehyperref{#1}{Assumption~\ref*{#1}}}
\newrefformat{ex}{\savehyperref{#1}{Example~\ref*{#1}}}
\newrefformat{fig}{\savehyperref{#1}{Figure~\ref*{#1}}}
\newrefformat{alg}{\savehyperref{#1}{Algorithm~\ref*{#1}}}
\newrefformat{rem}{\savehyperref{#1}{Remark~\ref*{#1}}}
\newrefformat{conj}{\savehyperref{#1}{Conjecture~\ref*{#1}}}
\newrefformat{prop}{\savehyperref{#1}{Proposition~\ref*{#1}}}
\newrefformat{table}{\savehyperref{#1}{Table~\ref*{#1}}}
\newrefformat{prob}{\savehyperref{#1}{Problem~\ref*{#1}}}
\newrefformat{claim}{\savehyperref{#1}{Claim~\ref*{#1}}}
\newrefformat{fact}{\savehyperref{#1}{Fact~\ref*{#1}}}
\newrefformat{sec}{\savehyperref{#1}{Section~\ref*{#1}}}
\newrefformat{ftnote}{\savehyperref{#1}{Footnote~\ref*{#1}}}
\newrefformat{exam}{\savehyperref{#1}{Example~\ref*{#1}}}
\newrefformat{exams}{\savehyperref{#1}{Examples~\ref*{#1}}}
\newrefformat{prob}{\savehyperref{#1}{(\ref*{#1})}}

\newcommand{\stab}{\operatorname{Stab}}

\usepackage{setspace}

\newcommand{\RR}{\mathbb{R}}
\newcommand{\EE}{\mathbb{E}}

\newcommand{\PP}{\mathbb{P}}

\newcommand{\NN}{\mathbb{N}}

\newcommand{\diag}{\mathrm{diag}}
\newcommand{\rk}{{\rm rk}}
\newcommand{\ploss}{{\cL}} 
\newcommand{\Od}[1]{{{\text{O}(#1)}}}

\newcommand{\bxi}{{\boldsymbol{\xi}}}

\newcommand{\grad}[1]{{\nabla {#1}}}

\newcommand{\sgn}{\mbox{sgn}}

\usepackage{breqn}

\newcommand{\tanset}{\mho}
\newcommand{\aut}{\mathrm{Aut}}

\usepackage[utf8]{inputenc} %
\usepackage[T1]{fontenc}    %
\usepackage{hyperref}       %
\usepackage{url}            %
\usepackage{booktabs}       %
\usepackage{amsfonts}       %
\usepackage{nicefrac}       %
\usepackage{microtype}      %
\usepackage{xcolor}         %

\newcommand{\weightspace}{\Theta}

\newcommand{\inspace}{X}
\newcommand{\outspace}{Y}
\newcommand{\universe}{\bA}
	
\newcommand{\id}{\mathrm{id}}

\newcommand{\grpstyle}[1]{\mathsf{#1}}
\newcommand{\model}{\grpstyle{L}}
\newcommand{\struc}{\grpstyle{S}}

\newcommand{\dist}{\grpstyle{D}}
\newcommand{\regis}{\grpstyle{R}}
\newcommand{\Model}{\Lambda}
\newcommand{\Struc}{\Sigma}
\newcommand{\Obj}{\Gamma}
\newcommand{\twoslashcircle}{%
	\mathord{\vcenter{\hbox{%
				\tikz[x=0.03em,y=0.03em,line width=0.02em,line cap=round]{
					\draw (0,0) circle (5);
					\draw (-5.8,-4.2) -- (4.2,5.8);
					\draw (-4.2,-5.8) -- (5.8,4.2);
				}%
	}}}%
}

\newcommand{\Sph}{\mathbb S}

\usepackage{tikz}					
\usepackage{ifthen}

\tikzset{%
	box/.style={
		align=center, minimum size=1cm, inner 
		sep=0pt, font=\tiny\bfseries,
		draw=black, #1
	},
	c1/.style={fill=yellow},
	c2/.style={fill=blue},
	c3/.style={fill=green},
	c3/.style={fill=red},
}

\newcommand{\blockmatrix}[9]{
	\draw[draw=#4,fill=#5] (0,0) rectangle( #1,#2);
	\ifthenelse{\equal{#6}{true}}
	{
		\draw[draw=#7,fill=#8] (0,#2) -- (#9,#2) -- ( #1,#9) -- ( #1,0) -- ( #1 
		- #9,0) -- (0,#2 -#9) -- cycle;
	}
	{}
	\draw ( #1/2, #2/2) node { #3};
}

\newcommand{\mblockmatrix}[4][none]{
	\begin{tikzpicture} 
		\ifthenelse{\equal{#1}{none}}
		{
			\blockmatrix{#2}{#3}{#4}{none}{none}{false}{none}{none}{0.0}
		}
		{
			\definecolor{fillcolor}{rgb}{#1}
			\blockmatrix{#2}{#3}{#4}{none}{fillcolor}{false}{none}{none}{0.0}
		}
	\end{tikzpicture}%
}

\newcommand{\fblockmatrix}[4][none]{
	\begin{tikzpicture} 
		\ifthenelse{\equal{#1}{none}}
		{
			\blockmatrix{#2}{#3}{#4}{black}{none}{false}{none}{none}{0.0}
		}
		{
			\definecolor{fillcolor}{rgb}{#1}
			\blockmatrix{#2}{#3}{#4}{black}{fillcolor}{false}{none}{none}{0.0}
		}
	\end{tikzpicture}%
}

\newcommand{\diagonalblockmatrix}[5][none]{
	\begin{tikzpicture} 
		
		\ifthenelse{\equal{#1}{none}}
		{
			\blockmatrix{#2}{#3}{#4}{black}{none}{true}{black}{none}{#5}
		}
		{
			\definecolor{fillcolor}{rgb}{#1}
			\blockmatrix{#2}{#3}{#4}{black}{none}{true}{black}{fillcolor}{#5}
		}
		
	\end{tikzpicture}%
}
\newcommand{\rowhead}[2]{\begin{tabular}[t]{@{}l@{}}#1\\#2\end{tabular}}

\newcommand{\netf}{\widetilde{N}}

\newcommand{\spT}{\cT_\star}
\DeclareMathOperator{\Unif}{Unif}
\DeclareMathOperator{\op}{op}

\newcommand{\E}{\mathbb E}
\newcommand{\Prob}{\mathbb P}

\newcommand{\sB}{\mathcal B}

\newcommand{\mset}[1]{\{\!\{#1\}\!\}}
\newcommand{\Mset}[1]{ \left\{\!\left\{#1\right\}\!\right\} }

\newcommand{\smloss}{\kappa}
\newcommand{\sqloss}{\mathtt{sq}}

\newcommand{\celoss}{\mathtt{ce}}
\newcommand{\Aff}{\mathrm{Aff}}
\newtheorem{setting}{Learning problem}
\newcommand{\convmap}{\Xi}

\usepackage{xspace}
\newcommand{\Integrated}{Composite\xspace}
\newcommand{\integrated}{composite\xspace}
\newcommand{\init}{\grpstyle{I}}
\newcommand{\avg}[1]{\overline{#1}}

\newcommand{\Span}{\operatorname{span}}\usepackage{lastpage}
\jmlrheading{}{2026}{}{}{}{}{Yossi Arjevani}

\ShortHeadings{The Origin of Symmetry and Its Breaking}{}
\firstpageno{1}
\begin{document}

\title{Eigenvalues of the Hessian in Deep Learning: \\
The Origin of Symmetry and Its Breaking}
\author{\name Yossi Arjevani \email yossi.arjevani@gmail.com \\
}
\thispagestyle{empty}
\editor{}

\maketitle

\begin{abstract}

Hessian spectra at trained models in deep learning exhibit a persistent pattern: eigenvalues organize into distinct clusters, including a large bulk near zero and a few isolated outliers. This paper shows that a natural account of these spectral phenomena emerges when the original setting is understood as a departure from a nearby, otherwise hidden, highly symmetric reference.

Modifications, including changes to the architecture, data distribution, or  parameter metric, expose a nearby reference configuration whose Hessian exhibits rich invariances---ones not accounted for by weight symmetries. There, symmetry enables a precise description of the spectra, forcing high-dimensional kernels and eigenvalues of large multiplicity. Returning to the original configuration breaks the Hessian symmetry and 
thereby produces the observed hierarchy of clusters and~outliers.

The framework is developed in some generality, with a detailed analysis of three-layer ReLU networks and applications to convolutional, graph, and transformer models, as well as to the NTK. The same mechanism is further shown to yield analogous spectral structures in layerwise Hessians and the Gauss--Newton matrix.

\end{abstract}

\setcounter{tocdepth}{2} 
\tableofcontents
\vspace{0.8cm}


\definecolor{matplotlibblue}{HTML}{1F77B4}

\begin{figure}[ht]
	\centering
	\includegraphics[width=\textwidth]{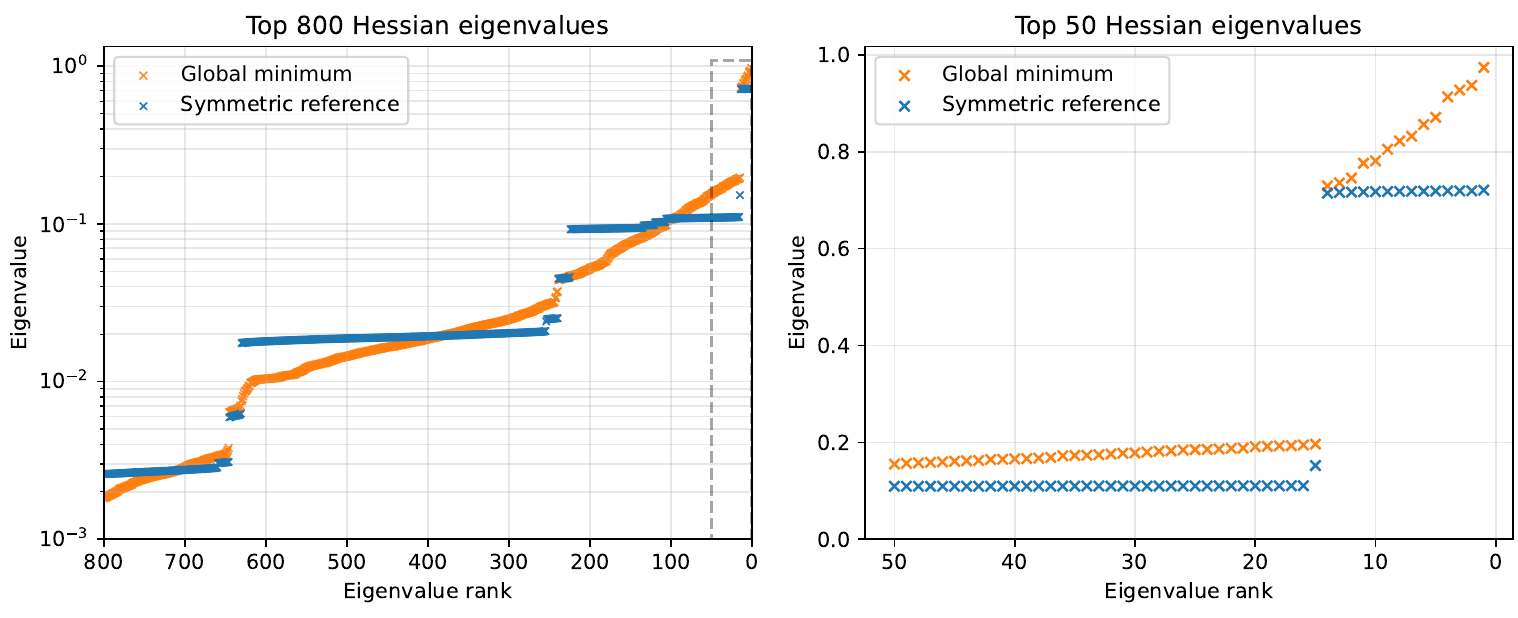}
	
	\smallskip
	{\footnotesize
		\setlength{\tabcolsep}{4pt}
		\renewcommand{\arraystretch}{1.15}
\begin{tabular}{c|p{.5\textwidth}|c}
	\multicolumn{3}{c}{
		\textbf{Spectrum at the {symmetric reference}}
	}
	\\[-1pt]
	\hline
	scale of eigenvalues
	&
	multiplicities forced by symmetry
	&
	total out of $8k^2$
	\\
	\hline
			$0$
			&
			one  eigenvalue with multiplicity $4k^2+4k$
			&
			$4k^2+4k$
			\\
			\hline
			$O(1/k)$
			&
			\begin{tabular}[t]{@{}l@{}}
				$3$ simple eigenvalues \\
				$6$ eigenvalues of multiplicity $k-1$ \\
				$4$ eigenvalues of multiplicity $\frac{k(k-3)}{2}$ \\
				$4$ eigenvalues of multiplicity $\frac{(k-1)(k-2)}{2}$
			\end{tabular}
			&
			$4k^2-6k+1$
			\\
			\hline
			$\Theta(1)$
			&
			\begin{tabular}[t]{@{}l@{}}
			one simple eigenvalue \\
			$2$ eigenvalues of multiplicity $k-1$
			\end{tabular}
			&
			$2k-1$
			\\
			\hline
		\end{tabular}

	}
	
\caption{
Hessian spectra of a three-layer network on a mixture of $k=15$ Gaussians with
random centers: at a global minimum of the original setting
(orange) and at a reference configuration (blue), close to the
original yet highly symmetric.\\
\textbf{At the symmetric reference.} The spectral structure falls into place. 
Symmetry forces the large kernel and the multiplicities in the table, and identifies invariant subspaces of directions weakly aligned with the data. For the family of global minima considered here, this results in only
\(2k-1\) eigenvalues remaining bounded away from~zero.\\
\textbf{Symmetry breaking.} Returning to the original configuration breaks
exact symmetry in a controlled manner. The resulting clusters arise as
perturbative remnants of the degenerate reference eigenvalues (left). \\
\textbf{Spectral multiplets.} A reference eigenvalue of multiplicity \(k-1\)
splits into \(k-1\) simple outliers (right), providing a mechanism for the
empirically observed scaling of the outlier count with the number of classes
\cite{sagun2016eigenvalues,sagun2017empirical}. With the parameter-space geometry taken into account, the spectra of a network trained on MNIST reveal the same organizing principle
(\pref{fig:mnist_n3_gn_metric_flow}).
	}
	\label{fig:original_vs_symmetric_mix}
\end{figure}

A recurring phenomenon in deep learning (DL) is the distinctive structure of the Hessian spectrum at trained networks. Empirical investigations \cite{lecun1998efficient,dauphin2014identifying,sagun2016eigenvalues,sagun2017empirical} found that the spectrum typically exhibits several  clusters, including a bulk concentrated near zero, alongside a set of outlier eigenvalues separated from these clusters. Far from being isolated observations, similar spectral features have since
been reported across a wide range of architectures and datasets
\cite{ghorbani2019investigation}, and for related quantities such as the
Gauss--Newton (GN) matrices \cite{papyan2018full}, neural tangent kernel (NTK) matrices \cite{fan2020spectra}, gradient covariance matrices \cite{jastrzebski2020break}, second-moment gradient matrices \cite{li2020hessian}, and network Jacobians
\cite{oymak2019generalization}, as well as in layerwise analyses of the Hessian
\cite{sankar2021deeper}.

Subsequent work sought to explain these spectral phenomena.
\cite{sagun2017empirical} observed that the number of outlier eigenvalues
closely matches the number of classes and, in overparameterized settings,
related part of the near-zero bulk to rank deficiency of the empirical GN
term. Later work further developed this use of the GN matrix \cite{fort2019emergent,papyan2018full} to provide a more refined account of the observed Hessian spectrum, including through connections with the  NTK \cite{jacot2019asymptotic}.  A parallel line of work has drawn on random matrix theory (RMT), often
under asymptotic or idealized assumptions. These analyses have captured several spectral features, but their scope
remains limited
\cite{pennington2017geometry,granziol2022learning,couto2025teacher}.
Taken together, these approaches provide valuable insight, but the
question of the origin and nature of these spectral phenomena
has remained~open.

In this work, we propose symmetry and its breaking as an organizing principle
for explaining the observed structure of Hessian spectra. Our analysis traces
the spectral clusters to the splitting of Hessian eigenvalues whose large
multiplicities are forced by symmetry at reference configurations. This mechanism, illustrated in \pref{fig:original_vs_symmetric_mix}, is established directly for the Hessian at fixed finite layer sizes.  A~central challenge is to identify the relevant Hessian invariances at these configurations, since they are generally not captured by weight symmetries. We develop the theoretical framework in some generality, while centering the detailed analysis on the settings studied empirically in \cite{sagun2016eigenvalues,sagun2017empirical}. The title was chosen in reference to these works, reflecting our aim to explain how exact symmetries and their subsequent breaking produce the
empirical phenomena reported there. Despite their
apparent simplicity, these settings already reveal the rich structure of the
underlying symmetry breaking (SB) phenomena. 

This extended version also considers applications beyond this core setting. Once the framework is in place, it can be applied to deeper networks, for example, convolutional networks, graph neural networks, and transformers. We include brief sketches of a few such settings in \pref{sec:dsb}.
Detailed accounts are in preparation as separate publications, to be incorporated later into a broader treatment \cite{arjevani2026dsb_tf}. Beyond the spectral
phenomena studied here, SB appears more broadly across DL, bearing on
optimization \cite{arjevani2025three}. We briefly discuss this connection in \pref{sec:concluding}.

\section{Introduction}
We begin by reviewing the relevant phenomenology and the limitations of
existing attempts to explain it. The main part of the introduction then gives
a high-level overview of the framework and demonstrates how SB accounts for the
observed spectral structure. The overview also serves as a guide to the
organization of the paper, with forward references to the sections where each
part of the framework is developed.

To avoid a lengthy preliminary section that could obscure the role of the core
ideas, we have sought to keep formal preliminaries in the introduction to a
minimum. The full formal development begins in \pref{sec:main}, while selected
background references are collected in the \hyperref[sec:reading]{Appendix}. Accordingly, some arguments are
presented at an informal level, with footnotes providing immediate technical
clarifications. This expository choice, \emph{toccata e fuga}, is intended to
isolate the mechanism from the technical framework that supports it, before
the subsequent sections return to develop that framework in full detail.

\subsection{Phenomenology}\label{sec:pheno}
The first step in studying the Hessian and its spectrum is, in principle,
their explicit computation. However, exact computation remains prohibitively difficult for full-scale architectures because of their large number of parameters. The settings considered in \cite{sagun2016eigenvalues,sagun2017empirical} were therefore chosen to permit precise and tractable computations and are described below.

We take inputs in $\inspace=\RR^d$ and outputs in $\outspace=\RR^k$,
with $d$ input coordinates and $k$ classes. The parameter space is
denoted by $\weightspace$, with $p=\dim\weightspace$. On these spaces,
we define the fully connected three-layer ReLU network
$N^3:\inspace\times\weightspace\to\outspace$ by
\begin{align}\label{net3}
	N^3(\x; W)
	&= W_3 \sigma\big(W_2 \sigma(W_1 \x + \b_1) + \b_2\big) + \b_3,
\end{align}
where \(\sigma\) denotes the ReLU activation function, applied coordinatewise,
and the parameter space \(\weightspace\)  consists of tuples
\(W=(W_1,\b_1,W_2,\b_2,W_3,\b_3)\) with
\(W_1 \in M(h,d)\), \(\b_1 \in \RR^h\), \(W_2 \in M(h,h)\),
\(\b_2 \in \RR^h\), \(W_3 \in M(k,h)\), and \(\b_3 \in \RR^k\). Here \(M(a,b)\) is the vector space of real \(a\times b\) matrices, and \(h\) is the number of neurons in the hidden layers.

In developing the framework, we also introduce intermediate architectures,
which we call \emph{diminished architectures}, to construct symmetric reference
configurations at which the relevant symmetries hold exactly. These architectures also provide simpler models in which to present the main
ideas. We write \(N^L(\x;W)\) for the \(L\)-layer version of \(N^3\) in
\eqref{net3}, and define 
	$N^{L+\nicefrac{1}{2}}(\x;W)
 =
	N^{L+1}(\x;(W,\bones^\top,0))$, with $k=1$ and \(\bones\) being the all-ones vector.
A superscript \(\circ\) indicates that all bias terms are removed, 
\(\semicircle\) indicates that only the final-layer bias is removed, and
\(\twoslashcircle\) indicates that the final-layer bias is restricted to be a
scalar multiple of \(\bones\). Thus,
\begin{alignat}{2}
	\label{net3circ}
	N^{2\nicefrac{1}{2},\circ}(\x; W)
	&=
	\bones^\top \sigma\big(W_2 \sigma(W_1 \x)\big),
	\quad\qquad&
	N^{3,\circ}(\x; W)
	&=
	W_3 \sigma\big(W_2 \sigma(W_1 \x)\big),
	\\
	N^{3,\semicircle}(\x; W)
	&=
	W_3 \sigma\big(W_2 \sigma(W_1 \x + \b_1) + \b_2\big),
	\quad&
	N^{3,\twoslashcircle}(\x; W)
	&=
	W_3 \sigma\big(W_2 \sigma(W_1 \x + \b_1) + \b_2\big)
	+
	b_3\bones .
	\notag
\end{alignat}
Define the pointwise loss associated with architecture \(N\) and loss function \(\ell^\ast:Y\times Y\to[0,\infty)\) by
\begin{align}\label{smloss_def}
\smloss\defeq	\smloss_N^\ast(\x,\y;W)
	=
	\ell^\ast\!\big(N(\x;W),\y\big).
\end{align}
Throughout, \(\ell^\ast\) is taken to be either the cross entropy loss \(\ell^\celoss\) or the squared loss \(\ell^\sqloss\). The expected loss is defined by
\begin{align}\label{eloss}
\ploss \defeq \ploss_{N,\dd}^*(W)
	=
	\mathbb{E}_{(\x,\y)\sim\dd}\!\left[\smloss_N^\ast(\x,\y; W)\right],
\end{align}
where $\dd$ denotes the data distribution.\footnote{Our use of \(\kappa\) for the pointwise loss reflects its role as an integral kernel in the theory of generalized functions. Later, we assume that \(\dd\) is a Radon measure and interpret the expected loss as a pushforward. This viewpoint allows us to treat Hessians with possible singular Radon measure components.} In general, the expected loss \eqref{eloss} and the objective function to be minimized may differ, for instance when the latter includes regularization. Moreover, some objectives are not naturally given as expected losses,
such as the tensor decomposition problem in \eqref{prob:tensor opt}.
Accordingly, unless the distinction matters, we usually use the term objective function.

The datasets considered in \cite{sagun2016eigenvalues,sagun2017empirical} are
MNIST and mixtures of Gaussians. We focus on the latter, for which the relevant
symmetries are explicit and exact (see \pref{fig:mnist_n3_gn_metric_flow} for results on MNIST).
\begin{setting}[Mixture of Gaussians]\label{setting:gaussian_mixture}
	The data distribution on
	\((X,Y)\in\mathbb{R}^d\times\{1,\dots,k\}\) is given by
	\begin{align}
		\mathbb{P}(Y=j)=\tfrac1k,
		\qquad
		X\mid Y=j \sim \mathcal{N}(\c_j,\tau^2 I_d),
		\qquad
		j=1,\dots,k,
	\end{align}
	where  \(\mathcal{N}\) denotes the Gaussian distribution,
	\(\c_1,\dots,\c_k\in\mathbb{R}^d\) are the class centers,
	\(I_d\) is the \(d\times d\) identity matrix, and \(\tau>0\).
	We use the cross-entropy loss \(\ell^\celoss\) unless otherwise stated. Following \cite{sagun2016eigenvalues}, we also consider the squared loss
	\(\ell^\sqloss\), identifying each class label \(j\) with its one-hot
	representation \(\e_j\in\mathbb{R}^k\). The same work considers the case \(k=2\), with
	\(\c_1=(-1,-1)\) and \(\c_2=(1,1)\). We denote the corresponding
	distribution by \(\cN_\pm\). In \cite{sagun2017empirical}, the centers
	are drawn randomly.
\end{setting}
We use He initialization \cite{he2015delving} and the Adam optimizer
\cite{kingma2015adam}, which partially adapts to local curvature. Training begins with batch sizes of 128, which are increased toward the end. Consistent with the findings of \cite{sagun2016eigenvalues},  \pref{fig:flat_directions} shows that after training, a large bulk of eigenvalues is concentrated near zero, together with a small set of positive eigenvalues separated from this bulk. Moreover, the Hessian itself exhibits a distinctive pattern reflecting
a sparse, low-dimensional structure (shown here for \(N^{3,\circ}\)). As further illustrated in the figure, increasing the dimension does not change this qualitative picture. The same behavior was reported when using the squared loss and when increasing the Gaussian variance. The phenomenon persists beyond the two-Gaussian case: for mixtures of
\(k\) Gaussians, \cite{sagun2017empirical} observed a similar 
spectral structure, with the number of outlier eigenvalues scaling with
the number of labels (see \pref{fig:original_vs_symmetric_mix}).

\begin{figure}[ht]
	\centering
\makebox[\textwidth][c]{\includegraphics[width=1.1\textwidth]{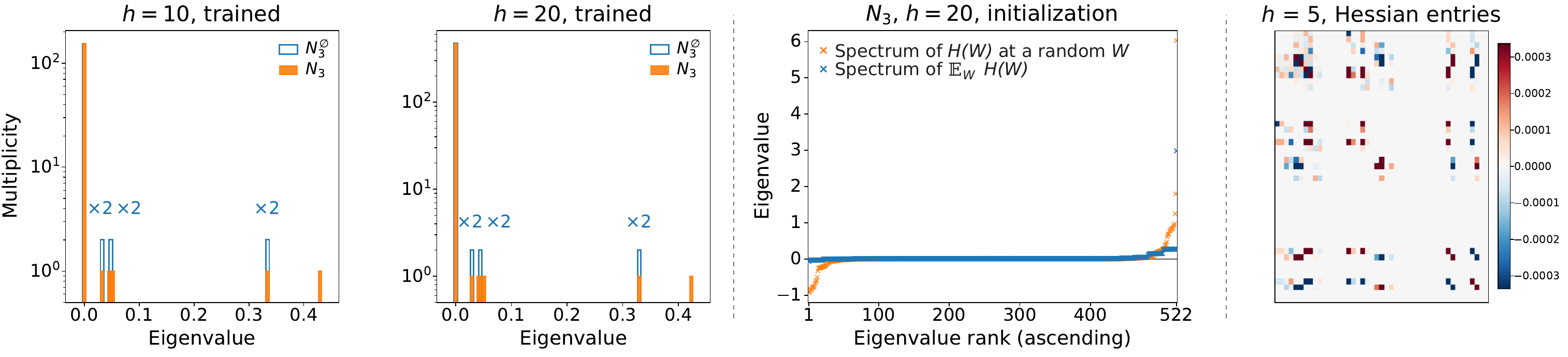}}
\caption{
	Hessian spectra (left and center) for the three-layer networks \(N^3\)
	and \(N^{3,\semicircle}\) (the latter obtained by removing the final-layer
	bias), and Hessian structure for the bias-free  \(N^{3,\circ}\) (right), on the
	two-Gaussian mixture in \pref{setting:gaussian_mixture}.
At trained models, their Hessian spectra (left) exhibit a large near-zero
bulk and a few outliers, consistent with \cite{sagun2016eigenvalues}. The spectrum at a single random draw, likewise examined in that work, is compared with that of the averaged Hessian~(center).\\
\textbf{At the symmetric reference.} Despite weight symmetries being
generically trivial, the argument illustrated in \eqref{struc_combine_with_dist} yields
a \(D_4\times C_2\) Hessian invariance. This invariance, richer than its \(k=2\) counterpart in
\pref{fig:original_vs_symmetric_mix}, forces all nonzero eigenvalues to have
multiplicity two (left, blue). The large kernels reflect symmetry-induced parameter degeneracy.\\
\textbf{Symmetry breaking.}
Restoring the final-layer bias breaks the
symmetry and splits two reference doublets, with the leading one producing
two isolated outliers (left, orange), while retaining the large reference kernel.
This explains the reported bulk-outlier pattern.
The detailed spectral description, including the full Cauchy interlacing
argument, is given in~\pref{table:n3_symmetry_breaking}.\\
	\textbf{Random initialization.} We show that the averaged Hessian has a rich
invariance that forces a small positive eigenvalue of multiplicity quadratic in $h$
	(center, blue), which in turn yields, with probability tending to one as
	\(h\to\infty\), a near-zero bulk containing a \(1-o(1)\) fraction of the
	spectrum at a random draw (center, orange).
}
	\label{fig:flat_directions}
\end{figure}

This spectral structure is not restricted to classification and appears
equally in regression. We include the following setting
in this extended version.
\begin{setting}\label{setting:gaussian}
The inputs are drawn from \(\mathcal{N}(0,I_d)\)
and labeled by $\x	\mapsto  \mathbf{1}^\top \sigma(\x)$.	The loss is the squared loss $\ell^\sqloss$.
\end{setting}
This learning problem has been studied for $N^{1\nicefrac12,\circ}$
and related variants in the statistical-physics 
literature \cite{saad1995online,goldt2019dynamics} and in the DL literature
\cite{tian2017analytical,du2018gradient,safran2021effects,xu2023overparameterization,simsek2023copyaverage,huang2026sharp,arjevani2019principle}. Here we study the same problem with
\(N^3\). Our framework extends to arbitrary depth, thereby resolving the
five- and six-layer architectures considered in the latter work. As in
the mixture of Gaussians, the spectra in \pref{fig:sq_evs}~(right) show
eigenvalue clusters together with separated outliers, whose number here is of
order \(d\). Importantly, some of these clusters lie close to zero but remain
positive. We return to this spectral structure below as a constraint
on candidate~explanations.

\begin{figure}[ht]
\centering
\includegraphics[width=\textwidth,scale=0.6]{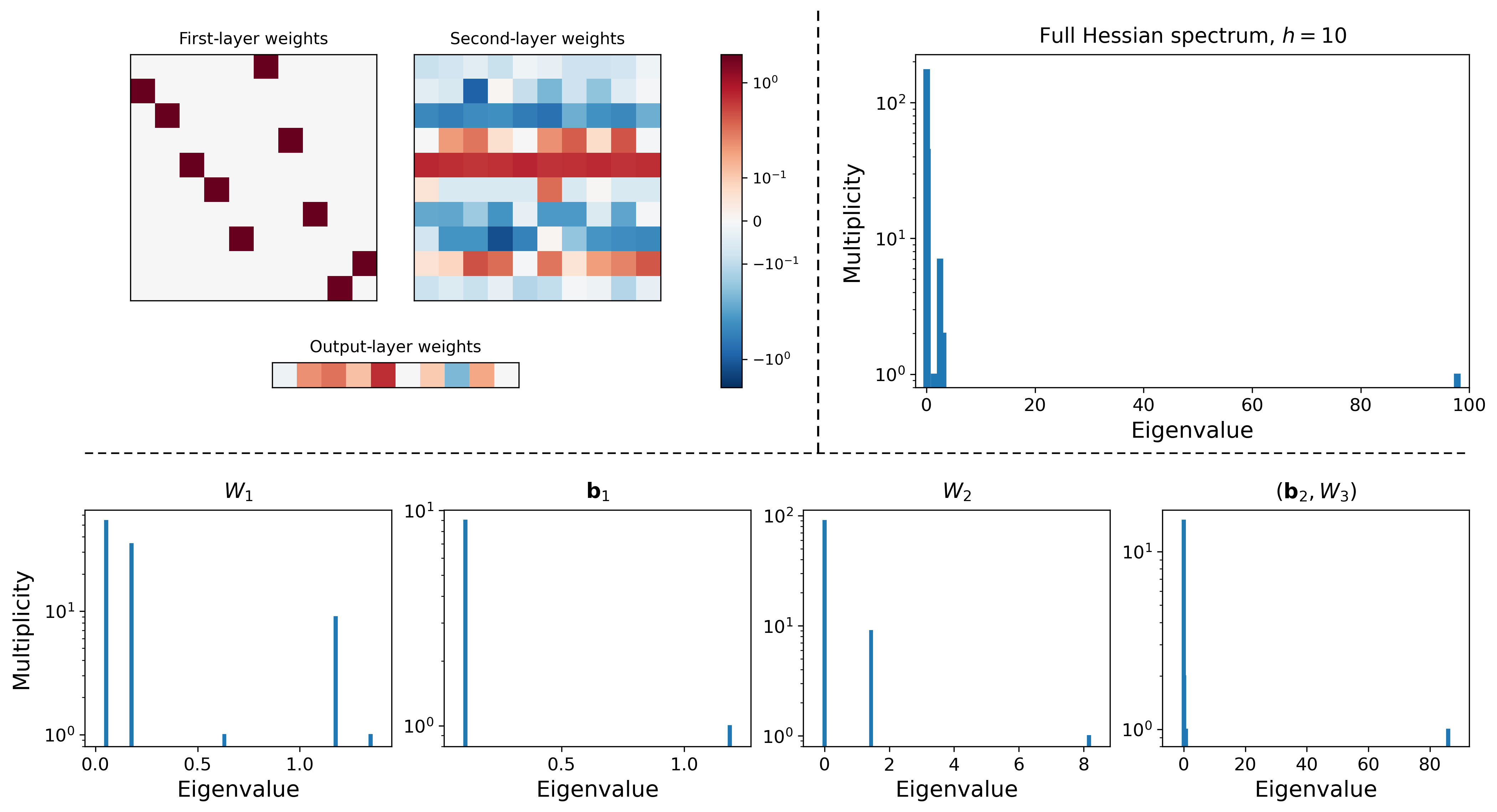}
\smallskip

{\footnotesize
	\setlength{\tabcolsep}{4pt}
	\renewcommand{\arraystretch}{1.15}
\begin{tabular}{c|p{.5\textwidth}|c}
	\multicolumn{3}{c}{
		\textbf{Spectrum of $N^3$ at the symmetric reference}
	}
	\\[-1pt]
	\hline
	scale of eigenvalues
	&
	multiplicities forced by symmetry
	&
	total out of $2d^2+3d+1$
	\\
	\hline
	$0$
	&
	one zero eigenvalue with multiplicity $d(d+2)$
	&
	$d(d+2)$
	\\
	\hline
	$\Theta(1)$
	&
	\begin{tabular}[t]{@{}l@{}}
		one eigenvalue with multiplicity $\frac{d(d-3)}{2}$ \\
		one eigenvalue with multiplicity $\frac{(d-1)(d-2)}{2}$ \\
		two eigenvalues, each with multiplicity $d-1$ \\
		one simple eigenvalue
	\end{tabular}
	&
	$d(d-1)$
	\\
	\hline
	$\Theta(d)$
	&
	\begin{tabular}[t]{@{}l@{}}
		two eigenvalues, each with multiplicity $d-1$ \\
		two simple eigenvalues
	\end{tabular}
	&
	$2d$
	\\
	\hline
	$\Theta(d^2)$
	&
	one simple eigenvalue
	&
	$1$
	\\
	\hline
\end{tabular}
	
	\smallskip
}
\caption{
	Weights and Hessian spectra of $N^{3}$ with $d=h\ge4$, trained on
	\pref{setting:gaussian}.
	\textbf{Top.} At the trained model, the Hessian has a rich invariance structure and
	nontrivial eigenvalue multiplicities---despite the learned weights having only
	trivial symmetry under the standard rescaling and permutation symmetries. Our
	framework instead traces these invariances to symmetries of the differential
	of $N^{3}$.
		The same obstruction appears for
		\pref{setting:gaussian_mixture} and for general $L$-layer fully connected
		networks in \cite{arjevani2019principle}. \textbf{Bottom.} 
	Nontrivial multiplicities also occur in the layerwise Hessian spectra. The limitation of weight symmetries is even more pronounced layerwise, since transformations coupling adjacent layers need not survive
	restriction to a single layer. Nevertheless, the framework
	captures the layerwise Hessian invariances and the resulting multiplicities
	recorded in~\pref{table:hess_kernel}.
}
\label{fig:sq_evs}
 \end{figure}

 \subsection{Previous attempts to identify an underlying mechanism}
We review representative approaches proposed to explain these spectral
structures, assess their limitations in light of the numerical results
above, and collect some of the main gaps in a list below.

\paragraph{Random matrix theory.}
A prominent line of work uses random matrix theory (RMT) to model Hessians
and related curvature matrices, drawing in part on spin-glass models of
high-dimensional loss landscapes. 
Representative examples include
\cite{choromanska2015loss,pennington2017geometry,liao2021hessian,baskerville2022universal,granziol2022learning,couto2025teacher}. These approaches can reproduce
broad bulk and outlier features, including through spiked models.
At the same time, the resulting theoretical picture does not fully capture the observed spectral structure, with major discrepancies reported, for example, in
\cite{granziol2020beyond}. Some modeling assumptions
motivated by studies of \emph{glassy} systems have also been criticized
\cite{baity2018comparing}, and developing more suitable random matrix models or
exploring alternative assumptions \cite{choromanska2015open} remains
an active research direction.

Set against the phenomenology above, the cited RMT analyses leave
several important gaps, summarized below, of which we highlight two here. First,
the same spectral organization already occurs in very small networks,
calling for an explanation that applies at finite dimensions.
Second, these analyses provide generative accounts by positing random
matrix ensembles or deriving them from probabilistic models, leaving
open what determines the organization of the Hessian spectrum at a
given critical point in the settings considered here, where this
structure is clearly observed.

\paragraph{Gauss--Newton decomposition.}
Another line of work uses the GN decomposition of the Hessian. As noted
earlier, rank deficiency of the empirical GN matrix was used in
\cite{sagun2017empirical} to account for some of the zero eigenvalues in the
overparameterized regime. Consider a network $N$ with $k$ outputs, and suppose that the pointwise loss is  twice differentiable in $W$ at the point considered. Write $J_WN(\x,W)$ for
the parameter Jacobian of $N(\x;W)$. Differentiating twice gives\par
{\small
	\begin{align}\label{gn_decomp}
		\nabla_W^2 \smloss_N^\ast(\x,\y;W)
		&=
		J_WN(\x,W)^{\!\top}
		\nabla_1^2\ell^\ast(N(\x;W),\y)
		J_WN(\x,W)
		+
		\sum_{a=1}^k
		\partial_{1,a}\ell^\ast(N(\x;W),\y)
		\nabla_W^2N_a(\x;W).
\end{align}}
 \par\noindent
 The first term is the \emph{pointwise} GN matrix. Averaging it over the
 training samples gives the empirical GN matrix. In the scalar-output case, for example,
 each pointwise GN matrix has rank at most one, so the empirical GN matrix has
 kernel dimension at least the number of parameters minus the number of
 samples. This rank argument, however, does not satisfactorily explain the
 observed Hessian structure. The large near-zero bulk persists when the number
 of samples far exceeds the number of parameters and also after passing to the
 expected loss, where the finite-sample rank bound no longer applies. Moreover, large Hessian kernels are observed even when the second term in
 \eqref{gn_decomp} is nonnegligible; see \pref{table:hess_kernel}. Similar
 considerations apply to the NTK matrix in \cite{jacot2019asymptotic}.
 
 Another feature of the GN matrix is that its structure permits
 decompositions into contributions from class-dependent gradient means
 and a residual measuring their spread. In
\cite{fort2019emergent}, the residual is assumed to follow a prescribed random
law, whereas \cite{papyan2020traces} proposes a more refined decomposition
and estimates its components directly from data to explain
 observed spectral features.

Taken together, these approaches provide valuable insight but do not account
for the full range of spectral phenomena displayed above. We now summarize
some of the principal gaps.

\smallskip
\par\noindent \textbf{$\bullet$ Spectral clusters.}
Beyond the large cluster near zero and the outliers, the spectrum exhibits a finer structure, with clusters that remain separated from zero. The number of these clusters is fixed and does not depend on the number of labels. 

 \smallskip
\par\noindent\textbf{$\bullet$ Spectral structure in small dimensions.}
Even if a random matrix model can be modified to capture the overall spectral
shape more faithfully, approaches based on such models typically operate in
high-dimensional or asymptotic regimes and therefore do not by themselves
explain why the same spectral structures already appear in very small
dimensions.
\smallskip
\par\noindent\textbf{$\bullet$ Large Hessian kernel at large sample sizes.}
At the other end of the scale, the large Hessian kernel is not an artifact of the small sample size. It persists when the sample size exceeds the number of
parameters and also in the Hessian of the expected loss. In these regimes,
bounds based on rank deficiency of the empirical GN matrix are vacuous.
\smallskip
\par\noindent\textbf{$\bullet$ Dependence on the GN decomposition.}
A limitation of the rank argument, shared by related approaches, is
its dependence on the GN matrix. This misses regimes in which the second term
in \eqref{gn_decomp}, involving second-order derivatives of the model, is
significant, including at nonglobal local minima and even on large sets of noncritical
points (see \pref{table:hess_kernel}).
\smallskip
\par\noindent\textbf{$\bullet$ Other curvature matrices.}
Related structures have been reported for
other curvature matrices, as discussed at the beginning of the introduction.
This suggests that these phenomena may be different manifestations of a
common underlying mechanism rather than consequences of a particular choice
of curvature matrix.
\smallskip
\par \noindent\textbf{$\bullet$ Layerwise spectra.}
The spectral cluster structure is not specific to the full Hessian. Related
structures are also observed in layerwise restrictions of the Hessian.

\smallskip
\noindent\textbf{$\bullet$ Regression.}
The spectral cluster structure is not specific to classification. Some explanations tie the existence and number of outliers to the network's output dimension and thereby, in classification, to the number of labels. This does not account for the regression setting, where there is only one output but the number of outliers scales as $\Theta(d)$; see \pref{table:hess_kernel}.

\smallskip
\noindent\textbf{$\bullet$ Mechanism.}
A final limitation is that existing approaches mainly encode spectral structure through a random
model or prescribed ansatz. Thus, they either fit the
observed spectrum through a generative model or decompose the empirical spectrum
post hoc, without fully clarifying whether a common mechanism explains why these structures arise and how they depend on the underlying learning problem.

In the next section, we explain how symmetry and its breaking provide
a natural account of the observed spectral structure and, in particular,
address the gaps identified above.


\subsection{Overview: origins of symmetry}\label{sec:overview_framework}
A central principle of the framework is to understand spectral
structure through Hessian invariances, either directly or relative to nearby configurations with richer invariances. By a configuration, we mean the collection of choices
determining the Hessian under consideration, including the
architecture, loss, data distribution, parameter point, parameter-space inner
product, and, where relevant, the subspace to which the Hessian is restricted. In some cases, multiple configurations may be simultaneously relevant, for example along an optimization trajectory.
Carrying out this program requires some preliminary work. We therefore begin with a high-level overview of the main ingredients of the framework and of the spectral analysis, referring along the way to the sections where they are developed in detail. 

We first make the notion of Hessian invariance precise. Let
\(\GL(\weightspace)\) denote the group of invertible linear transformations of~\(\weightspace\), and write \(\GL_p=\GL(\RR^p)\).  Identify
\(\weightspace\) with \(\RR^p\) and equip it with the standard Euclidean inner
product. We define the invariance group of the Hessian 	$\nabla^2\ploss(W)$ at \(W\) as the set of all \(g\in \GL_p\) satisfying
\begin{align}\label{hess_cong}
	g^\top
	\nabla^2\ploss(W)
	g
	=
	\nabla^2\ploss(W),
\end{align}
where \(g^\top\) denotes the usual transpose. When the Hessian is restricted to a
subspace \(V\subseteq\weightspace\), the same definition is applied with
\(V\) in place of \(\weightspace\).

\paragraph{Intrinsic and metric-dependent quantities.}
We distinguish between properties that are intrinsic and those that depend
on the choice of inner product or coordinates. This distinction motivates the formulation used later.

 The matrix $\nabla^2\ploss(W)$ represents the self-adjoint operator
associated with the second differential $d_W^2\ploss(W)$ for the
standard Euclidean inner product (using coordinates fixed by the
identification of $\weightspace$ with $\RR^p$).
The relation \eqref{hess_cong} in fact expresses invariance of the
second differential itself, meaning that
\begin{align}
d_W^2\ploss(W)[g\delta W,g\delta W']
=d_W^2\ploss(W)[\delta W,\delta W'], \quad
\delta W,\delta W'\in T_W\weightspace.
\end{align}
Here $T_W\weightspace\simeq\weightspace$ canonically.
Thus invariance is \emph{intrinsic}
in that it does not depend on the choice of inner product or coordinates.
The spectrum, by contrast, is metric-dependent, meaning that it depends
on the inner product used to define the associated self-adjoint operator. This metric dependence is not merely a formal matter. Rather, it is one of the mechanisms behind the formation of spectral clusters (see \pref{fig:mnist_n3_gn_metric_flow} and \pref{table:alexnet_parameter_spectrum}). We shall refer to this as \emph{spectral refraction} (see \pref{rem:spectral_refraction}). Coordinates serve a separate purpose, allowing us to express the
differential directly or represent the operator by a matrix. 

Another intrinsic feature is the kernel of the associated Hessian operator.
For every choice of inner product, the operator has the same kernel, namely
the radical of $d_W^2\ploss(W)$ (see the \hyperref[sec:reading]{Appendix}
for background references on these and other notions used throughout
this work). The null cone of $d_W^2\ploss(W)$, which will also be of interest later,
is intrinsic by definition. Riemannian settings are considered in
\pref{sec:layerwise}.

At suitable reference configurations, the Hessian matrix admits
large invariance groups. When these actions are orthogonal, standard group-theoretic constraints force a
rigid spectral structure with only a few distinct eigenvalues, some
necessarily of large multiplicity (see \pref{sec:failure}). The spectral
clusters at the original configuration are then analyzed as deviations from
this reference structure induced by controlled SB effects (see \pref{fig:sb_concepts} below).

\paragraph{A symmetric reference.} 
With these preliminaries in place, we turn to a simple setting to make
the framework concrete. Consider the mixture of two Gaussians specified in \pref{setting:gaussian_mixture}, with $\tau=1$ and the network \(N^3\). The associated class of model functions \(\x\mapsto N^3(\x;W)\) contains Bayes-optimal predictors. A simple family of weights realizing such a predictor is

{\footnotesize \begin{align}\label{Wstar}
	\spT
	=
	\left\{
	(W_1^\star,0,W_2^\star,0,W_3^\star,0)
	\,\middle|\,
	\begin{array}{@{}l@{}}
		W_1^\star
		=
		\begin{pmatrix}
			1 & 1\\
			-1 & -1
		\end{pmatrix},
		\quad
		W_2^\star
		=
		\begin{pmatrix}
			1 & -a_{1,2}\\
			-a_{2,1} & 1
		\end{pmatrix},
		\quad
		W_3^\star
		=
		\begin{pmatrix}
			-1 & 1\\
			1 & -1
		\end{pmatrix},\\
			a_{1,2},a_{2,1}>0
	\end{array}
	\right\}.
\end{align}}
\par\noindent \par\noindent The weights in $\spT$ also give global minima for each of the diminished
architectures $N^{3,\circ}$, $N^{3,\semicircle}$, and
$N^{3,\twoslashcircle}$, with the absent bias terms omitted. It turns out that these
architectural restrictions yield richer Hessian invariance groups. We later construct, at each $W^\star\in\spT$, the following invariance subgroups acting orthogonally with respect to the standard
Euclidean inner product
\begin{align}\label{intro:inv_grps}
	\begin{array}{c|c}
		\text{Architecture} & \text{Group} \\
		\hline
		N^{3,\circ},\ N^{3,\semicircle},\ N^{3,\twoslashcircle}
		&
		D_4\times C_2
		\\
		N^3
		&
		C_2^2
	\end{array}.
\end{align}
Here $D_4$ denotes the dihedral group of order eight and $C_2$ the cyclic
group of order two. In the present setting, we construct symmetric references by restricting the
architecture in this way. For all three diminished architectures, the $D_4\times C_2$ symmetry
forces the nonzero eigenvalues to occur in pairs. For
$N^{3,\semicircle}$, the reference used in \pref{fig:flat_directions}, there
are three such pairs of eigenvalues. Returning from $N^{3,\semicircle}$ to $N^3$
results in what we refer to as a \emph{resolutive} form of SB, in which the splitting is
constrained by the reference spectrum, as illustrated in the same figure and
further discussed in \pref{sec:analysis_overview}~below. The full set of global minima for $N^{3,\circ}$ is characterized in \eqref{k2_all_T}, with related spectral effects
analyzed in \pref{sec:two_gaussians_analysis}.

\begin{remark}
A modified $D_4\times C_2$ action nevertheless preserves
$d_W^2\ploss(W^\star)$ for $N^3$, but is no longer orthogonal
with respect to the standard Euclidean parameter geometry;
see \pref{rem:spectral_refraction}.
\end{remark}

\paragraph{The failure of weight symmetries.} Establishing the Hessian invariances in \eqref{intro:inv_grps} is an important component of the analysis and draws on several parts of the framework; spectral constraints then follow directly from standard representation theory. The reason is that the standard route proceeds by looking for global symmetries of the objective that fix the parameter point \(W^\star\in\spT\). 
In many areas of mathematics and physics, this approach has long served as a central and highly effective organizing principle \cite{golubitsky1988singularities}. However, this
fails already in the present example. For generic choices of $a_{1,2}$ and $a_{2,1}$ in \eqref{Wstar},
the subgroup fixing $W^\star$ among the standard positive rescaling
and permutation symmetries of the hidden units is trivial, and hence so is the subgroup of Hessian
invariances it accounts for, whereas the full Hessian invariance group
is nontrivial. This limitation is sharper layerwise, since these global  symmetries
couple adjacent layers and therefore need not survive restriction to
a single layer. Even so, in \pref{sec:layerwise} we establish a rich invariance structure
for the restricted Hessians, illustrated in \pref{fig:sq_evs}, and
determine their exact eigenvalue multiplicities, reported in
\pref{table:hess_kernel}. These results provide an analytical counterpart to the empirical layerwise spectral structure reported in \cite{sankar2021deeper}. We return to the failure of the weight-symmetry approach in
\pref{sec:failure}, where we examine it in more detail.

\begin{remark}[Identifiability]\label{rem:identifiability}
The global minima in \eqref{Wstar}, and still more clearly those in \eqref{sq_loss_global} for \pref{setting:gaussian} (as illustrated in \pref{fig:sq_evs}), reveal a basic limitation of \emph{identifiability}. 
Even modulo the global symmetries of the objective, the sets of global
minimizers remain positive-dimensional, reflecting substantial local degrees
of freedom; see \pref{sec:spec_analysis_gaussian}. Nevertheless, the Hessian invariances and the relevant spectral phenomena persist. Thus, both weight symmetries and exact identifiability reflect a view of the loss landscape that is too fine-grained,~asking us to look at it \emph{too~closely}.
\end{remark}

\paragraph{Generalized functions and presheaves of symmetry groups.}
The failure of the weight-symmetry approach discussed above requires us to
establish Hessian invariance directly. When the pointwise loss is $C^2$, suitable moment conditions ensure, for example via dominated convergence \cite{folland1999real}, that the Hessian
exists and is given by the expectation of the Hessian of the pointwise loss.
For ReLU networks, an additional complication arises from nonsmooth
activation boundaries, which can contribute singular terms that are absent from the classical pointwise Hessian.
Consequently, in our setting one generally has
\begin{align}\label{diff_under_expec}
	\nabla^2 \ploss(W)
	\neq
	\EE_{(\x,\y)\sim\dd}
	[
	\nabla_W^2 \smloss(\x,\y;W)
	]
		\quad \text{classically}.
\end{align}
The correct identity may be recovered in the sense of
\emph{generalized functions}\footnote{Also known as \emph{distributions}. We avoid this term because phrases such as the \emph{distributional} Hessian could
	otherwise be confused with the Hessian of the expected loss (even if \emph{generalized}
	may evoke statistical generalization).  Explicit expressions
	are supplied by \emph{geometric measure theory}, realizing 
	the singular terms as Radon measures (see \pref{sec:further_framework}).} by including the singular contributions
identified by the Gauss--Green formula, giving
\begin{align}\label{interchange}
	\nabla^2\ploss
	&=
	\EE\!\left[
	\nabla^2_W \smloss
	\right]
	\quad \text{as generalized functions}.
\end{align}
This formulation, discussed in \pref{sec:dist_der}, lets us establish invariances of the generalized Hessian on open sets, which are then generally inherited by the classical Hessian wherever it exists. It also leads naturally to \emph{presheaves} of symmetry groups, the subject of \pref{sec:formal_sym}. Readers unfamiliar with generalized functions may interpret the remainder of this section as referring to the classical Hessian.

\begin{remark}	The generalized Hessian has practical significance beyond its role
	in the invariance analysis. In direct or automatic differentiation, the pointwise
	Hessian captures only the regular part of the generalized Hessian. For example, let $X\sim\cN(0,1)$, write
 $\phi(x)=(2\pi)^{-1/2}e^{-x^2/2}$ and 	$f(x;w) = \sigma(x-w)$, and set
	$F(w)=\EE[f(X;w)]$. First differentiation commutes with expectation and
	gives $F'(w)=\EE[\partial_w f(X;w)]
	=-\int_w^\infty\phi(x)\,dx$.
	Classically, however, $\partial_w^2 f(X;w)=0$ almost surely, so automatic
	differentiation gives zero after averaging. The one-dimensional Gauss--Green formula instead gives
	\begin{align*}
			F''(w)
			&=
			\frac{d}{dw}F'(w)
			=
			-\bigl(-\phi(w)\bigr)
			+
			\int_w^\infty\partial_w\bigl(-\phi(x)\bigr)\,dx
			=
			\phi(w)
			=
			\frac{1}{\sqrt{2\pi}}e^{-w^2/2}.
		\end{align*}
	Consequently, the sample-averaged Hessians commonly computed in empirical
	studies need not be consistent estimators of the classical Hessian of the
	expected loss. In this work, calculations involving the singular terms were validated
	numerically using finite differences and Monte Carlo estimates of the corresponding surface integrals obtained with
	mollifiers.
\end{remark}

\paragraph{Differentials and layer coordinates.}
The interchange of expectation and differentiation reduces the analysis of Hessian invariance to pointwise quantities. Writing out the  pointwise Hessian matrix in
flattened coordinates makes its invariances difficult to read off. Instead, we work with the pointwise differentials themselves, while writing them in the natural block coordinates of the layers. For instance, for the diminished architecture $N^{3,\circ}$, the first
differential evaluated on a tangent vector $\delta W$, namely
$d_W N^{3,\circ}(\x;W)[\delta W]$, is given in layer coordinates~by
\begin{align}\label{dn3circ}
	\delta W_3\,\sigma\!\left(W_2\sigma(W_1\x)\right)
	+
	W_3\,D_2\,\delta W_2\,\sigma(W_1\x)
	+
	W_3\,D_2\,W_2\,D_1\,\delta W_1\,\x,
\end{align}
where $D_1(\x;W)=D\sigma(W_1\x)$ and
$D_2(\x;W)=D\sigma\!\left(W_2\sigma(W_1\x)\right)$.\footnote{
For the networks considered in this work (see \pref{def:definable_dag}), the first differential
	\(d_WN\) has no singular part. Thus pointwise evaluations of \(d_WN\) refer
	to its regular representative, defined up to a null set.}
If $\operatorname{vec}$ denotes row-wise stacking, then in flattened coordinates, the same differential is given by
{\begin{align*}
\bigl(I_k\otimes\sigma(W_2\sigma(W_1\x))^\top\bigr)\operatorname{vec}(\delta W_3)
	+
	\left(W_3D_2\otimes \sigma(W_1\x)^\top\right)
	\operatorname{vec}(\delta W_2)
	+
	\left(W_3D_2W_2D_1\otimes \x^\top\right)
	\operatorname{vec}(\delta W_1).
\end{align*}}
\par\noindent This comparison shows the advantage of layer coordinates:
they preserve the ordered compositional structure and make explicit how
each perturbation enters between the surrounding layerwise factors.
This allows one to read off the differential's invariances, much as
permutation and rescaling symmetries can be read off from the expression
defining the network architecture.
For the second differential
$d_W^2N^{3,\circ}$, this loss of structure is more severe and becomes still
more pronounced for deeper architectures; see \pref{sec:dsb}.

\paragraph{The Gauss--Newton and Fisher information forms.}
In general, the regular and singular terms of the pointwise second
differential $d_W^2\smloss$ are analyzed directly (see
\pref{table:hessian_quantities} for their expressions). A~natural route to establishing the Hessian invariances of $\ploss$ would therefore start from those of $d_W^2\smloss$.
For $W^\star\in\spT$ and $N^{3,\circ}$, however, the condition later formalized in \pref{sec:overview_gn} as
\emph{well specification}\footnote{The notion of a well-specified model is standard in statistics, where it means
	that the data-generating distribution lies in the model class. Neither being a
	global minimum nor being well specified implies the other.} holds. By
\pref{prop:well_specified_gn_identity},  this implies the identity
\begin{align}\label{well_spec_iden}
	d_W^2\ploss(W^\star)
	&=\EE_{(\x,\y)\sim \cN_\pm}
	\left[d_W^{\mathrm{GN}}\smloss(\x,\y;W^\star)\right],
\end{align}
where, for a general network $N$, the \emph{pointwise Gauss--Newton (GN) form} on the right is given by
\begin{align}\label{gn_def}
	d_W^{\mathrm{GN}}\smloss(\x,\y;W^\star)[\delta W]
	&\defeq
	d_1^2\ell(N(\x;W^\star),\y)
	\big[d_W N(\x;W^\star)[\delta W]\big].
\end{align}
Here and throughout, a single argument denotes diagonal evaluation
of an $r$-linear form $A$, i.e., $A[\v]\defeq A[\v,\ldots,\v]$,
with $\v$ repeated $r$ times.
The (expected) GN form is closely related to the \emph{Fisher information (FI) form}.
For a conditional model \(p_W(\y\mid\x)\), the FI form is defined by
\begin{align}\label{fi_def}
	d_W^{\mathrm{FI}}(W)[\delta W]
	&\defeq
	\EE_{\x}
	\EE_{\y\sim p_W(\cdot\mid\x)}
	\left[
	d_W\log p_W(\y\mid\x)[\delta W]^2
	\right].
\end{align}
For the likelihood models considered here, with matching normalization, $d_1^2\ell(\z,\y)$ is independent
of $\y$, so the GN and FI forms coincide
\cite{pascanu2013revisiting,martens2020new}. 

Beyond its role as an analytic simplification, the GN form also plays an important role in optimization. It underlies classical and stochastic methods that seek more efficient ways to exploit curvature
\cite{nocedal2006numerical,sra2011optimization,tran2020stochastic}, including methods developed for deep learning
\cite{martens2010deep,vinyals2012krylov,botev2017practical,gargiani2020promise,korbit2024exact}. 
The FI form likewise underlies natural gradient descent, where it captures
the local curvature of KL divergence and supports tractable implementations
\cite{amari1998natural,desjardins2015natural,martens2015optimizing,martens2020new}.
These roles make the GN and FI forms important objects of study in their own right, independently of their relation to the Hessian. The framework and methods introduced in this paper apply directly and show that their spectra exhibit the same qualitative organization as the Hessian~spectrum, with the underlying invariance mechanism illustrated next.

\paragraph{Structural and distributional symmetries.} 
With  \eqref{well_spec_iden} in hand, we can see how Hessian
invariances arise from symmetries of the network, the loss, and the data
distribution. We distinguish between \emph{structural} symmetries, coming
from the pointwise loss, and \emph{distributional}
symmetries, which come from the data distribution.
For $N^{3,\circ}$, a direct substitution into \eqref{dn3circ} verifies that,
for every $W^\star\in\spT$, $d_WN^{3,\circ}$ is invariant under the
transformation below, and hence so is the associated pointwise GN form \eqref{gn_def}:
\begin{flalign}
	&\hspace*{7em}
	\makebox[13em][l]{$\text{Structural symmetry }\delta\sfa\colon\quad$}
	(\x,\delta W_1)
	\mapsto
	(P_{(1,2)}\x,\delta W_1P_{(1,2)}),
	&&
\end{flalign}
where $P_{(1,2)}$ exchanges the first two input coordinates. 
In addition, for the $k=2$ instance of
\pref{setting:gaussian_mixture}, the data distribution $\cN_\pm$  is invariant under
\begin{flalign}\label{eq:distributional_swap}
	&\hspace*{7em}
	\makebox[13em][l]{$\text{Distributional symmetry }\sfd^\sfs\colon\quad$}
	\x
	\mapsto
	P_{(1,2)}\x .
	&&
\end{flalign}
The structural and distributional symmetries now combine to produce a
Hessian invariance at every $W^\star\in\spT$, as the following chain of
identities shows
\begin{align}
d_W^2\ploss(W^\star)[\delta W]
&=
\EE_{\cN_\pm}\!\left[
d_W^{\mathrm{GN}}\smloss(\x,\y;W^\star)[\delta W]
\right]
&&\boxed{\text{\small Well-specification identity \eqref{well_spec_iden}}}
\nonumber\\
	&=
	\EE_{\cN_\pm}\!\left[
	d_W^{\mathrm{GN}}\smloss(P_{(1,2)}\x,\y;W^\star)
	[
	\delta W_1P_{(1,2)},\delta W_2,\delta W_3
	]
	\right]
	&&\boxed{\text{\small Structural symmetry $\delta\sfa$ of $d^{\mathrm{GN}}_W\smloss$} } 
	\nonumber\\
	&=
	\EE_{\cN_\pm}\!\left[
	d_W^{\mathrm{GN}}\smloss(\x,\y;W^\star)
	[
	\delta W_1P_{(1,2)},\delta W_2,\delta W_3
	]
	\right]
	&&\boxed{\text{\small Distributional symmetry $\sfd^\sfs$}}
	\nonumber\\
&=
d_W^2\ploss(W^\star)
\big[
\delta W_1P_{(1,2)},\delta W_2,\delta W_3
\big].
&&\boxed{\text{\small Well-specification identity}}
	\label{struc_combine_with_dist}
\end{align}
The resulting invariance  is an example of what we call a \emph{\integrated symmetry},
\begin{flalign}\label{comp_intro}
	&\hspace*{7em}
	\makebox[13em][l]{$\text{\Integrated symmetry }\sfc_1\colon\quad$}
	\delta W_1
	\mapsto
	\delta W_1P_{(1,2)}.
	\quad \boxed{\text{Invariance of $d_W^2 \ploss(W^\star)$}}
\end{flalign}
This generates a subgroup
of the invariance group of $\nabla^2\ploss(W^\star)$ isomorphic to $C_2$.
The remaining \integrated invariances in \eqref{intro:inv_grps}, now acting
simultaneously on several perturbation variables, are derived similarly.
Thus, we find that combining structural symmetries derived directly from the
network differential with distributional symmetries allows us to establish
rich Hessian invariances even when the weight symmetries are trivial, as is
generically the case here. The general principle behind this construction is formalized in \pref{thm:local_symmetry_inv}. 

\begin{remark}[Hidden-feature symmetries]
	The same argument applies after an intermediate layer, with the joint
	hidden feature--label law replacing $\dd$ for the downstream network.
	Symmetries of this law, paired with compatible structural symmetries of
	the downstream network's pointwise differential form, yield \integrated invariances
	as above. This extension is of interest, for example, when approximate symmetries
	emerge in the hidden feature--label laws of trained networks, a phenomenon
	we call \emph{neural registration} (\pref{sec:hidden_feature_symmetry}).
\end{remark}

Symmetries of the initialization law can likewise produce invariances
of the \emph{averaged} Hessian through the composite invariance argument
in \eqref{struc_combine_with_dist}. Here symmetry forces an eigenvalue
of multiplicity at least $(h-1)^2$, whose magnitude we bound
by $\bigO{h^{-1}}$.
We use this to establish, with high probability, a near-zero bulk
in the spectrum of a random Hessian; see
\pref{fig:flat_directions} (center) and~\pref{thm:streamlined-N3-spectrum}.

\begin{remark}[Random labels]\label{rem:random_labels}
	The experiments with random labels in \cite{zhang2016understanding}
	motivate a related application of the framework. Uniform labels
	independent of the inputs give a distributional symmetry under
	class permutations. Pairing this with the structural symmetry
	that simultaneously permutes the output-layer weights
	gives invariance of the expected loss under the corresponding
	parameter action, by the composite invariance argument in
	\eqref{struc_combine_with_dist}. 
\end{remark}

\paragraph{Flat directions and the Hessian kernel.}
Symmetry also provides mechanisms, ranging in complexity, that force the
Hessian kernel to be nontrivial and, in our settings, often large. The most
elementary and well-known mechanism occurs at critical points: every vector
tangent to an orbit of global continuous symmetries lies in the Hessian
kernel.\footnote{These vectors are precisely those generated infinitesimally
	by the corresponding Lie algebra action.} In the settings considered in this work, such vectors generally account for
	only a vanishing fraction of the Hessian kernel, although at points in
	$\spT$ they account for most of it.
For the positive rescaling symmetries between adjacent layers in
$N^{3,\circ}$, these include the following four vectors:
\begin{align}\label{orb_vects}
	\bigl(E_{ii}W_1^\star,-W_2^\star E_{ii},0\bigr),
	\qquad i=1,2,
	\qquad \text{and} \qquad
	\bigl(0,E_{jj}W_2^\star,-W_3^\star E_{jj}\bigr),
	\qquad j=1,2,
\end{align}
where $E_{ij}$ denotes the matrix with a single nonzero entry in position
$(i,j)$. These are examples of what we call \emph{flat} vectors, roughly speaking,
vectors tangent to curves along which the objective is constant. Such vectors form a cone.\footnote{When a level set is definable in an o-minimal structure, this cone
	coincides with the \emph{tangent cone}, a central object in variational analysis
	\cite{van1998tame,rockafellar1998variational}.\label{ftnote:tangent_cone}}
We shall also consider a relaxed notion of flatness, in which the function difference is required only to vanish to higher order along a curve.	When the curve is an orbit of continuous
symmetries of the objective, as in \eqref{orb_vects}, we call the corresponding flat vector an
\emph{orbit vector}. Two further orbit vectors arise from the shift invariance
of the logits,
\begin{align}
	\bigl(0,0,\bones \e_j^\top\bigr),
	\qquad j=1,2,
\end{align}
where $\e_j$ denotes the $j$-th standard basis vector.

\begin{remark}
Under cross-entropy loss, logit-shift symmetry forces at least $h$ zero
Hessian eigenvalues at all points—critical or not—for any network with a fully
connected final layer; see \pref{thm:at_init_oblivious}.
\end{remark}

Flat vectors transverse to the orbit vectors are called
\emph{transverse flat vectors}. In the present setting, there are two, obtained
by perturbing $W_2$ along $E_{12}$ and $E_{21}$. These directions yield
Hessian kernel vectors by a similar mechanism and are summarized in
\pref{table:n3_symmetry_breaking}, together with the associated eigenvalue
multiplicities forced by the invariance groups in \eqref{intro:inv_grps}, using the notation introduced~below.
\begin{definition}[Multiplicity structure]\label{def:multi_def}
	We record an eigenvalue multiplicity structure using the multiset notation
	$\mset{m_1^{[q_1]},\ldots,m_r^{[q_r]}}$, meaning that multiplicity $m_i$
	occurs $q_i$ times. Coincidences are allowed, in which case the corresponding multiplicities add.
	When clear from context, we omit the double braces.
	Moreover, we suppress the exponent when $q_i=1$. For $s\leq q$, the notation
	$m^{[q]}|_{\alpha_1,\ldots,\alpha_s}$ means that $s$ of the $q$ eigenvalues
	have values $\alpha_1,\ldots,\alpha_s$.
\end{definition}

The mechanisms enforcing large Hessian kernels are more varied than this
brief discussion can convey. The two cases in \pref{table:hess_kernel}
illustrate distinct mechanisms, both reflecting local degrees of freedom
rather than finite-sample effects. In the first, symmetry forces large
Hessian kernels throughout open regions (whose points are generally
noncritical) as established in \pref{thm:at_positive}. In the second, well
specification implies that $\ker d_W N^{3}$ is contained in the Hessian
kernel, accounting for roughly half of the spectrum, as shown in
\pref{sec:spec_analysis_gaussian}. 

\begin{remark}[The flat minima conjecture]
	The notion used above should be distinguished from the flat minima
	conjecture. This literature investigates whether optimization tends to favor flatter
	regions of the loss landscape
	\cite{hochreiter1997flat,keskar2016large,jastrzkebski2017three,
		chaudhari2019entropy,simsekli2019tail}. Our notion of flatness is instead intended to capture the local degrees of 	freedom responsible for Hessian degeneracy and is invariant under invertible linear changes of coordinates. 
\end{remark}

\newcommand{\vpad}[1]{
	\raisebox{0pt}[\dimexpr\height+0.8ex\relax][\dimexpr\depth+0.8ex\relax]{#1}}
\begin{table}[h]
	\centering
	\small
	\renewcommand{\arraystretch}{1.35}
	\setlength{\tabcolsep}{6pt}
	\begin{tabular}{@{}
			>{\raggedright\arraybackslash}m{0.2\textwidth}|
			>{\centering\arraybackslash}m{0.23\textwidth}|
			>{\centering\arraybackslash}m{0.48\textwidth}
			@{}}
		\textbf{Parameter set}
		&
		\textbf{Dimension of the Hessian kernel}
		&
		\textbf{Multiplicity structure}
		\\
		\hline
		
		A.e. point with \(W_2,\b_2>0\), \quad\quad \emph{critical or not}
		&
		\vpad{
			{\boldmath 
				\(d(d-1)\)}
			{\scriptsize (all transverse)}
		}
		&
		$d(d-1)|_0,1^{[d^2+4d+1]}$
		\\
		\hline
		
		Global minimizers
		&
			{\boldmath \(d(d+2)\)}
			{\scriptsize (\(d^2\) transverse)}
		&
			\(
			d(d+2)|_0,\;
			\frac{(d-1)(d-2)}{2},\;
			\frac{d(d-3)}{2},
			(d-1)^{[4]},\;
			1^{[4]}
			\)
		\\
		\hline
\multicolumn{2}{c|}{}&
\begin{tabular}[t]{@{}l@{\quad}|l@{}}
	$W_1$ & ${\frac{(d-1)(d-2)}{2}, \frac{d(d-3)}{2},(d-1)^{[3]},1^{[2]}}$ \\\hline
	$\b_1$ & ${d-1,1}$ \\\hline
	$W_2$ & ${d(d-1)|_0,d-1,1}$ \\\hline
	$(\b_2,W_3)$ & ${(d-1)|_0,d-1,1^{[2]}}$ \\\hline
	$\b_3$ & ${1}$ 
\end{tabular}
	\end{tabular}
\caption{
Lower bounds on Hessian kernel dimensions for $N^3$ under \pref{setting:gaussian}, together with multiplicity structures for $d=h\ge4$. \textbf{Top.} Throughout an open region, points are typically noncritical, yet the
Hessian kernel is large and generically consists only of directions transverse
to the orbits of global symmetries
(see \pref{table:gauss_flat_directions}).
The proof uses generalized functions. Replacing ReLU by GELU removes these
transverse directions from the kernel
(see \pref{fig:lb_zeros}).
\textbf{Bottom.} At the global minimizers in \eqref{sq_loss_global}, the kernel dimension is larger and, under an invariant
inner product, the Hessian has large eigenvalue multiplicities, established
in \pref{thm:learning_gauss_analysis}, as do the corresponding layerwise
restrictions (see \pref{table:layerwise_restrictions}).
}
	\label{table:hess_kernel}
\end{table}

\paragraph{Alignment with the data.}
So far, we have emphasized one visible spectral consequence of Hessian
invariance, namely that the spectrum contains only a few distinct eigenvalues,
some necessarily with large multiplicity. This visible structure reflects a sharper representation-theoretic constraint
encoded by the so-called \emph{isotypic decomposition}
\cite{fulton1991representation}, recalled briefly in
\pref{sec:failure}. This decomposition also identifies parameter subspaces
preserved by the Hessian. 
For the groups considered here, a basis adapted to this decomposition
puts the Hessian in the block form
\begin{align}\label{block_iso}
	\begin{pmatrix}
		B^{(1)}\otimes I_{d_1} & & \\
		& \ddots & \\
		& & B^{(r)}\otimes I_{d_r}
	\end{pmatrix},
\end{align}
where each $B^{(i)}$ is a symmetric $m_i\times m_i$ matrix. Thus, each
eigenvalue of $B^{(i)}$ is repeated $d_i$ times in its block.
The restriction of the Hessian to each corresponding subspace can
then be analyzed separately, with its eigenvalue scales estimated from
alignment with the data. We
illustrate this procedure here, with detailed estimates in
\pref{sec:GN_scale}. In layer coordinates, the second differential has the
schematic~form
	\begin{align}\label{hessian_into_contributions}
		d_W^2\ploss(W)[\delta W]
		=
		\sum
		\E_\z\left[
		\mathrm{Left}_i(\z)\,\delta W_i\,\mathrm{Center}_{ij}(\z)\,\delta W_j\, \mathrm{Right}_j(\z)
		\right],
	\end{align}
where $\mathrm{Left}_i(\z)$, $\mathrm{Center}_{ij}(\z)$, and $\mathrm{Right}_j(\z)$ encode the data-dependent
transformations associated with the network layers. For $N^{3,\circ}$ at the
well-specified points $W^\star\in\spT$,
\eqref{hessian_into_contributions} specializes to
\begin{align}
	c_1
	\left\|
	\operatorname{diag}(\e_1-\e_2)
	\left(
	(\delta W_3)^\top(\e_1-\e_2)-\delta W_1\bones
	\right)
	-
	2\operatorname{diag}(\delta W_2)
	\right\|_2^2
	+
	c_2\|\delta W_1(\e_1-\e_2)\|_2^2,
\end{align}
where $c_1,c_2>0$ are explicit coefficients. One
Hessian-invariant subspace identified by the $D_4\times C_2$ action is
$\{(\delta W_1,0,0)\mid\delta W_1\bones=0\}$. On this subspace, the quadratic
form reduces to $c_2\|\delta W_1(\e_1-\e_2)\|_2^2$ and yields the smaller
nonzero doublet.

The underlying structure is already nontrivial in this example and becomes considerably more intricate in general. For instance, in \pref{setting:gaussian}, at the global minima with
$W_1^\star=I_d$, $W_3^\star=\bones^\top$, and $W_2^\star$ any positive
left-stochastic matrix, one finds
{\small
	\begin{align}\label{second_diff_setting_gauss}
d_W^2\ploss(W^\star)&[\delta W]
		=
		\frac12\|\delta W_1\|_F^2
		+
		\frac12\|\bones^\top\delta W_1\|_2^2
		+
		\frac1\pi
		\left(
		(\operatorname{tr}\delta W_1)^2
		+
		\operatorname{tr}((\delta W_1)^2)
		-
		2\|\operatorname{diag}(\delta W_1)\|_2^2
		\right)
		\\
		&\quad+
		\q^\top
		\left[
		(\delta W_1)^\top\bones
		+
		\left(1-\frac2\pi\right)
		\operatorname{diag}(\delta W_1)
		+
		\frac{2\operatorname{tr}(\delta W_1)}{\pi}\bones
		\right]
		+
		\q^\top
		\left[
		\left(1-\frac1\pi\right)I_d
		+
		\frac1\pi\bones\bones^\top
		\right]\q,\nonumber
	\end{align}}
\par\noindent	where $\q=(\delta W_2)^\top\bones+(W_2^\star)^\top\delta W_3^\top$.	
In \pref{sec:spec_analysis_gaussian}, we show how symmetry exposes the structure hidden in this expression by decomposing 
it into a sum of simple tensor-product forms (see \eqref{block_iso}). This perspective underlies the estimates reported in
\pref{fig:original_vs_symmetric_mix} and \pref{fig:sq_evs}, based on
\pref{prop:gn_expression_mix_gauss} and
\pref{prop:exact_gn_expression_gauss}, respectively.

\paragraph{NTK and Gram forms.}
The network map $N$ induces the map
$\netf:\weightspace\to L^2(\inspace,\dd_{\inspace};\outspace)$, where
$\dd_{\inspace}$ is the input marginal of $\dd$, by
$\netf(W)(\x)=N(\x;W)$. For
$f\in L^2(\inspace,\dd_{\inspace};\outspace)$, the \emph{NTK form}
\cite{jacot2018neural} is defined by
\begin{align}\label{ntkform}
	d_W^{\mathrm{NTK}}N[f]
	&\defeq
	\|(d_W\netf)^*f\|_{T_W\weightspace}^2,
\end{align}
where the adjoint is taken with respect to the Euclidean inner
product on $T_W\weightspace$ and the standard $L^2$ inner product.
We similarly define the \emph{network Gram form}
and the \emph{loss Gram form} by,~respectively,
\begin{align}\label{network_gram_form}
	&\|d_W\netf[\delta W]\|_{L^2(\inspace,\dd_{\inspace};\outspace)}^2,
	\qquad
	\left\|
	d_W\smloss(\,\cdot\,,\,\cdot\,;W)[\delta W]
	\right\|_{L^2(\dd)}^2 .
\end{align}
The operators associated with these forms, and related matrices, have been linked to optimization and generalization
\cite{oymak2019generalization,cohen2021gradient,jastrzebski2020break,li2020hessian}.

We illustrate the framework for the network Gram form
and the NTK form, taking $N^3$ under
\pref{setting:gaussian_mixture} with $d=k=2$,
$\dd=\cN_\pm$, and $\tau=1$.
The pointwise differential $d_WN^3(\x;W)[\delta W]$ admits the
architectural symmetry $\bar{\delta}\sfa_0(R)$, acting by
$(\x,W_1,\delta W_1)\mapsto(R\x,W_1R^{-1},\delta W_1R^{-1})$,
leaving the biases and their perturbations unchanged
(the argument is the same as in the bias-free case, where the symmetry
can be read off directly from \eqref{dn3circ}).
For the pointwise form $\|d_WN^3(\x;W)[\delta W]\|^2$,
we have the additional symmetry
$(W_3,\b_3,\delta W_3,\delta\b_3)
\mapsto(QW_3,Q\b_3,Q\delta W_3,Q\delta\b_3)$.
Together, these yield the
structural symmetry 
\begin{align}
	\bar{\delta}\sfs\colon
	(\x,W_1,W_3,\b_3,\delta W_1,\delta W_3,\delta\b_3)
	&\mapsto
	(R\x,W_1R^{-1},QW_3,Q\b_3,
	\delta W_1R^{-1},Q\delta W_3,Q\delta\b_3),
\label{eq:gram_structural_action}
\end{align}
Here $R\in\GL_2(\RR)$ and $Q\in\Od{2}$. Unlike the action in
\eqref{struc_combine_with_dist}, this action transforms $W$ along
with the data and tangent perturbations. Write
$\bar{\delta}_{\weightspace}\sfs$ for the parameter action and
$\bar{\delta}_{T_W\weightspace}\sfs$ for its tangent action.
By the definitions and orthogonality of $Q$, the pointwise network
Gram and NTK forms satisfy, respectively,
\begin{align}
	\left\|
	\bigl(d_{\bar{\delta}_{\weightspace}\sfs W}\netf
	[\bar{\delta}_{T_W\weightspace}\sfs\delta W]\bigr)(R\x)
	\right\|^2
	&=
	\left\|\bigl(d_W\netf[\delta W]\bigr)(\x)\right\|^2,
	\label{eq:network_gram_pointwise}\\
	\left\|
	\bigl((d_{\bar{\delta}_{\weightspace}\sfs W}\netf)
	(R\x)\bigr)^*[Q\v]
	\right\|^2
	&=
	\left\|\bigl((d_W\netf)(\x)\bigr)^*\v\right\|^2,
	\label{eq:pointwise_ntk_structural}
\end{align}
for $\v\in\RR^2$. The second identity additionally requires $R$ to be orthogonal.
In addition to the distributional symmetry $\sfd^\sfs$ in
\eqref{eq:distributional_swap}, we use
$\sfd^\sfr:(\x,\y)\mapsto(-\x,P_{(1,2)}\y)$
(see \pref{table:symmetry_examples}). Their input components generate $\tri{P_{(1,2)},-I_2}$,
preserving the input marginal $\dd_{\inspace}$ of $\dd=\cN_\pm$.
Restrict $R$ to this group. The induced action on
$f\in L^2(\inspace,\dd_{\inspace};\outspace)$ is
$(\bar{\delta}\sfs f)(\x)\defeq Qf(R\x)$. Applying the same composite argument as in
\eqref{struc_combine_with_dist} to the differential and its
adjoint, we obtain
\begin{align}
	\|d_{\bar{\delta}_{\weightspace}\sfs W}\netf
	[\bar{\delta}_{T_W\weightspace}\sfs\delta W]\|_{L^2}^2
	&=\|d_W\netf[\delta W]\|_{L^2}^2,
	\qquad
	d_{\bar{\delta}_{\weightspace}\sfs W}^{\mathrm{NTK}}N^3
	[\bar{\delta}\sfs f]
	=d_W^{\mathrm{NTK}}N^3[f].
	\label{eq:ntk_composite}
\end{align}
For the He initialization law $\nu^{\mathrm{He}}$
(explicitly given in \pref{sec:streamlined-three-layer}), we also have that
$\bar{\delta}_{\weightspace}\sfs$ preserves the initialization law
(see the initialization group in \pref{table:symmetry_groups}). Averaging, we find that
\begin{align}
	\EE_{W\sim\nu^{\mathrm{He}}}
	\|d_W\netf[\bar{\delta}_{T_W\weightspace}\sfs\delta W]\|_{L^2}^2
	&=\EE_{W\sim\nu^{\mathrm{He}}}\|d_W\netf[\delta W]\|_{L^2}^2,\\
	\EE_{W\sim\nu^{\mathrm{He}}}d_W^{\mathrm{NTK}}N^3[\bar{\delta}\sfs f]
	&=\EE_{W\sim\nu^{\mathrm{He}}}d_W^{\mathrm{NTK}}N^3[f].
	\label{eq:gram_ntk_initialization_invariance}
\end{align}
The resulting composite invariance group is isomorphic to
$\Od{2}\times C_2^2$.
For the initialization-averaged network Gram form, the hidden-unit
permutations enlarge this group to	$S_h^2\times\Od{2}\times C_2^2$.

The same composite argument applies to a deterministic infinite-width limit. Let $W_h$ have the He initialization law at width $h$,
and suppose that, for deterministic scaling factors $a_h>0$,
\begin{align}
	a_h d_{W_h}^{\mathrm{NTK}}N^3[f]
	&\xrightarrow{\mathbb P}d_\infty^{\mathrm{NTK}}[f]
	\qquad\text{for every }f\in L^2(\inspace,\dd_{\inspace};\RR^2),
	\label{eq:ntk_infinite_width_convergence}
\end{align}
where $d_\infty^{\mathrm{NTK}}$ is a deterministic limiting form.
The same identity and initialization symmetry give
\begin{align}
	a_h d_{W_h}^{\mathrm{NTK}}N^3[\bar{\delta}\sfs f]
	&\overset{d}{=}
	a_h d_{\bar{\delta}_{\weightspace}\sfs W_h}^{\mathrm{NTK}}N^3[\bar{\delta}\sfs f]
	=a_h d_{W_h}^{\mathrm{NTK}}N^3[f].
	\label{eq:ntk_random_form_symmetry}
\end{align}
Equality in law implies equality of the deterministic limits:
\begin{align}
	d_\infty^{\mathrm{NTK}}
	[\bar{\delta}\sfs f]
	&=d_\infty^{\mathrm{NTK}}[f].
	\label{eq:ntk_infinite_width_invariance}
\end{align}
Thus, averaging and passage to a deterministic limit yield invariance under the same composite action. The corresponding isotypic decomposition and spectral properties are summarized in \pref{table:ntk_limit_structure}.

\begin{table}[htbp]
	\centering
	\small
	\setlength{\tabcolsep}{3pt}
	\renewcommand{\arraystretch}{1.25}
	\begin{tabular*}{\linewidth}{@{\extracolsep{\fill}}ccc@{}}
		\toprule
		Invariance analysis
		& \makecell{Isotypic decomposition on $(\ker K_\infty)^\perp$}
		& \makecell{Positive eigenvalue decay}\\
		\midrule
		\makecell{Output invariance only $\Od{2}$}
		& $\mathrm{std}_{\Od{2}}^{\oplus\infty}$
		& \makecell{$\lambda_j=\Theta(j^{-2})$}\\
		\makecell{Full composite invariance $\Od{2}\times C_2^2$}
		& $\displaystyle
		\bigoplus_{\chi\in\widehat{C_2^2}}
		(\mathrm{std}_{\Od{2}}\otimes\chi)^{\oplus\infty}$
		& \makecell{$\lambda_{\chi,j}=\Theta(j^{-2})$}\\
		\bottomrule
	\end{tabular*}
\caption{Spectral analysis of the nonzero bounded operator $K_\infty$
	associated with $d_\infty^{\mathrm{NTK}}$. Incorporating distributional symmetries
	adds the $C_2^2$ input factor, refining the isotypic decomposition
	into four character sectors. Here $\widehat{C_2^2}$ is the set of
	four characters of $C_2^2$, and $\mathrm{std}_{\Od{2}}$ is the
	standard two-dimensional representation. Additional kernel analysis
	for the two-Gaussian law establishes $\Theta(j^{-2})$ decay of the
	positive eigenvalues in each sector and the multiplicity structure
	$\mset{\aleph_0|_0,\;2|_{\lambda_{\max}},\;2^{[\aleph_0]}}$.}
	\label{table:ntk_limit_structure}
\end{table}

In related work, \citet{maiti2021symmetry} derive the output symmetries relevant to the first
row of \pref{table:ntk_limit_structure}, with eigenvalue decay results due to \citet{bietti2021deep}. Here, the composite invariance
mechanism yields the additional $C_2^2$ input factor in the second row,
refining the isotypic decompositions of both the initialization-averaged
and limiting NTK operators and imposing stronger spectral constraints.
We defer to \cite{arjevani2027optsym} the complete analysis,
results analogous in purpose to \pref{thm:streamlined-N3-spectrum},
further examples, and a fuller treatment of the random-label
application in \pref{rem:random_labels}.

\subsection{Overview: symmetry breaking}\label{sec:analysis_overview}
We have shown how structural and distributional symmetries expose rich
Hessian invariances at suitable reference configurations, so far obtained
by passing to diminished architectures, and thereby impose a rigid
spectral structure. The spectral structure in the original settings,
here for $N^3$, is then recovered through distinct local and
nonlocal SB mechanisms (see \pref{fig:sb_concepts}). We now discuss further
ways to construct symmetric references and outline how SB accounts for the
phenomena in \pref{sec:pheno}. The discussion follows representative classes
of critical points, distinguishing mechanisms that may act simultaneously.

\begin{figure}[ht]
	\centering
	\includegraphics[width=0.85\textwidth]{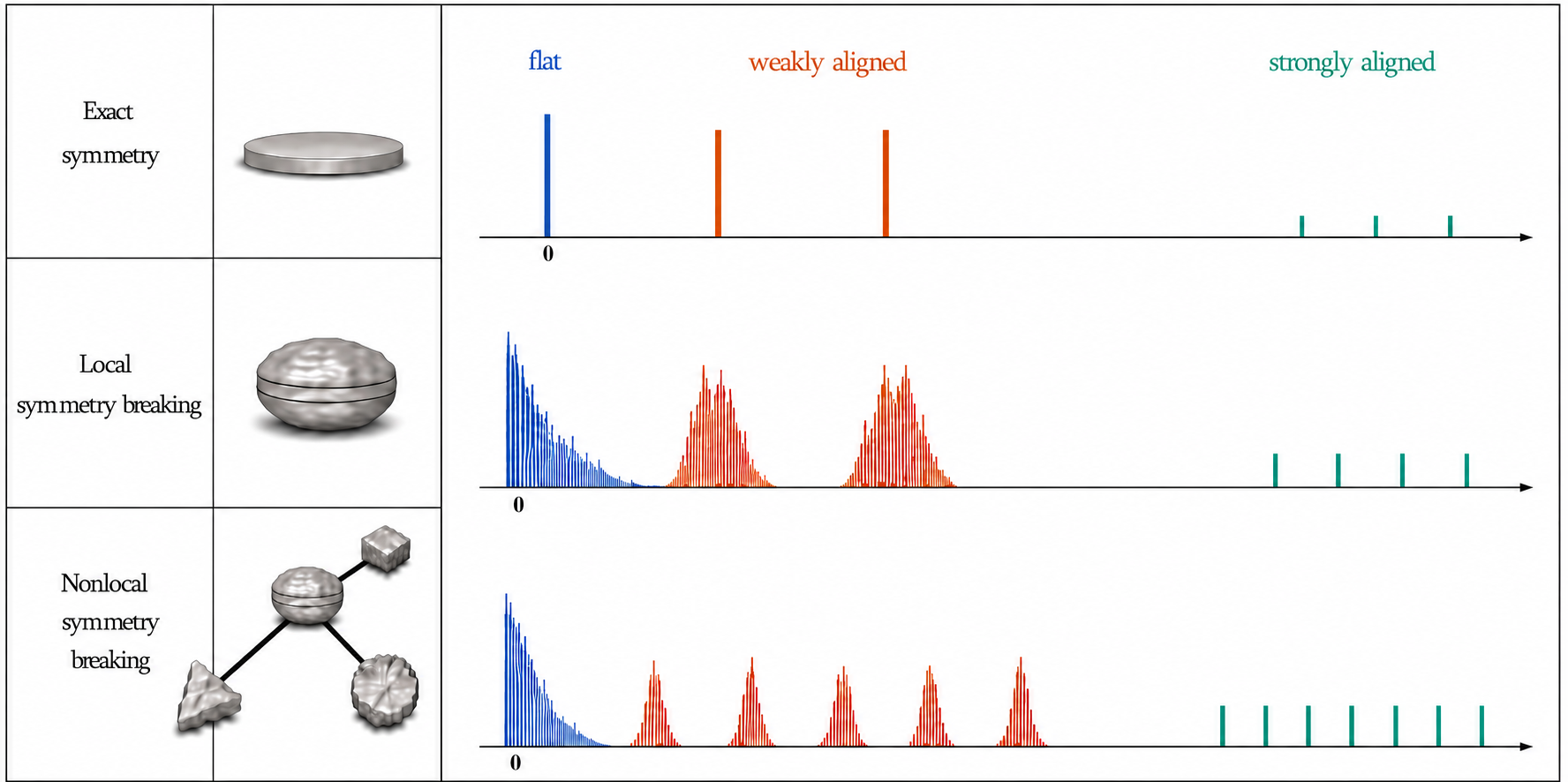}
	\caption{
			Schematic illustration of the SB mechanisms developed in this work.\\
			\textbf{Exact symmetry.}
A symmetric reference configuration is constructed by modifying the
setting. There, structural symmetries identified directly from
derivatives of the pointwise loss combine with symmetries of the data
distribution to produce Hessian invariances, even when weight symmetries are
trivial. The large Hessian kernel, the high eigenvalue multiplicities, and the
different eigenvalue scales now become transparent through the resulting rich
symmetry structure.\\
			\textbf{Local SB.}
			The spectra of the original models unfold as controlled SB splits these
			degeneracies into the observed hierarchy of clusters and outliers.\\
			\textbf{Nonlocal SB.}
			Depending on the initialization and algorithm, optimization may
			select critical points of lower symmetry, producing a more refined spectrum.
	}
	\label{fig:sb_concepts}
\end{figure}

\paragraph{Mixture of two Gaussians: resolutive SB.}
Recall that, properly embedded, $\spT$ in \eqref{Wstar} consists of global
minima for \pref{setting:gaussian_mixture} with $k=2$ and $\tau=1$ for each
of the networks $N^{3,\circ}$, $N^{3,\semicircle}$, and
$N^{3,\twoslashcircle}$ in \eqref{net3circ}. We demonstrated how symmetries
of the network, loss function, and data distribution combine to produce
the orthogonal Hessian invariance subgroup $D_4\times C_2$ recorded in
\eqref{intro:inv_grps}, thereby forcing all nonzero eigenvalues to have
multiplicity two. Varying $\tau>0$ in \pref{setting:gaussian_mixture},
with the realizing weights adjusted accordingly, preserves these Hessian
invariances and the kernel dimension
(see \pref{table:mixture_gaussians_resolution}).
For $N^3$, however, adding a full final-layer bias breaks the orthogonal
realization of the $D_4\times C_2$ action on the full parameter space. This is the \emph{resolutive} form of SB introduced earlier. Consequently, some multiplicity-two eigenvalues split into simple
eigenvalues, whose displacements are generally controlled by perturbation
theory. Beyond this perturbative control, the resolutive mechanism imposes additional
algebraic constraints on the splitting, as shown in
\pref{table:n3_symmetry_breaking} below. In particular, these constraints
bring the splitting within the scope of Cauchy's interlacing theorem. The full technical analysis,
including the kernel and invariant-subspace structure, is given in
\pref{sec:two_gaussians_analysis}.  At random initialization in the same setting,
we also establish a rich invariance group for the averaged Hessian that
explains the macroscopic near-zero bulk at a single draw; see
\pref{fig:flat_directions} and~\pref{thm:streamlined-N3-spectrum}.

\newcommand{\eigsTwoA}{
	\begin{tikzpicture}[x=0.72cm,y=0.72cm,baseline=-0.45ex]
		\draw[gray!45,line width=0.25pt] (0,0)--(3.15,0);
		\foreach \x in {0.85,1.55}
		\draw[line width=1.65pt,line cap=round] (\x,-0.18)--(\x,0.18);
	\end{tikzpicture}
}

\newcommand{\eigsTwoB}{
	\begin{tikzpicture}[x=0.72cm,y=0.72cm,baseline=-0.45ex]
		\draw[gray!45,line width=0.25pt] (0,0)--(3.15,0);
		\foreach \x in {1.55,2.05}
		\draw[line width=1.65pt,line cap=round] (\x,-0.18)--(\x,0.18);
	\end{tikzpicture}
}

\newcommand{\eigsThree}{
	\begin{tikzpicture}[x=0.72cm,y=0.72cm,baseline=-0.45ex]
		\draw[gray!45,line width=0.25pt] (0,0)--(3.15,0);
		\foreach \x in {0.55,1.55,2.55}
		\draw[line width=1.65pt,line cap=round] (\x,-0.18)--(\x,0.18);
	\end{tikzpicture}
}

\newcommand{\eigsSplit}{
	\begin{tikzpicture}[x=0.72cm,y=0.72cm,baseline=-0.45ex]
		\draw[gray!45,line width=0.25pt] (0,0)--(3.15,0);
		\draw[line width=1.65pt,line cap=round]
		(0.55,-0.18)--(0.55,0.18);
		\foreach \x/\y in {1.55/1.88,2.55/2.88} {
			\draw[line width=0.75pt,line cap=round]
			(\x,-0.18)--(\x,0.18);
			\draw[line width=0.75pt,line cap=round]
			(\y,-0.18)--(\y,0.18);
		}
	\end{tikzpicture}
}

\begin{table}[!htbp]
	\centering
	\small
	\renewcommand{\arraystretch}{1.22}
	\setlength{\tabcolsep}{4pt}
	\begin{tabular}{@{}
			l
			>{\raggedright\arraybackslash}p{0.29\textwidth}
			c
			c
			c
			@{}}
		\toprule
		Subarchitecture
		&  Orthogonal symmetries of $H$
		& $\dim\ker H$
		& Nonzero multiplicities
		& Schematic spectrum \\
		\midrule
		$N^{2,\circ}$
		& $D_4 \times C_2$
		& $4$
		& $2^{[2]}$
		& \eigsTwoA \\
		$\vdots$&$\vdots$&$\vdots$&$\vdots$&$\vdots$\\
		$N^{3,\circ}$
		& $D_4 \times C_2$
		& $8$
		& $2^{[2]}$
		& \eigsTwoB \\
		$\vdots$&$\vdots$&$\vdots$&$\vdots$&$\vdots$
		\\
		$N^{3,\semicircle}$
		& $D_4 \times C_2$
		& $10$
		& $2^{[3]}$
		& \eigsThree \\
		$N^{3,\twoslashcircle}$
		&  $D_4 \times C_2$
		& $11$
		& $2^{[3]}$
		& \eigsThree \\
		$N^3$
		& broken; but retained under a codimension-one restriction
		& $12$
		& $2,1^{[4]}$
		& \eigsSplit \\
		\bottomrule
	\end{tabular}
	\vspace{-0.1cm}
{\begin{align*}
	\text{Nested subspaces:}\quad
	\Theta_1=\Theta_{N^{2,\circ}}
	\subseteq \cdots \subseteq
	\Theta_{N^{3,\semicircle}}
	\subseteq
	\Theta_{N^{3,\twoslashcircle}}
	\subseteq
	\Theta_{N^3}=\Theta_r,
	\;
	\dim\bigl(\Theta_{i+1}/\Theta_i\bigr)=1,
\end{align*}}	\vspace{-0.8cm}
	\caption{Hessian spectra on $\spT$ along nested subarchitectures of $N^3$, viewed as
subspaces in its parameter space.
		In the symmetric cases, \(D_4 \times C_2\)-invariance
		forces the nonzero eigenvalues to have multiplicity two.
		Passing to \(N^3\) splits the reference eigenvalues, with the splitting
	constrained by Cauchy interlacing. This multiplicity structure is observed in
		\pref{fig:flat_directions}. Successive applications of Cauchy interlacing along the full chain give a sharper localization of the \(N^3\) eigenvalues, now relative to the spectra of the subarchitectures.}
	\label{table:n3_symmetry_breaking}
\end{table}

We now consider the global minima outside $\spT$. In
\pref{sec:two_gaussians_analysis}, we give a complete characterization of the
global minimizers of $N^{3,\circ}$ when $d=h=2$. For these other minima, the
spectrum again follows from resolutive SB, although with respect to a
different subspace (see \pref{table:sym_per_model}). Our focus on $\spT$ is
motivated by the empirical observation that Adam consistently converges to
this subset. A distinguishing feature of $\spT$ is its substantially smaller
curvature transverse to the kernel, indicating a concrete instance of an implicit bias of Adam (for studies of implicit bias in other settings, see \cite{safran2022linearregions,soudry2018implicit}). 

The same behavior persists when $h$ is increased, with Adam converging to embedded copies of $\spT$. This empirical observation is in line with a general principle followed throughout the paper. We first carry out the analysis in a small ambient parameter space
that is nevertheless sufficiently large to realize the relevant Bayes-optimal predictors. Passing to a larger ambient parameter space then embeds this model into a larger network, where the additional degrees of freedom account for extra flat directions. For more detailed treatments of such embeddings, higher-order flatness and implicit bias, see \cite{arjevani2026dsb_tf}.

\paragraph{Mixture of $k$ Gaussians: distributional SB.}
We now turn to mixtures of $k$ Gaussians described in
\pref{setting:gaussian_mixture}. We take the Gaussian centers to be drawn
uniformly on the sphere. Related choices lead to the same
qualitative results. Unlike in the preceding two-Gaussian setting, these centers are
almost surely in general position. To isolate SB arising from the data distribution from the resolutive effects
discussed above, we work with the bias-free model $N^{3,\circ}$. 
We take $d=k$ and $h=2k$ and work throughout with an explicit realization of
the Bayes-optimal predictor within $N^{3,\circ}$. Full details are given in
\pref{sec:spec_analysis_mix}.

 In this setting, a natural symmetric reference is constructed by solving the
 orthogonal Procrustes problem in operator norm. Organize the $k$ centers as
the rows of a matrix $C\in\mathbb{R}^{k\times k}$, and let $U$ be a nearest
orthogonal matrix to $C$. In our construction, we use the orthogonal factor
obtained from a singular value decomposition:
 \begin{align}\label{svd_reference}
C=P D_C Q^\top,
\qquad
U=PQ^\top .
 \end{align}
The rows of \(U\) are then used as the centers of the reference Gaussian
mixture, and the analysis proceeds through this symmetric case.  The exact isotypic decomposition, provided in \pref{sec:spec_analysis_mix}, identifies invariant subspaces of total dimension
$\Theta(k^2)$ which consist of directions weakly aligned with the data.
On these subspaces, \eqref{hessian_into_contributions} yields the
corresponding small-eigenvalue scales (see \pref{thm:mix_spectrum}). The resulting eigenvalue multiplicities and scales at the symmetric reference are summarized
in \pref{fig:original_vs_symmetric_mix}.  The spectral estimates then transfer to the original setting
through standard perturbation theory (the reference decomposition can
also be used directly).

\paragraph{Spontaneous SB.}
The SB mechanisms presented above were local in nature. An asymmetric configuration was analyzed by passing to a nearby symmetric configuration. We now turn to a nonlocal SB mechanism in which the problem parameters remain fixed, but the optimization dynamics selects points whose 
symmetric references have lower Hessian symmetry. The resulting spectrum is more refined and can be understood relative to these references~(\pref{fig:sb_concepts}). The corresponding symmetry types are encoded by \emph{stabilizers} under the invariance group $\Obj$ of the
objective. 

Given a point $W$ in the weight space $\weightspace$, its stabilizer is
\begin{align}\label{def_stabilizer}
	\stab_\Obj(W)\defeq \{\gamma\in \Obj:\gamma W=W\}.
\end{align}
The same terminology applies to arbitrary group actions, including linear actions on vector spaces such as spaces of symmetric matrices, tensors, linear functionals, or distributions, as well as actions on sets, such as the action of a permutation group on pairs of indices. When referring specifically to a parameter, we may also use the terms \emph{point stabilizer} or \emph{isotropy group} and sometimes write $\Gamma_W$.

Given a $\Gamma$-invariant objective, a central question is which symmetries in
$\Gamma$ are retained by its local minima and, more generally, by its critical
points.  This belongs to a classical theme in mathematics and physics concerning how the isotropy groups of solutions relate to the symmetries of the governing
equations. Here, the relevant equation is the gradient equation
$\nabla \ploss(W)=0$. A critical point is called \emph{SB} if its isotropy group is a
proper subgroup of $\Obj$. We say that \emph{spontaneous SB} occurs when such
a point has been selected by optimization. Our primary interest lies in cases in
which the residual symmetry remains substantial.

We illustrate this with $N^{3,\circ}$ under \pref{setting:gaussian}. As in the
earlier analyses, structural and distributional symmetries combine to produce
\integrated invariances, now at the level of the objective itself rather than
of its differential.  The resulting invariance
subgroup is $S_h^2\times S_d$, with the explicit actions of its factors given in
\pref{sec:symmetry_grps}. At the level of the objective, the composite invariance argument
in \eqref{loss_obj} follows the same underlying reasoning as that
used for $N^{1\nicefrac{1}{2}}$ in
\cite[Section~4.1]{arjevani2019principle}, with the conceptual
ingredients abstracted and formalized. As noted earlier, when used to identify Hessian invariances, the point-stabilizer approach employed there fails for the deeper networks already
considered in that work, whose Hessians retain rich invariances despite
trivial weight stabilizers (see \pref{table:gaussian_regression_resolution}).

When $h=d$, fix $0<\beta<1/d$ and set
$\alpha=1-d\beta$. Then the global~minimizer
\begin{align}\label{specificWstar}
	W^\star
	=
	\bigl(
	I_d,
	\alpha I_d+\beta\bones\bones^\top,
	\bones^\top
	\bigr)
\end{align}
is not fixed by the full group $S_h^2\times S_d$, but still retains
substantial symmetry. Its stabilizer contains
$\Delta S_d\defeq\{(\pi,\pi,\pi):\pi\in S_d\}$, the \emph{diagonal}
embedding of $S_d$ in $S_d^3$. Thus, $W^\star$ is an SB global minimizer.
We note in passing that even a substantial point stabilizer under the invariance group $\Obj$ need not recover the full Hessian invariance group, which may be strictly larger.

\paragraph{Goursat's lemma and O'Nan--Scott.}
We next discuss tools for examining possible isotropy groups under the $S_h^2\times S_d$-action.

Although some mechanisms can force or explain critical points with particular
stabilizers, $\Gamma$-invariance alone imposes no general lower bound on
the stabilizers of critical points.\footnote{For instance, suppose
$\Gamma$ is finite and a point $p$ has trivial stabilizer. Averaging a local bump
function with a nondegenerate critical point at $p$ gives a
$\Gamma$-invariant function for which the orbit $\Gamma p$ consists of
critical points, all with trivial stabilizer. Moreover, this critical orbit
persists under small $\Gamma$-invariant $C^2$-perturbations.} As discussed above in connection with the failure of weight symmetries
(\pref{sec:failure}), our settings provide many such instances in which
degrees of freedom transverse to $\Obj$-orbits reduce point stabilizers
without necessarily reducing the Hessian invariance group. Indeed, it was precisely this limitation that required us to develop a different route to establishing Hessian invariance. Nevertheless, the stabilizer lattice remains useful for an efficient organization of an otherwise highly complex set of critical points through the possible symmetry types (taking into consideration transverse flat directions), with the subgroup lattice serving as a tractable proxy when needed. Here, we limit ourselves to a high-level discussion of some general considerations rather than carrying out a systematic analysis, and refer to \cite{arjevani2024sc} for a fuller treatment.

A~convenient tool for understanding subgroups of \(S_h^2\times S_d\),
a direct product, is Goursat's lemma. Informally, it describes
them through subgroups of the factors and possible couplings between their
quotients. In the example in \eqref{specificWstar}, the same subgroup $S_d$ occurs
in each factor, with the three copies coupled diagonally. Representative subgroup
types of $S_d$ are shown in \pref{fig:onan_scott}, with the primitive case
further organized by the O'Nan--Scott theorem. These seemingly far-removed group-theoretic considerations organize the
symmetry types of lower-symmetry minima found by standard optimization
methods (see \pref{fig:sb_points} and \pref{rem:iso}).  Some of the isotropy groups are realized by minima with dead neurons, as discussed in~\pref{sec:spec_analysis_gaussian}.

\newcommand{\BoxScale}{1.12}   %
\newcommand{\ArrowScale}{0.88} %

\definecolor{diagpink}{RGB}{247,202,201}
\definecolor{blockgreen}{RGB}{198,239,206}
\definecolor{blockyellow}{RGB}{255,242,204}
\definecolor{blockpurple}{RGB}{217,210,233}
\definecolor{blockgray}{RGB}{217,217,217}
\colorlet{blockyellow}{white}

\newcommand{\RootGrid}{%
	\begin{tikzpicture}[scale={0.24*\BoxScale},baseline=(current bounding box.center)]
		\foreach \i in {0,...,4}{
			\foreach \j in {0,...,4}{
				\ifnum\i=\j
				\fill[diagpink] (\j,-\i) rectangle ++(1,-1);
				\else
				\fill[white] (\j,-\i) rectangle ++(1,-1);
				\fi
			}
		}
		\draw[line width=0.35pt] (0,0) rectangle (5,-5);
		\foreach \k in {1,...,4}{
			\draw[black!45,line width=0.22pt] (\k,0)--(\k,-5);
			\draw[black!45,line width=0.22pt] (0,-\k)--(5,-\k);
		}
	\end{tikzpicture}%
}

\newcommand{\IntransGrid}{%
	\begin{tikzpicture}[scale={0.24*\BoxScale},baseline=(current bounding box.center)]
		\foreach \i in {0,...,4}{
			\foreach \j in {0,...,4}{
				\ifnum\i=\j
				\fill[diagpink] (\j,-\i) rectangle ++(1,-1);
				\else
				\ifnum\i<4
				\ifnum\j<4
				\fill[white] (\j,-\i) rectangle ++(1,-1);
				\else
				\fill[blockgreen] (\j,-\i) rectangle ++(1,-1);
				\fi
				\else
				\fill[blockgreen] (\j,-\i) rectangle ++(1,-1);
				\fi
				\fi
			}
		}
		\draw[line width=0.35pt] (0,0) rectangle (5,-5);
		\foreach \k in {1,...,4}{
			\draw[black!45,line width=0.22pt] (\k,0)--(\k,-5);
			\draw[black!45,line width=0.22pt] (0,-\k)--(5,-\k);
		}
		\draw[line width=0.6pt] (4,0)--(4,-5);
		\draw[line width=0.6pt] (0,-4)--(5,-4);
	\end{tikzpicture}%
}

\newcommand{\ImprimitiveGrid}{%
	\begin{tikzpicture}[scale={0.20*\BoxScale},baseline=(current bounding box.center)]
		\foreach \i in {0,...,5}{
			\foreach \j in {0,...,5}{
				\pgfmathtruncatemacro{\bi}{floor(\i/3)}
				\pgfmathtruncatemacro{\bj}{floor(\j/3)}
				\ifnum\i=\j
				\fill[diagpink] (\j,-\i) rectangle ++(1,-1);
				\else
				\ifnum\bi=\bj
				\fill[white] (\j,-\i) rectangle ++(1,-1);
				\else
				\fill[blockgreen] (\j,-\i) rectangle ++(1,-1);
				\fi
				\fi
			}
		}
		\draw[line width=0.35pt] (0,0) rectangle (6,-6);
		\foreach \k in {1,...,5}{
			\draw[black!45,line width=0.18pt] (\k,0)--(\k,-6);
			\draw[black!45,line width=0.18pt] (0,-\k)--(6,-\k);
		}
		\draw[line width=0.6pt] (3,0)--(3,-6);
		\draw[line width=0.6pt] (0,-3)--(6,-3);
	\end{tikzpicture}%
}

\newcommand{\PaleyThirteenGrid}{%
	\begin{tikzpicture}[scale={0.105*\BoxScale},baseline=(current bounding box.center)]
		\foreach \i in {0,...,12}{
			\foreach \j in {0,...,12}{
				\pgfmathtruncatemacro{\diff}{mod(\j-\i+13,13)}
				\ifnum\diff=0
				\fill[diagpink] (\j,-\i) rectangle ++(1,-1);
				\else
				\ifnum\diff=1 \fill[blockgreen] (\j,-\i) rectangle ++(1,-1);
				\else\ifnum\diff=3 \fill[blockgreen] (\j,-\i) rectangle ++(1,-1);
				\else\ifnum\diff=4 \fill[blockgreen] (\j,-\i) rectangle ++(1,-1);
				\else\ifnum\diff=9 \fill[blockgreen] (\j,-\i) rectangle ++(1,-1);
				\else\ifnum\diff=10 \fill[blockgreen] (\j,-\i) rectangle ++(1,-1);
				\else\ifnum\diff=12 \fill[blockgreen] (\j,-\i) rectangle ++(1,-1);
				\else \fill[white] (\j,-\i) rectangle ++(1,-1);
				\fi\fi\fi\fi\fi\fi
				\fi
			}
		}
		\draw[line width=0.35pt] (0,0) rectangle (13,-13);
		\foreach \k in {1,...,12}{
			\draw[black!35,line width=0.12pt] (\k,0)--(\k,-13);
			\draw[black!35,line width=0.12pt] (0,-\k)--(13,-\k);
		}
	\end{tikzpicture}%
}

\newcommand{\PetersenGrid}{%
	\begin{tikzpicture}[scale={0.13*\BoxScale},baseline=(current bounding box.center)]
		\foreach \i in {0,...,9}{
			\foreach \j in {0,...,9}{
				\ifnum\i=\j
				\fill[diagpink] (\j,-\i) rectangle ++(1,-1);
				\else
				\fill[white] (\j,-\i) rectangle ++(1,-1);
				\fi
			}
		}
		\foreach \a/\b in {
			0/7,0/8,0/9,
			1/5,1/6,1/9,
			2/4,2/6,2/8,
			3/4,3/5,3/7,
			4/9,
			5/8,
			6/7
		}{
			\fill[blockgreen] (\b,-\a) rectangle ++(1,-1);
			\fill[blockgreen] (\a,-\b) rectangle ++(1,-1);
		}
		\draw[line width=0.35pt] (0,0) rectangle (10,-10);
		\foreach \k in {1,...,9}{
			\draw[black!35,line width=0.12pt] (\k,0)--(\k,-10);
			\draw[black!35,line width=0.12pt] (0,-\k)--(10,-\k);
		}
	\end{tikzpicture}%
}

\newcommand{\ProductActionGrid}{%
	\begin{tikzpicture}[scale={0.055*\BoxScale},baseline=(current bounding box.center)]
		\foreach \i in {0,...,24}{
			\foreach \j in {0,...,24}{
				\pgfmathtruncatemacro{\ri}{floor(\i/5)}
				\pgfmathtruncatemacro{\ci}{mod(\i,5)}
				\pgfmathtruncatemacro{\rj}{floor(\j/5)}
				\pgfmathtruncatemacro{\cj}{mod(\j,5)}
				\pgfmathtruncatemacro{\dr}{ifthenelse(\ri==\rj,0,1)}
				\pgfmathtruncatemacro{\dc}{ifthenelse(\ci==\cj,0,1)}
				\pgfmathtruncatemacro{\dist}{\dr+\dc}
				\ifnum\dist=0
				\fill[diagpink] (\j,-\i) rectangle ++(1,-1);
				\else
				\ifnum\dist=1
				\fill[blockgreen] (\j,-\i) rectangle ++(1,-1);
				\else
				\fill[white] (\j,-\i) rectangle ++(1,-1);
				\fi
				\fi
			}
		}
		\draw[line width=0.35pt] (0,0) rectangle (25,-25);
	\end{tikzpicture}%
}

\newcommand{\DiagonalTypeGrid}{%
	\begin{tikzpicture}[scale={0.026*\BoxScale},baseline=(current bounding box.center)]
		\input{diagonal_type_A5_fills.tex}
		\draw[line width=0.3pt] (0,0) rectangle (60,-60);
	\end{tikzpicture}%
}

\newcommand{\StructuralBox}[1]{%
	\begin{tikzpicture}[baseline=(current bounding box.center),scale=\BoxScale]
		\node[
		draw,
		rounded corners=2pt,
		inner sep=4pt,
		minimum width=1.25in,
		minimum height=0.62in,
		align=center
		] {\scriptsize #1};
	\end{tikzpicture}%
}

\def\labelone{-5.15}
\def\gridone{-4.35}
\def\groupone{-5.45}

\def\labeltwo{-7.35}
\def\gridtwo{-8.85}
\def\grouptwo{-10.15}

\newcommand{\TopLabel}[1]{%
	\parbox[c][8mm][c]{3.4cm}{\centering\small #1}%
}
\newcommand{\BottomLabel}[1]{%
	\parbox[c][8mm][c]{3.1cm}{\centering\small #1}%
}

\begin{figure}[ht]
	\centering
	\resizebox{\textwidth}{!}{%
		\begin{tikzpicture}[
			>=stealth,
			every node/.style={align=center},
			edge/.style={->,line width=0.55pt},
			lab/.style={font=\small,inner sep=0pt},
			box/.style={inner sep=0pt}
			]
			
			\def\yroot{0}
			\def\labelone{-3.10}
			\def\gridone{-4.20}
			\def\groupone{-5.25}
			
			\def\labeltwo{-7.25}
			\def\gridtwo{-8.55}
			\def\grouptwo{-9.90}
			
			\def\xstart{-1.8}
			\def\xgap{3.3}
			
			\node[box] (root) at ({0*\ArrowScale},{\yroot*\ArrowScale}) {
				\RootGrid\\[1.5mm]
				{\small \(S_d\)}
			};
			
			\node[lab] (intransLabel) at ({-6*\ArrowScale},{\labelone*\ArrowScale})
			{\TopLabel{Intransitive}};
			\node[box] (intransGrid) at ({-6*\ArrowScale},{\gridone*\ArrowScale})
			{\IntransGrid};
			\node[lab] (intrans) at ({-6*\ArrowScale},{\groupone*\ArrowScale})
			{\small \(S_4\times S_1\)};
			
			\node[lab] (imprimLabel) at ({0*\ArrowScale},{\labelone*\ArrowScale})
			{\TopLabel{Transitive imprimitive}};
			\node[box] (imprimGrid) at ({0*\ArrowScale},{\gridone*\ArrowScale})
			{\ImprimitiveGrid};
			\node[lab] (imprim) at ({0*\ArrowScale},{\groupone*\ArrowScale})
			{\small \(S_3\wr S_2\)};
			
			\node[lab] (prim) at ({6*\ArrowScale},{\labelone*\ArrowScale})
{\TopLabel{Primitive\\[-0.3mm] not 2-transitive}};
			
			\node[lab] (affineLabel) at ({(\xstart+0*\xgap)*\ArrowScale},{\labeltwo*\ArrowScale})
			{\BottomLabel{Affine}};
			\node[box] (affineGrid) at ({(\xstart+0*\xgap)*\ArrowScale},{\gridtwo*\ArrowScale})
			{\PaleyThirteenGrid};
			\node[lab] (affine) at ({(\xstart+0*\xgap)*\ArrowScale},{\grouptwo*\ArrowScale})
			{\small \(C_{13}\rtimes C_6\)};
			
			\node[lab] (almostLabel) at ({(\xstart+1*\xgap)*\ArrowScale},{\labeltwo*\ArrowScale})
			{\BottomLabel{Almost simple}};
			\node[box] (almostGrid) at ({(\xstart+1*\xgap)*\ArrowScale},{\gridtwo*\ArrowScale})
			{\PetersenGrid};
			\node[lab] (almost) at ({(\xstart+1*\xgap)*\ArrowScale},{\grouptwo*\ArrowScale})
			{\small \(S_5\curvearrowright\binom{[5]}2\)};
			
			\node[lab] (prodLabel) at ({(\xstart+2*\xgap)*\ArrowScale},{\labeltwo*\ArrowScale})
			{\BottomLabel{Product action}};
			\node[box] (prodGrid) at ({(\xstart+2*\xgap)*\ArrowScale},{\gridtwo*\ArrowScale})
			{\ProductActionGrid};
			\node[lab] (prod) at ({(\xstart+2*\xgap)*\ArrowScale},{\grouptwo*\ArrowScale})
			{\small \(S_5\wr S_2\)};
			
			\node[lab] (diagLabel) at ({(\xstart+3*\xgap)*\ArrowScale},{\labeltwo*\ArrowScale})
			{\BottomLabel{Diagonal type}};
			\node[box] (diagGrid) at ({(\xstart+3*\xgap)*\ArrowScale},{\gridtwo*\ArrowScale})
			{%
	\begin{tikzpicture}[scale={0.026*\BoxScale},baseline=(current bounding box.center)]
		\input{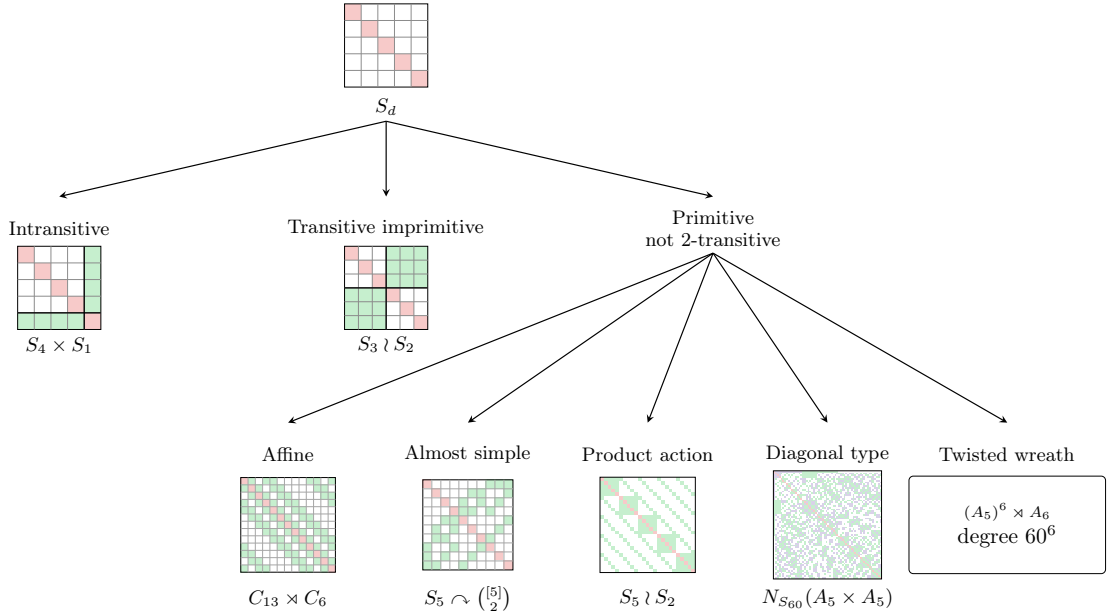}
		\draw[line width=0.3pt] (0,0) rectangle (60,-60);
	\end{tikzpicture}%
};
			\node[lab] (diag) at ({(\xstart+3*\xgap)*\ArrowScale},{\grouptwo*\ArrowScale})
			{\small \(N_{S_{60}}(A_5\times A_5)\)};
			
			\node[lab] (twistLabel) at ({(\xstart+4*\xgap)*\ArrowScale},{\labeltwo*\ArrowScale})
			{\BottomLabel{Twisted wreath}};
			\node[box] (twistGrid) at ({(\xstart+4*\xgap)*\ArrowScale},{\gridtwo*\ArrowScale})
			{\StructuralBox{$(A_5)^6\rtimes A_6$\\degree $60^6$}};
		\coordinate (rootArrowStart) at ([yshift=-0.1cm]root.south);
		
		\draw[edge] (rootArrowStart) -- ([yshift=0.10cm]intransLabel.north);
		\draw[edge] (rootArrowStart) -- ([yshift=0.10cm]imprimLabel.north);
		\draw[edge] (rootArrowStart) -- ([yshift=0.10cm]prim.north);
			
			\draw[edge] (prim.south) -- ([yshift=0.10cm]affineLabel.north);
			\draw[edge] (prim.south) -- ([yshift=0.10cm]almostLabel.north);
			\draw[edge] (prim.south) -- ([yshift=0.10cm]prodLabel.north);
			\draw[edge] (prim.south) -- ([yshift=0.10cm]diagLabel.north);
			\draw[edge] (prim.south) -- ([yshift=0.10cm]twistLabel.north);

		\end{tikzpicture}%
	}
\caption{Schematic overview of the subgroup structure relevant to nonlocal
	SB under the diagonal action of $S_d$ on $M(d,d)$, together with
	representative examples. The top node represents the full $S_d$ symmetry,
	while the lower nodes organize the possible symmetry types of
	SB points. The first level distinguishes the standard
	cases of intransitive, transitive imprimitive, and primitive subgroups.
	The primitive case is further organized according to the O'Nan--Scott~theorem.}
	\label{fig:onan_scott}
	
\end{figure}

\paragraph{Symmetric tensor decomposition: tangency arcs and 
	parameterized families of functions.}	
Symmetry of critical points may also be understood comparatively, relative to a given reference point. We conclude this section by describing a nonlocal SB mechanism, developed in \cite{arjevani2024sc}, which explains how nearby critical points may retain a substantial part of the reference point’s symmetry. To keep the manuscript focused, we do not develop the full theory but instead provide a very brief, high-level overview and representative examples.

Let \(f:\mathbb{R}^d\to\mathbb{R}\) be a \(C^1\) function, and let
\(\mathbf{c}\in\mathbb{R}^d\) be a critical point of \(f\). The
\emph{tangency set} of \(f\) at \(\mathbf{c}\) is defined\footnote{The tangency equation also admits the coordinate-free form
	\(df_{\x}\wedge(\x-\mathbf{c})^\flat=0\), where \(\flat\) is taken with
	respect to the chosen inner product, here the Euclidean one.}  by
\begin{align}
\mho_{\mathbf{c}}(f)
	\defeq
	\{\mathbf{x}\in\mathbb{R}^d
	\mid
	D_i f(\mathbf{x})(\mathbf{x}-\mathbf{c})_j
	=
	D_j f(\mathbf{x})(\mathbf{x}-\mathbf{c})_i,\ i,j\in[d]\}.
\end{align}
In particular, every critical point of \(f\) belongs to
\(\mho_{\mathbf{c}}(f)\). Arcs contained in this set and emanating from $\c$ are called
\emph{tangency arcs}.\footnote{These arcs play a role analogous to bifurcation branches. Global
	continuation of such branches can be studied using topological methods
	\cite{rabinowitz1971some}.} Such arcs, generically finitely many under mild regularity, may retain a substantial part of the symmetry of~$\c$, and critical points reached along them inherit this symmetry. 
Another quantity preserved along each arc is its index, which allows conclusions to be drawn about the extremal character of such critical points on spheres centered at $\c$.	 For maximal isotropy subgroups (see \pref{fig:onan_scott}), the existence of such arcs is guaranteed, while arcs with submaximal isotropy often appear as well. Applied iteratively at newly identified critical points, this construction reveals aspects of the broader symmetry structure of the landscape.

\begin{remark}\label{rem:iso}
	The admissible isotropy types of tangency arcs for the restriction to
	$M(d,d)$ of the diagonal $S_d$-action retained by $W^\star$ in
	\eqref{specificWstar} are analyzed in \cite{arjevani2024sc}.
	Some occur generically, including 
	$S_q\times S_{d-q}$, $S_q\wr S_2$ when $d=2q$, and
	$\prod_{i=1}^3 S_{n_i}$ for positive block sizes 
	or $(S_n^3)\rtimes C_3$ instead when the three blocks have equal
	size $n$. Generically occurring types occupy different positions in the subgroup
	and isotropy lattices, as illustrated in the table below.
	\begin{center}
		{\small
			\begin{tabular}{@{}ll@{}}
				\toprule
				Isotropy type & Position \\
				\midrule
				$S_{d-1}\times S_1$ ($d=5$ instance in \pref{fig:onan_scott})
				& Maximal subgroup and maximal isotropy subgroup \\
				$D_5$, $d=5$ (bottom panel of \pref{fig:sb_points})
				& Maximal isotropy subgroup, but not maximal subgroup \\
				$C_3\times S_2$, $d=5$ (middle panel of \pref{fig:sb_points})
				& Submaximal isotropy subgroup \\
				\bottomrule
		\end{tabular}}
	\end{center}
Here $D_5$ is the order-$10$ dihedral group. Stable occurrence need not,
however, be universal, as illustrated at $d=5$ by $\langle(123)\rangle$.
This latter type is submaximal even within its irreducible summand,
providing another counterexample to the maximal-isotropy conjecture
\cite{golubitsky1983benard} (cf. its formulation for minima \cite{michel1980minima}).
\end{remark}

If $f$ is definable in an o-minimal
structure, the arcs can be constructed definably
\cite{arjevani2023hidden}. Moreover, they always contain enough information to
determine whether a critical point is a local minimum, local maximum, or
saddle. Analogous conclusions hold generically for smooth
functions in the equivariant category. To illustrate the
construction in the semialgebraic case, consider the problem of decomposing a
real symmetric tensor into a sum of rank-one terms
\cite{arjevanifield2021tensor,arjevani2026geometry}. Standard approaches involve solving the following nonconvex optimization problem associated with an order-$n$ symmetric tensor $A$ on $\RR^d$: 
\begin{align} \label{prob:tensor opt}
	 \min_{ \stackrel{W\in M(k,d),}{\balpha\in\reals^k}}
	\nrm*{\sum_{i=1}^k \alpha_i \w_i^{\tens n} - A}^2. 
\end{align} 
For $k=d$ and $n=3$, we absorb each $\alpha_i$ into $\w_i$ and 
set $\balpha=\bones$. Taking
$A\defeq\sum_{i=1}^d\e_i^{\tens 3}$ then gives an $S_d^2$-invariance,
corresponding to independent permutations of the rows and columns of~$W$. Explicit constructions of real-analytic tangency arcs are given in
\pref{table:tensor_tan_arcs} below. A point on a numerically traced tangency arc for
\pref{setting:gaussian} is shown in \pref{fig:sb_points}.

\begin{table}[h]
	\renewcommand{\arraystretch}{1.2}
	\begin{scriptsize}
		\begin{tabular}{|c|c|c|c|c|c|c|c|c}
			\hline
			\textbf{Tangency arc}& \multicolumn{3}{c|}{\textbf{Source}}&  \multicolumn{3}{c|}{\textbf{Target}}
			\\\hline
$\left[
\begin{smallmatrix}
	0_{2,d}\\
	t\bones_{d-2}\bones_d^\top
\end{smallmatrix}
\right]$
&
$0$
&
Saddle
&
$S_d\times S_d$
&
$\left[
\begin{smallmatrix}
	0_{2,d}\\
	\Theta\bigl(d^{-1}\bigr)\bones_{d-2}\bones_d^\top
\end{smallmatrix}
\right]$
&
Saddle
&
$(S_2\times S_{d-2})\times S_d$
\\\hline
			$tI_d$ &$0$&Saddle& $S_d\times S_d$& $I_d$ &\makecell{Global \\minimum}&$\Delta S_d$\\
			\hline
			$tI_{d-1}\oplus [0]$&$0$&Saddle& $S_d\times S_d$& $I_{d-1}\oplus [0]$&Saddle &$\Delta S_{d-1}$\\
			\hline
			$I_{d-1}\oplus [t]$&$I_{d-1}\oplus [0]$&Saddle&$\Delta S_{d-1}$& $I_d$ &\makecell{Global \\minimum} &$\Delta S_d$\\
			\hline
			\tiny
			\makecell{$I_{d-2} \oplus 
				\begin{bsmallmatrix} a(t) & 
					0\\ t & 0\end{bsmallmatrix},$ \\ $~a(t) = \sum_{m=0}^\infty \binom{1/3}{m} 
				(-1)^mt^{3m}$}&$I_{d-1}\oplus [0]$&Saddle&$\Delta S_{d-1}$ & 	$I_{d-2} \oplus 
			\begin{bsmallmatrix} 0 & 
				0\\ 1 & 0\end{bsmallmatrix}$ &Saddle &$\Delta S_{d-1}$\\
			\hline
		\end{tabular}
	\end{scriptsize}
	\vspace*{0.1in}
\caption{Tangency arcs for the symmetric tensor decomposition problem 
	\eqref{prob:tensor opt}. Isotropy groups are listed up to conjugacy. The examples illustrate recursive application of the
	construction, leading to critical points with maximal and submaximal isotropy
	groups for the \(S_d^2\)-action. These are related, via Goursat's lemma, to the corresponding subgroup structures in the \(S_d\) factors, as illustrated in \pref{fig:onan_scott}.}
\label{table:tensor_tan_arcs}
\end{table}

	\begin{figure}[t]
			\includegraphics[scale=0.5]{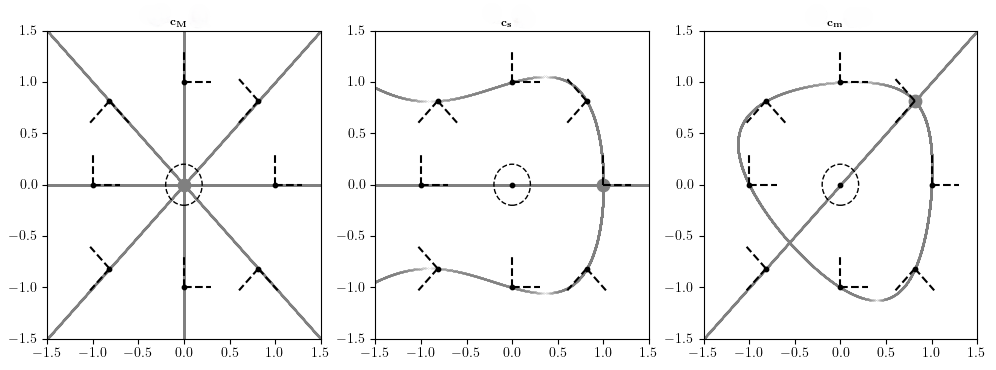}
		\caption{Tangency sets of $h$ in \eqref{eq:h} relative to different
			critical points. The tangency arcs may \emph{break} the symmetry of the reference critical point. Any critical point lying
			on an arc has at least the symmetry of that arc. 
			\label{fig:hc_tan}}
	\end{figure}

As a concrete low-dimensional illustration of how tangency arcs organize critical points by symmetry type, consider the following simple $B_2$-invariant function, with $B_n$ generally denoting the hyperoctahedral group (Weyl group of type~$B$),
\begin{equation}\label{eq:h}
	h(x,y) = x^4 + x^2y^2 + y^4 -2x^2 - 2y^2.
\end{equation}
The function $h$ is invariant under $(x,y) \mapsto (y,x)$, $x\mapsto-x$ and $y\mapsto-y$,  hence $B_2$-invariant. Computing, we find that $h$ has nine critical points: $\c_M = (0,0)$ (a maximum), $\c_s = (1,0)$ giving four critical points (saddles) in total owing to its \emph{orbit} $B_2(1,0) = \{ g(1,0)~|~g\in B_2\}$, and $\c_m = (\sqrt{2/3}, \sqrt{2/3})$ (a minimum) again yielding $|B_2(\sqrt{2/3}, \sqrt{2/3})|=4$ critical points. By equivariance, it suffices to compute the tangency set for the orbit representatives $\c_M, \c_s$ and $\c_m$. Although in this case the tangency set can be described by simple algebraic expressions, our interest lies rather in the geometric situation. Referring to \pref{fig:hc_tan}, the isotropy groups of tangency arcs
are either the full point stabilizer $(B_2)_{\c_\times}$ or its maximal
proper isotropy subgroups, $\times\in\{M,m,s\}$. In addition, we find that $\tanset_{\c_M}(h) \setminus \{\c_M\}$ has eight tangency arcs passing through the eight critical points away from the origin.  Iterating the analysis of the tangency set at these saddles reveals tangency arcs of index zero which connect to $\c_m$.

Tangency arcs can also be viewed as paths of critical points of the family of
functions $f(\x)-{\lambda}/{2}\|\x-\c\|^2$. Similarly, depending on the
phenomenon one wants to investigate, one can consider other parameterized
families of objectives and track the corresponding variation in local
geometry. For example, one may vary the input dimension and hidden-layer widths, as in
\cite{sagun2016eigenvalues,sagun2017empirical} and in the present work, or replace ReLU with leaky ReLU and vary the leakiness parameter continuously, as studied in \cite{ArjevaniField2020} (see also \cite{liu2025nonlinear}).
Layer widths can likewise be varied to track transitions from saddles to
minima and vice versa \cite{arjevani2021analytic,arjevani2022annihilation} (cf. \cite{simsek2025gaussianmultiindex}).
Other parameterized families of objective functions may arise from variations
in the data distribution, such as changing the number of classes or the
cluster variances (see \pref{table:mixture_gaussians_resolution}). Further
continuous deformations include taking
$D_C(t)=(1-t)D_C+tI$ in \eqref{svd_reference}.

\begin{figure}[ht]
	\centering
	\includegraphics[width=0.7\textwidth]{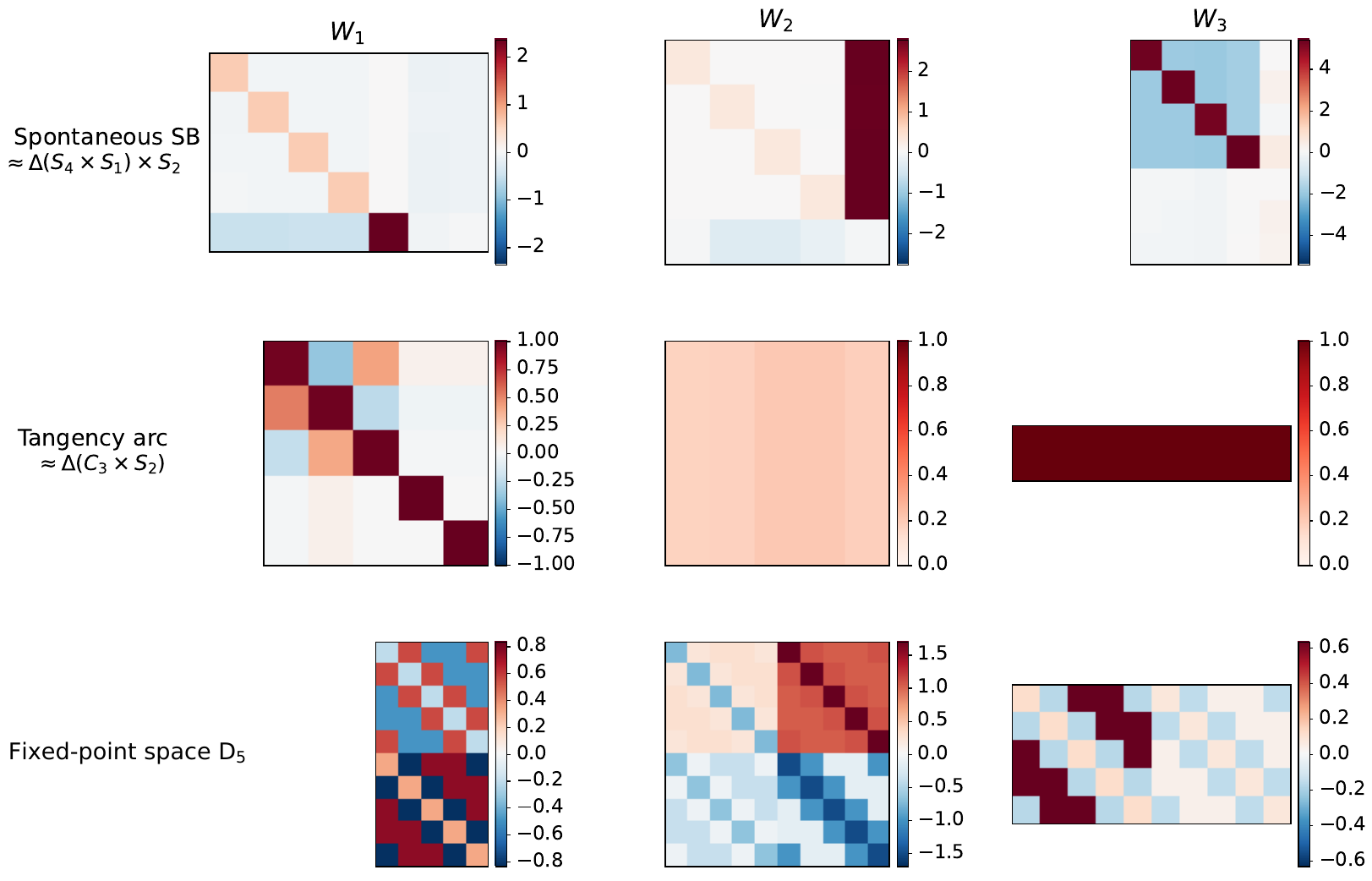}
\caption{Weight matrices illustrating three forms of SB. \textbf{Top.} Weights of a trained \(N^{3,\circ}\) network with \(d=k=7\) and
	\(h=5\), learned on the mixture of seven Gaussians in
	\pref{setting:gaussian_mixture}. The observed symmetry breaking is from
	\(S_7\times S_5^2\) to approximately
	\(\Delta(S_4\times S_1)\times S_2\).	\textbf{Middle.} A point on a tangency arc
	emanating from \(W^\star\) in \eqref{specificWstar}, with approximate isotropy
	\(\Delta(C_3\times S_2)\) (see \pref{rem:iso}). \textbf{Bottom.} Weight matrices in the fixed-point space of
	$D_5\leq \operatorname{AGL}(1,5)\leq S_5$, illustrating forced SB. Although $D_5$ is not maximal as a subgroup of
	$S_5$, it is a maximal \emph{isotropy} group for the diagonal
	action on $M(5,5)$. Indeed, $\operatorname{AGL}(1,5)$ is 2-transitive, so its
	fixed-point space in $M(5,5)$ coincides with the $S_5$-fixed-point space. Similar phenomena also occur on parameter manifolds, as illustrated in \pref{sec:layerwise} for parameters constrained to lie on a sphere.}
	\label{fig:sb_points}
\end{figure}

\section{Deep Symmetry Breaking}\label{sec:dsb}

Although the main analysis of this paper concerns three-layer networks, the framework itself is formulated in some generality. This extended version illustrates its application to selected deeper architectures and benchmark datasets. Detailed accounts of these applications
are in preparation as separate publications, to be collected later
in a broader treatment \cite{arjevani2026dsb_tf}.

We begin with a brief extension to deep fully connected networks and then
analyze AlexNet in greater detail, illustrating resolutive SB and spectral
refraction (\pref{sec:alexnet_superposition}). GNNs exhibit depth-dependent
eigenvalue scales (\pref{sec:gnns}), while a minimal Transformer example illustrates
Hessian invariance in in-context learning (\pref{sec:icl_symmetry}).
The unconstrained features model connects neural collapse with the
orthogonal realization of composite Hessian invariance
(\pref{sec:neural_collapse_composite_refraction}). Hidden feature--label laws of trained networks can themselves develop approximate symmetries, a phenomenon we refer to as \emph{neural
	registration} and examine in $N^5$ and ResNet-$18$
(\pref{sec:hidden_feature_symmetry}). These emerging symmetries provide
a further source of composite invariance, illustrated for $N^3$ trained
on MNIST in \pref{fig:mnist_n3_gn_metric_flow}. Finally, we ask what eigenvalue
multiplicities reveal about the symmetries that force them
(\pref{sec:can_hear}). Apart from AlexNet, these applications are presented
as very brief previews.

\begin{figure}[!t]
	\centering
	\makebox[\linewidth][c]{
		\includegraphics[width=\linewidth]{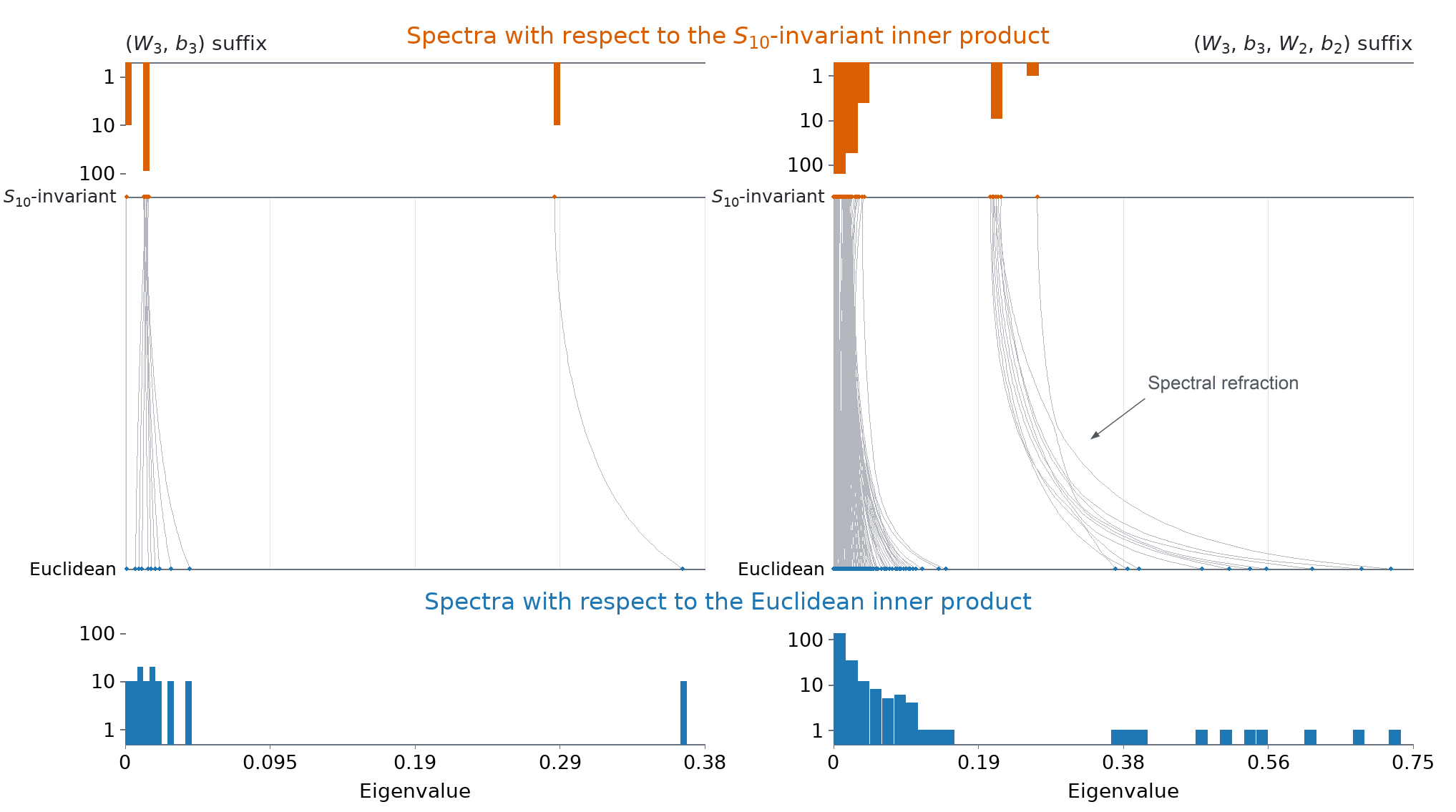}}
\caption{When Hessian invariances are compatible with the
	parameter-space geometry, they force large eigenvalue multiplicities.
	Here we follow the empirical spectra for $N^3$ trained  on MNIST
	as the parameter-space inner product varies from 
	 an $S_{10}$-invariant one to Euclidean. \textbf{Top.} Approximate invariance produces
	narrow bands at the invariant endpoint. \textbf{Bottom.} The action is incompatible
	with the Euclidean parameter-space geometry, and the bands refract
	into spectral clusters (\pref{rem:spectral_refraction}). The left
	and right panels concern the suffixes $(W_3,\b_3)$ and
	$(W_3,\b_3,W_2,\b_2)$, respectively. The approximate invariance arises from a phenomenon we refer to as
	\emph{neural registration}, whereby approximate symmetries emerge in
	hidden feature--label laws. The construction and
	analysis are given in \pref{sec:hidden_feature_symmetry}.}
	\label{fig:mnist_n3_gn_metric_flow}
\end{figure}

\subsection{AlexNet}
\label{sec:alexnet_superposition}

The first extension changes depth, not the learning problem: in
\pref{setting:gaussian}, $N^3$ is replaced by a fully connected network
$N^L$. The construction of Hessian invariances in
\pref{sec:spec_analysis_gaussian} carries over, mutatis mutandis. This
yields sharp asymptotic estimates for the Hessian spectrum as width and
depth vary. Spectral refraction takes a more involved form. In this case, the relevant permutations act on the full parameter space.
The AlexNet analysis below shows how this picture changes: restoring
convolutional directions restricts this permutation action on the full
quotient to $D_4$, producing resolutive SB.

Assume henceforth that $d\ge3$.  To demonstrate the core mechanism, we consider an AlexNet variant with a single channel in each
convolutional layer and without normalization, dropout, or biases. The first convolutional layer uses a
$3\times3$ kernel, unit stride, and zero padding.  These choices keep the
analysis transparent and allow the delta kernel to realize the target.
When this layer has a single channel, the absence of downsampling is
essential.  A stride larger than one makes the linear map defined by this
layer noninjective, while
$\bones_{d^2}^{\top}\sigma(\operatorname{vec}(\x))$ is not constant on its
fibers.  More general choices for the size of the kernel, the padding, and
the stride may be accommodated by introducing sufficiently many channels
in the first layer \cite{arjevani2026dsb_tf}.  Let $L_c$ and $L_f$ denote the numbers of convolutional
and fully connected layers, respectively, with the final readout included in $L_f$.  The dimensions of the
remaining layers are arbitrary.  Set $L=L_c+L_f$ and $n=d^2$.

Throughout the analysis, $\h_i$ denotes the vectorization of the output of
layer $i$, while $A_i$ denotes the corresponding linear map, with
\begin{align}
	A_i
	&=
	\begin{cases}
		\convmap(K_i),
		& 1\leq i\leq L_c,\\
		W_i,
		& L_c+1\leq i\leq L.
	\end{cases}
\end{align}
Here $K_i$ denotes the convolutional kernel and $W_i$ the weight matrix,
with $K_1\in\mathbb R^{3\times3}$, and by a slight abuse of notation, we use
$\convmap$ for the map induced by convolution in every convolutional layer,
with its domain, codomain, stride, padding, and number of channels
understood from context.   The network is then given by
\begin{align}
	\h_0
	&=
	\operatorname{vec}(\x),
	&
	\h_i
	&=
	\sigma(A_i\h_{i-1}),
	\quad
	1\leq i\leq L-1,
	&
	N^{\mathrm{AlexNet},\circ}(\x;W)
	&=
	A_L\h_{L-1}.
\end{align}
Let
\begin{align}
	\Omega^+
	&=
	\left\{
	W
	\ \middle|\
	K_2,\ldots,K_{L_c}>0,\quad
	W_{L_c+1},\ldots,W_{L-1}>0
	\right\},
\end{align}
with the readout unrestricted.  Define the tail map
$A_{L:2}=A_L\cdots A_2$.  We consider the following set of well-specified
global minima at which $dA_{L:2}$ is surjective,\footnote{If there exists one choice of the layer parameters for which
	$\operatorname{rank}dA_{L:2}=n$, then surjectivity is generic.  The
	$n\times n$ minors of the matrix representing $dA_{L:2}$ are polynomial
	functions of the layer parameters, and at least one of them is then not
	identically zero.  The rank-deficient parameters therefore form a
	proper algebraic subset, so the parameters for which $dA_{L:2}$ is
	surjective form an open dense subset of the positive orthant.}
\begin{align}
	\cG_{\mathrm{alexnet}}^{+}
	&=
	\left\{
	W\in\Omega^+
	\ \middle|\
	\begin{aligned}
		&K_1=sK_\delta,\quad
		A_{L:2}=s^{-1}\bones_n^\top
		\text{ for some }s>0,\quad
		\operatorname{rank}dA_{L:2}=n
	\end{aligned}
	\right\}.
	\label{eq:alexnet_positive_minima}
\end{align}
For any $W^+\in\cG_{\mathrm{alexnet}}^{+}$, all nonlinearities after the
first act linearly, and
\begin{align}
	N^{\mathrm{AlexNet},\circ}(\x;W^+)
	&=
	A_{L:2}\sigma\!\left(\convmap(K_1)\h_0\right).
\end{align}
Hence,
\begin{align}
	\label{alex_plus_diff}
	d_WN^{\mathrm{AlexNet},\circ}(\x;W^+)[\delta W]
	&=
	s\,dA_{L:2}[\delta W]\,\sigma(\h_0)
	+
	s^{-1}\bones_n^\top
	\operatorname{diag}\!\left(\mathbf 1_{\h_0>0}\right)
	\convmap(\delta K_1)\h_0.
\end{align}
The flat vectors coincide with the kernel of the model differential,
\begin{align}
	\ker d_WN^{\mathrm{AlexNet},\circ}(\,\cdot\,;W^+)
	&=
	\left\{
	\delta W
	\ \middle|\
	\delta K_1=\alpha K_1,\quad
	dA_{L:2}[\delta W]=-\alpha A_{L:2},\quad
	\alpha\in\mathbb R
	\right\}.
	\label{eq:alexnet_flat_directions}
\end{align}
The surjectivity of $dA_{L:2}$ implies that
\begin{align}
	\Phi
	\colon
	T_{W^+}\weightspace
	&\longrightarrow
	\mathbb R^{3\times3}\oplus\mathbb R^{1\times n},
	&
	\Phi(\delta W)
	&=
	\left(
	\delta K_1,\,
	dA_{L:2}[\delta W]
	\right)
\end{align}
is surjective and induces the isomorphism
\begin{align}
	\frac{
		T_{W^+}\weightspace
	}{
		\ker d_WN^{\mathrm{AlexNet},\circ}(\,\cdot\,;W^+)
	}
	&\cong
	\frac{
		\mathbb R^{3\times3}\oplus\mathbb R^{1\times n}
	}{
		\operatorname{span}\{(K_1,-A_{L:2})\}
	}.
	\label{eq:alexnet_kernel_quotient}
\end{align}
Under this identification, the model differential descends to
\begin{align}
	\overline{d_WN}_{\x}
	\bigl[(\delta K_1,\u^\top)\bigr]
	&=
	s\,\u^\top\sigma(\h_0)
	+
	s^{-1}\bones_n^\top
	\operatorname{diag}\!\left(\mathbf 1_{\h_0>0}\right)
	\convmap(\delta K_1)\h_0,
	\label{eq:alexnet_quotient_differential}
\end{align}
where $\u^\top=dA_{L:2}[\delta W]$.  The numerator on the right side of
\eqref{eq:alexnet_kernel_quotient} has dimension $n+9$, and the quotient
removes the one-dimensional subspace
$\operatorname{span}\{(K_1,-A_{L:2})\}$.  The parameter quotient therefore has
dimension $n+8$.  For the squared loss, well specification and the full
support of the Gaussian distribution imply that the kernels of the second
differential, the GN form, and the model differential coincide.  Hence, the
second differential descends to a nondegenerate bilinear form on the
quotient in~\eqref{eq:alexnet_kernel_quotient}.

\paragraph{Symmetries.}
In addition to the structural symmetries of the fully connected layers,
we have 
\begin{align}
	g\,\sigma(K*\x)
	&=
	\sigma\!\left((gK)*(g\x)\right),
	&
	\convmap(gK)
	&=
	P_g\convmap(K)P_g^\top,
	\qquad
	g\in D_4,
	\label{eq:alexnet_conv_equivariance}
\end{align}
where $P_g$ is the induced permutation of the spatial coordinates.  With
several channels, this action operates diagonally on the channels, while
permutations of channels can be compensated in adjacent~layers. On the subspace of the quotient represented by $(0,\u^\top)$,
\eqref{eq:alexnet_quotient_differential} admits the architectural symmetry
\begin{align}
	\bigl(
	\h_0,[(0,\u^\top)]
	\bigr)
	&\mapsto
	\bigl(
	P\h_0,[(0,\u^\top P^\top)]
	\bigr),
	\qquad
	P\in S_n.
\end{align}
On the entire quotient, convolutional equivariance restricts this action to
\begin{align}
	\bigl(
	\x,[(\delta K_1,\u^\top)]
	\bigr)
	&\mapsto
	\bigl(
	g\x,[(g\delta K_1,\u^\top P_g^\top)]
	\bigr),
	\qquad
	g\in D_4.
	\label{eq:alexnet_d4_quotient_action}
\end{align}
The distribution of $\h_0$ is invariant under $S_n$, while the Gaussian
input distribution is invariant under~$D_4$.  Canceling the corresponding
actions on the input therefore gives the composite actions
\begin{align}
	[(0,\u^\top)]
	&\mapsto
	[(0,\u^\top P^\top)],
	\qquad
	P\in S_n,
	\nonumber\\
	[(\delta K_1,\u^\top)]
	&\mapsto
	[(g\delta K_1,\u^\top P_g^\top)],
	\qquad
	g\in D_4.
	\label{alexnet_inv}
\end{align}
The first preserves the restrictions of the Hessian and the GN form to the
subspace represented by $(0,\u^\top)$, while the second preserves both
forms on the entire quotient.  Both transformations arise from orthogonal
actions on
$\mathbb R^{3\times3}\oplus\mathbb R^{1\times n}$ equipped with the
standard Euclidean inner product and preserve
$\operatorname{span}\{(K_1,-A_{L:2})\}$.  Hence, they descend to orthogonal
actions on the quotient.

\paragraph{Resolutive structure.}
A Gaussian moment calculation as in \eqref{second_diff_setting_gauss} gives
the second differential in the quotient coordinates as
\begin{align}
	d_W^2\ploss(W^+)
	\bigl[(\delta K_1,\u^\top)\bigr]
	&=
	\frac1{2s^2}
	\left(
	\lVert\convmap(\delta K_1)\rVert_F^2
	+
	\lVert\bones_n^\top\convmap(\delta K_1)\rVert_2^2
	\right)
	\nonumber\\
	&\quad+
	\frac1{\pi s^2}
	\left(
	\operatorname{tr}\!\left(\convmap(\delta K_1)^2\right)
	-
	2\left\lVert
	\operatorname{diag}\!\left(\convmap(\delta K_1)\right)
	\right\rVert_2^2
	+
	\left(
	\operatorname{tr}\!\left(\convmap(\delta K_1)\right)
	\right)^2
	\right)
	\nonumber\\
	&\quad+
	2\u^\top
	\left[
	\frac12\convmap(\delta K_1)^\top\bones_n
	+
	\left(
	\frac12-\frac1\pi
	\right)
	\operatorname{diag}\!\left(\convmap(\delta K_1)\right)
	+
	\frac{
		\operatorname{tr}\!\left(\convmap(\delta K_1)\right)
	}{\pi}\bones_n
	\right]
	\nonumber\\
	&\quad+
	s^2
	\left(
	1-\frac1\pi
	\right)
	\lVert\u\rVert_2^2
	+
	\frac{s^2}{\pi}
	\left(
	\u^\top\bones_n
	\right)^2.
	\label{eq:alexnet_effective_second_differential}
\end{align}
On the subspace represented by $(0,\u^\top)$, this reduces to
\begin{align}
	d_W^2\ploss(W^+)[(0,\u^\top)]
	&=
	s^2
	\left(
	1-\frac1\pi
	\right)
	\lVert\u\rVert_2^2
	+
	\frac{s^2}{\pi}
	\left(
	\u^\top\bones_n
	\right)^2.
\end{align}
Under the first action in \eqref{alexnet_inv},
$\mathbb R^n
=
\operatorname{span}\{\bones_n\}
\oplus
\bones_n^\perp$
is the decomposition into the trivial and standard representations of
$S_n$.  The preceding display gives the corresponding eigenvalues
\begin{align}
	\lambda_{\mathrm{triv}}
	&=
\left(s^2+\frac{n}{s^2}\right)\left(1+\frac{n-1}{\pi}\right),
	\quad
	\lambda_{\mathrm{std}}
	=
	s^2
	\left(
	1-\frac1\pi
	\right).
	\label{eq:alexnet_sn_eigenvalues}
\end{align}
Restoring $\delta K_1$ enlarges the quotient by only eight dimensions.
Cauchy interlacing, see \pref{table:n3_symmetry_breaking}, therefore implies
that $\lambda_{\mathrm{std}}$ retains multiplicity at least 	$(n-1)-8
	=
	n-9
	=
	d^2-9$. The pinned eigenspace can be identified explicitly.  Define the subspace of $\RR^n$ coupled to first-layer perturbations,
\begin{align}
U_c
	&=
	\left\{
	\frac12\convmap(\delta K_1)^\top\bones_n
	+
	\left(
	\frac12-\frac1\pi
	\right)
	\operatorname{diag}\!\left(\convmap(\delta K_1)\right)
	+
	\frac{
		\operatorname{tr}\!\left(\convmap(\delta K_1)\right)
	}{\pi}\bones_n
	\ \middle|\
	\delta K_1\in\mathbb R^{3\times3}
	\right\}.
	\label{eq:alexnet_coupling_space}
\end{align}
For the zero-padded $3\times3$ convolution and $d\geq3$, this space has
dimension nine and contains $\bones_n$.  The mixed term in
\eqref{eq:alexnet_effective_second_differential} vanishes on
$U_c^\perp$, which is also orthogonal to $\bones_n$.  Consequently, every
vector in $U_c^\perp$ is an eigenvector with eigenvalue
$\lambda_{\mathrm{std}}$, and	$\dim U_c^\perp
	=
	n-9$. 
	The full quotient therefore splits as
\begin{align}
	U_c^\perp
	\oplus
	\frac{
		\mathbb R^{3\times3}\oplus U_c
	}{
		\operatorname{span}\{(K_1,-A_{L:2})\}
	}.
	\label{eq:alexnet_resolutive_split}
\end{align}

We now pass from the restricted $S_n$ structure to the $D_4$-isotypic
decomposition of the full quotient. Using the standard irreducible representations $\mathrm A_1,\mathrm A_2,\mathrm B_1,\mathrm B_2,\mathrm E$ of $D_4$ \cite{mulliken1955report}, the character $\chi_d$ of its permutation representation on $\mathbb R^n$ decomposes as
\begin{align}
	\chi_d
	&=
	\frac{(d+1)(d+3)}{8}\chi_{\mathrm A_1}
	+
	\frac{(d-1)(d-3)}{8}\chi_{\mathrm A_2}
	+
	\frac{d^2-1}{8}
	\left(
	\chi_{\mathrm B_1}+\chi_{\mathrm B_2}
	\right)
	+
	\frac{d^2-1}{4}\chi_{\mathrm E},
	\qquad
	d\text{ odd},
	\nonumber\\
	\chi_d
	&=
	\frac{d(d+2)}{8}\chi_{\mathrm A_1}
	+
	\frac{d(d-2)}{8}\chi_{\mathrm A_2}
	+
	\frac{d(d-2)}{8}\chi_{\mathrm B_1}
	+
	\frac{d(d+2)}{8}\chi_{\mathrm B_2}
	+
	\frac{d^2}{4}\chi_{\mathrm E},
	\qquad
	d\text{ even}.
\end{align}
The character of the action on the perturbations of the kernel in the
first convolutional layer is
\begin{align}
	\chi_3
	&=
	(9,1,1,3,3)
	=
	3\chi_{\mathrm A_1}
	+
	\chi_{\mathrm B_1}
	+
	\chi_{\mathrm B_2}
	+
	2\chi_{\mathrm E}.
\end{align}
The quotient in \eqref{eq:alexnet_kernel_quotient} removes one copy of
$\mathrm A_1$.  Consequently, its isotypic decomposition is
{\small
\begin{align}
	&
	\begin{cases}
		\displaystyle
		\frac{d^2+4d+19}{8}\mathrm A_1
		\oplus
		\frac{d^2-4d+3}{8}\mathrm A_2
		\oplus
		\frac{d^2+7}{8}
		(\mathrm B_1\oplus\mathrm B_2)
		\oplus
		\frac{d^2+7}{4}\mathrm E,
		& d\text{ odd},
		\\[0.4cm]
		\displaystyle
		\frac{d^2+2d+16}{8}\mathrm A_1
		\oplus
		\frac{d^2-2d}{8}\mathrm A_2
		\oplus
		\frac{d^2-2d+8}{8}\mathrm B_1
		\oplus
		\frac{d^2+2d+8}{8}\mathrm B_2
		\oplus
		\frac{d^2+8}{4}\mathrm E,
		& d\text{ even}.
	\end{cases}
	\label{eq:alexnet_d4_isotypic_decomposition}
\end{align}
}
\par\noindent To describe the resolutive structure within the full $D_4$-isotypic
decomposition, we determine how the pinned eigenspace $U_c^\perp$ and the
remaining summand decompose.  
The map defining $U_c$ is $D_4$ equivariant
and, since $\dim U_c=9$, injective.  Hence,
	$U_c
	\cong
	3\mathrm A_1
	\oplus
	\mathrm B_1
	\oplus
	\mathrm B_2
	\oplus
	2\mathrm E$.
Consequently,  the remaining summand in
\eqref{eq:alexnet_resolutive_split} has dimension seventeen and decomposes
as
\begin{align}
	\frac{
		\mathbb R^{3\times3}\oplus U_c
	}{
		\operatorname{span}\{(K_1,-A_{L:2})\}
	}
	&\cong
	5\mathrm A_1
	\oplus
	2\mathrm B_1
	\oplus
	2\mathrm B_2
	\oplus
	4\mathrm E.
	\label{eq:alexnet_residual_d4}
\end{align}
Comparing this decomposition with
\eqref{eq:alexnet_d4_isotypic_decomposition} shows that $U_c^\perp$
contains all but five copies of $\mathrm A_1$, every copy of
$\mathrm A_2$, all but two copies of each of $\mathrm B_1$ and
$\mathrm B_2$, and all but four copies of $\mathrm E$.  All these copies
carry the common eigenvalue $\lambda_{\mathrm{std}}$.  
The generic multiplicity structures obtained from $D_4$ alone and after
taking the resolutive structure into account are
\begin{align}
	{\renewcommand{\arraystretch}{2}
		\begin{array}{c|c|c}
			D_4,\ d\text{ odd}
			&
			D_4,\ d\text{ even}
			&
			D_4\text{ with the resolutive structure}
			\\ \hline
			\displaystyle
			2^{\left[\frac{d^2+7}{4}\right]},
			1^{\left[\frac{d^2+9}{2}\right]}
			&
			\displaystyle
			2^{\left[\frac{d^2+8}{4}\right]},
			1^{\left[\frac{d^2+8}{2}\right]}
			&
			n-9,\,
			2^{[4]},\,
			1^{[9]}
	\end{array}}.
\end{align}
Thus, $D_4$ alone permits order $d^2$ distinct eigenvalues, whereas the
resolutive structure permits at most fourteen.  When $d=3$, the
multiplicity $n-9$ is absent, leaving at most thirteen distinct
eigenvalues. The restricted $S_n$ spectrum also locates the remaining eigenvalues
relative to $\lambda_{\mathrm{std}}$ and $\lambda_{\mathrm{triv}}$.  For
the $\mathrm B_1$ and $\mathrm B_2$ components, Cauchy interlacing~gives
\begin{align}
	\mu_{\mathrm B_1}^-
	&\leq
	\lambda_{\mathrm{std}}
	\leq
	\mu_{\mathrm B_1}^+,
	&
	\mu_{\mathrm B_2}^-
	&\leq
	\lambda_{\mathrm{std}}
	\leq
	\mu_{\mathrm B_2}^+.
	\label{eq:alexnet_B_interlacing}
\end{align}
If
$\mu_1^{\mathrm E}\leq\cdots\leq\mu_4^{\mathrm E}$
are the four eigenvalues of type $\mathrm E$, each occurring with
multiplicity two, then
\begin{align}
	\mu_1^{\mathrm E}
	&\leq
	\mu_2^{\mathrm E}
	\leq
	\lambda_{\mathrm{std}}
	\leq
	\mu_3^{\mathrm E}
	\leq
	\mu_4^{\mathrm E}.
	\label{eq:alexnet_E_interlacing}
\end{align}
Finally, if
$\mu_1^{\mathrm A_1}\leq\cdots\leq\mu_5^{\mathrm A_1}$
are the eigenvalues of type $\mathrm A_1$, then
\begin{align}
	\mu_1^{\mathrm A_1}
	\leq
	\lambda_{\mathrm{std}}
	\leq
	\mu_3^{\mathrm A_1},
	\qquad
	\mu_2^{\mathrm A_1}
	\leq
	\lambda_{\mathrm{std}}
	\leq
	\mu_4^{\mathrm A_1},
	\qquad
	\mu_3^{\mathrm A_1}
	\leq
	\lambda_{\mathrm{triv}}
	\leq
	\mu_5^{\mathrm A_1}.
	\label{eq:alexnet_A1_interlacing}
\end{align}

\paragraph{Asymptotic structure on the quotient.}
The asymptotic orders follow directly from
\eqref{eq:alexnet_effective_second_differential}.  Assume throughout that
$s=\Theta(1)$ and consider the seventeen-dimensional summand in
\eqref{eq:alexnet_resolutive_split} with respect to the invariant Euclidean
metric on the quotient.  The subspace
$\operatorname{span}\{(K_\delta,0),(0,\bones_n^\top)\}$ is contained in
the $\mathrm A_1$-isotypic component.  Before quotienting, the compression
of the representing operator to this subspace, relative to the ordered
basis $(K_\delta,0)$ and $(0,n^{-1/2}\bones_n^\top)$, is
\begin{align}
	\left(
	1+\frac{n-1}{\pi}
	\right)
	\begin{pmatrix}
		n/s^2 & \sqrt n\\
		\sqrt n & s^2
	\end{pmatrix}.
	\label{eq:alexnet_center_trivial_block}
\end{align}
In this ordered basis, the kernel is spanned by the coordinate vector
$(s,-\sqrt n/s)^\top$ and corresponds to the scaling direction removed in
\eqref{eq:alexnet_kernel_quotient}.  Its orthogonal complement within the
two-dimensional subspace projects onto the line spanned by the unit vector
\begin{align}
	\v_0
	&=
	\frac{
		\left(
		\sqrt n\,s^{-1}K_\delta,\,
		sn^{-1/2}\bones_n^\top
		\right)
	}{
		\sqrt{n/s^2+s^2}
	}.
\end{align}
On this unit vector, the second differential takes the value
\begin{align}
	d_W^2\ploss(W^+)[\v_0]
	&=
	\left(
	1+\frac{n-1}{\pi}
	\right)
	\left(
	\frac{n}{s^2}+s^2
	\right)
	=
	\Theta(n^2)
	=
	\Theta(d^4).
	\label{eq:alexnet_distinguished_quotient_value}
\end{align}
This gives a lower bound of order $n^2$ for the largest eigenvalue on the
seventeen-dimensional summand.  To determine the number of eigenvalues at
this scale and the scales of the remaining eigenvalues, we estimate the
restriction to $\operatorname{span}\{\v_0\}^{\perp}$ within this summand
and its mixed terms with $\operatorname{span}\{\v_0\}$.

The seventeen-dimensional summand is the orthogonal direct sum of
$\operatorname{span}\{\v_0\}$, the classes $[(\delta K_1,0)]$ with
$\delta K_1\perp K_\delta$, and the classes $[(0,\u^\top)]$ with
$\u\in U_c\cap\bones_n^\perp$.  These three summands have dimensions
$1$, $8$, and $8$, respectively.  The first carries $\mathrm A_1$, while
each of the other two carries
$2\mathrm A_1\oplus\mathrm B_1\oplus\mathrm B_2\oplus2\mathrm E$.

If $\delta K_1\perp K_\delta$, then
$\operatorname{tr}(\convmap(\delta K_1))=0$,
$\operatorname{diag}(\convmap(\delta K_1))=0$, and
$\lvert\operatorname{tr}(\convmap(\delta K_1)^2)\rvert
\leq\lVert\convmap(\delta K_1)\rVert_F^2$.  Index the kernel coefficients
by offsets $\a=(a_1,a_2)\in\{-1,0,1\}^2$.  The coefficient at offset
$\a$ occurs in exactly
$(d-\lvert a_1\rvert)(d-\lvert a_2\rvert)$ entries of the convolution
matrix, and the sets of entries associated with distinct offsets are
disjoint.  Moreover, each entry of
$\bones_n^\top\convmap(\delta K_1)$ is a sum of at most eight kernel
coefficients.  Consequently,
\begin{align}
	\frac{n}{4}\lVert\delta K_1\rVert_F^2
	&\leq
	\lVert\convmap(\delta K_1)\rVert_F^2
	\leq
	n\lVert\delta K_1\rVert_F^2,
	&
	\left\lVert
	\bones_n^\top\convmap(\delta K_1)
	\right\rVert_2^2
	&\leq
	8n\lVert\delta K_1\rVert_F^2.
	\label{eq:alexnet_convolution_counting_bounds}
\end{align}
Since $1/2-1/\pi>0$, there are constants $c_K,C_K>0$, independent of
$d$, such that
\begin{align}
	c_Kn\lVert\delta K_1\rVert_F^2
	&\leq
	d_W^2\ploss(W^+)[(\delta K_1,0)]
	\leq
	C_Kn\lVert\delta K_1\rVert_F^2.
	\label{eq:alexnet_kernel_uniform_bounds}
\end{align}
The restriction to the second eight-dimensional summand satisfies
\begin{align}
	d_W^2\ploss(W^+)[(0,\u^\top)]
	&=
	\lambda_{\mathrm{std}}\lVert\u\rVert_2^2,
	\qquad
	\u\in U_c\cap\bones_n^\perp.
	\label{eq:alexnet_tail_uniform_bound}
\end{align}

The mixed term between the two eight-dimensional summands is
$\frac12\u^\top\convmap(\delta K_1)^\top\bones_n$.  For every coordinate
$(i,j)$ with $2\leq i,j\leq d-1$, the corresponding entry of
$\convmap(\delta K_1)^\top\bones_n$ equals the sum of the kernel
coefficients.  Subtracting this value from every coordinate does not change
its inner product with $\u$, since $\u\perp\bones_n$.  The resulting
vector is supported on the coordinates satisfying
$i\in\{1,d\}$ or $j\in\{1,d\}$, of which there are $4d-4$, and its entries
are $O(\lVert\delta K_1\rVert_F)$.  Hence, for some constant
$C_{\mathrm{mix}}>0$ independent of $d$,
\begin{align}
	\left|
	d_W^2\ploss(W^+)
	\left[
	(\delta K_1,0),(0,\u^\top)
	\right]
	\right|
	&\leq
	C_{\mathrm{mix}}\sqrt d\,
	\lVert\delta K_1\rVert_F
	\lVert\u\rVert_2
	=
	C_{\mathrm{mix}}n^{1/4}\,
	\lVert\delta K_1\rVert_F
	\lVert\u\rVert_2.
	\label{eq:alexnet_mixed_uniform_bound}
\end{align}

For $\v\in\operatorname{span}\{\v_0\}$ and
$\u\in U_c\cap\bones_n^\perp$,
$d_W^2\ploss(W^+)[\v,(0,\u^\top)]=0$.
By \eqref{eq:alexnet_effective_second_differential},
\eqref{eq:alexnet_convolution_counting_bounds}, and the Cauchy--Schwarz
inequality, the mixed term between $\operatorname{span}\{\v_0\}$ and the
first eight-dimensional summand satisfies, for some constant $C_0>0$
independent of $d$,
\begin{align}
	\left|
	d_W^2\ploss(W^+)
	\left[
	\v,(\delta K_1,0)
	\right]
	\right|
	&\leq
	C_0n
	\lVert\v\rVert_2
	\lVert\delta K_1\rVert_F,
	\qquad
	\v\in\operatorname{span}\{\v_0\}.
	\label{eq:alexnet_v0_mixed_bound}
\end{align}
Relative to the diagonal quadratic form
$n^2\lVert\v\rVert_2^2+
n\lVert\delta K_1\rVert_F^2+
\lVert\u\rVert_2^2$,
the normalized mixed terms in
\eqref{eq:alexnet_mixed_uniform_bound} and
\eqref{eq:alexnet_v0_mixed_bound} have orders $O(n^{-1/4})$ and
$O(n^{-1/2})$, respectively.
Young's inequality therefore gives constants
$c_{\mathrm{res}},C_{\mathrm{res}}>0$, independent of $d$, such that, for
all sufficiently large $d$,
\begin{align}
	c_{\mathrm{res}}
	\left(
	n^2\lVert\v\rVert_2^2
	+
	n\lVert\delta K_1\rVert_F^2
	+
	\lVert\u\rVert_2^2
	\right)
	&\leq
	d_W^2\ploss(W^+)
	\left[
	\v+[(\delta K_1,\u^\top)]
	\right]
	\nonumber\\
	&\leq
	C_{\mathrm{res}}
	\left(
	n^2\lVert\v\rVert_2^2
	+
	n\lVert\delta K_1\rVert_F^2
	+
	\lVert\u\rVert_2^2
	\right),
	\label{eq:alexnet_residual_uniform_bounds}
\end{align}
for $\v\in\operatorname{span}\{\v_0\}$,
$\delta K_1\perp K_\delta$, and
$\u\in U_c\cap\bones_n^\perp$.  With respect to the invariant Euclidean metric on the quotient, the
min--max principle applied to the preceding two-sided comparison gives the
asymptotic spectral structure summarized in
\pref{table:alexnet_parameter_spectrum}. Since the operator representing
the second differential and the comparison matrix
$\operatorname{diag}(n^2,nI_8,I_8)$ both preserve the
$D_4$-isotypic components, the same argument on each component gives the
multiplicity structure in the table.

Let $H_{\mathrm{quo}}$ denote the
operator representing the descended second differential with respect to
the invariant Euclidean metric on the quotient. Realize the parameter
quotient as
$\ker d_WN^{\mathrm{AlexNet},\circ}(\cdot;W^+)^\perp$, and denote the
isomorphism induced by $\Phi$ on this subspace by $\widehat{\Phi}$. The
restriction $H$ of the ordinary Euclidean Hessian then satisfies
$H=\widehat{\Phi}^*H_{\mathrm{quo}}\widehat{\Phi}$, and the parameter
metric transported to the quotient is represented by
$(\widehat{\Phi}\widehat{\Phi}^*)^{-1}$.

Before quotienting, the parameter perturbations split orthogonally into
$\delta K_1$ and the perturbations of the remaining layers. Relative to
this decomposition, $\Phi=I_9\oplus dA_{L:2}$ and
$\Phi\Phi^*=I_9\oplus dA_{L:2}dA_{L:2}^*$. Since $I_9$ commutes with the action
on $\mathbb R^{3\times3}$, the condition
$P_g^\top dA_{L:2}dA_{L:2}^*P_g=dA_{L:2}dA_{L:2}^*$ for every $g\in D_4$ is sufficient
for the transported metric to be $D_4$ invariant. Surjectivity of $dA_{L:2}$
implies that $dA_{L:2}dA_{L:2}^*$ is positive definite, but does not imply this
invariance. For example, when $L=3$, take $A_{3:2}=A_3A_2$, where $A_2>0$ is an
invertible square fully connected map and
$A_3=s^{-1}\bones_n^\top A_2^{-1}$. Then
$dA_{3:2}dA_{3:2}^*=A_2^\top A_2+\lVert A_3\rVert_F^2I_n$, which does not commute
with some $P_g$ for generic $A_2$. Thus, $D_4$ invariance of the
transported metric is not guaranteed, and neither is the preservation of
the $D_4$-isotypic components by $H$. Regardless, the min--max principle
gives, for $1\leq j\leq n+8$,
\begin{align}
	\sigma_{\min}(\widehat{\Phi})^2\lambda_j(H_{\mathrm{quo}})
	&\leq
	\lambda_j(H)
	\leq
	\sigma_{\max}(\widehat{\Phi})^2\lambda_j(H_{\mathrm{quo}}).
	\label{eq:alexnet_global_singular_value_bound}
\end{align}
For every $W^+\in\cG_{\mathrm{alexnet}}^+$, the ordinary Euclidean
Hessian has kernel dimension $p-n-8$ and $n+8$ positive eigenvalues.
Membership in $\cG_{\mathrm{alexnet}}^+$ makes $\widehat{\Phi}$ invertible
for each fixed $d$, but does not control its singular values uniformly in
$d$. Since $\Phi=I_9\oplus dA_{L:2}$ before quotienting, it suffices that the
largest and smallest nonzero singular values of $dA_{L:2}$ be bounded above
and away from zero, respectively, uniformly in $d$. Under this condition,
\eqref{eq:alexnet_global_singular_value_bound} gives the same asymptotic
orders in parameter space.

\begin{table}[ht]
	\centering
	{\small
		\setlength{\tabcolsep}{4pt}
		\renewcommand{\arraystretch}{1.15}
		\begin{tabular}{c|p{.5\textwidth}|c}
			\multicolumn{3}{c}{
				\textbf{Spectrum on the parameter quotient}
			}
			\\[-1pt]
			\hline
			scale of eigenvalues
			&
			multiplicities forced by symmetry
			&
			total out of $n+8$
			\\
			\hline
			$\Theta(1)$
			&
			\begin{tabular}[t]{@{}l@{}}
				one eigenvalue of multiplicity $n-9$ \\
				two eigenvalues of multiplicity $2$ \\
				four simple eigenvalues
			\end{tabular}
			&
			$n-1$
			\\
			\hline
			$\Theta(d^2)$
			&
			\begin{tabular}[t]{@{}l@{}}
				two eigenvalues of multiplicity $2$ \\
				four simple eigenvalues
			\end{tabular}
			&
			$8$
			\\
			\hline
			$\Theta(d^4)$
			&
			one simple eigenvalue
			&
			$1$
			\\
			\hline
			\multicolumn{3}{c}{
				The kernel in parameter space has dimension $p-n-8$.
			}
		\end{tabular}
	}
	\caption{Asymptotic spectral structure on the parameter quotient.}
	\label{table:alexnet_parameter_spectrum}
\end{table}

\begin{figure}[h]
\centering
 \includegraphics[width=.7\textwidth]{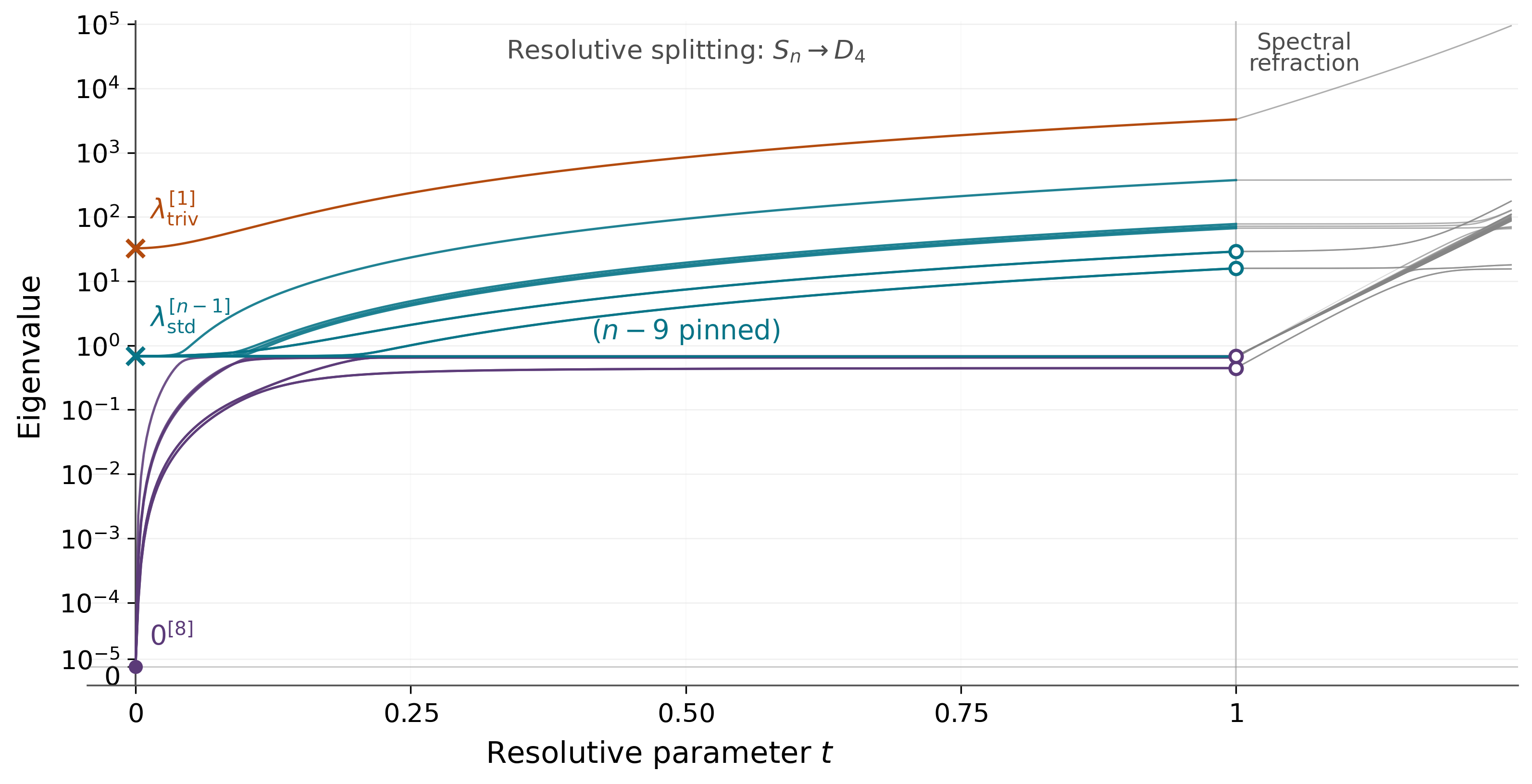}
	
\caption{
	Distinct eigenvalue branches of the Hessian on the parameter quotient at a
global minimum of a thin AlexNet, shown for $d=10$, $n=d^2=100$, and $s=1$.
The parameter $t\in[0,1]$ rescales $\delta K_1\mapsto t\,\delta K_1$.
At each $t$, the spectrum is computed modulo
$\operatorname{span}\{(K_\delta,-t\bones_n^\top)\}$.
At $t=0$, the positive spectrum is	$\lambda_{\mathrm{std}}^{[n-1]}\oplus\lambda_{\mathrm{triv}}$, with
	$\lambda_{\mathrm{std}}=1-\frac{1}{\pi}$ and $\lambda_{\mathrm{triv}}=1-\frac{1}{\pi}+\frac{n}{\pi}$.
	Eight branches emanate from zero, eight split from 	$\lambda_{\mathrm{std}}$, and the branch from
	$\lambda_{\mathrm{triv}}$ evolves into the largest eigenvalue, while 	$n-9=91$ eigenvalues remain exactly pinned at $\lambda_{\mathrm{std}}$. The  right-hand region shows spectral refraction as the	$D_4$-compatible quotient metric is deformed to the parameter-induced metric. 	Hollow circles mark the four $D_4$-forced multiplicity-two eigenvalues. Their subsequent splitting is below	the scale in the figure.
}

\end{figure}

\subsection{GNNs on Erd\H{o}s--R\'enyi graphs}
\label{sec:gnns}

We next consider a thin version of the Erd\H{o}s--R\'enyi GNN setting of
\citet{yehudai2021local}. For simplicity, we omit biases and the fully
connected suffix layers, and set the graph size, feature dimension, and
hidden width equal. The corresponding full analysis, including the
additional architectural components and the more involved spectral calculations, is given in \cite{arjevani2026dsb_tf}, which also treats the broader formulation within the standard message-passing framework
of \citet{gilmer2017neural,battaglia2018relational}.

Fix $n\geq3$, a depth $L\geq1$, and an edge probability $\rho\in(0,1)$. Let $A\sim\dd^{\mathrm{ER}}=G(n,\rho)$ be the adjacency matrix of an Erd\H{o}s--R\'enyi graph. We take the student and teacher to have common width and choose features and teacher weights that preserve simultaneous relabeling of vertices and feature coordinates. For $0<\beta<1$, define $X_0=(1-\beta)I_n+\beta\bones_n\bones_n^\top$, and
\begin{align}
	X_\ell=\relu\left(X_{\ell-1}W_{\ell,\mathrm{self}}+AX_{\ell-1}W_{\ell,\mathrm{nbr}}\right),\quad \ell=1,\ldots,L,\qquad N^{\mathrm{GNN},\circ}(A;W)=\bones_n^\top X_L\q.
	\label{eq:gnn-network}
\end{align}
Since $X_0$ is fixed throughout, we suppress it from the notation $N^{\mathrm{GNN},\circ}(A;W)$. Here $W_{\ell,\mathrm{self}},W_{\ell,\mathrm{nbr}}\in\RR^{n\times n}$ and $\q\in\RR^n$, so $\dim\weightspace=2Ln^2+n$. We take
\begin{align}
	W_{\ell,\mathrm{self}}^\star=W_{\ell,\mathrm{nbr}}^\star=I_n,\qquad \ell=1,\ldots,L,\qquad \q^\star=n^{-1}\bones_n.
	\label{eq:gnn-teacher-weights}
\end{align}
The resulting teacher satisfies
\begin{align}
	N^{\mathrm{GNN},\circ}(A;W^\star)=\frac1n\bones_n^\top(I_n+A)^LX_0\bones_n=\frac{1+(n-1)\beta}{n}\bones_n^\top(I_n+A)^L\bones_n.
	\label{eq:gnn-teacher-output}
\end{align}
The expected squared loss is then
\begin{align}
	\ploss(W)=\EE_{A\sim\dd^{\mathrm{ER}}}\left[\left(N^{\mathrm{GNN},\circ}(A;W)-N^{\mathrm{GNN},\circ}(A;W^\star)\right)^2\right].
	\label{eq:gnn-loss}
\end{align}
At $W^\star$, $\ploss$ is locally $C^2$, and its second differential equals the GN form.

\paragraph{Symmetries, kernel and spectrum.}
For every permutation matrix $P$, consider the (noninternal) structural symmetry $\grpstyle{s}^{\mathrm{GNN}}(P)=(\grpstyle{s}^{\mathrm{GNN}}_Z(P),\grpstyle{s}^{\mathrm{GNN}}_\Theta(P))$ given by
\begin{align}
	\grpstyle{s}^{\mathrm{GNN}}_Z(P)(A,X_0,y)&=(PAP^\top,PX_0P^\top,y),\notag\\
	\grpstyle{s}^{\mathrm{GNN}}_\Theta(P)W&=\left(\bigl(PW_{\ell,\mathrm{self}}P^\top,PW_{\ell,\mathrm{nbr}}P^\top\bigr)_{\ell=1}^L,P\q\right).
	\label{eq:gnn-structural-symmetry}
\end{align}
For the labels generated by the teacher above, the joint data law has the distributional symmetry
\begin{align}
	\grpstyle{d}^{\mathrm{ER}}(P):(A,X_0,y)\mapsto(PAP^\top,PX_0P^\top,y),
	\label{eq:gnn-distributional-symmetry}
\end{align}
since $A\overset{\mathrm d}=PAP^\top$, $PX_0P^\top=X_0$, and the teacher output in \eqref{eq:gnn-teacher-output} is unchanged by $A\mapsto PAP^\top$. Together, these yield the composite invariance of the objective $\grpstyle{c}^{\mathrm{GNN}}(P):=\grpstyle{s}^{\mathrm{GNN}}_\Theta(P)$, so $\ploss(\grpstyle{c}^{\mathrm{GNN}}(P)W)=\ploss(W)$. The teacher weights $W^\star$ are stabilized by $\left\langle\grpstyle{c}^{\mathrm{GNN}}(P):P\in S_n\right\rangle\cong S_n$. Consequently, $d_W^2\ploss(W^\star)$ is invariant under this $S_n$ action.

\begin{remark}
	The same composite invariance holds for random features whenever the
	joint law of $(A,X_0)$ is invariant under the $S_n$ action in
	\eqref{eq:gnn-structural-symmetry}.
\end{remark}
For a perturbation $\delta W$, set
\begin{align}
	\delta\bxi=\frac1n\left(\sum_{\ell=1}^L\delta W_{\ell,\mathrm{self}}-\sum_{\ell=1}^L\delta W_{\ell,\mathrm{nbr}}\right)\bones_n,\qquad \delta\boldsymbol{\eta}=\frac1n\left(\sum_{\ell=1}^L\delta W_{\ell,\mathrm{nbr}}\right)\bones_n+\delta\q.
	\label{eq:gnn-effective-coordinates}
\end{align}
A direct calculation gives
\begin{align}
	d_WN^{\mathrm{GNN},\circ}(A;W^\star)[\delta W]=\bigl(X_0(I_n+A)^{L-1}\bones_n\bigr)^\top\delta\bxi+\bigl(X_0(I_n+A)^L\bones_n\bigr)^\top\delta\boldsymbol{\eta}.
	\label{eq:gnn-model-differential}
\end{align}
The map $\delta W\mapsto(\delta\bxi,\delta\boldsymbol{\eta})$ is surjective onto the effective coordinate space $\RR^n\oplus\RR^n$. The composite $S_n$ action induces the diagonal permutation action on this space, whose isotypic decomposition is
\begin{align}
	\RR^n\oplus\RR^n\cong2\fs_{(n)}\oplus2\fs_{(n-1,1)}.
	\label{eq:gnn-effective-isotypic}
\end{align}
Quotienting the effective coordinate space by the kernel of the descended GN form gives the Hessian quotient. The resulting kernel, symmetry-forced multiplicities, and asymptotic scales are summarized in \pref{table:gnn-er-spectrum}. For a balanced stochastic block model with even $n$, the same analysis applies mutatis mutandis after replacing the $S_n$ symmetry by the wreath product $S_{n/2}\wr S_2$.

\begin{table}[t]
	\centering
	{\small
		\setlength{\tabcolsep}{4pt}
		\begin{tabular}{@{}ccccc@{}}
			\toprule
			Depth
			&
			Multiplicity pattern
			&
			$\fs_{(n)}$
			&
			$\fs_{(n-1,1)}$
			&
			$\dim\ker d_W^2\ploss$
			\\
			\midrule
			$L=1$
			&
			$(2n^2-1)|_0,\ (n-1),\ 1^{[2]}$
			&
			$\Theta(n^5),\ \Theta(1)$
			&
			$\Theta(n)$
			&
			$2n^2-1$
			\\
			$L=2$
			&
			$(4n^2-n)|_0,\ (n-1)^{[2]},\ 1^{[2]}$
			&
			$\Theta(n^7),\ \Theta(n^2)$
			&
			$\Theta(n^3),\ \Theta(n^{-1})$
			&
			$4n^2-n$
			\\
			$L\geq3$
			&
			$(2Ln^2-n)|_0,\ (n-1)^{[2]},\ 1^{[2]}$
			&
			$\Theta(n^{2L+3}),\ \Theta(n^{2L-2})$
			&
			$\Theta(n^{2L-1}),\ \Theta(n^{2L-6})$
			&
			$2Ln^2-n$
			\\
			\bottomrule
	\end{tabular}}
	\caption{Asymptotic spectrum of the Euclidean Hessian at $W^\star$ for fixed $\rho,\beta\in(0,1)$ and fixed depth $L$ as $n\to\infty$. Each $\fs_{(n)}$ branch is simple, while each $\fs_{(n-1,1)}$ branch has multiplicity $n-1$. The reference point $W^\star$ belongs to a larger local family of global minima, similarly to the AlexNet setting above. }
	\label{table:gnn-er-spectrum}
\end{table}

\subsection{Transformers: In-context learning} \label{sec:icl_symmetry}
As noted above, the same framework extends to general message-passing architectures, with the analysis applied to their message, aggregation, and update operators. We next apply this perspective to in-context learning in transformers. Following \citet{garg2022can}, we consider prompts comprising in-context examples followed by a query input. For a minimal illustration, we choose a function class exactly realizable by a single-head causal Transformer block with a residual connection.

Fix $K\geq1$. Draw $a$ uniformly from $\{-1,1\}$ and
$x_1,\ldots,x_{K+1}$ independently from $\Unif[-1,1]$,
independently of $a$. Set $f_a(x)=a+x$ and $y_i=f_a(x_i)$.
For $1\leq k\leq K$, the prompt is
$\prompt_k=(x_1,y_1,\ldots,x_k,y_k,x_{k+1})$, and the target is
$y_{k+1}$. Fix $h\geq6$ and orthonormal vectors
$\u_{\mathrm x},\u_{\mathrm c},\u_{\mathrm y}\in\bones_h^\perp$.
Encode each in-context example as an input token and an output token by
\begin{align}
	\h_{2i-1}
	&=x_i\u_{\mathrm x}+\sqrt{1-x_i^2}\,\u_{\mathrm c},
	&
	\h_{2i}
	&=\frac{y_i}{2}\u_{\mathrm y}
	+\sqrt{1-\frac{y_i^2}{4}}\,\u_{\mathrm c},
\end{align}
and encode the query input using the input-token formula with $x_i$
replaced by $x_{k+1}$. Every token is centered and has unit norm, so layer normalization with
fixed $\epsilon>0$ multiplies it by $c=(h^{-1}+\epsilon)^{-1/2}$.
With a residual connection and fixed output projection and readout, 
the model output and a realizing parameter choice are given by
\begin{align}
	N_k(\prompt_k;W)
	&=x_{k+1}+\sum_{j=1}^{2k}\alpha_j(W)\v^\top(c\h_j),
	&
	W^\star
	&=\left(0,0,\frac{2}{c}(-\u_{\mathrm x}+2\u_{\mathrm y})\right),
\end{align}
where $W=(\q,\k,\v)$ and $\alpha_j(W)$ are the softmax weights
obtained from the scalar query and the keys of the preceding $2k$ tokens. At $W^\star$, the model realizes $y_{k+1}$ exactly for every prompt. Every $R\in O(h)$ fixing $\operatorname{span}\{\bones_h,\u_{\mathrm x},
\u_{\mathrm c},\u_{\mathrm y}\}$
pointwise acts simultaneously on $\q,\k,\v$, preserves the loss, and
fixes $W^\star$. Together with joint sign reversal of $\q$ and $\k$,
this gives an orthogonal Hessian-invariance subgroup
$C_2\times O(h-4)$. The expected squared loss vanishes at $W^\star$, and so its Hessian there
equals the GN form. A direct calculation gives the multiplicity structure
$\mset{(3h-3)|_0,1^{[3]}}$.

\subsection{Neural collapse, UFM, and spectral refraction}
\label{sec:neural_collapse_composite_refraction}
We now turn from architectures used in practice to the symmetries and
invariances of a theoretical model motivated by neural collapse (NC)
\cite{papyan2020prevalence}. Concretely, we consider a bias-free version of
the unregularized unconstrained features model (UFM) of
\cite{mixon2022neural}. We recast the setting in an equivalent form that brings the SB framework naturally into view and makes the composite invariance argument transparent.

\paragraph{The unconstrained features model.} Let $k\geq2$, let $k\mid m$,
and put $q=m/k$. Assume $k\leq h<m$.  Let $\weightspace\defeq M(k,h)\times M(h,m)$ be the weight space,
with its standard Euclidean inner product. The network and data
distribution are given by
\begin{align}
	N^{\mathrm{UFM}}(\x;W,F)
	&\defeq WF\x,
	\qquad (W,F)\in\weightspace,
	&
	\dd^{\mathrm{UFM}}
	&\defeq\frac1m\sum_{c=1}^k\sum_{j=1}^q
	\delta_{(\e_{(c-1)q+j},\e_c)}.
\end{align}
Put $\mathbf Y\defeq I_k\otimes\bones_q^\top$. The expected squared loss is
\begin{align}
	\ploss^{\mathrm{UFM}}(W,F)
	&\defeq\EE_{(\x,\y)\sim\dd^{\mathrm{UFM}}}
	\|N^{\mathrm{UFM}}(\x;W,F)-\y\|^2
	=\frac1m\|WF-\mathbf Y\|_F^2.
\end{align}
Its global minimizers are exactly the pairs satisfying $WF=\mathbf Y$.
In particular, every global minimizer interpolates and has $\rk W=k$.

To describe the symmetries, let $\boldsymbol\tau=(\tau_1,\ldots,\tau_k)\in S_q^k$ and $\pi\in S_k$, and set
$P_{\boldsymbol\tau}\defeq\diag(P_{\tau_1},\ldots,P_{\tau_k})$ and
$P_\pi^{(m)}\defeq P_\pi\otimes I_q$. Write $P_\omega(\boldsymbol\tau,\pi)\defeq
P_\pi^{(m)}P_{\boldsymbol\tau}$, abbreviated to $P_\omega$
when the pair is clear. These matrices realize
the standard imprimitive action of
$S_q\wr S_k=S_q^k\rtimes S_k$ on the $m$ samples (see \pref{fig:onan_scott}).
\pref{table:ufm_symmetries} records the relevant architectural, structural, and
distributional symmetries, with the relevant general symmetries in \pref{table:symmetry_examples} repeated
for convenience.

\begin{table}[!htbp]
	\centering
	\renewcommand{\arraystretch}{1.2}
	{\small
		\begin{tabular}{cllll}
			\toprule
			\# & Symmetry type & Definition & Isomorphism class
			& Invariance of $\ploss^{\mathrm{UFM}}$ \\
			\midrule
			
			$\sfa_0^{\mathrm{UFM}}(M_0)$ & Architectural
			& $(FM_0^{-1},M_0\x)$
			& $\GL_m$ & No \\
			
			$\sfa_1^{\mathrm{UFM}}(M_1)$ & Architectural
			& $(WM_1^{-1},M_1F)$
			& $\GL_h$ & Yes \\
			
			$\sfs^{\mathrm{UFM}}(U)$ & Structural
			& $(UW,U\y)$
			& $O(k)$ & No \\
			
			$\sfd_{\mathrm{within}}^{\mathrm{UFM}}(\boldsymbol\tau)$
			& Distributional
			& $P_{\boldsymbol\tau}\x$
			& $S_q^k$ & No \\
			
			$\sfd_{\mathrm{between}}^{\mathrm{UFM}}(\pi)$
			& Distributional
			& $(P_\pi^{(m)}\x,P_\pi\y)$
			& $S_k$ & No \\
			
			$\sfc^{\mathrm{UFM}}$
			& \Integrated
			& $(P_\pi WM_1^{-1},M_1FP_\omega^\top)$
			& $\GL_h\times(S_q\wr S_k)$ & Yes \\
			\bottomrule
	\end{tabular}}
	\caption{Architectural, structural, and distributional
		symmetries of the UFM and the resulting composite invariance.
		Here $M_0\in\GL_m$, $M_1\in\GL_h$, and $U\in O(k)$.}
	\label{table:ufm_symmetries}
\end{table}
Following \pref{sec:failure}, we study Hessian invariance at global
minima directly through composite invariances rather than point stabilizers. The pointwise GN form is
\begin{align}
	d_{(W,F)}^{\mathrm{GN}}\smloss^{\mathrm{UFM}}
	(\x,\y;W,F)[\delta W,\delta F]
	&=2\|(\delta WF+W\delta F)\x\|^2.
\end{align}
At every global minimizer, the Hessian equals the expected GN form, and
we have
\begin{align*}
	d_{(W,F)}^2\ploss^{\mathrm{UFM}}[\delta W,\delta F]
	&=\frac2m\|\delta WF+W\delta F\|_F^2,
	&
	\ker d_{(W,F)}^2\ploss^{\mathrm{UFM}}
	&=\bigl\{(\delta W,\delta F):\delta WF+W\delta F=0\bigr\}.
\end{align*}
In particular, the kernel has dimension $h(k+m)-km$, the number of
independent flat directions along the manifold $WF=\mathbf Y$ of global
minimizers.

Although the Hessian here admits a direct analysis, this setting is
deliberately chosen to bring the SB framework into focus. The structural symmetries of the pointwise GN form compatible with the
distribution group in \pref{table:ufm_symmetries} have a complete
description. For $(\boldsymbol\tau,\pi)\in S_q\wr S_k$, every such symmetry induces
on the quotient by the kernel an action of the form
\begin{align}\label{eq:ufm_gn_quotient_action}
	\delta WF+W\delta F
	&\mapsto
	Q(\delta WF+W\delta F)P_\omega(\boldsymbol\tau,\pi)^\top,
	\qquad Q\in O(k).
\end{align}
The choice
$Q=I_k$ yields a composite invariance subgroup of
$\Obj_{(W,F)}^{[\mathrm{GN}]}$ isomorphic to $S_q\wr S_k$.
Under the Euclidean parameter metric, the largest subgroup whose action
can be chosen orthogonal is
\begin{align}\label{eq:ufm_orthogonal_quotient}
	\bigl\{(\boldsymbol\tau,\pi)\in S_q\wr S_k:
	P_\omega F^\top FP_\omega^\top=F^\top F\bigr\}.
\end{align}
When this condition fails for some $(\boldsymbol\tau,\pi)\in S_q\wr S_k$, spectral refraction may split the symmetry-forced multiplicities in the Euclidean Hessian spectrum.

\paragraph{Neural collapse and spectral refraction.} 
The distinction between orthogonal and nonorthogonal actions of this
Hessian invariance has a precise interpretation in terms of NC1 and NC2. To state their exact forms, define
\begin{align}
	\bmu_c&\defeq\EE_{\dd^{\mathrm{UFM}}}[F\x\mid\y=\e_c],
	&
	\bar\bmu&\defeq\frac1k\sum_{c=1}^k\bmu_c,\notag\\
	\Sigma_W
	&\defeq\frac1k\sum_{c=1}^k
	\EE_{\dd^{\mathrm{UFM}}}
	[(F\x-\bmu_c)(F\x-\bmu_c)^\top\mid\y=\e_c],
	&
	\Sigma_B
	&\defeq\frac1k\sum_{c=1}^k
	(\bmu_c-\bar\bmu)(\bmu_c-\bar\bmu)^\top.
\end{align}
NC1 is exact within-class collapse, that is, $\Sigma_W=0$.
NC2 says that the centered class means form a simplex equiangular tight frame (ETF), equivalently
\begin{align}
	\Sigma_B
	=\frac{\tr(\Sigma_B)}{k-1}
	\Pi_{\Span\{\bmu_c-\bar\bmu:c\in[k]\}},
	\quad
	\rk\Sigma_B=k-1,
\end{align}
where \(\Pi_V\) generally denotes orthogonal projection onto a subspace $V$.
For $k\leq h<m$, the full $S_q\wr S_k$ distribution group
gives orthogonal composite GN invariances if and only if the subgroup
in \eqref{eq:ufm_orthogonal_quotient} is the full group. This occurs precisely when NC1 and NC2 hold and the class means have equal norms. 

This gives a natural hierarchy. Composite GN invariance under the full wreath
group is universal across the global minima. NC1, a proper restriction on global minima when $h>k$, already makes the within-class factor $S_q^k$ act orthogonally, since
$FP_{\boldsymbol\tau}=F$ for every $\boldsymbol\tau\in S_q^k$. Adding NC2
and equal norms extends this orthogonal realization to the between-class factor and hence to the full wreath~group. 

At the NC1 level, restricting the representation in the second row
of \pref{table:ufm_quotient_isotypic} to $S_q^k$ splits each of
its $k$ induced copies into $k$ within-class standard components,
one per class and each of dimension $q-1$. The remaining
$k\triv\oplus k\stdsd_{k-1,1}$ restricts to $k^2\triv$.
When the between-class factor also acts orthogonally, the $k$
within-class components in each copy form an induced representation
of dimension $m-k$, as in the second row.  

At the minimal width $h=k$, if the subgroup in
\eqref{eq:ufm_orthogonal_quotient} is the full wreath group, then
\begin{align}\label{eq:ufm_classifier_gram}
	WW^\top
	&=\alpha \Pi_{\Span\{\bones_k\}}+\beta \Pi_{\bones_k^\perp},
	\qquad \alpha,\beta>0.
\end{align}
This permits an additional orthogonal $O(k-1)$ action on the GN form.
For $h>k$, \eqref{eq:ufm_classifier_gram} is no longer automatic.
If the subgroup in \eqref{eq:ufm_orthogonal_quotient} is the full
wreath group and \eqref{eq:ufm_classifier_gram} holds, the same
enlarged group $\Od{k-1}\times(S_q\wr S_k)$ acts. The third row records its quotient decomposition and multiplicity
structure in either case.

\begin{table}[!htbp]
	\centering
	\small
	\setlength{\tabcolsep}{2pt}
	\begin{tabular}{@{}p{.2\linewidth}p{.44\linewidth}p{.34\linewidth}@{}}
		\toprule
		Space & Isotypic decomposition & Multiplicity structure \\
		\midrule
		$T_{(W,F)}\weightspace$
		&
		$\begin{aligned}[t]
			&\bigl(h(k+m)-km+k\bigr)\triv
			\oplus k\stdsd_{k-1,1}\\
			&\oplus k
			\operatorname{Ind}_{S_q^k\rtimes S_{k-1}}^{S_q\wr S_k}
			\stdsd_{q-1,1}
		\end{aligned}$
		&
		$\begin{aligned}[t]
			&\bigl(h(k+m)-km\bigr)|_0,\\
			&1^{[k]},\quad (k-1)^{[k]},\quad (m-k)^{[k]}
		\end{aligned}$
		\\[3pt]
		\midrule
		$\dfrac{T_{(W,F)}\weightspace}
		{\ker d^2\ploss^{\mathrm{UFM}}}$
		&
		$
			k\triv\oplus k\stdsd_{k-1,1}
			\oplus k
			\operatorname{Ind}_{S_q^k\rtimes S_{k-1}}^{S_q\wr S_k}
			\stdsd_{q-1,1}
		$
		&
		$
			1^{[k]},\quad (k-1)^{[k]},
			(m-k)^{[k]}
$
		\\
		\midrule
		\makecell[l]{Quotient under\\
			$\Od{k-1}\times(S_q\wr S_k)$}
		&
		$\begin{aligned}[t]
			&(\triv\oplus\mathrm{std}_{O(k-1)})\\
			&\qquad\otimes\Bigl(
			\triv\oplus\stdsd_{k-1,1}
			\oplus
			\operatorname{Ind}_{S_q^k\rtimes S_{k-1}}^{S_q\wr S_k}
			\stdsd_{q-1,1}\Bigr)
		\end{aligned}$
		&
		$\begin{aligned}[t]
			&1,\quad (k-1)^{[2]},\quad (k-1)^2,\\
			&m-k,\quad (k-1)(m-k)
		\end{aligned}$
		\\
		\bottomrule
	\end{tabular}
		\caption{Isotypic decompositions at global minima under the full composite
			$S_q\wr S_k$ action and, in the third row where
			\eqref{eq:ufm_classifier_gram} holds, its $O(k-1)$ enlargement. Here $\mathrm{std}_{O(k-1)}$ is the standard representation of $O(k-1)$,
			and $\fs_\lambda$ denotes irreducible representations of symmetric groups as in \pref{sec:failure}.
				The first row uses a lift acting trivially on the Hessian kernel.
				In the induced representations, $\fs_{q-1,1}$ acts through the first
				$S_q$ factor, with the other $S_q$ factors and $S_{k-1}$ acting trivially.
				The resulting multiplicity structures use the notation of
				\pref{def:multi_def}. When the subgroup in
				\eqref{eq:ufm_orthogonal_quotient} equals $S_q\wr S_k$, the stated 
				multiplicities apply to the Euclidean Hessian;
				otherwise, spectral refraction may split them.}
	\label{table:ufm_quotient_isotypic}
\end{table}

The same analysis extends mutatis mutandis to deep linear UFMs. A fuller treatment of this setting and nonlinear UFM models, including the implications of neural collapse for Hessian invariance and questions of optimization and convergence, is deferred to \cite{arjevani2026dsb_tf}.

\subsection{Neural registration:  ResNet-$18$}
\label{sec:hidden_feature_symmetry}
So far, the applications have used symmetries of the data distribution. Here we bring into focus a phenomenon we call  \emph{neural registration}, in
which symmetries emerge in the joint hidden feature--label law. Paired with structural symmetries of a network suffix, these symmetries of the feature law yield composite invariances.
As these symmetries emerge across layers, the quality of registration is
reflected in the resulting composite invariance, which we use to
understand the structure of the original spectrum. The fuller treatment, including that of the layers preceding registration, is deferred to \cite{arjevani2026dsb_tf}.

Under mild conditions, approximate NC1, discussed in
\pref{sec:neural_collapse_composite_refraction}, gives rise to
approximate registration. 
A ResNet-$18$ trained on MNIST illustrates this connection:
registration strengthens as the NC1 diagnostic, defined in
\eqref{eq:trace_nc1_defect}, decreases toward
the final blocks (left panel of
\pref{fig:resnet18_mnist_registration_nc1_layers}). Registration is broader than collapse, however. At an explicit
five-layer global minimizer in \pref{setting:gaussian_mixture} with
$k=5$, the registration diagnostic decreases across hidden layers
while the NC1 diagnostic remains high (right panel of
\pref{fig:resnet18_mnist_registration_nc1_layers}).

We begin with some notation. For a network $N$ composed of $L$ modules
with parameter blocks $W_1,\ldots,W_L$, and for $1\leq i\leq j\leq L$,
write $N_{j:i}(\,\cdot\,;W_{j:i})$ for the composition of modules $i,\ldots,j$, where
$W_{j:i}=(W_i,\ldots,W_j)$. Write $N_{i:}$ for the composition of
modules $1,\ldots,i$ and $N_{:i}$ for the composition of modules
$i,\ldots,L$. Empty compositions are identity maps. For $i<j$,
\begin{align*}
	N(\,\cdot\,;W)
	&=
	N_{:j+1}(\,\cdot\,;W_{:j+1})
	\circ N_j(\,\cdot\,;W_j)
	\circ N_{j-1:i+1}(\,\cdot\,;W_{j-1:i+1})
	\circ N_i(\,\cdot\,;W_i)
	\circ N_{i-1:}(\,\cdot\,;W_{i-1:}).
\end{align*}
Let $\inspace_0\defeq\inspace$, and for $1\leq i\leq L$ let
$\inspace_i\defeq\RR^{d_i}$ be the output space of module $i$.
Write $\h_0(\x)=\x$ and
$\h_i(\x)=N_{i:}(\x;W_{i:})\in\inspace_i$ for $1\leq i\leq L$. For $0\leq i<L$, we have
$N(\x;W)=N_{:i+1}(\h_i(\x);W_{:i+1})$.
If perturbations are confined to the parameters of modules
$i+1,\ldots,L$, the pointwise loss and its parameter differentials depend
on $(\x,\y)$ only through $(\h_i(\x),\y)$.

\begin{definition}[Registration group]
	\label{def:feature_law_registration}
	Write 
	\begin{align*}
	\operatorname{Law}(\h_i,\y)\defeq \bigl(N_{i:}(\,\cdot\,;W_{i:}),
	\id_{\outspace}\bigr)_*\dd
	\end{align*}
	for the
	joint law of $(\h_i(\x),\y)$. Set $\aut(\inspace_i\times\outspace)
	\defeq\aut(\inspace_i)\times\aut(\outspace)$, with the usual
	identification of tuples of automorphisms with their componentwise actions. The \emph{registration group} at $i$ is defined~by
	\begin{align}
		\regis_i
		&\defeq\dist(\operatorname{Law}(\h_i,\y))
		\leq\aut(\inspace_i\times\outspace).
		\label{eq:feature_law_registration}
	\end{align}
\end{definition}
The composite invariance argument in \eqref{struc_combine_with_dist}
applies to the suffix $N_{:i+1}$, with $\regis_i$ as the distribution
group of its inputs and labels. For GN invariance, for example, we need
$\regis_i$-compatible structural symmetries of the pointwise GN form.
We demonstrate this on an $N^3$ network trained on MNIST with squared loss
(\pref{fig:mnist_n3_gn_metric_flow}). At a hidden layer of dimension $d$,
let $M=(\widehat{\bmu}_1,\ldots,\widehat{\bmu}_k)$ collect the empirical
class means and define
\begin{align}\label{eq:mean_matching_action}
	\rho(\pi)&=I_d+(MP_\pi-M)M^\dagger,\qquad \pi\in S_k.
\end{align}
These maps form an exact representation whenever
$P_\pi\ker M=\ker M$ for every permutation in the chosen group.
This holds automatically for full column rank, but also allows
dependent means, including centered simplex means. Write $\rho_j(\pi)$
for the action at hidden layer $j$ and set $\rho_3(\pi)=P_\pi$ at
the output. The baseline layerwise action is
\begin{align}
	\delta W_j&\mapsto
	\rho_j(\pi)\delta W_j\rho_{j-1}(\pi)^{-1},
	&
	\delta\b_j&\mapsto\rho_j(\pi)\delta\b_j,
	\qquad i<j\leq3.
\end{align}
For $N_{:3}$, we use this action directly. For $N_{:2}$, we conjugate
the representation by a fitted common change of coordinates to improve
its pointwise GN compatibility while preserving the group law.
The figure illustrates the resulting approximate composite invariance
and spectral refraction.

\paragraph{Collapse and registration.}
To quantify registration, here, we use the squared sliced
$2$-Wasserstein distance, which averages squared one-dimensional
$2$-Wasserstein distances uniformly over unit directions
\cite{bonneel2015sliced}, for its computational tractability and its
direct connection to NC1. Assume balanced classes and finite second moments. At feature cut $i$,
take a finite group $G\leq\aut(\inspace_i\times\outspace)$ whose label
actions permute the class encodings, and write
$\mathsf r=(\mathsf r_{\inspace_i},\mathsf r_{\outspace})\in G$.
Write $\h=\h_i$ and use the class-mean and covariance notation
$\bmu_a$, $\bar\bmu$, $\Sigma_W$, and $\Sigma_B$ from
\pref{sec:neural_collapse_composite_refraction}.
For $\tr(\Sigma_B)>0$, define
\begin{align}
	\operatorname{Reg}^{\mathrm{class}}_G
	\!\left(\operatorname{Law}(\h,\y)\right)
	&\defeq
	\frac{(k-1)\dim\inspace_i}{2k^2|G|\tr(\Sigma_B)}
	\sum_{\mathsf r\in G}\sum_{a=1}^k
	\operatorname{SW}_2^2
	\!\left((\mathsf r_{\inspace_i})_*
	\operatorname{Law}(\h\mid\y=\e_a),
	\operatorname{Law}(\h\mid\y=\mathsf r_{\outspace}\e_a)\right).
	\label{eq:feature_registration_defect}
\end{align}
The normalization uses the mean squared distance between distinct
class means, $2k\tr(\Sigma_B)/(k-1)$.

NC1 is also generally approximate.
For intermediate features, rather than the original diagnostic of
\cite{papyan2020prevalence}, we follow the trace normalization of
\cite{rangamani2023feature} 
\begin{align}
	\overline{\operatorname{NC1}}(\h)
	&\defeq\frac{\tr(\Sigma_W)}{\tr(\Sigma_W+\Sigma_B)}.
	\label{eq:trace_nc1_defect}
\end{align}
For linear feature actions satisfying
$\mathsf r_{\inspace_i}\bmu_a=\bmu_b$ whenever
$\mathsf r_{\outspace}\e_a=\e_b$, we have
\begin{align}
	\operatorname{Reg}^{\mathrm{class}}_G
	\!\left(\operatorname{Law}(\h,\y)\right)
	&\leq
	\frac{k-1}{2k}
	\left(
	\frac1{|G|}\sum_{\mathsf r\in G}
	\bigl(\|\mathsf r_{\inspace_i}\|_{\op}^2+1\bigr)
	\right)
	\frac{\overline{\operatorname{NC1}}(\h)}
	{1-\overline{\operatorname{NC1}}(\h)}.
	\label{eq:hidden_feature_wasserstein_defect}
\end{align}
Thus vanishing $\overline{\operatorname{NC1}}$ gives vanishing
registration defect when the average squared operator norm remains
bounded. \pref{fig:resnet18_mnist_registration_nc1_layers} places these two
diagnostics side by side across hidden layers of a five-layer ReLU
network on five Gaussians and a ResNet-$18$ trained on MNIST.
The ResNet-$18$ checkpoint comes from \cite{papyan2019measurements}.
We use the maps in \eqref{eq:mean_matching_action}. The sliced
discrepancies are averaged over adjacent class transpositions, using
$16$ projection directions per transposition. In the Gaussian case, the registration diagnostic
decreases across layers while $\overline{\operatorname{NC1}}$ remains
high, separating registration from collapse.

\begin{figure}[!htbp]
	\centering
	\includegraphics[width=\linewidth]{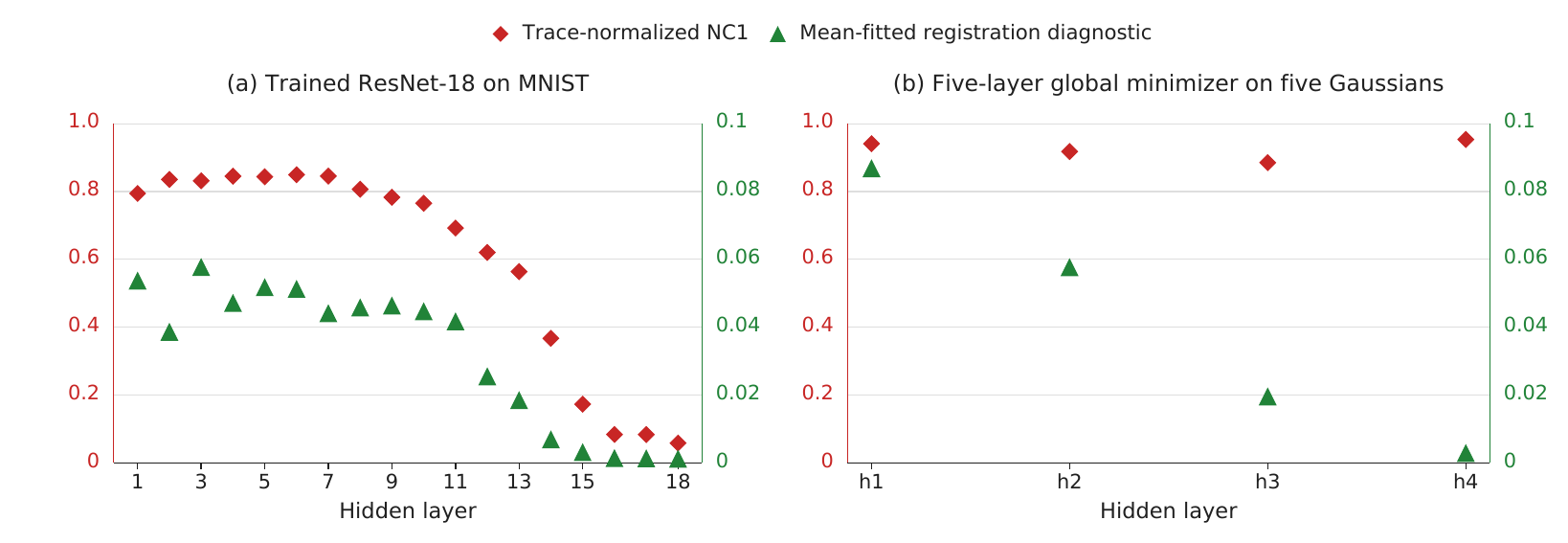}
	\caption{Registration and neural collapse across hidden layers.
		Red shows $\overline{\operatorname{NC1}}$ from
		\eqref{eq:trace_nc1_defect}, and green the normalized empirical
		registration diagnostic. The left uses $6{,}313$ MNIST images per class for a trained ResNet-$18$
		(indices denote feature cuts before successive parameter blocks). Both diagnostics
		decrease toward its final blocks. The right uses $10{,}000$
		independent samples per class from
		$\x\mid\y=\e_c\sim\cN(\e_c,I_5)$ at the explicit $N^5$
		global minimizer. Registration improves while the NC1 diagnostic
		remains high.}
	\label{fig:resnet18_mnist_registration_nc1_layers}
\end{figure}

\subsection{Can One Hear the Symmetry of a Critical Point?}\label{sec:can_hear}
The preceding applications lead from symmetries to spectral structure.
Specifically, they show how architectural, structural, and distributional
symmetries force large Hessian eigenvalue multiplicities. Inspired by Mark Kac’s celebrated question “Can One Hear the Shape of a Drum” \cite{kac1966can}, we now turn to the inverse problem:
\begin{center}
	What do eigenvalue multiplicities reveal about the symmetries that force them?
\end{center}
This question requires several qualifications before it can be approached.
For example, one should specify the ambient group in which to look,
here obtained by projecting the relevant structure group onto its
parameter component. Genericity is another qualification. For the natural
action of $S_d$ on $\RR^d$ by coordinate permutations, a fairly detailed account can be given (cf. \pref{table:isotropy_example}), provided there are no
eigenvalue coincidences beyond those forced by symmetry.

\section{Concepts and Methods}\label{sec:main}
We now give a  formal account of the overview, intentionally retaining some overlap with the earlier discussion for continuity. Selected proof sketches are included here to illuminate the sources of the phenomena considered, with the remaining proofs and technical details given in \pref{sec:further_framework}. The sketches may be omitted without interrupting the exposition.

\subsection{Architectural, structural, and distributional symmetries}
\label{sec:symmetry_grps}
The overview demonstrated the central role of Hessian invariance
properties in constraining the spectrum, in particular through the invariant
subspaces and eigenvalue degeneracies they impose. This leads to the basic
question of how to identify the Hessian invariances responsible for these
spectral features.

\paragraph{The point-stabilizer approach.} The standard approach to establishing such invariance proceeds as follows.
Let \(\Gamma\) act linearly on \(V\), and let \(f\colon V\to\RR\) be twice
differentiable and \(\Gamma\)-invariant, meaning
\begin{align}\label{inv_of_f}
	f(g\v)
	=
	f(\v)
	\qquad
	\text{for all } g\in\Gamma,\ \v\in V .
\end{align}
The action of \(\Gamma\) on \(V\) induces an action on the space of symmetric
matrices by \(H\mapsto g^{-\top} H g^{-1}\). Accordingly, we
may consider the corresponding stabilizer \(\stab_\Gamma(H)\) of a symmetric matrix
\(H\)
 (see
\eqref{def_stabilizer} for the stabilizer notation).
Differentiating the identity \eqref{inv_of_f} twice at \(\v\) gives \(g^\top\nabla^2 f(\v)g=\nabla^2 f(\v)\) for every \(g\in\stab_\Gamma\prn{\v}\). Therefore,
\begin{align}\label{point_in_hessian}
	\stab_\Gamma\prn{\v}
	\leq
	\stab_{\Gamma}\prn{\nabla^2 f(\v)}
	\leq
	\stab_{\GL(V)}\prn{\nabla^2 f(\v)} .
\end{align}
The inclusions in \eqref{point_in_hessian} therefore place the point stabilizer
inside the Hessian stabilizer. Thus, identifying the loss invariance group
$\Gamma$ provides a route to identifying a subgroup of the Hessian invariance
group.  As we will see shortly, this route generally fails to capture the relevant Hessian invariances in our setting. 

In many cases, the relevant 
invariances are apparent from the form of the objective, or are built into the
construction. For example, for the $n$-variate quadratic polynomial
\begin{align}\label{q}
	q(\v)=\sum_{i<j}v_i v_j,
\end{align}
it is immediate that permuting the coordinates of $\v$ leaves $q$ unchanged.
Thus, $q$ is invariant under the standard action of $S_n$ by coordinate
permutations. In our setting, where the objective comes from fitting neural
networks \eqref{eloss}, structural symmetries are likewise apparent in layer coordinates. \Integrated symmetries, however, arise from a subtler interaction
between structural and distributional symmetries, as in
\eqref{struc_combine_with_dist}.  We now formally introduce the concepts and methods developed to analyze these
invariances. 

\paragraph{Symmetry groups of a learning problem.} We define the \emph{data space} and the \emph{joint space} by
\begin{align}
	\dataspace &\defeq \inspace \times \outspace,
\qquad
	\universe \defeq \dataspace \times \weightspace .
\end{align}
The joint space is the natural domain of the pointwise loss~\eqref{smloss_def}.
Models such as $N^3$ are viewed on it through their dependence on $(\x,W)$. By \emph{symmetries} we mean elements of the automorphism groups of
$\dataspace$ and $\universe$, defined to be
$\aut(\dataspace)\defeq\aut(\inspace)\times\aut(\outspace)$ and
$\aut(\universe)\defeq\aut(\dataspace)\times\aut(\weightspace)$,
respectively, with the usual identification of tuples of automorphisms
with their componentwise actions. The factors consist of admissible
invertible transformations in the chosen category. In this work, unless
otherwise stated, these transformations are affine. Once fixed, these groups induce natural actions on the relevant objects. For
$g\in\aut(\universe)$, functions and maps $f$ on $\universe$ transform by pullback,
$g^* f \defeq f\circ g$. The action is likewise defined for $g\in\aut(\weightspace)$.
For $g\in\aut(\dataspace)$, distributions
transform by pushforward, $g_*\dd \defeq \dd\circ g^{-1}$. We write
$g f\defeq (g^{-1})^*f$ and $g\dd\defeq g_*\dd$, with the intended domain, codomain,
and meaning determined by the object. 

\begin{definition}[Internal and compatible symmetries]\label{def:type_of_symmetries}
	We define the architecture (or model) group, the structure group,
	and the distribution group, respectively, by
	\begin{align}
		\model(N)
		&\defeq \stab_{\aut(\universe)}(N\times \id_\outspace),\quad
		\struc(N,\ell)
		\defeq \stab_{\aut(\universe)}(\smloss),\quad
		\dist(\dd)
		\defeq \stab_{\aut(\dataspace)}(\dd).
	\end{align}
	Elements of these groups are called architectural (or model), structural,
	and distributional symmetries, respectively.
	For $g\in\aut(\universe)$, we write $g_{\dataspace}$ and
	$g_{\weightspace}$ for its components in $\aut(\dataspace)$ and
	$\aut(\weightspace)$. A~transformation $g\in\aut(\universe)$ is \emph{internal} if
	$g_{\dataspace}=\id_{\dataspace}$ and is \emph{$\dd$-compatible} if
	$g_{\dataspace}\in\dist(\dd)$.
\end{definition}

\begin{table}[!htbp]
	\centering
	\renewcommand{\arraystretch}{1.2}
	{\small\begin{tabular}{clll}
		\toprule
		\# & Symmetry type & Action & Invariance of $\ploss$ \\ 
		\midrule
		
		$\mathsf{a}_0(M)$ & Architectural 
		& $(W_1,\x)\mapsto (W_1M^{-1},M\x)$
		& No \\
		
		$\mathsf{a}^\pi_{i}(P_\pi)$ & Architectural 
		& $(W_i,W_{i+1})\mapsto (P_\pi W_i,W_{i+1} P_\pi^{-1})$
		& Yes \\
		
		$\mathsf{a}^{\lambda}_{i}(\blambda)$ & Architectural 
		& $(W_i,W_{i+1})\mapsto (\diag_\blambda W_i,W_{i+1}\diag_\blambda^{-1})$
		& Yes \\
		$\mathsf{s}^\celoss_1(\v)$ & Structural 
		& $W_3\mapsto W_3 + \bones \v^\top$
		& Yes \\
		
		$\mathsf{s}^\celoss_2$ & Structural 
		& $(W_3,\y)\mapsto(P_{(1,2)}W_3,P_{(1,2)}\y)$
		& No \\
		
		$\mathsf{s}^\celoss_3$ & Structural 
		& $W_3 \mapsto -P_{(1,2)}W_3$
		& Yes\\
		
		$\mathsf{d}^{\mathsf{s}}$ & Distributional 
		& $\x \mapsto P_{(1,2)}\x$
		& No\\[0.1cm]
		$\mathsf{d}^{\mathsf{r}} $ & Distributional 
		& $(\x,\y)\mapsto(-\x,P_{(1,2)}\y)$
		& No \\
		
		$\mathsf{c}^{\mathsf{s}}$ & \Integrated
		& $\sfa_0\circ \sfd^\sfs: W_1 \mapsto W_1P_{(1,2)}$ &Yes\\
		$\mathsf{c}^{\mathsf{r}}$ & \Integrated
		& 
		\makecell[lt]{
			$\mathsf{s}^\celoss_2\circ \sfs^\celoss_3\circ 
			\mathsf{a}^{\pi}_{2} \circ \mathsf{a}^{\pi}_{1}
			\circ \sfa_0\circ\sfd^\sfr:$\\$(W_1,W_2,W_3)\mapsto (-P_{(1,2)}W_1, P_{(1,2)}W_2P_{(1,2)},-W_3P_{(1,2)})$}
		&Yes
		\\
		\bottomrule
	\end{tabular}}
\caption{Examples of architectural symmetries of $N^{3,\circ}$, structural
	symmetries under the cross-entropy loss, and distributional symmetries of
	the mixture of two Gaussians in \pref{setting:gaussian_mixture},
	together with the resulting \integrated invariances. The final
	column indicates whether each symmetry is an invariance of the expected
	loss $\ploss$, as determined by \pref{prop:inv_of_eloss}.
	Here \(M\in\GL_d\), \(P_\pi\) denotes the permutation matrix associated with
	\(\pi\), \(\blambda\in\RR_{>0}^h\), and \(\v\in\RR^h\).
}
	\label{table:symmetry_examples}
\end{table}

When clear from context, we suppress the arguments of symmetry groups and of
the invariance groups introduced below. Since every architectural symmetry is structural, we often use the latter term
for both. Examples of symmetries and groups defined in
\pref{def:type_of_symmetries} appear in \pref{table:symmetry_examples} and
\pref{table:obj_symmetry_groups},~respectively.

\begin{proposition}\label{prop:inv_of_eloss}
If \(g \in \struc\)  is $\dd$-compatible, then the objective \(\ploss\) is invariant under
\(g_{\weightspace}\). In particular, if $g$ is internal, this invariance holds independently of the data
distribution. In each case, the transformations satisfying the corresponding condition form a group.
\end{proposition}
This proposition gives a basic mechanism for deriving invariance properties of
the objective, as illustrated in \pref{table:symmetry_examples}. Its formulation
extracts the conceptual ingredients of the argument in
\cite[Section~4.1]{arjevani2019principle}, and the proof, provided in \pref{sec:formal_sym},  adapts it to the present setting.

\begin{definition}[Invariance groups]\label{def:invs}
	We define the architectural (or model), structural, and $\dd$-\integrated invariance
	groups of the objective, respectively, by
	\begin{align}
		\Model(N)
		&\defeq
		\{g_{\weightspace}: g\in\model(N) \text{ is internal}\}				= \pi_\weightspace\bigl(\model(N) \cap (\id_{\dataspace}\times \aut(\weightspace))\bigr),
		\nonumber\\
		\Struc(N,\ell)
		&\defeq
		\{g_{\weightspace}: g\in\struc(N,\ell) \text{ is internal}\} 				= \pi_\weightspace\bigl(\struc(N,\ell) \cap (\id_{\dataspace}\times \aut(\weightspace))\bigr),
		\nonumber\\
		\Obj(N,\ell,\dd)
		&\defeq
		\{g_{\weightspace}: g\in\struc(N,\ell) \text{ is } \dd\text{-compatible}\}
		= \pi_\weightspace\bigl(\struc(N,\ell) \cap (\dist(\dd)\times \aut(\weightspace))\bigr).
	\end{align}
\end{definition}
We omit the modifier $\dd$ from \emph{compatible} and \emph{\integrated}
when clear from context. Moreover, we generally reserve \emph{symmetry group}
for groups preserving the model, pointwise loss, pointwise forms,
or data distribution, and \emph{invariance group} for those preserving
the objective or expected forms derived from it.

For $N^3$, the symmetry groups and their internal subgroups are given in
\pref{table:obj_symmetry_groups}. The factor
$\Aff(d)=\RR^d\rtimes\GL_d$ represents affine input changes compensated in the first layer, while $\RR_{>0}\wr S_h$ represents the usual positive rescalings (see, e.g., \cite{freeman2017topology}) and permutations between adjacent layers.\footnote{Here $N\rtimes G$ denotes the semidirect product associated with an
	action of $G$ on $N$. Thus, in
	$\Aff(d)=\RR^d\rtimes\GL_d$, the group $\GL_d$ acts on $\RR^d$ by its
	standard linear action. For the scaling symmetries, $S_h$ permutes the factors
	of $\RR_{>0}^h$, and the resulting semidirect product
	$\RR_{>0}^h\rtimes S_h$ is the wreath product
	$\RR_{>0}\wr S_h$ (see \cite[p.~187]{dummit2004abstract}
	for the definition and \cite[Appendix A, Table~3]{arjevani2025three}
	for a wreath-product isotropy type in a network with bias terms). }
The remaining factors arise from the loss. For squared loss,
$\RR\rtimes C_2$ acts by simultaneous translations of the scalar output and
label, and by simultaneous sign changes. For binary cross-entropy, the loss
depends on the final-layer parameters only through the difference between the
two rows of $[W_3,\b_3]$. The average of the two rows may therefore be
transformed affinely, with its translation depending linearly on the remaining
$p-h-1$ parameter coordinates. These transformations form a group isomorphic to
$\RR^{(h+1)(p-h)}\rtimes\GL_{h+1}$. The architecture groups and the $C_2$ factor that exchanges the labels
preserve this group under conjugation, producing the remaining semidirect
products in the table.

\begin{table}[!htbp]
	\centering
	\renewcommand{\arraystretch}{1.25}
	\setlength{\tabcolsep}{4pt}
	\makebox[\textwidth][c]{\small
		\begin{tabular}{@{}l l@{\qquad}l l@{}}
			\toprule
			\multicolumn{2}{c}{Symmetry group}
			&
			\multicolumn{2}{c}{Internal subgroup}
			\\
			\midrule
			$\model(N^3)$
			&
			$(\RR_{>0}\wr S_h)^2\times\Aff(d)$
			&
			$\Model(N^3)$
			&
			$(\RR_{>0}\wr S_h)^2$
			\\
			\midrule
			$\struc(N^3,\ell^\sqloss)$
			&
			$(\RR\rtimes C_2)\times\model(N^3)$
			&
			$\Struc(N^3,\ell^\sqloss)$
			&
			$\Model(N^3)$
			\\
			\midrule
			$\struc(N^3,\ell^\celoss)$
			&
			$\bigl(\RR^{(h+1)(p-h)}\rtimes\GL_{h+1}\bigr)
			\rtimes\bigl[C_2\times\model(N^3)\bigr]$
			&
			$\Struc(N^3,\ell^\celoss)$
			&
			$\bigl(\RR^{(h+1)(p-h)}\rtimes\GL_{h+1}\bigr)
			\rtimes\Model(N^3)$
			\\
			\bottomrule
		\end{tabular}
	}
	\caption{Architecture and structure groups of
		$N^3$, together with their internal subgroups. The squared-loss row concerns
		scalar regression ($k=1$), whereas the cross-entropy row concerns binary
		classification ($k=2$). The internal subgroups induce objective invariances independently of the data
		distribution. For cross-entropy, the internal structure group strictly
		extends the internal architecture group.}
	\label{table:obj_symmetry_groups}
\end{table}
Proving that a given set of structural symmetries is exhaustive is 
tractable in many cases, with the proof usually straightforward, if lengthy. As this is not needed below, we leave the issue outside the scope of the
present work and defer its treatment to future work.

\subsection{The failure of weight symmetries}\label{sec:failure}
The inclusion in \eqref{point_in_hessian} identifies
$\stab_\Obj\prn{\v}$ as a subgroup of the Hessian invariance group. In the settings considered here, however, this point stabilizer fails to capture the relevant Hessian invariances, as illustrated by
\pref{fig:sq_evs}. To see why these missing symmetries matter, we first recall how symmetry constrains the
spectrum of a Hessian. 

A \emph{representation} of a group \(G\) on a Euclidean vector space \(V\) is a homomorphism \(G\to\GL(V)\). We sometimes say instead that \(G\) acts \emph{linearly} on \(V\), and when the homomorphism is clear from context,
we write the action of $g\in G$ on $\mathbf v\in V$ as $g\mathbf v$.  A subspace \(V'\subseteq V\) is called \(G\)-\emph{invariant} if
\begin{align}
	gV'
	=
	V'
	\qquad
	\text{for all } g\in G .
\end{align}
A $G$-invariant subspace, equipped with the restricted action, is itself a
representation of $G$, called a $G$-\emph{subrepresentation} of $V$.
A $G$-subrepresentation is \emph{irreducible} if it contains no proper
nonzero $G$-subrepresentations. Assume that $V$ is a finite direct sum of
irreducible $G$-subrepresentations. Let
$\fv_1,\ldots,\fv_m$ be a complete list of pairwise nonisomorphic
irreducible summands,\footnote{Two $G$-representations $V_1$ and
	$V_2$ are isomorphic if there exists a linear isomorphism
	$T\colon V_1\to V_2$ such that
	$T(g\mathbf v)=gT(\mathbf v)$ for all $g\in G$ and $\mathbf v\in V_1$.} and let $r_i$ denote the multiplicity of the isomorphism class represented by
	$\fv_i$. For each
$i$, let $V_i$ be the sum of all irreducible subrepresentations of $V$
isomorphic to $\fv_i$. Then
\begin{align}\label{iso_comps}
	V
	&=
	V_1
	\oplus
	\cdots
	\oplus
	V_m,
	\quad 
	V_i
	\cong
	r_i\fv_i .
\end{align}
Here
$r_i\fv_i$ denotes the direct sum of $r_i$ copies of $\fv_i$.  The subspaces $V_i$ are called the \emph{isotypic components}, and the
displayed direct sum is called the \emph{isotypic decomposition}. 
A decomposition into irreducible subrepresentations need not exist or be unique, but when it exists, the isotypic components and the isotypic  decomposition are uniquely determined.

Let $H$ be self-adjoint and commute with the $G$-action, meaning
\begin{align}\label{commutation}
	Hg
	=
	gH
	\qquad
	\text{for all } g\in G.
\end{align}

\begin{fact}\label{fact:iso_dec_commuting_maps}
	If $V$ has the isotypic decomposition \eqref{iso_comps} and $H$ is as above, each isotypic component \(V_i\) is invariant under \(H\), that is, \(HV_i\subseteq V_i\). The restriction of \(H\) to \(V_i\) has real eigenvalues, each with multiplicity a multiple of \(\dim(\fv_i)\). Finally, \(\ker H\) is a \(G\)-invariant subspace.
\end{fact}
Let $V=T_{W}\weightspace$, equipped with its
Euclidean inner product, and let $\nabla_W^2\ploss(W)$ denote the Hessian matrix representing
$d^2\ploss(W)$ in the standard coordinates on $\weightspace$. For \emph{orthogonal} representations,\footnote{As
	indicated in \pref{sec:overview_framework}, invariance is a property of the second-order
	differential itself. The spectrum, on the other hand, also depends on the
	inner product used to identify the second-order differential with a
	self-adjoint endomorphism of $V$. Orthogonality ensures compatibility with
	this identification.} namely representations whose image lies in
$\mathrm O(V)$, the group of orthogonal transformations of $V$, the
congruence relation \eqref{hess_cong} is equivalent to $gH=Hg$, and
\pref{fact:iso_dec_commuting_maps} applies. This is the representation-theoretic origin of the eigenvalue multiplicities
and of the less visible but equally rigid invariant structure of the Hessian
(see the block form \eqref{block_iso}).

\begin{table}[!htbp]
	\centering
{\small\begin{tabular}{c|c|c|c}
	& \(\bzeros\) & \(\e_i\) & Generic point \\
	\hline
	\(\stab_{S_n}(\x)\)
	& \(S_n\)
	& \(S_{n-1}\)
	& \(S_{1}^n\) \\
	Stabilizer multiplicity structure
	& \(1,\ n-1\)
	& \(1^{[2]},\ n-2\)
	& \(1^{[n]}\) \\
	Hessian multiplicity structure 
	& \(1,\ n-1\)
	& \(1,\ n-1\)
	& \(1,\ n-1\)
\end{tabular}}
\caption{Point stabilizers and  multiplicity structure
	(\pref{def:multi_def}) for the natural action of \(S_n\) on \(\RR^n\). 
The quadratic example \(q\) in \eqref{q}
	illustrates one limitation of the point-stabilizer approach: its Hessian is
	constant, so the same spectral structure persists even at points with
	trivial stabilizer (see also \pref{sec:can_hear}).  In the ReLU settings considered here,
	the Hessian need not be constant, yet weight stabilizers fail to recover
	the rich invariance structure that forces its nontrivial eigenvalue
	multiplicities.}
	\label{table:isotropy_example}
\end{table}

For $q$ in \eqref{q}, both $q$ and its Hessian are $S_n$-invariant. At the origin,
$\stab_{S_n}\prn{\bzeros}=S_n$, so the point stabilizer exactly recovers the
multiplicity structure of the Hessian spectrum, made explicit below. At a generic point of $\RR^n$, the Hessian retains its rich $S_n$-invariance and the multiplicity structure it forces. However,  the point
stabilizer collapses to the trivial group. The point-stabilizer
approach then becomes entirely silent, imposing no nontrivial
representation-theoretic constraints on the Hessian.

Before turning to an alternative route to identifying Hessian invariances,
we spell out the representation-theoretic details of this example. As a representation of $\stab_{S_n}\prn{\bzeros}=S_n$, $T_{\bzeros}\RR^n\cong\RR^n$  decomposes
isotypically as
	$\RR^n
	=
	\RR\bones_n
	\oplus
	\bones_n^\perp $.
Both summands are irreducible representations of \(S_n\). In the general
theory \cite{fulton1991representation}, the irreducible representations of \(S_n\) correspond to partitions \(\lambda\) of \(n\), here \((n)\) and \((n-1,1)\), respectively, and are
denoted by \(\fs_\lambda\), with \(\ft \) reserved for the
trivial representation of the relevant symmetric group.
\pref{fact:iso_dec_commuting_maps} now forces the
Hessian to preserve these two components and to have one eigenvalue on each, with multiplicities \(1\) and \(n-1\), respectively (see \pref{table:isotropy_example}). Note that the derivation uses only
symmetry, and therefore applies to any Hessian with the same invariance. In this example, the families of irreducible representations that appear, and their multiplicities, stabilize for $n\geq 2$.
This is a general phenomenon for standard constructions involving
representations of \(S_d\), see \cite{arjevani2024sc}. 

The same failure of the point-stabilizer approach arises in our setting,
albeit for different reasons. For example, for generic $W\in\spT$ in \eqref{Wstar}, the composite transformation
$\mathsf c^{\mathsf r}$ interchanges $a_{1,2}$ and $a_{2,1}$ and
therefore does not fix $W$. Nevertheless, the Hessian at $W$ has the
$D_4\times C_2$ invariance. This discrepancy is not specific to $\spT$ but extends to the full
set of global minimizers in \eqref{k2_all_T} and, more broadly, arises in all
settings considered in this work.

\begin{remark}[Spectral refraction]\label{rem:spectral_refraction}
Applying \pref{fact:iso_dec_commuting_maps} to Hessian invariance hinges on
the action being orthogonal with respect to the chosen inner product. As noted in \pref{sec:overview_framework}, the second differential and its invariance group are intrinsic, but its realization as a self-adjoint operator, and hence its spectrum, depends on the chosen inner product. For a nonorthogonal representation, the congruence relation \eqref{hess_cong} need not imply the commutation relation \eqref{commutation}. Thus, although the invariance group of \(d^2\ploss\) is unaffected by the choice of inner product—and may still be large—the ordinary Euclidean Hessian need not exhibit the multiplicities in \pref{fact:iso_dec_commuting_maps}. We refer to this phenomenon as \emph{spectral refraction}. This viewpoint enters only briefly in \pref{sec:spec_analysis_gaussian},
but plays a substantial role in the applications in \pref{sec:dsb}.
\end{remark}

\subsection{Flat directions and symmetry}\label{sec:non_relu_flat}
Structural invariances of the objective have the advantage of explicitness:
they can be read directly from the network architecture and the loss, without
reference to the data distribution. Even at this level of generality, they
have significant implications for the Hessian kernel.
\begin{definition}[Flat directions]
	Let $f$ be a real-valued function and let $\p$ be a point in its domain.
	The tangent cone
		$T_\p\bigl(f^{-1}(f(\p))\bigr)$
	to the level set $f^{-1}(f(\p))$ at $\p$ is called the
	\emph{level-set cone}. Its elements are called \emph{flat vectors}. 
	Unit-length flat vectors are called \emph{flat directions}. If $f$ is invariant under a group action, flat vectors tangent to the group
	orbit are called \emph{orbit} flat vectors, and otherwise~\emph{transverse}.
\end{definition}

Flat directions help explain the geometric origins of large Hessian
kernels. Flatness alone, however, does not imply membership in the
Hessian kernel. Identifying conditions under which flat directions lie in the Hessian
kernel, particularly those imposed by symmetry, is a central theme
of this work.

We illustrate the basic relationship between flat directions and the Hessian
kernel through smooth curves. This covers the constructions considered in this work. A parallel argument applies to directions defined through the sequential characterization of the tangent cone  \cite{arjevani2026dsb_tf}.
Let $f:V\to\RR$ be smooth, let $\p\in V$, and let
$\gamma:[0,1)\to V$ be a smooth curve satisfying $\gamma(0)=\p$ and
$f(\gamma(t))=f(\p)$ for all $t\in[0,1)$.
 Then $\dot\gamma(0)$ is  a flat vector. Differentiating,  we find that 
$\inner{\grad f(\gamma(t)),\dot{\gamma}(t)}=0$. Differentiating once more
yields
\begin{align}
\inner{\nabla^2 f(\gamma(t))\,\dot{\gamma}(t),\dot{\gamma}(t)}+\inner{\grad f(\gamma(t)),\ddot{\gamma}(t)}=0.
\end{align}
In particular, if $\p$ is a critical point or if
$\ddot{\gamma}(0)=0$,\footnote{On a Riemannian manifold \cite{boumal2023introduction,burgisser2019noncommutative}, as in \pref{sec:layerwise}, the
	corresponding condition is
	$\nabla_{\dot\gamma}\dot\gamma|_{t=0}=0$, where $\nabla$ is the
	Levi-Civita connection. In particular, this holds when $\gamma$ is a
	geodesic.} this reduces to
\begin{align}\label{gamma_dot_in_null_cone}
	\inner{\nabla^2 f(\p)\,\dot{\gamma}(0),\dot{\gamma}(0)}=0.
\end{align}
Thus, $\dot\gamma(0)$ lies in the \emph{null cone} of $\nabla^2 f(\p)$, i.e., the set of all $\v$ such that
$\inner{\nabla^2 f(\p)\v,\v}=0$. Although the Hessian kernel is always contained in the null cone, the inclusion may be
strict. 
For example, if $f(\xi,\zeta)=\xi^2-\zeta^2$, then at the origin both the level-set cone
and the null cone consist of the lines $\xi=\pm\zeta$, whereas the Hessian
kernel is trivial. In general, the null cone equals the Hessian kernel if and only if the
Hessian is positive or negative semidefinite.

\paragraph{Flatness at critical points and local extrema.}
A sufficient condition for a flat direction to give a kernel vector is its
realization through a curve of critical points. This applies in particular to
group orbits through critical points, since for a $\Gamma$-invariant $f$, all
points on a given orbit are simultaneously critical or noncritical. 
If $\Gamma$ is a Lie group acting smoothly, then both the isotropy group
$\Gamma_\p$ and the orbit $\Gamma\p$ are manifolds,\footnote{Although not proper, the Lie group actions considered here are semialgebraic, so their orbits are embedded rather than merely immersed submanifolds.}
 and their dimensions satisfy
\begin{align}\label{dim_formula}
	\dim \Gamma = \dim \Gamma_\p + \dim \Gamma \p.
\end{align}
The following summarizes the discussion so far.
\begin{proposition}[Notation as above]\label{prop:basic_flat}
	Assume the relevant derivatives exist.\\
	(i) Any flat direction lies in the kernel of the first differential.\\
	(ii) At a critical point $\p$, all orbit flat directions lie in the Hessian
	kernel, spanning a $(\dim \Gamma-\dim \Gamma_\p)$-dimensional subspace.\\
	(iii) At a local extremum, all flat directions, whether orbit or
	transverse, lie in the Hessian kernel.
\end{proposition}
The inclusion $\model(N)\leq \struc(N,\ell)$ implies that
$\Model(N)\leq \Struc(N,\ell)\leq \Obj(N,\ell,\dd)$. Therefore, for any $W$,
\begin{align}
	T_W\bigl(\Model(N)\cdot W\bigr)
	\subseteq
	T_W\bigl(\Struc(N,\ell)\cdot W\bigr)
	\subseteq
	T_W\bigl(\Obj(N,\ell,\dd)\cdot W\bigr).
\end{align}
For the internal groups of $N^3$ in \pref{table:obj_symmetry_groups}, these inclusions and \pref{prop:basic_flat} give the following.

\begin{corollary}\label{corr:order_h_zero_evs}
(i) Irrespective of the data distribution, loss, or weight point $W$, the
kernel of the GN matrix of the objective \eqref{eloss} contains
$T_W(\Model(N^3)\cdot W)$, which has dimension $2h$ for generic $W$ and is
described in \pref{table:gauss_flat_directions}. At critical points, the same inclusion holds for the
Hessian.\\
	(ii) Under the cross-entropy loss, independently of the data distribution, the same
	holds with $T_W(\Struc(N^3,\celoss)\cdot W)$, which has dimension $3h+1$ for
	generic $W$.
\end{corollary}
In \pref{sec:relu_flat}, transverse flat directions yield subspaces of the kernel whose dimension grows quadratically, rather than only linearly. 

\paragraph{Flatness at noncritical points.}
At noncritical points, symmetry can identify flat directions annihilated by the Hessian by restricting its action to individual isotypic components. The general argument is as follows. Let $E\subseteq V\subseteq\RR^d$ be subspaces such that (i) $E$ is contained
in the null cone of $\nabla^2 f(\x)$ and (ii) is invariant under the Hessian
of $f|_{\x+V}$ at $\x$.  Although a vector $\v\in E$
need not itself lie in $\ker\nabla^2 f(\x)$, the vector
$\nabla^2 f(\x)\v$ is orthogonal to $E$. Indeed, let $\w\in E$. Then, since
$\inner{\nabla^2 f(\x)\u,\u}=0$ for
$\u\in\{\v,\w,\v+\w\}$, we get\footnote{This is the polarization identity. The subspace $E$ is then said to be \emph{isotropic} for the associated bilinear form.}
$	\inner{\nabla^2 f(\x)\v,\w}=0$. Hence $\nabla^2 f(\x)E\subseteq E^\perp$.

\begin{proposition}\label{prop:flat_directions}
	With $V$ and $E$ as above, $\nabla^2 f(\x)E\subseteq V^\perp$, and hence
	\begin{align}\label{lb_zero_evs}
		\dim\bigl(\ker\nabla^2 f(\x)\cap E\bigr)
		\geq \max\{\dim E-\operatorname{codim}V,0\}.
	\end{align}
Both hypotheses are satisfied under the following conditions.\\
(i) The level set of $f$ through $\x$ contains a neighborhood of $\x$
within $\x+E$.\footnote{Under this condition, every $\v\in E$ is realized by the
	affine flat curve so 
	\eqref{gamma_dot_in_null_cone} applies.}\\
(ii) $E$ is a sum of isotypic components
for an orthogonal subgroup of the
invariance group of the Hessian of $f|_{\x+V}$.
\end{proposition}
Although the following result admits a more direct proof, we give a representation-theoretic argument to illustrate how, through \pref{prop:flat_directions}, symmetry and flat directions jointly enforce zero eigenvalues at \emph{noncritical} points under the cross-entropy loss, independently of the data and, this time, also independently of the architecture. Denote the weights and biases of the last layer at a point $W$ by \(W_{\mathrm{last}}\in\RR^{2\times h}\) and \(\b_{\mathrm{last}}\in\RR^{2}\), and recall that  $\sfs^\celoss_1$ and $\sfs^\celoss_3$ in \pref{table:symmetry_examples} are invariances of the objective (here extended to $[W_{\mathrm{last}},\b_{\mathrm{last}}]$). Since the common-logit shift $\sfs^\celoss_1$ leaves the Hessian unchanged, we may assume that  \([W_{\mathrm{last}},\b_{\mathrm{last}}]=(1,-1)^\top \c\) for some \(\c\in\RR^{h+1}\).  Then, \(\langle \sfs^\celoss_3\rangle \le \stab_{\Struc}(W) \le \stab_{\Obj}(W) \le \stab_{\Obj}\prn{\nabla^2 \ploss(W)}\), the last inclusion by \eqref{point_in_hessian}. The isotypic decomposition of $\weightspace$ under $\langle \sfs^\celoss_3\rangle$ consists of a trivial component of codimension $h+1$ and a sign component of dimension $h+1$, namely the subspace $E$ of perturbations supported in the final layer and satisfying $[\delta W_{\mathrm{last}},\delta\b_{\mathrm{last}}] =(1,1)^\top\v^\top$ for $\v\in\RR^{h+1}$,  which also coincides with the space of orbit flat directions $T_W(\tri{\sfs^\celoss_1}\cdot W)$. Taking \(V=T_W\weightspace\),  \pref{prop:flat_directions} shows that
$E$ is an $(h+1)$-dimensional subspace of the Hessian~kernel.

\begin{theorem}\label{thm:at_init_oblivious}
Under the cross-entropy loss, for any network whose final layer is fully
connected with two outputs, the Hessian has at least $h+1$ zero eigenvalues
(resp.\ $h$ without a final-layer bias), where \(h\) is the width of the last hidden layer, at any point where the Hessian exists, critical or not.
\end{theorem}
The theorem above concerns orbit flat directions, as does \pref{corr:order_h_zero_evs}. In \pref{sec:relu_flat}, after introducing higher-order structure groups, we turn to transverse flat directions.

\begin{remark}
	In the example above, an isotypic component coincides with the space of orbit
	flat directions. This illustrates the general phenomenon whereby the invariance
	group of the objective acts on level-set cones, providing a representation-theoretic
	description of both orbit and transverse flat vectors 
	\cite{arjevani2026dsb_tf}. The same work also studies flatness using the functional formulation underlying
	\eqref{network_gram_form}.
\end{remark}

\subsection{Generalized derivatives and  higher-order symmetry groups} \label{sec:dist_der}
In \pref{sec:failure} and \pref{sec:non_relu_flat}, we derived results on
Hessian invariance and the Hessian kernel, respectively, directly from the
objective function. We found, however, that this approach may capture too little to explain the observed spectra. In the sequel, we therefore turn to a direct analysis of the Hessian. The discussion can be followed without familiarity with the theory of generalized functions by taking the identities below
as given. Their derivation and the more rigorous development of the underlying theory in our setting
are deferred to \pref{sec:further_framework}.

Since expectation and differentiation cannot be interchanged naively, as indicated in \eqref{diff_under_expec}, we interpret derivatives as generalized functions. Then the generalized Hessian $D_W^2\smloss$ splits into a regular part, which agrees with $\nabla_W^2\smloss$ on every open set where the pointwise loss is $C^2$, and a singular part supported outside the $C^2$ region. While singular terms may be highly irregular in general, they are more
structured for ReLU DAGs, namely networks represented by directed acyclic
graphs with ReLU activations, a special case of the definable DAG networks in
\pref{def:definable_dag}. For such networks, standard results from geometric measure theory show that the singular part is a Radon measure concentrated on the codimension-one
interfaces between activation regions, with no Cantor part. More precisely,
it admits a density with respect to Hausdorff measure restricted to these
interfaces. Together, the regular and singular terms give the identity \eqref{interchange} of generalized functions.
\begin{proposition}[Informal]
	\label{prop:scheme_hessian_rest}
	Assume that $N$ is a ReLU DAG network and that the loss is either the squared loss
	or the cross-entropy loss. Then, as generalized functions,
	\begin{align}\label{holy_eqn}
		\partial_W^\balpha \ploss
		=
		\EE_{\z\sim\dd}\!\left[
		\partial_W^\balpha \smloss(\z;\cdot)
		\right],
	\end{align}
	for any multiindex $\balpha$. Moreover, the first
	derivatives of $\smloss$ and $N$ are regular.\footnote{The regularity of the first derivatives  ultimately follows from the activation regions being semialgebraic, and hence admitting a semialgebraic stratification on whose full-dimensional strata the pointwise loss is smooth. For definable, locally DC networks, see \pref{sec:further_framework}.}
\end{proposition}
In this formulation, once invariance is established on an invariant open set, it
transfers to the classical Hessian almost everywhere, and everywhere when the
objective function is $C^2$.  To make this precise, we now define the action of $\aut (\universe)$ formally. Since $\weightspace$ is Euclidean, we may identify
$T_W\weightspace\simeq\weightspace$, and use the induced linear action on
perturbations:\footnote{On a general  manifold of weights, identifying $T_W\Theta$ with
	$T_{gW}\Theta$ is not canonical and requires further structure. In the
	Euclidean setting used here, $T_W\weightspace\simeq\weightspace$ is
	canonical, and hence so is $T_W\Theta\simeq T_{gW}\Theta$.}
\begin{align}
	g\cdot(\z,W,\delta W)
	&=
	(g\z,W,g\,\delta W),
	\qquad
	\delta W\in T_W\weightspace .
\end{align}
If $\kappa$ is smooth on $\dataspace\times\weightspace$, then
$d_W^k\kappa$ is a section of $\mathrm{Sym}^k(T^*\weightspace)$. The action on
functions is by pullback, while the action on $d_W^k\kappa$ is the induced
dual action on symmetric covariant $k$-tensors. For example, for the first
differential and any $g\in\aut(\universe)$,
\begin{align}
	(g\cdot d_W\kappa)(\z,W)[\delta W]
	&=
	d_W\kappa(g^{-1}\z,W)[g^{-1}\delta W].
\end{align}
The same description applies to generalized derivatives in the parameter
variables.\footnote{\label{ftnote:actionon_distr}
	On a product open set
	$\Omega=\Omega_{\dataspace}\times\Omega_{\weightspace}$,
	we regard $d_W^k\kappa$ as the family
	$\z\mapsto d_W^k\kappa(\z,\cdot)$ of tensor-valued distributions
	in $W$. The action replaces $\z$ by $g_{\dataspace}^{-1}\z$
	and applies the inverse linear action to the perturbations.}	 The extended actions yield the natural higher-order extensions of the symmetry 	groups in \pref{def:type_of_symmetries} and \pref{def:invs}, as detailed in
	\pref{table:symmetry_groups}. The analogue of \pref{prop:inv_of_eloss} follows from
	\pref{prop:scheme_hessian_rest}, with higher-order structural and \integrated invariances
	implying invariance of the corresponding higher-order differentials.

\newcommand{\AutW}{\aut_W \weightspace}
\newcommand{\AutU}{\aut_W \universe}
\newcommand{\piw}{\pi_{W}}
\newcommand{\StabU}[1]{\stab_{\AutU}\!\left(#1\right)}
\newcommand{\StabGlobal}[1]{
	\stab_{\aut(\universe)}\!\left(#1\right)}

\begin{table}[!t]
	\centering
	\begin{tabular}{lll}
		\toprule
		& Notation & Definition \\
		\midrule
		
		Local automorphism group
		& $\AutW$
		& $\aut(T_W\weightspace)$
		\\
		Weight-component projection
		& $\piw$
		& Projection onto $\AutW$ \\
		Joint local automorphism group
		& $\AutU$
		& $\aut(\dataspace)\times \AutW$ \\
		Distribution group
		& $\dist(\dd)$
		& $\stab_{\aut(\dataspace)}{\dd}$ \\
Initialization group
& $\init(\nu)$
& $\stab_{\aut(\weightspace)}{\nu}$ \\
		\midrule
		\multicolumn{3}{c}{\textbf{Symmetry groups of pointwise quantities}} \\
		\midrule
		
		Order-$r$ model/architecture
		& $\model_W^{[r]}$
		& $\StabU{d_W^r (N\times \id_\outspace)}$ \\
		
		Order-$r$ structure
		& $\struc_W^{[r]}$
		& $\StabU{d_W^r \smloss}$ \\
		
		GN structure
		& $\struc_W^{[\mathrm{GN}]}$
		& $\StabU{\qgn_W \smloss}$ \\
		
		\midrule
		\multicolumn{3}{c}{\textbf{Invariance groups of the objective}} \\
		\midrule
		
		Order-$r$ model/architectural invariances
& $\Model_W^{[r]}$
& $\piw\bigl(\model_W^{[r]} \cap (\id_{\dataspace}\times \AutW)\bigr)$ \\

Order-$r$ structural invariances
& $\Struc_W^{[r]}$
& $\piw\bigl(\struc_W^{[r]} \cap (\id_{\dataspace}\times \AutW)\bigr)$ \\

Order-$r$ $\dd$-\integrated invariances
& $\Obj_W^{[r]}$
& $\piw\bigl(\struc_W^{[r]} \cap (\dist(\dd) \times \AutW)\bigr)$ \\

GN structural invariances
& $\Struc_W^{[\mathrm{GN}]}$
& $\piw\bigl(\struc_W^{[\mathrm{GN}]} \cap (\id_{\dataspace}\times \AutW)\bigr)$ \\

GN $\dd$-\integrated invariances
& $\Obj_W^{[\mathrm{GN}]}$
& $\piw\bigl(\struc_W^{[\mathrm{GN}]} \cap (\dist(\dd)\times \AutW)\bigr)$ \\
		\midrule
		\multicolumn{3}{c}{
			\textbf{Invariance groups of averaged quantities}
		} \\
		\midrule
Order-$r$ model/architectural invariances
& $\avg{\Model}^{[r]}$
& $\pi_{W}\Bigl(
\model^{[r]}
\cap
(\id_{\dataspace}\times\init(\nu))
\Bigr)$ \\

Order-$r$ structural invariances
& $\avg{\Struc}^{[r]}$
& $\pi_{W}\Bigl(
\struc^{[r]}
\cap
(\id_{\dataspace}\times\init(\nu))
\Bigr)$ \\

Order-$r$ $(\dd,\nu)$-\integrated invariances
& $\avg{\Obj}^{[r]}$
& $\pi_{W}\Bigl(
\struc^{[r]}
\cap
(\dist(\dd)\times\init(\nu))
\Bigr)$ \\

GN structural invariances
& $\avg{\Struc}^{[\mathrm{GN}]}$
& $\pi_{W}\Bigl(
\struc^{[\mathrm{GN}]}
\cap
(\id_{\dataspace}\times\init(\nu))
\Bigr)$ \\

GN $(\dd,\nu)$-\integrated invariances
& $\avg{\Obj}^{[\mathrm{GN}]}$
& $\pi_{W}\Bigl(
\struc^{[\mathrm{GN}]}
\cap
(\dist(\dd)\times\init(\nu))
\Bigr)$ \\
		\bottomrule
	\end{tabular}
	
\caption{Definitions of local automorphism groups and of symmetry and
	invariance groups for higher derivatives and the GN form, extending
	\pref{def:type_of_symmetries} and \pref{def:invs}. At fixed~$W$, each
	classical pointwise derivative is viewed as a partially defined map on
	$\dataspace$ with values in symmetric multilinear forms, and its stabilizer
	is taken under the induced action. For generalized derivatives, the corresponding groups are defined on open
sets; see \pref{sec:formal_sym}. Here
	$\aut_W(\weightspace)=\aut(T_W\weightspace)$, whereas
	$\aut(\weightspace)_W$ is the stabilizer of $W$ in
	$\aut(\weightspace)$. Without a subscript $W$, $\model^{[r]}$, $\struc^{[r]}$, and
	$\struc^{[\mathrm{GN}]}$ denote the corresponding global stabilizers under
	the induced action
	$g\cdot(\z,W,\delta W)=(g\z,gW,g\delta W)$.
	In the barred rows, $\pi_W$ extracts the induced linear action on
	perturbations, while the intersections impose compatibility with $\nu$
	and, in the composite case, with $\dd$. These groups preserve the
	corresponding initialization~averages.}

	\label{table:symmetry_groups}
\end{table}

\subsection{Gauss--Newton regularity and well specification}\label{sec:overview_gn}
The identity \eqref{holy_eqn} can be refined at second order, separating the Gauss--Newton (GN) term, always regular, from a second-order model term that carries all singularities. This corresponds to the second-order chain rule applied to \(\kappa(\x,\y;W)=\ell(N(\x;W),\y)\), which gives the coordinate-free version of \eqref{gn_decomp}
\begin{align}\label{aut_universe_action}
	d_W^2 \kappa(\x,\y;W)[\delta W]
	&=
	d_1^2\ell(N(\x;W),\y)\big[d_W N(\x;W)[\delta W]\big]
	+
	d_1\ell(N(\x;W),\y)\, d_W^2 N(\x;W)[\delta W].
\end{align}
The first term is the \emph{pointwise GN form}, which we denote by \(\qgn_W\kappa\). The \emph{(expected) GN form} is obtained as the pushforward of the pointwise GN form,
	$d_W^{\mathrm{GN}}\ploss = \EE_{\dd}\, d_W^{\mathrm{GN}}\kappa$.
The symmetry groups \(\struc_W^{[\mathrm{GN}]}\), \(\Struc_W^{[\mathrm{GN}]}\), and \(\Obj_W^{[\mathrm{GN}]}\) are defined analogously and are listed in \pref{table:symmetry_groups}.
The coefficient $d_1\ell(N(\x;W),\y)$ in the second-order model term isolates a natural condition under which the Hessian reduces to the GN term.

\begin{definition}[Well-specified model]\label{def:wellspec}
	Set \(\dd_{\inspace}:=(\pi_{\x})_*\dd\). A model \(N\) is
	\emph{well-specified} at \(W^\star\) if
	\begin{align}
		\EE_{(\x,\y)\sim\dd}
		\bigl[d_1\ell(N(\x;W^\star),\y)\mid\x\bigr]
		=0
		\quad\text{for }\dd_{\inspace}\text{-almost every }\x.
	\end{align}
\end{definition}

\begin{proposition}[informal]\label{prop:gn_reduction}
	(i) The pointwise GN form and the  GN form are regular.
	\\
	(ii) If the objective is twice differentiable at \(W^\star\) and the model is well-specified, then
	\begin{align}
		d_{W^\star}^{\mathrm{GN}}\ploss = d_{W^\star}^2 \ploss.
	\end{align}
\end{proposition}
In this work, all global minima considered are well-specified. Nonglobal local minima and noncritical points, such as those treated in \pref{sec:relu_flat}, need not be, and often are not.

\subsection{Invariance of the Averaged Hessian and the Spectrum of a Random Hessian}
\label{sec:streamlined-internal-symmetry}

The barred groups introduced above record symmetries preserved by averaging
over initialization. We now apply them to the \emph{regular} second differential, as considered in \cite{sagun2016eigenvalues}. We
first give general results for its average and for the spectrum of a random
Hessian. Their application to $N^3$ appears in \pref{sec:streamlined-three-layer}.

\paragraph{Structural invariances of the averaged second differential.}
Using the canonical identification $T_W\weightspace\simeq\weightspace$, let
$\nu$ be the initialization law. We assume that the expectations below exist
and that the pointwise loss may be differentiated twice in the indicated
regular sense for $\dd\otimes\nu$-almost every $(\z,W)$, excluding the singular terms 
described in \pref{prop:scheme_hessian_rest}.  
All averages are iterated, first over the data and then over the
initialization. Whether the resulting form also equals
$\EE_Wd_{W,\mathrm{reg}}^2\ploss(W)$ is a separate identification and is
not assumed in the argument.

\begin{proposition}
	\label{prop:streamlined-internal-average}
	We have
	\begin{align}
		\avg{\Struc}^{[2]}(N,\ell,\nu)
		&\leq
		\avg{\Obj}^{[2]}(N,\ell,\dd,\nu)
		\leq
		\stab_{\aut_W(\weightspace)}\left(
		\EE_{W\sim\nu}\EE_{\z\sim\dd}
		d_{W,\mathrm{reg}}^2\smloss(\z;W)
		\right).
	\end{align}
\end{proposition}
The proof is the standard composite invariance argument, applied first to
$\dd$ and then to $\nu$.

Equip $\weightspace$ with its Euclidean
inner product, let $H(W)$ be the self-adjoint operator representing
$\EE_{\z\sim\dd}d_{W,\mathrm{reg}}^2\smloss(\z;W)$, and set
$M=\EE_{W\sim\nu}H(W)$. \pref{fact:iso_dec_commuting_maps} applies after
restricting the groups in \pref{prop:streamlined-internal-average} to their
orthogonal elements, facilitating the analysis of the averaged Hessian through the associated isotypic components and the resulting constraints
on eigenvalue multiplicities. We now ask what  transfers to a random Hessian $H(W)$. 

We give here a general spectral count. For a subspace $E\subseteq\weightspace$, let $\Pi_E$ denote the orthogonal projection
onto $E$. For a self-adjoint operator $A$ and an interval $I$, let $n_-(A)$
denote its negative inertia and let $N_A(I)$ count its eigenvalues in $I$,
with multiplicity.
\begin{proposition}
	\label{prop:streamlined-random-bulk}
	Let $H=Q+R$ be a random self-adjoint operator on the Euclidean space
	$\weightspace$. Suppose $U,B\subseteq\weightspace$ are fixed orthogonal
	subspaces such that, almost surely, $Q\succeq0$,
	$U\subseteq\ker Q\cap\ker R$, and $\Pi_BR\Pi_B=0$. 	Fix a nonzero subspace $E\subseteq B$ and assume
	$\EE\tr(\Pi_EQ\Pi_E)<\infty$. Then, almost surely,
	\begin{align}
		n_-(H)
		&\leq
		\codim(U\oplus B).
		\label{eq:streamlined-general-inertia}
	\end{align}
	For every $T>0$ and $\delta\in(0,1)$, with probability at least
	$1-\delta$,
\begin{align}
	N_H([0,T])
	&\geq
	\dim(U\oplus E)-\codim(U\oplus B)
	-\frac{\EE\tr(\Pi_EQ\Pi_E)}{\delta T}.
	\label{eq:streamlined-general-bound}
\end{align}
\end{proposition}
\begin{proof}
	$\Pi_{U\oplus B}H\Pi_{U\oplus B}
	=\Pi_{U\oplus B}Q\Pi_{U\oplus B}\succeq0$.
	Interlacing gives
	$n_-(H)\leq\codim(U\oplus B)$. For the spectral count,
	$\Pi_{U\oplus E}H\Pi_{U\oplus E}=\Pi_EQ\Pi_E\succeq0$.
	Markov's inequality gives
	$\tr(\Pi_EQ\Pi_E)\leq\EE\tr(\Pi_EQ\Pi_E)/\delta$ with probability at least
	$1-\delta$. On this event, at most
	$\EE\tr(\Pi_EQ\Pi_E)/(\delta T)$ eigenvalues of this compression exceed
	$T$. Hence at least
	$\dim U+\dim E-\EE\tr(\Pi_EQ\Pi_E)/(\delta T)$ of its eigenvalues lie in
	$[0,T]$. Interlacing gives at least as many eigenvalues of $H$ below or
	equal to $T$. At most $\codim(U\oplus B)$ of them are negative. Since
	$U\perp E$, we have $\dim U+\dim E=\dim(U\oplus E)$, which gives
	\eqref{eq:streamlined-general-bound}.
\end{proof}


\section{Spectral Analysis} \label{sec:spectra}
We now turn to analyzing the Hessian spectra for \pref{setting:gaussian_mixture} and \pref{setting:gaussian}. Our general strategy is to construct a symmetric reference when needed,
identify the relevant structural and distributional symmetries, and determine
the induced composite invariance groups. We then derive the spectral structure forced by symmetry and show how it gives rise to the spectra of the original trained models through SB.
Some of the results are purely structural, as the next subsection~demonstrates.

\subsection[The world is (quite) flat]
{The world is (quite) flat\protect\footnotemark}
\label{sec:relu_flat}
\footnotetext{The title echoes Thomas L. Friedman’s \emph{The World Is Flat}.}

Both \pref{corr:order_h_zero_evs} and \pref{thm:at_init_oblivious} guarantee
$\Omega(h)$ zero eigenvalues under fairly general conditions by identifying orbit
flat directions. The former applies to $N^3$ at critical points, independently of the data distribution, and the latter, under cross-entropy loss, at arbitrary points of networks regardless of activation.
 This generality comes at a price. By accounting only for the effects of global symmetries of the loss, these results miss important local degrees of freedom
and therefore fall short of the quadratic growth in kernel dimension observed in \pref{fig:original_vs_symmetric_mix} and \pref{fig:sq_evs}. Using the higher-order symmetry groups (\pref{sec:dist_der}), we now show
that transverse flat directions recover the missing degeneracy, yielding
$\Omega(h^2)$ zero eigenvalues. Throughout this subsection, we set $h=d$ and assume $k=1$ for simplicity.
\begin{theorem}\label{thm:at_positive}
	Under these standing assumptions, at $W_2,\b_2>0$, independently of
	the loss function and data distribution,  the numbers of zero eigenvalues contributed by flat directions are almost everywhere given by
	\[
	\begin{array}{l|ccc}
		& N^{2\nicefrac{1}{2},\circ} & N^{3,\circ} & N^3 \\
		\hline
		\ker\mathrm{GN} & h^2 & h(h+1) & h(h+2) \\
		\ker\mathrm{Hessian} & h(h-1) & (h-1)^2 & h(h-1) \\
		\dim\Theta & 2h^2 & h(2h+1) & (2h+1)(h+1)
	\end{array}.
	\]
	For the GN matrix of $N^3$, the constituent subspaces of the kernel are
	described in \pref{table:gauss_flat_directions}. Similar descriptions
	hold for the other cases.
\end{theorem}

For the GN matrix, part~(i) of \pref{prop:basic_flat} suffices to establish
the above through the relevant flat directions in \pref{table:gauss_flat_directions}. The points considered in \pref{thm:at_positive}, however,
are generally not critical, so
\pref{prop:basic_flat} does not yield the corresponding conclusions for the
Hessian. Nevertheless, symmetry forces Hessian kernels of dimension $\Omega(h^2)$ at these points.
A visual indication that rich symmetry is indeed present is given in \pref{fig:lb_zeros}.
The next lemma makes this precise.
\begin{figure}[ht]
	\centering
	\includegraphics[width=\textwidth]{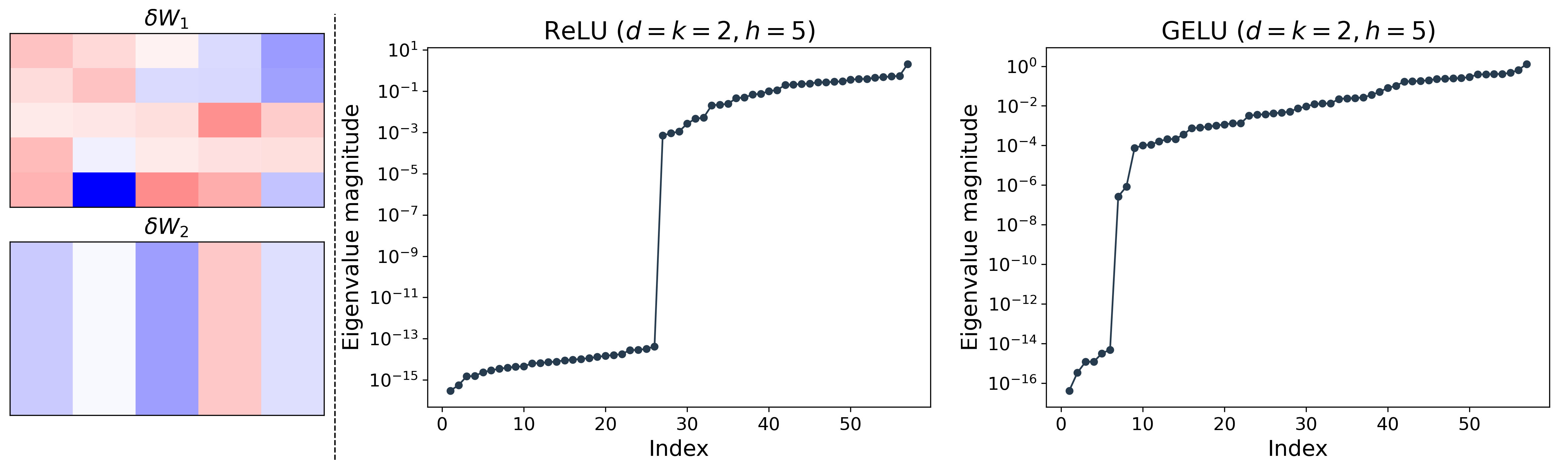}
	
	\caption{
		\textbf{Left.} For the network $N^{2\nicefrac{1}{2},\circ}$ in \eqref{net3circ}, consider
		any point $W$ at which all second-layer weights are positive. One finds that
		the $\delta W_2$-component of
		$\nabla^2\ploss(W)\delta W$ always has identical rows. 
		This reflects a rich Hessian invariance group
		(\pref{lem:hessian_inv_despite}), present despite $W$ generically having
		trivial symmetry.
	\textbf{Middle.} Under the conditions of \pref{thm:at_positive}, the Hessian of $N^3$ \eqref{net3} with ReLU has $\Theta(h^2)$ zero eigenvalues---independently of the data distribution. The result extends directly to the setting shown here, namely
	$d=k=2$, $h=5$ for $\cN_\pm$ in \pref{setting:gaussian_mixture}. \textbf{Right.} If ReLU is replaced by GELU, this result no longer holds, but \pref{thm:at_init_oblivious}, which holds regardless of the activation, still applies and guarantees $\Omega(h)$ zero eigenvalues. The distinct ``knee'' locations in spectra evaluated at the same weights are numerically consistent with the corresponding lower bounds.
}
	\label{fig:lb_zeros}
\end{figure}

\begin{lemma}\label{lem:hessian_inv_despite}
	Let \(W\in\Theta\) with \(W_2,\b_2 >0\). Then, 
	\(d_W N^{2\nicefrac{1}{2},\circ}, d_W N^{3,\circ}\) and \(d_W N^3\) are invariant under
	\begin{align}\label{deltas1}
		\delta\sfs_1(M_1,M_2):\quad 
		(\delta W_2,\delta \b_2)
		\mapsto
		(M_1 \delta W_2,\, M_2 \delta \b_2),
		\qquad M_1,M_2\in \stab_{\GL(M(1,d))}\prn{W_3},
	\end{align}
	 with  $W_3 = \bones^\top_h$ for $N^{2\nicefrac{1}{2},\circ}$. Thus
		$\delta \sfs_1 \in \Model_W^{[1]}(N), N\in \{N^{2\nicefrac{1}{2},\circ}, N^{3,\circ}, N^3\}$.
	We also have $\delta \sfs_1 \in \Model_W^{[2]}(N^{2\nicefrac{1}{2},\circ})$. However, for $N \in \{N^{3,\circ}, N^3\}$, this invariance holds only for perturbations satisfying~$\delta W_3 \parallel W_3$.
\end{lemma}
We treat $\Model_W^{[1]}(N^3)$. The remaining cases follow similarly. As noted  in the overview, the argument becomes transparent when working in layer coordinates.
For $N^3$, a direct computation gives
\begin{align}\label{dwn3}
	d_W N^3(\x;W)[\delta W]
	&=
	W_3 D_2\Bigl(
	\delta W_2\,\sigma(W_1 \x + \b_1)
	+ W_2 D_1\bigl(\delta W_1 \x + \delta \b_1\bigr)
	+ \delta \b_2
	\Bigr) 
	\notag \\
	&\quad
	+ \delta W_3\,\sigma\bigl(W_2 \sigma(W_1 \x + \b_1) + \b_2\bigr)
	+ \delta \b_3,
\end{align}
where $D_1 \defeq D\sigma(W_1\x+\b_1)$ and
$D_2 \defeq D\sigma\bigl(W_2\sigma(W_1\x+\b_1)+\b_2\bigr)$,
with their dependence on $\x$ suppressed in the notation.
Under the assumption of \pref{lem:hessian_inv_despite}, we have \(D_2 = I_d\). Inspection of \eqref{dwn3} then readily shows that \(d_W N^3\) is invariant under \eqref{deltas1}. For $\Model_W^{[2]}$, one must also take into account singular terms, see \pref{sec:formal_sym}.

The directions listed in \pref{table:gauss_flat_directions}
lie in $\ker d_WN^3$. Since the GN form factors through $d_WN^3$, they
lie in its kernel and give the $N^3$ GN bound in
\pref{thm:at_positive}. The other GN bounds follow similarly. For the Hessian of $N^3$, let
$V=\{\delta W:\delta W_3\parallel W_3\}$.
By \pref{lem:hessian_inv_despite}, the orthogonal transformations in
\eqref{deltas1} preserve the Hessian of $\ploss|_{W+V}$.
Writing $\std$ for the standard representation of
$\stab_{O(M(1,d))}\prn{W_3}$ on $W_3^\perp$, the action on $V$ has
the isotypic decomposition
\begin{align}\label{deltas1_iso_dec}
	V\cong
	(h^2+2h+3)(\triv,\triv)
	\oplus h(\std,\triv)
	\oplus(\triv,\std).
\end{align}

\proof[\pref{thm:at_positive}, Sketch].
Let $E\subset V$ be the sum of the two nontrivial isotypic components
in \eqref{deltas1_iso_dec}, so $\dim E=h^2-1$. On $E$,
$W_3\delta W_2=W_3\delta\b_2=0$. Since $W_2,\b_2>0$, the network,
and hence $\ploss$, is locally constant along $W+E$.
Both hypotheses of \pref{prop:flat_directions} therefore hold.
As $\operatorname{codim}V=h-1$, \eqref{lb_zero_evs} gives
$\dim(\ker\nabla^2\ploss(W)\cap E)\ge h(h-1)$.
The other architectures follow~similarly. \qed

\begin{table}[!htbp]
	\centering
	{\footnotesize \begin{tabular}{l l c l}
		\hline
		Type & Space & Dimension & Description \\
		\hline
		Transverse & $(\std,\triv)$-isotypic component & $h(h-1)$ & $\delta W_2$ with columns orthogonal to $W_3$ \\
		Transverse & $(\triv,\std)$-isotypic component & $h-1$ & $\delta \mathbf{b}_2$ orthogonal to $W_3^\top$ \\
		Transverse & in the $(\triv,\triv)$-isotypic component & $1$ & $\delta \mathbf{b}_2 \in \mathrm{Span}(W_3^\top)$ and $W_3\,\delta \mathbf{b}_2 = -\,\delta \mathbf{b}_3$ \\
		Orbit & $T_W(\Model(N^3) W)$ & $2h$ & $\delta W = (A W_1, A\mathbf{b}_1, B W_2 - W_2 A, B\mathbf{b}_2, -W_3 B, 0)$ \\
		\hline
		Total & & $h(h+2)$ & \\
		\hline
	\end{tabular}}
\caption{Subspaces of $\ker d_WN^3$ yielding the GN bound in
	\pref{thm:at_positive}, with $A,B$ diagonal.
	Orbit flat directions contribute only $\Theta(h)$ zero eigenvalues, compared
	with $\Theta(h^2)$ from transverse directions.
	The two nontrivial isotypic components contain all but one transverse
	direction and, through \pref{prop:flat_directions}, yield the Hessian
	kernel bound.}

	\label{table:gauss_flat_directions}
\end{table}

\subsection{Regression }\label{sec:spec_analysis_gaussian}
The preceding analysis uses only structural Hessian invariances, hence holds
independently of the data distribution. When present, distributional
symmetries may combine with them to form \integrated invariances (see
\pref{def:invs}), yielding the richer invariance group needed for a faithful
account of the observed spectra. Throughout this subsection and the next, we assume $h=d\ge4$ to avoid separate treatment of low-dimensional cases. 

The labeling function in \pref{setting:gaussian} is realizable within the $N^3$ model class, for example by any point in the following set of global minima,
{\small \begin{align}\label{sq_loss_global}
	\cG_0 = \left\{ W \ \middle| \ 
	\begin{aligned}
		& W_1 = P_{\pi, \s} \text{ for } \pi \in S_d, \s \in \RR_{>0}^d, \\
		& W_2 \in \GL_d \text{ and } \b_2 \in \RR^d \text{ have strictly positive entries}, \\
		& W_3 W_2 W_1 = \bones^\top, \quad \b_1 = \bzeros, \quad W_3 \b_2 = -\b_3 
	\end{aligned}
	\right\},
\end{align}}
\par\noindent where $P_{\pi, \s} \coloneqq \diag(\s)P_\pi$ 
is a positive \emph{monomial} matrix, consisting of exactly one positive entry in each row and column.  Arguments similar to those in the proof of \pref{prop:k2_global_minima} apply in this setting. Since the pointwise loss vanishes identically at every \(W \in \cG_0\), the GN form coincides with the second differential by \pref{prop:gn_reduction}. Therefore, for the squared loss used here, compatible symmetries of $d_WN^3$ suffice to establish the Hessian invariances considered below.

Although the inputs in \pref{setting:gaussian} follow a Gaussian distribution, the distributional symmetries are not given by all orthogonal transformations, since the labels must also be taken into account. The relevant distribution group is therefore
\begin{align}\label{dist_sym_gaussian}
	\sfd_3(P_\pi): \quad \x \mapsto P_\pi \x, \quad \text{for } \pi \in S_d. 
\end{align}
We can offset the input change through the architectural symmetry $\mathsf{a}_0$ in \pref{table:symmetry_examples}. For illustration, take $N^{3,\circ}$ with $W_1=I_d$, $W_2\in\GL_d$ entrywise positive and $W_3W_2=\bones^\top$. The identities
\begin{align}
	\sigma(P_\pi\u)=P_\pi\sigma(\u), \qquad
	D\sigma(P_\pi\u)=P_\pi D\sigma(\u)P_\pi^\top
\end{align}
(where the derivatives exist) yield the structural symmetry
\begin{align}
	\delta\sfs_2(P_\pi):
	(\x,\delta W_1,\delta W_2,\delta W_3)
	\mapsto
	(P_\pi\x,\,
	P_\pi\delta W_1P_\pi^\top,\,
	\delta W_2P_\pi^\top,\,
	\delta W_3W_2P_\pi^\top W_2^{-1})
\end{align}
of $d_WN^{3,\circ}$. Thus,
$\delta\sfs_2\in\model_W^{[1]}(N^{3,\circ})$ is not internal but is
compatible with the data distribution $(\operatorname{id},\bones^\top\sigma)_*\cN$, since its data component is given by \eqref{dist_sym_gaussian}. It therefore induces \integrated invariance of the GN form at these
global minima (\pref{thm:local_symmetry_inv}). The following
lemma gives the corresponding invariance for $N^3$, including the bias
correction.
\begin{lemma}\label{lem:hessian_inv_despite3}
On $\cG_0$, $d_WN^3$ admits the structural symmetry $\delta\sfs_1$
in \eqref{deltas1}, and the GN form admits the following \integrated
invariance:
{\small\begin{align}
		\delta \sfc_1(P_\pi):~
		(\delta W_1, &\delta\b_1,\delta W_2, \delta \b_2, \delta W_3)
		\mapsto
		\bigl(
		W_1 P_\pi W_1^{-1}\delta W_1 P_\pi^\top,\;
		W_1 P_\pi W_1^{-1}\delta \b_1,\;
		\\
		&\quad 
		\delta W_2 W_1 P_\pi^\top W_1^{-1},\;
		\delta \b_2
		-
		\delta W_3 A^{-1}(I-A)\b_2
		\frac{W_3^\top}{W_3 W_3^\top},\;
		\delta W_3 W_2 W_1 P_\pi^\top W_1^{-1} W_2^{-1}
		\bigr),
		\nonumber
	\end{align}}
\par\noindent	where $\pi\in S_d$ and $A \defeq W_2 W_1 P_\pi W_1^{-1} W_2^{-1}$. 
For every $W\in\cG_0$,
\begin{align}
\tri{\delta \sfs_1, \delta \sfc_1}
  \leq \Obj_W^{[\mathrm{GN}]}(N^3,\ell^\sqloss,(\operatorname{id},\mathbf{1}^{\top}\sigma)_{*}\cN)&\leq
  \stab_{\aut(T_W\weightspace)}\prn{d_W^2\ploss}.
\end{align}
\end{lemma}
We find in particular that, at global minima $W\in\cG_0$, the Hessian
has rich invariances that the point stabilizer generally fails to capture (see \pref{rem:identifiability}).

Dead neurons are often encountered in practice
(see \pref{fig:sq_evs}) and give rise to a distinct class of global
minima at which the composite action above may no longer preserve
the Hessian. This occurs for example when \(e\) rows of \(W_2\), together with the corresponding entries of \(\b_2\), are strictly negative, while the remaining rows of 
\(W_2\) and entries of \(\b_2\) are strictly positive, and the first and last lines of \eqref{sq_loss_global} hold with
$W_2$ and $\b_2$ replaced by $D_2W_2$ and $D_2\b_2$. Denote this set by~$\cG_e$.  In this regime, $\delta \sfs_1$ persists, with $W_3$ now replaced by $W_3D_2$ in \pref{lem:hessian_inv_despite}, whereas $\delta \sfc_1$ fails. 
Nevertheless, for every $W\in\cG_e$, the image of $d_WN^3$ is isomorphic,
as an $S_d$-module, to the corresponding image on $\cG_0$. Explicitly, it is
spanned by 
$\mathbf 1_{\{(W_1\x)_i>0\}}x_j$,
$\mathbf 1_{\{(W_1\x)_i>0\}}$, and $1$.
Consequently, $\operatorname{rank}d_WN^3=d^2+d+1$, the Hessian kernel has
dimension $d(d+2)$, and the quotient by the kernel has the same isotypic
decomposition as for $\cG_0$. Thus, under a group-invariant inner product,
the same multiplicities are forced in the nonzero spectrum.

We next state the isotypic decomposition for the subgroup of
$\tri{\delta\sfs_1,\delta\sfc_1}$ obtained by restricting $M_1,M_2$
in \eqref{deltas1} to $\stab_{O(M(1,d))}\prn{W_3}$.
 In \pref{sec:failure}, we encountered \(\triv\), of degree
\(1\), and \(\stdsd_{(d-1,1)}\), of degree \(d-1\). Here, we also need
\(\stdsd_{(d-2,2)}\), of degree \(\frac{d(d-3)}{2}\), and
\(\stdsd_{(d-2,1,1)}\), of degree \(\frac{(d-1)(d-2)}{2}\). Since $\delta\sfc_1$ mixes $\delta\b_2$ with $\delta W_3$, set
$\delta\widetilde{\b}_2\defeq
\delta\b_2+(\delta W_3\b_2)W_3^\top/(W_3W_3^\top)$.
In these coordinates, the blocks carry the following representations
\begin{align}\label{p1_iso_dec}
	\delta W_1 &\cong 2(\triv,\triv,\triv) \oplus 3(\triv,\triv,\stdsd_{d-1,1})
	\oplus (\triv,\triv,\stdsd_{d-2,2}) \oplus (\triv,\triv,\stdsd_{d-2,1,1}),
	&\delta \b_1 &\cong (\triv,\triv,\triv) \oplus (\triv,\triv,\stdsd_{d-1,1}),\nonumber\\
	\delta W_2 &\cong (\triv,\triv,\triv) \oplus (\triv,\triv,\stdsd_{d-1,1})
	\oplus (\std,\triv,\triv) \oplus (\std,\triv,\stdsd_{d-1,1}),
&\delta\widetilde{\b}_2
&\cong (\triv,\triv,\triv)\oplus(\triv,\std,\triv),\nonumber\\
	\delta W_3 &\cong (\triv,\triv,\triv) \oplus (\triv,\triv,\stdsd_{d-1,1}),
	&\delta \b_3 &\cong (\triv,\triv,\triv).
\end{align}
The three entries refer, in order, to the orthogonal factors
acting through $M_1$ and $M_2$, and the permutation factor $S_d$.
The decomposition follows from standard tensor product constructions  \cite{fulton1991representation} (see also \cite{thiede2020general} for related
constructions in the context of equivariant network theory, with further
discussion in \pref{sec:equivariance}). An explicit description of the isotypic components is readily obtained by
adapting the analysis of \pref{setting:gaussian_mixture} in
\pref{sec:spec_analysis_mix}.
\begin{theorem}\label{thm:learning_gauss_analysis}
In the setting of \pref{setting:gaussian},
\begin{enumerate}[leftmargin=*,label=(\roman*)]
\item For any $W\in\cG_0$, the Hessian is invariant under $\tri{\delta \sfs_1, \delta \mathsf{c}_1}$. The corresponding isotypic decomposition of $T_W\weightspace$ \eqref{p1_iso_dec} and the $d(d+2)$-dimensional subspace of zero eigenvectors arising from flat directions (\pref{thm:at_positive}) give 
{\footnotesize
	\begin{align}
		&T_W\weightspace \cong\;
		7(\triv,\triv,\triv)
		\oplus
		6(\triv,\triv,\stdsd_{d-1,1}) 
		\oplus
		(\triv,\triv,\stdsd_{d-2,2})
		\oplus
		(\triv,\triv,\stdsd_{d-2,1,1})
		\oplus
		(\std,\triv,\triv)
		\oplus
		(\std,\triv,\stdsd_{d-1,1})
		\oplus
		(\triv,\std,\triv),\nonumber\\
		&T_W(\ploss^{-1}(0)) \supseteq
		3(\triv,\triv,\triv)
		\oplus
		2(\triv,\triv,\stdsd_{d-1,1})
		\oplus
		(\std,\triv,\triv)
		\oplus
		(\std,\triv,\stdsd_{d-1,1})
		\oplus
		(\triv,\std,\triv).
	\end{align}
}
The multiplicity pattern corresponding to this isotypic decomposition is given in \pref{table:hess_kernel}.
\item For any $W\in\cG_e$ with $1\leq e\leq d-1$,
the Hessian kernel has dimension $d(d+2)$, and the quotient by this kernel
has the same isotypic decomposition, and hence the same symmetry-forced
nonzero multiplicity structure under an invariant inner product, as in
part~(i).
\end{enumerate}
\end{theorem}
For simplicity, we restrict the discussion to the normalized realization of
$\cG_0$ in \eqref{second_diff_setting_gauss} by the bias-free network
$N^{3,\circ}$.  Removing the bias parameters and quotienting by the $d(d+1)$-dimensional Hessian kernel gives
\begin{align}
	2(\triv,\triv,\triv)
	\oplus
	3(\triv,\triv,\stdsd_{d-1,1})
	\oplus
	(\triv,\triv,\stdsd_{d-2,2})
	\oplus
	(\triv,\triv,\stdsd_{d-2,1,1}),
\end{align}
of dimension $d^2$ (as for $N^{1\nicefrac12,\circ}$, reflecting the same
$S_d$-representation on $M(d,d)$; this quotient viewpoint is developed
further in \pref{sec:alexnet_superposition}). At the normalized realization, the action on the blocks corresponding to $\delta W_1$ and $\delta W_2$ is orthogonal with respect to the Euclidean
inner product. In particular, the last two components, each of dimension $\Theta(d^2)$,
are supported in $\delta W_1$, so the metric induced on the quotient is
invariant on them and they consequently retain their exact multiplicities,
while \pref{prop:exact_gn_expression_gauss} shows that their eigenvalues
remain uniformly bounded.

The complementary
$2(\triv,\triv,\triv)\oplus
3(\triv,\triv,\stdsd_{d-1,1})$
subspace has dimension $3d-1$ and is preserved as a whole, although the
Euclidean realization of the Hessian may mix its isotypic components and
split their forced multiplicities. Identifying the quotient with
$\RR^{d\times d}$, we use the invariant inner product that agrees with the
Frobenius inner product off the diagonal and weights the diagonal subspace
by $(d+1)^{-1}$. This metric is uniformly equivalent within a factor of two
to the metric induced by the Euclidean inner product on parameter space,
so the min--max principle gives the same comparison for the ordered
eigenvalues and preserves their asymptotic scales. A refinement of
\pref{prop:exact_gn_expression_gauss} gives one $\Theta(d^2)$ outlier and
places all other unbounded eigenvalues at order $\Theta(d)$. This is an instance of spectral
refraction confined to $3d-1=O(d)$ directions.

\subsection{Analysis by layer and on manifolds: spontaneous SB}\label{sec:layerwise}

In the sequel, we discuss how the framework adapts to restrictions of the Hessian. Such restrictions are relevant to several applications, as illustrated below.

Let $H:V\times V\to\RR$ be a symmetric bilinear form invariant under a group $G \leq \aut(V)$. For $E\leq V$, define
\begin{align}
	G\!\mid_E
	&\defeq
	\{g|_E~|~g\in G,\ gE=E\}.
\end{align}
Then $G\!\mid_E\leq\stab_{\aut(E)}(H|_{E\times E})$. In the present setting, let $E\leq T_W\Theta$. There are two natural ways to extract symmetries on $E$ that induce \integrated invariances: first restrict the pointwise differential to $E$ and then retain only those whose $\dataspace$ component lies in $\dist$, or reverse the order. Write $\struc_{W;E}^{[r]}(N,\ell)$ for the stabilizer of
$d_W^r\smloss|_{E^r}$ in $\aut(\dataspace)\times\aut(E)$. In the notation of \pref{table:symmetry_groups}, these correspond, respectively, to
\begin{align*}
	\Obj_{W\!;E}^{[r]}(N,\ell,\dd)
	&\defeq
	\pi_E\!\left(
	\struc_{W;E}^{[r]}(N,\ell)
	\cap
	\bigl(\dist(\dd)\times\aut(E)\bigr)
	\right),
	\nonumber\\
	\Obj_W^{[r]}(N,\ell,\dd)\!\mid_E
	&\defeq
	\pi_E\!\left(
	\bigl(
	\struc_W^{[r]}(N,\ell)
	\cap
	\bigl(\dist(\dd)\times\aut_W(\weightspace)\bigr)
	\bigr)
	\!\mid_{\dataspace\times E}
	\right),
\end{align*}
where $\pi_E\colon\aut(\dataspace)\times\aut(E)\to\aut(E)$ denotes the
projection onto the second factor. Whenever the relevant expectation identity holds at $W$, the
order-$2$ \integrated invariances satisfy 
\begin{align}
	\Obj_W^{[2]}\!\mid_E
	\,\,\leq\,\,
	\Obj_{W\!;E}^{[2]}
	\,\,\leq\,\,
	\stab_{\aut(E)}\!\left((d^{2}_W\ploss)|_{E\times E}\right),
\end{align}
and both inclusions can be strict. The remaining symmetry groups are defined likewise, and the corresponding relations generalize directly.

\paragraph{Layerwise analysis.}
We now apply the preceding framework to the layerwise Hessian spectra using the symmetries inherited from $\tri{\delta\sfs_1,\delta\sfc_1}\leq\Obj_W^{[\mathrm{GN}]}$ at $W\in\cG_0$. Since $\Theta$ is Euclidean, we identify each $T_{W_i}\Theta_i$ with its canonical image in $T_W\Theta$. Thus, for example,
$T_{W_1}\Theta_1
=
\{(\delta W_1,0,0,0,0,0)\}
\leq
T_W\Theta$.
The \integrated symmetry $\delta\sfc_1(P_\pi)$ sends
	$(\delta W_1,0,0,0,0,0)$
to
	$(W_1P_\pi W_1^{-1}\delta W_1P_\pi^\top,0,0,0,0,0)$.
Therefore $\delta\sfc_1(P_\pi)T_{W_1}\Theta=T_{W_1}\Theta$ and
$\delta\sfc_1\!\mid_{T_{W_1}\Theta}\in \Obj_W^{[\mathrm{GN}]}\!\mid_{T_{W_1}\Theta}$. The full analysis of the layerwise Hessian is summarized in \pref{table:layerwise_restrictions}.

\begin{table}[!htbp]
	\centering
	\footnotesize
	\setlength{\tabcolsep}{3pt}
	\begin{tabular}{l | l | l | l | l | l}
		\hline
		& $T_{W_1}\Theta$ 
		& $T_{\b_1}\Theta$ 
		& $T_{W_2}\Theta$ 
		& $T_{\b_2,W_3}\Theta$
		& $T_{\b_3}\Theta$ \\
		\hline
		
		$\delta\sfs_1$
		& trivial
		& trivial
		& $M_1\delta W_2$
		& $\bigl(M_2\delta\b_2,\delta W_3\bigr)$
		& trivial
		\\
		
		$\delta\sfc_1$
		& $B_\pi\delta W_1P_\pi^\top$
		& $B_\pi\delta\b_1$
		& $\delta W_2C_\pi$
		& $\Phi_\pi$
		& trivial
		\\
		\hline
		
		$\tri{\delta\sfs_1,\delta\sfc_1}\!\mid_E$
		& $S_d$
		& $S_d$
		& $\stab_{O(M(1,d))}(W_3)\times S_d$
		& $\stab_{O(M(1,d))}(W_3)\times S_d$
		& trivial
		\\
		\hline
		
		\rowhead{Isotypic}{dec.}
		& \begin{tabular}[t]{@{}l@{}}
			$2\triv\oplus 3\stdsd_{d-1,1}$ \\
			$\oplus\,\stdsd_{d-2,2}\oplus\,\stdsd_{d-2,1,1}$
		\end{tabular}
		& $\triv\oplus\,\stdsd_{d-1,1}$
		& \begin{tabular}[t]{@{}l@{}}
			$(\triv,\triv)\oplus(\triv,\stdsd_{d-1,1})$ \\
			$\oplus\,(\std,\triv)\oplus\,(\std,\stdsd_{d-1,1})$
		\end{tabular}
		& \begin{tabular}[t]{@{}l@{}}
			$2(\triv,\triv)\oplus(\std,\triv)$ \\
			$\oplus\,(\triv,\stdsd_{d-1,1})$
		\end{tabular}
		& $\triv$
		\\
		\hline
		
		\rowhead{Multi.}{structure}
		& \begin{tabular}[t]{@{}l@{}}
			$\mset{1^{[2]},(d-1)^{[3]},$ \\
				$\frac{d(d-3)}{2},\frac{(d-1)(d-2)}{2}}$
		\end{tabular}
		& $\mset{1,d-1}$
		& $\mset{1,(d-1)^{[2]},(d-1)^2}$
		& $\mset{1^{[2]},(d-1)^{[2]}}$
		& $\mset{1}$
		\\
		
		\hline
	\end{tabular}
	\caption{Restrictions of $\delta\sfs_1$ and $\delta\sfc_1$ to the relevant parameter subspaces, together with the corresponding restricted invariance groups, isotypic decompositions, and multiplicity structures. Here $B_\pi\defeq W_1P_\pi W_1^{-1}$, $C_\pi\defeq W_1P_\pi^\top W_1^{-1}$, $R_\pi\defeq W_2W_1P_\pi^\top W_1^{-1}W_2^{-1}$, $A_\pi\defeq W_2W_1P_\pi W_1^{-1}W_2^{-1}$, and $\Phi_\pi(\delta\b_2,\delta W_3)\defeq(\delta\b_2-\delta W_3A_\pi^{-1}(I-A_\pi)\b_2\frac{W_3^\top}{W_3W_3^\top},\delta W_3R_\pi)$. The two $\stab_{O(M(1,d))}(W_3)$ factors in $\delta\sfs_1(M_1,M_2)$ correspond, respectively, to $M_1$ and $M_2$. Trivial factors are suppressed in the displayed product groups and isotypic decompositions where no confusion results. Since $\delta\sfc_1$ maps a $W_3$ variation to one with a $\b_2$ component, the natural invariant block is $T_{\b_2,W_3}\Theta$ rather than $T_{W_3}\Theta$.}
	\label{table:layerwise_restrictions}
\end{table}

\paragraph{SB on the sphere.}
The framework developed in this work extends naturally to settings in which the weights are constrained to a manifold. As a simple example, consider \pref{setting:gaussian} with the model
$N^{1\nicefrac{1}{2},\circ}$ (\pref{sec:pheno}), with parameter space given by the Frobenius sphere in $M(d,d)$,
\begin{align}
	\Theta
	\defeq
	\Sph(\rho)
	=
	\{W_1\in M(d,d):\|W_1\|_F=\rho\},
	\qquad
	\rho\defeq\|I_d\|_F=\sqrt d.
\end{align}
We find that $I_d\in\Theta$ remains a global minimizer. Set
$E\defeq T_{I_d}\Sph(\rho)$, the space of trace-zero matrices. We consider the \integrated invariances
$\Obj_{I_d;E}^{[2]}(N^{1\nicefrac{1}{2},\circ},\ell^\sqloss,\dd)$
of the Hessian restricted to $E$. Since $I_d$ is also an ambient critical point, the Hessian of
$\ploss|_{\Sph(\rho)}$ at $I_d$ is the restriction of the ambient Hessian to
$T_{I_d}\Sph(\rho)$.\footnote{For an embedded submanifold \(\mathcal M\), the Riemannian
	Hessian of \(f|_{\mathcal M}\) at \(W\in\mathcal M\) differs from the restriction
	of the ambient Hessian to \(T_W\mathcal M\) by a term involving the second
	fundamental form. In the spherical case, this term is a scalar multiple of the
	metric induced on \(T_W\Sph(\rho)\). Thus, at points critical only for
	\(\ploss|_{\Sph(\rho)}\), it can shift eigenvalues but does not change the
	isotypic decomposition or the multiplicities forced by the relevant orthogonal 
	symmetries.}
	Consequently, we may obtain invariances of the spherical Hessian by
	restricting ambient \integrated invariances that preserve $E$.

The invariance properties of
\(N^{1\nicefrac{1}{2},\circ}\) on \(M(d,d)\) were first studied in
\cite[Section 4.1]{arjevani2019principle}, where the objective was shown to be
\(S_d\times S_d\)-invariant, with the two factors independently permuting the
rows and columns of \(W_1\).  Although the point-stabilizer approach generally yields only trivial invariances in our settings, here it gives a diagonal copy of $S_d$ in $\Obj_{I_d}^{[2]}$, acting by $\delta W_1\mapsto P_\pi\delta W_1P_\pi^\top$. This action preserves $T_{I_d}\Sph(\rho)$ since $\tr(P_\pi\delta W_1P_\pi^\top)=\tr(\delta W_1)$. Estimates of the Hessian spectrum were obtained in
\cite{arjevani2020analytic,arjevani2023hidden}, and later for
$N^{2,\circ}(\x;W_1,\w_2)$ in \cite{arjevani2021analytic}. Although analogous estimates are possible, our focus here is on Hessian invariances. The $S_d$-representation on $M(d,d)$ agrees with that on the $\delta W_1$ component in \eqref{p1_iso_dec}. Since $T_{I_d}\Sph(\rho)=(\RR I_d)^\perp$, its isotypic decomposition is obtained by removing one trivial copy, yielding the following decomposition and multiplicity structure for $d\ge4$:
\begin{align}\label{eq:n1ahalf_sphere_iso_mult}
T_{I_d}\Sph(\rho)
	&\cong
	\triv
	\oplus
	3\stdsd_{d-1,1}
	\oplus
	\stdsd_{d-2,2}
	\oplus
	\stdsd_{d-2,1,1}
	\enspace\text{and}\enspace
	\Mset{
		1,
		(d-1)^{[3]},
		\frac{d(d-3)}{2},
		\frac{(d-1)(d-2)}{2}
	}.
\end{align}
This multiplicity structure is illustrated numerically in the top row of
\pref{fig:sb_n1ahalf_sphere}. The bottom row shows a spontaneous SB point, whose spectrum is likewise constrained by its stabilizer $\Delta (S_{d-2}\times S_2)$.

\begin{figure}[t]
	\centering
	\includegraphics[width=0.5\textwidth]{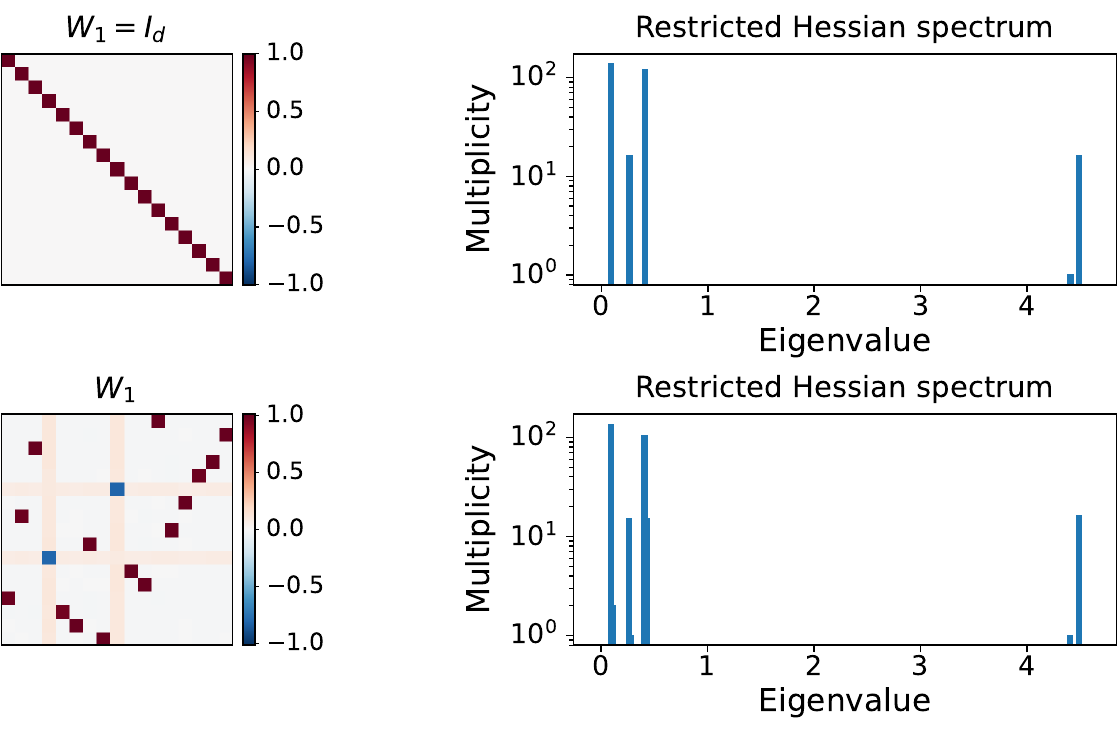}
\caption{
	Weights $W_1$ and corresponding spectra of the Hessian of
	$\ploss|_{\Sph(\rho)}$ for $N^{1\nicefrac{1}{2},\circ}$. The top row corresponds to
	the global minimum $W_1=I_d$, whose observed eigenvalue multiplicities agree
	with those predicted by \eqref{eq:n1ahalf_sphere_iso_mult}, up to the numerical
	coalescence of nearby eigenvalues. The bottom row shows a trained SB point
	with stabilizer $\Delta(S_{d-2}\times S_2)$.}
	\label{fig:sb_n1ahalf_sphere}
\end{figure}

\subsection{Two Gaussians: resolutive SB}
\label{sec:two_gaussians_analysis}
We next study global minima in the $k=2$ case of
\pref{setting:gaussian_mixture}, taking $\tau=1$ for the explicit
formulas, and allowing $d$ and $h$ to vary independently. The Bayes-optimal risk is the conditional entropy of the labels given the input,
$\mathbb E\!\left[
h_{\mathrm{bin}}\!\big(\mathbb P(Y=1\mid X)\big)
\right]
\approx 0.193$,
where \(h_{\mathrm{bin}}(u)=-u\log u-(1-u)\log(1-u)\). This value is attained by predictors whose two outputs differ by \(2(x_1+x_2)\). The symmetry arguments apply
for every $\tau>0$, with the Bayes-optimal logit difference scaled by
$\tau^{-2}$. Here, as throughout this work, the Bayes-optimal predictor is characterized by semialgebraic conditions on the network output. Since ReLU networks  (\pref{def:definable_dag}) are semialgebraic in $(\x,W)$, the set of weights satisfying these conditions for all $\x$ is semialgebraic.

Referring to the symmetries in \pref{table:symmetry_examples}, if \(d=h=2\), modulo the rescaling symmetries $\sfa_i^\lambda$, permutation
symmetries $\sfa_i^\pi$ ($i=1,2$), and common-logit shifts
$\sfs_1^\celoss$, every global minimum implemented by \(N^{3,\circ}\) is of the form
\begin{align}\label{k2_all_T}
	\cT =
	\left\{
	(W_1,W_2,W_3)
	~\middle|\;
	W_1=
	\begin{pmatrix}
		1&1\\
		-1&-1
	\end{pmatrix},
	~
	\sigma(W_2)\in \GL_2,
	~
	W_3 =
	\begin{pmatrix}
		-1&1\\
		1&-1
	\end{pmatrix}
	\sigma(W_2)^{-1}
	\right\}.
\end{align}
The characterization is proved in \pref{prop:k2_global_minima}.
At $W^\star\in\cT$ where the relevant classical derivatives exist,
well specification gives equality of the Hessian and GN matrix. In particular, all flat directions  lie in the kernel of the Hessian by \pref{prop:basic_flat}.  To identify the orbit directions $T_{W^\star}(\Obj\cdot W^\star)$, observe from
 \pref{table:obj_symmetry_groups} that for the bias-free model the two
 rescaling families contribute $2h$ directions and the common-logit
 transformations contribute $h$. Thus, $T_{W^\star}(\Struc\cdot W^\star)$ has dimension $3h=6$ and equals
 $T_{W^\star}(\Obj\cdot W^\star)$ in the present case.
There are two additional transverse flat directions at $W^\star$, characterized by $\delta W_1=0$, 
$d\sigma(\u_j)[\delta\u_j]\in\mathrm{span}\{(1,1)^\top\}$, $j=1,2$, and
$\delta W_3=-W_3d\sigma(W_2)[\delta W_2]\sigma(W_2)^{-1}$, where $\u_j$ and $\delta\u_j$ are the $j$th columns of $W_2$ and $\delta W_2$, respectively. Together, these yield \(8\) zero eigenvalues out of \(12\) total eigenvalues.

To analyze the $4$-dimensional complement of the flat-direction space described above, we now take into account the relevant \integrated symmetries: $\sfc^\sfs$ and $\sfc^\sfr$. Note that \(\sfc^\sfs\in\stab_{\Gamma}\prn{W^\star}\) for all \(W^\star\in\cT\), hence \(\sfc^\sfs\in\Gamma_{W^\star}^{[\mathrm{GN}]}\). Its sign component, denoted by \(C^{(-1)}\), is \(\delta W_1= (a,b)^\top (1,-1),a,b\in\RR\), which realizes the \(\stdsd_{(1,1)}\) representation of \(S_2\cong\tri{\sfc^\sfs}\).
This yields a \(2\)-dimensional invariant subspace orthogonal to the flat directions. 
The symmetry $\sfc^\sfr$ preserves the GN form on $C^{(-1)}$,
but not on all of $T_{W^\star}\Theta$. Inspection of $d_WN^{3,\circ}$
gives another structural symmetry of the pointwise GN form:
\begin{align}
	\delta\sfs_2:
	(\delta W_1,\delta W_2,\delta W_3)
	\mapsto
	(\operatorname{diag}(1,-1)\delta W_1,\delta W_2\operatorname{diag}(1,-1),
	\delta W_3\sigma(W_2)\operatorname{diag}(1,-1)\sigma(W_2)^{-1}).
\end{align}
It satisfies
\begin{align}
	d_WN^{3,\circ}(\x;W^\star)[\delta\sfs_2\delta W]
	=
	\sgn(x_1+x_2)\,
	d_WN^{3,\circ}(\x;W^\star)[\delta W].
\end{align}
Hence, $\delta\sfs_2$ preserves the pointwise GN form, giving
\begin{align}
	\sfc^\sfs, \delta \sfs_2
	\in
	\Obj_{W^\star}^{[\mathrm{GN}]}(N^{3,\circ}, \ell^\celoss, \cN_{\pm}) \quad\text{ and }\quad
	\sfc^\sfr \in
	\Obj_{W^\star;C^{(-1)}}^{[\mathrm{GN}]}
	(N^{3,\circ},\ell^\celoss,\cN_\pm).
\end{align}
The restricted generators yield an irreducible two-dimensional representation
on $C^{(-1)}$, with image isomorphic to $D_4$. In our experiments, Adam consistently converges to the subfamily
$\spT\subseteq\cT$ described in \eqref{Wstar}. At these points,
$\sfc^\sfr$ extends to the full tangent space, and $\delta\sfs_2$
extends to $N^{3,\semicircle}$ by also sending
$(\delta\b_1,\delta\b_2)$ to $(\operatorname{diag}(1,-1)\delta\b_1,\operatorname{diag}(1,-1)\delta\b_2)$.
The corresponding GN invariance extends to $N^{3,\twoslashcircle}$
as well.

\begin{table}[!htbp]
	\centering
	\begin{tabular}{c|c|c|c|c|c|c|c}
		&
		Points &
		$\dim\mathrm{Flat}$&
		$\codim\,\mathrm{Flat}$&
		$\sfc^{\sfs}$&
		$\sfc^{\sfr}$&
		$\delta \sfs_2$ &
		isotypic decomposition of $T_W\Theta/\mathrm{Flat}$
		\\ \hline
		$N^{3,\circ}$ &$\cT$&
		$8$ & $4$ &
		$T_W\Theta$ & $C^{(-1)}$ & $T_W\Theta$ &
		$\bigl(\fs_{(2)}\oplus \fs_{(1,1)}\bigr)_{\sfc^\sfs}
		\otimes
		\bigl(\fs_{(2)}\oplus \fs_{(1,1)}\bigr)_{\delta\sfs_2}$ 
		\\
		$N^{3,\circ}$ &$\spT$&
		$8$ & $4$ &
		$T_W\Theta$ & $T_W\Theta$ & $T_W\Theta$ &
		$(\fs_{(2)} \oplus \fs_{(1,1)})_{\sfc^\sfs}\otimes \stdih $ 
		\\
		$N^{3,\semicircle}$ &$\spT$&
		$10$ & $6$ &
		$T_W\Theta$ & $T_W\Theta$ & $T_W\Theta$ &
		$\left(\fs_{(1,1)} \oplus 2\fs_{(2)}\right)_{\sfc^\sfs}
		\otimes \mathrm{Std}_{D_4}$ \\
		$N^{3,\twoslashcircle}$ &$\spT$&
		$11$ & $6$ &
		$T_W\Theta$ & $T_W\Theta$ & $T_W\Theta$ &
		$\left(\fs_{(1,1)} \oplus 2\fs_{(2)}\right)_{\sfc^\sfs}
		\otimes \mathrm{Std}_{D_4}$ 
		\\
		$N^3$ &$\spT$&
		$12$ & $6$ &
		$T_W\Theta$ & $T_W\Theta$ & $\delta\b_3\|\bones$ &
		$\left(
		\fs_{(1,1)} \oplus 2\fs_{(2)}
		\right)_{\sfc^\sfs}
		\otimes
		\left(
		\fs_{(2)} \oplus \fs_{(1,1)}
		\right)_{\sfc^\sfr}$ 
	\end{tabular}
\caption{Isotypic decompositions of $T_W\Theta/\mathrm{Flat}$ for the
	indicated models under \pref{setting:gaussian_mixture} with $k=2$,
	where $\mathrm{Flat}$ is the space of flat directions. The table also
	gives the dimension and codimension of $\mathrm{Flat}$ and the subspaces
	on which $\sfc^{\sfs}$, $\sfc^{\sfr}$, and $\delta\sfs_2$ act.}
	\label{table:sym_per_model}
\end{table}

For $N^3$, the Euclidean-orthogonal realization of the action on the
full space reduces from $C_2\times D_4$ to
$\langle\sfc^\sfs,\sfc^\sfr\rangle\cong C_2^2$.
However, on $\delta\b_3\|\bones$, the full Euclidean-orthogonal
action is recovered. Since this subspace has codimension one,
Cauchy interlacing constrains the resulting eigenvalue splitting,
see \pref{table:n3_symmetry_breaking}.   Moreover, spectral
displacement can be bounded by standard perturbation-theoretic
estimates in terms of
$\|\hgn_{W^*}\ploss-
(\hgn_{W^*}\ploss\big|_{\delta\b_3\|\bones}\oplus 0)\|_2$.
The doublet on $C^{(-1)}$ remains intact. The splitting of the largest eigenvalue of $N^{3,\twoslashcircle}$
provides a structural mechanism for the two isolated eigenvalues
observed in \cite{sagun2016eigenvalues}. The Hessian matrix in \pref{fig:flat_directions} is shaped by the flat directions described earlier and by its restriction to $\delta\b_3\|\bones$ lying in the commutant algebra of $D_4\times C_2$. The same qualitative picture persists for different values of the
variance parameter and for squared loss; see
\pref{table:mixture_gaussians_resolution}.

\subsection{The Spectrum of a Random Hessian}
\label{sec:streamlined-three-layer}
We now use the results of \pref{sec:streamlined-internal-symmetry} to
study the Hessian spectrum of the full architecture $N^3$ \eqref{net3}
for $h\geq2$ and $k=2$ under cross-entropy loss at He
initialization,
\begin{align}
	(W_1)_{i\alpha}&\sim\cN(0,2/d),
	&(W_2)_{ji}&\sim\cN(0,2/h),
	&(W_3)_{aj}&\sim\cN(0,2/h),
\end{align}
with all weight entries sampled independently and
$\b_1=\b_2=\b_3=0$. Denote this initialization law by
$\nu^{\mathrm{He}}$. Let $H(W)$ denote the regular Hessian as considered
in that section, and set $M\defeq\EE_{W\sim\nu^{\mathrm{He}}}H(W)$. All expectations in this subsection are assumed to be finite.
For $\pi_1,\pi_2\in S_h$, let
\begin{align}
	\mathsf{i}^{\mathrm{He}}(P_{\pi_1},P_{\pi_2})
	(W_1,\b_1,W_2,\b_2,W_3,\b_3)
	&\defeq
	\bigl(\sfa_2^\pi(P_{\pi_2}^\top)\circ
	\sfa_1^\pi(P_{\pi_1}^\top)\bigr)
	(W_1,\b_1,W_2,\b_2,W_3,\b_3)
	\notag\\
	&=
	(P_{\pi_1}^\top W_1,P_{\pi_1}^\top\b_1,
	P_{\pi_2}^\top W_2P_{\pi_1},
	P_{\pi_2}^\top\b_2,W_3P_{\pi_2},\b_3).
	\label{eq:streamlined-permutation-action}
\end{align}

\begin{lemma}
	\label{lem:streamlined-N3-invariances}
	Regardless of the data distribution $\dd$,
	\begin{align}
	\tri{\mathsf{i}^{\mathrm{He}}(P_{\pi_1},P_{\pi_2}):
	\pi_1,\pi_2\in S_h}
		&\leq\avg{\Struc}^{[2]}(N^3,\ell,\nu^{\mathrm{He}}).
		\label{eq:streamlined-structural-invariance}
	\end{align}
	For the mixture of two Gaussians $\cN_\pm$ from
	\pref{setting:gaussian_mixture}, the composite symmetries $\sfc^\sfs$ and
	$\sfc^\sfr$ join the structural symmetries above in
	$\avg{\Obj}^{[2]}(N^3,\ell,\cN_\pm,\nu^{\mathrm{He}})$.
\end{lemma}
The argument below uses only the structural invariance. The representation
on $\weightspace$ of the structure group in the lemma
contains a unique copy of
$\fs_{(h-1,1)}\boxtimes\fs_{(h-1,1)}$, namely
\begin{align}
	B&\defeq
	\{(0,0,Z,0,0,0):Z\in M(h,h),\ Z\bones=0,\ Z^\top\bones=0\}.
\end{align}
The subspace $B$ has codimension $\bigO{h}$ in
$\weightspace$ for fixed $d$. 
\begin{proposition}
	\label{prop:streamlined-N3-quadratic-eigenvalue}
	The averaged Hessian satisfies $M|_B=\lambda_B I_B$ for some scalar
	$\lambda_B$. If
	$m_2\defeq\EE_{(\x,\y)\sim\dd}\|\x\|^2<\infty$, then
	\begin{align}
		\lambda_B
		&\leq
		\frac{m_2}{2dh}\left(1-\frac1\pi\right).
		\label{eq:streamlined-quadratic-He-bound}
	\end{align}
\end{proposition}

\begin{proof}
	The subspace $B$ is irreducible and occurs with multiplicity one in
	$\weightspace$, so
	\pref{lem:streamlined-N3-invariances} and
	\pref{fact:iso_dec_commuting_maps} give
	$M|_B=\lambda_B I_B$. Set $\bepsilon\defeq(1,-1)^\top\in\RR^2$,
	$\u\defeq(1,-1,0,\ldots,0)^\top/\sqrt2\in\RR^h$, and
	$Z_\star\defeq\u\u^\top$. The embedding of $Z_\star$ in the $W_2$
	coordinate is a unit vector in $B$. The network is affine in $W_2$ on each
	ReLU region, so $\nabla_{W_2}^2N^3(\x;W)=0$. For binary cross entropy,
	\begin{align}
		\nabla_1^2\ell(N^3(\x;W),\y)
		&=
		\bigl[\operatorname{softmax}(N^3(\x;W))\bigr]_1
		\bigl[\operatorname{softmax}(N^3(\x;W))\bigr]_2
		\bepsilon\bepsilon^\top.
	\end{align}
	Evaluating the regular chain rule on the unit direction $Z_\star$ therefore
	gives
	{\small \begin{align}
			\lambda_B
			&=
			\EE_{(\x,\y)\sim\dd,\,W\sim\nu^{\mathrm{He}}}\!\left[
			\bigl[\operatorname{softmax}(N^3(\x;W))\bigr]_1
			\bigl[\operatorname{softmax}(N^3(\x;W))\bigr]_2
			\left|\bepsilon^\top d_{W_2}N^3(\x;W)[Z_\star]\right|^2
			\right].
	\end{align}}
	\par\noindent Under He initialization, averaging successively over $W_3$,
	$W_2$, and $W_1$ gives
	\begin{align}
		\EE_W\left|
		\bepsilon^\top d_{W_2}N^3(\x;W)[Z_\star]
		\right|^2
		&=
		\frac{2\|\x\|^2}{dh}\left(1-\frac1\pi\right).
	\end{align}
	Since the product of the two softmax probabilities is at most $1/4$, this
	proves \eqref{eq:streamlined-quadratic-He-bound}.
\end{proof}
As $\tr(\Pi_BM\Pi_B)=(h-1)^2\lambda_B$, this estimate supplies the trace
bound required by \pref{prop:streamlined-random-bulk}.
\begin{theorem}\label{thm:streamlined-N3-spectrum}
	Assume $m_2<\infty$. For $W\sim\nu^{\mathrm{He}}$, we have
	$n_-(H(W))\leq(d+5)h$ almost surely. For every
	$T>0$ and every $\delta\in(0,1)$, the following bound holds without any
	symmetry assumption on $\dd$,
	\begin{align}
		\PP\left\{
		N_{H(W)}([0,T])
		\geq
		h^2-(d+6)h+2
		-\frac{(h-1)^2m_2}{2dh\,\delta T}
		\left(1-\frac1\pi\right)
		\right\}
		&\geq1-\delta.
		\label{eq:streamlined-structural-count}
	\end{align}
\end{theorem}

\begin{proof}
	Let $U$ be the $(h+1)$-dimensional subspace of
	perturbations with all coordinates zero except
	$(\delta W_3,\delta\b_3)=(\bones_2\a^\top,c\bones_2)$, where
	$\a\in\RR^h$ and $c\in\RR$. Then $U\perp B$ and
	$\codim(U\oplus B)=(d+5)h$. Moreover,
	$\dim(U\oplus B)-\codim(U\oplus B)=h^2-(d+6)h+2$. The decomposition
	\eqref{gn_decomp} gives
	$H(W)=Q(W)+R(W)$, where $Q(W)$ is the GN matrix and $Q(W)\succeq0$. The
	common logit shift symmetry $\sfs_1^\celoss$ in
	\pref{table:symmetry_examples} gives
	$U\subseteq\ker Q(W)\cap\ker R(W)$. On each fixed ReLU region,
	\eqref{dwn3} shows that $N^3$ is affine in $W_2$, and hence
	$\Pi_BR(W)\Pi_B=0$. Therefore, by
	\pref{prop:streamlined-N3-quadratic-eigenvalue},
	\begin{align}
		\EE_W\tr(\Pi_BQ(W)\Pi_B)
		&=
		\tr(\Pi_BM\Pi_B)
		=(h-1)^2\lambda_B.
	\end{align}
	Applying \pref{prop:streamlined-random-bulk} with $E=B$, followed by
	\eqref{eq:streamlined-quadratic-He-bound}, proves the probability bound.
\end{proof}
Symmetry guides the proof by identifying the quadratic component $B$ and
reducing its expected trace per dimension to $\lambda_B$. The bound is
asymptotically sharp in its leading term, although its lower-order constants
can be improved. A generalization to additional architectures and data
distributions is developed in~\cite{arjevani2026dsb_tf}.

\begin{remark}[Spherical initialization]
Also relevant to the experiments in \cite{sagun2016eigenvalues}
is the following initialization scheme.
	Let $\nu^{\mathrm{sph}}_{r_h}$ be the uniform law on the sphere of radius
	$r_h$ in $\weightspace$ and set
	$M_{\mathrm{sph}}\defeq\EE_{W\sim\nu^{\mathrm{sph}}_{r_h}}H(W)$. Its
	permutation invariance gives
	$M_{\mathrm{sph}}|_B=\lambda_B^{\mathrm{sph}}I_B$, and a similar derivation
	gives
	\begin{align}
		0\leq\lambda_B^{\mathrm{sph}}
		&\leq
		\frac{2r_h^4(m_2+1)}
		{\dim\weightspace(\dim\weightspace+2)}.
	\end{align}
	Thus fixed $r_h$ gives $\lambda_B^{\mathrm{sph}}=\bigO{h^{-4}}$. 
\end{remark}

\subsection{$k$ Gaussians: distributional SB}\label{sec:spec_analysis_mix}
We next consider the mixture of \(k\) Gaussians from
\pref{setting:gaussian_mixture}, with equal priors, common covariance \(I_d\),
and centers \(\{\c_i\}_{i=1}^k \subset \mathbb{R}^k\) drawn uniformly on the
sphere. The corresponding Bayes logit map is \(\x \mapsto C\x\), where
\(C \in \mathbb{R}^{k \times k}\) has rows \(\c_i^\top\). This map can
therefore be implemented without a bias term. For simplicity, we restrict to
\(N^{3,\circ}\) and the Bayes-optimal realization
\(W^{(k)}\), well-specified, where
\begin{align}\label{W_star_mix}
	W^{(k)}_1(C) &= C \otimes 
	\begin{pmatrix}
		1\\
		-1
	\end{pmatrix},
	\qquad
	W^{(k)}_2 = I_k \otimes
	\begin{pmatrix}
		1 & -1\\
		-1 & 1
	\end{pmatrix},
	\qquad
	W^{(k)}_3 = I_k \otimes
	\begin{pmatrix}
		1 & -1
	\end{pmatrix}.
\end{align}
As in the earlier \(k=2\) case, the off-diagonal entries of the second factor in 
\(W_2^{(k)}\) may be chosen arbitrarily negative without changing the Hessian
invariances, again demonstrating that the point-stabilizer approach is insufficient to explain the observed spectra.  The symmetric reference is constructed as in  \eqref{svd_reference} with the orthogonal center matrix \(U\). The construction extends to unequal priors and common nonidentity covariance
matrices, though the symmetry conclusions may weaken accordingly.
With $C=U$, the mixture admits an \(S_k\)-symmetry, 
\begin{align}
	(\x,\y)
	\mapsto
	(U^\top P_\pi U \x,\pi(\y)),
	\qquad \pi\in S_k .
\end{align}
Here \(U^\top P_\pi U\) is orthogonal and satisfies
\(U^\top P_\pi U \c_i=\c_{\pi(i)}\). Combining this distributional
symmetry with the structural symmetries of \(N^{3,\circ}\)
(\pref{table:symmetry_examples}) gives the \integrated invariance below.
\begin{lemma}\label{lem:mix_gauss_inv}
	Let \(\sfc^\pi\) be the (linear) action of \(S_k\) defined as
	follows. For every \(\pi\in S_k\),
	{\footnotesize
		\begin{align}
			\sfc^\pi(P_\pi):
			(\delta W_1,\delta W_2,\delta W_3)
			\mapsto
			\bigl(
			(P_\pi\otimes I_2)\delta W_1 (U^\top P_\pi U)^{-1},\ 
			(P_\pi\otimes I_2)\delta W_2(P_\pi^\top\otimes I_2),\ 
			P_\pi\delta W_3(P_\pi^\top\otimes I_2)
			\bigr).
	\end{align}}
	Then \(\sfc^\pi\leq \stab_{\Obj}\prn{W^{(k)}}\).
\end{lemma}
We now describe the corresponding isotypic decomposition. For simplicity,
we work with \(U=I_k\). The general orthogonal case follows
by conjugation. We first split \(\delta W\) into eight \(\sfc^\pi\)-invariant
subspaces, each naturally identified with \(\mat_k\defeq M(k,k)\). Let
\(S_+=I_k\otimes\e_1\) and \(S_-=I_k\otimes\e_2\). The first layer splits as
\(\delta W_1=\delta W_1^+\oplus \delta W_1^-\). The second layer splits into the four blocks
\(\delta W_2^{++},\delta W_2^{+-},\delta W_2^{-+},\delta W_2^{--}\). The third layer splits as
\(\delta W_3=\delta W_3^+\oplus \delta W_3^-\).
In these coordinates,
\begin{align}\label{A_eight_blocks}
	\delta W_1^\epsilon
	&=
	S_\epsilon A_1^\epsilon,
	&
	\delta W_2^{\epsilon\epsilon'}
	&=
	S_\epsilon A_2^{\epsilon\epsilon'}S_{\epsilon'}^\top,
	&
	\delta W_3^\epsilon
	&=
	A_3^\epsilon S_\epsilon^\top,
	\qquad
	\epsilon,\epsilon'\in\{-,+\}.
\end{align}
Here $A_1^\epsilon,A_2^{\epsilon\epsilon'},A_3^\epsilon\in\mat_k$ are the
corresponding block coordinate matrices. Thus
	$\delta W
	=
	\delta W_1\oplus \delta W_2\oplus \delta W_3
	\cong
	8\,\mat_k $.
Each of the eight copies of $\mat_k$ is invariant under the action of $S_k$,
and on each copy the induced action is $A\mapsto P_\pi A P_\pi^\top$. It
therefore remains to decompose a single copy of $\mat_k$ under this action. 
Let \(\Pi_1=\frac1k\bones\bones^\top\) and
\(\Pi_0=I_k-\Pi_1\), the projections onto
\(\RR\bones\) and \(\bones^\perp\), respectively.  For
\(k\ge4\), \(\mat_k\) decomposes as
\begin{align}\label{iso_dec_mix}
	\mat_k
	&=
	\mathbb R \Pi_1
	\oplus
	\mathbb R \Pi_0
	\oplus
	\fR
	\oplus
	\fC
	\oplus
	\fD
	\oplus
	\fS
	\oplus
	\fA ,
\end{align}
where
\begin{align}\label{mix_iso_comps}
	\begin{array}{rcl@{\qquad}rcl}
		\fR
		&=&
		\{\bones \v^\top:\v^\top\bones=0\},
		&
		\fC
		&=&
		\{\v\bones^\top:\v^\top\bones=0\},
		\\[0.7em]
		\fS
		&=&
		\left\{
		A\in\mat_k:
		A^\top=A,\;
		A\bones=0,\;
		\operatorname{diag}(A)=0
		\right\},
		&
		\fA
		&=&
		\left\{
		A\in\mat_k:
		A^\top=-A,\;
		A\bones=0
		\right\},
		\\[0.9em]
		\fD
		&=&
		\left\{
		\operatorname{diag}(\v)
		-
		\frac{1}{k}
		\left(
		\v\bones^\top+\bones \v^\top
		\right)
		:
		\v^\top\bones=0
		\right\}.
	\end{array}
\end{align}
The two spaces \(\mathbb R \Pi_1\) and \(\mathbb R \Pi_0\) are trivial
representations. The three spaces \(\fR\), \(\fC\), and \(\fD\) are copies of
\(\fs_{(k-1,1)}\). The spaces \(\fS\) and \(\fA\) realize
\(\fs_{(k-2,2)}\) and \(\fs_{(k-2,1,1)}\), respectively. Hence
\begin{align}
	\mat_k
	\cong
	2\,\fs_{(k)}
	\oplus
	3\,\fs_{(k-1,1)}
	\oplus
	\fs_{(k-2,2)}
	\oplus
	\fs_{(k-2,1,1)} .
\end{align}
The isotypic components of \(\delta W\) are obtained by grouping, across the eight
\(\mat_k\)-blocks, all copies of the same irreducible
representation, as summarized in \pref{table:iso_dec_summary_mix}. As for the kernel of the GN form, a perturbation satisfies
$\delta W\in\ker d_W^{[\mathrm{GN}]}\ploss$ precisely when
$d_WN^{3,\circ}(\x;W^{(k)})[\delta W]\in\RR\bones$ for a.e. $\x$. Consequently,
	$\dim\ker d_W^{[\mathrm{GN}]}\ploss
	=
	4k^2+4k $, and the quotient representation $T_{W^{(k)}}\weightspace/\ker d_W^{[\mathrm{GN}]}\ploss$ has
dimension $4k^2-4k$ and decomposes as
\begin{align}
	T_{W^{(k)}}\Theta/\ker d_W^{[\mathrm{GN}]}\ploss
	&\cong
	\begin{cases}
		4\,\fs_{(2)}
		\oplus
		4\,\fs_{(1,1)},
		& k=2,
		\\
		4\,\fs_{(3)}
		\oplus
		8\,\fs_{(2,1)}
		\oplus
		4\,\fs_{(1,1,1)},
		& k=3,
		\\
		4\,\fs_{(k)}
		\oplus
		8\,\fs_{(k-1,1)}
		\oplus
		4\,\fs_{(k-2,2)}
		\oplus
		4\,\fs_{(k-2,1,1)},
		& k\ge4 .
	\end{cases}
\end{align}
Computing, one finds that the large $\Theta(k^2)$ components corresponding to
$\fS$ and $\fA$ carry only small eigenvalues. Their common structural feature,
visible from \eqref{mix_iso_comps}, is that the associated block matrices have
zero row sum. The decomposition in \eqref{hessian_into_contributions} permits a blockwise
analysis and shows that the zero-row-sum condition itself, independently of the
isotypic component, forces small eigenvalues.
\begin{proposition}\label{prop:mix_estimates}
	Let $k\ge2$. At $W^{(k)}$,
	\begin{align}
		d_W^{[\mathrm{GN}]}\ploss(W^{(k)})[\delta W]
		&\le
		\left(16+\frac{32}{k}\right)\|\delta W\|_F^2 .
	\end{align}
On the subspace $\fZ_0$ of perturbations 	for which all block matrices in
\eqref{A_eight_blocks} have zero row sum, one~has
	\begin{align}
		d_W^{[\mathrm{GN}]}\ploss(W^{(k)})[\delta W]
		&\le
		\frac{32}{k}\|\delta W\|_F^2 .
	\end{align}
In particular, eigenvectors in $\fZ_0$ have eigenvalues $O(1/k)$, whereas the general bound is $O(1)$.
\end{proposition}
The estimates are proved in \pref{prop:gn_expression_mix_gauss}. The method is
quite general and applies readily to the other networks considered here and to \pref{setting:gaussian}. 

The isotypic
decomposition of the remaining quotient
$T_{W^{(k)}}\Theta/(\fZ_0+\ker d_W^{[\mathrm{GN}]})$ follows directly. In each copy of
$\mat_k$, the condition $A\bones=0$ excludes $\mathbb R \Pi_1$ and
$\fC$ from $\fZ_0$, while retaining $\mathbb R \Pi_0$, $\fR$, $\fD$, $\fS$, and
$\fA$. The resulting decompositions are summarized in
\pref{table:iso_dec_summary_mix}.
\begin{table}[!htbp]
	\centering
{\small	\begin{tabular}{c|c}
		space & isotypic decomposition \\
		\hline
		$T_{W^{(k)}}\Theta$
		&
		$16\,\fs_{(k)}
		\oplus
		24\,\fs_{(k-1,1)}
		\oplus
		8\,\fs_{(k-2,2)}
		\oplus
		8\,\fs_{(k-2,1,1)}$
		\\[0.2em]
		$T_{W^{(k)}}\Theta/\ker d_W^{[\mathrm{GN}]}$
		&
		$4\,\fs_{(k)}
		\oplus
		8\,\fs_{(k-1,1)}
		\oplus
		4\,\fs_{(k-2,2)}
		\oplus
		4\,\fs_{(k-2,1,1)}$
		\\[0.2em]
		$\fZ_0/(\fZ_0\cap\ker d_W^{[\mathrm{GN}]})$
		&
		$4\,\fs_{(k)}
		\oplus
		6\,\fs_{(k-1,1)}
		\oplus
		4\,\fs_{(k-2,2)}
		\oplus
		4\,\fs_{(k-2,1,1)}$
		\\[0.2em]
		$T_{W^{(k)}}\Theta/(\fZ_0+\ker d_W^{[\mathrm{GN}]})$
		&
		$2\,\fs_{(k-1,1)}$
	\end{tabular}}
	\caption{Isotypic decompositions for the $k$-class Gaussian mixture with
		$k\ge 4$.}
	\label{table:iso_dec_summary_mix}
\end{table}
A further calculation gives the remaining eigenvalue bounds in the trivial and standard isotypic components. We leave the details to the reader. Together with the preceding decomposition, these bounds yield the following result.
\begin{theorem}\label{thm:mix_spectrum}
	For $k\ge4$ and $U=I_k$, the  Hessian spectrum at
	$W^{(k)}$ for the symmetric reference is given in the table of
	\pref{fig:original_vs_symmetric_mix}.
\end{theorem}


\section{Further Concepts and Proofs}\label{sec:further_framework}
We now develop in detail several concepts and results used in the preceding
spectral analyses, restating definitions where a more precise formulation is
needed.

\subsection{Generalized derivatives}
To analyze Hessian symmetries for nonsmooth models, we regard the pointwise
and expected losses as generalized functions in the parameter variables.

The notation $V'$ denotes the
continuous dual of a topological vector space $V$. Following standard
convention, when $V$ is equipped with an inner product, we write $V^*$ for
this dual. The definitions below apply to open sets and have
local analogues on manifolds. We state them in Euclidean form for simplicity.
Relevant definitions and notation are collected in
\pref{table:basic_func_defs}.

\begin{table}[!htbp]
	\centering
	{\footnotesize
		\begin{tabular}{lll}
			\toprule
			Name & Notation & Definition \\
			\midrule

Lebesgue space &
$L^p_{\mathrm{loc}}(\Omega;Y)$, $1\le p\le\infty$ &
$\left\{u:\Omega\to Y \text{ measurable} \ \middle|\
\|u\|_{L^p(K;Y)}<\infty
\text{ for all } K\Subset\Omega\right\}$
\\[0.1em]
			
			Sobolev space &
			$W^{k,p}_{\mathrm{loc}}(\Omega;Y)$ &
			$ \left\{ u \in L^p_{\mathrm{loc}}(\Omega;Y) \ \middle|\ D^\balpha u \in L^p_{\mathrm{loc}}(\Omega;Y),\ |\balpha|\le k \right\}$ \\[0.2em]

			Radon measures &
			$\mathcal{M}_{\mathrm{loc}}(\Omega;Y)$ &
			locally finite \(Y\)-valued Radon measures on \(\Omega\) \\[0.2em]
			
			Bounded variation &
			$BV_{\mathrm{loc}}(\Omega;Y)$ &
			$ \left\{ u \in L^1_{\mathrm{loc}}(\Omega;Y) \ \middle|\ Du \in \mathcal{M}_{\mathrm{loc}}\bigl(\Omega;Y\otimes E^*\bigr) \right\}$ \\[0.2em]
			
			Special bounded variation &
			$SBV_{\mathrm{loc}}(\Omega;Y)$ &
			$ \left\{ u \in BV_{\mathrm{loc}}(\Omega;Y) \ \middle|\ D^c u=0 \right\}$ \\[0.2em]
			
			Bounded Hessian &
			$BH_{\mathrm{loc}}(\Omega;Y)$ &
			$ \{ u \in W^{1,1}_{\mathrm{loc}}(\Omega;Y) \ |\ 
			D^2u\in
			\mathcal{M}_{\mathrm{loc}}\bigl(\Omega;Y\otimes\operatorname{Sym}^2(E^*)\bigr)
			\}$ 
			\\[0.2em]
			
			Special bounded Hessian &
			$SBH_{\mathrm{loc}}(\Omega;Y)$ &
			$ \left\{ u \in BH_{\mathrm{loc}}(\Omega;Y) \ \middle|\ (D^2u)^c=0 \right\}$ \\[0.2em]
			
			Test functions &
			$C_c^\infty(\Omega;Y)$ &
			$ \left\{ \varphi \in C^\infty(\Omega;Y) \ \middle|\ \operatorname{supp}(\varphi)\Subset\Omega \right\}$ \\[0.2em]

			Distributions &
			$\cD'(\Omega;Y)$ &
			$ \bigl(C_c^\infty(\Omega;Y^*)\bigr)'$\\
			\bottomrule
	\end{tabular}}
	\caption{Definitions from functional analysis, the theory of generalized functions, and geometric measure theory used in this section. Here $E$ and $Y$ are Euclidean spaces, $\Omega\subseteq E$ is open, and the
		superscript \(c\) denotes the Cantor part.}
	
	\label{table:basic_func_defs}
\end{table}

We take the input, output, and parameter
spaces to be Euclidean spaces, equipped with their standard inner products and
Borel $\sigma$-algebras. As usual, we write $d=\dim\inspace$,
$p=\dim\weightspace$, and $k=\dim\outspace$. The data distribution $\dd$ is a
Radon probability measure on $\dataspace=\inspace\times\outspace$, allowing,
in particular, empirical distributions on finite samples, and $\weightspace$
carries Lebesgue measure $dW$. The joint space is
$\universe=\dataspace\times\weightspace$, equipped with the product Borel
structure and the product Radon measure $\dd\otimes dW$. Let the model $N:\inspace\times\weightspace\to\outspace$ and the
loss function $\ell:\outspace\times\outspace\to[0,\infty)$ be Borel
measurable and locally bounded,\footnote{This permits, for example, networks with Heaviside activations \cite{cho2009kernel} and the $0$--$1$ classification loss.} and let the
pointwise loss $\smloss$ be defined as in \eqref{smloss_def}. Then $\smloss$
is nonnegative and belongs to
$L^1_{\mathrm{loc}}(\universe,\dd\otimes dW)$. By abuse of notation, we
identify $\smloss$ with the distribution of order zero induced by the Radon
measure $\smloss\,(\dd\otimes dW)$ on $\universe$. Assume further that the expected loss $\ploss$ in \eqref{eloss} belongs to
$L^1_{\mathrm{loc}}(\weightspace)$. We adopt the same convention for
$\ploss$, identifying it with the Radon measure
$\ploss(W)\,dW
=(\pi_W)_*\bigl(\smloss(\z,W)(\dd\otimes dW)\bigr)$.

\begin{proposition}[Distributional differentiation under expectation]
	\label{prop:partial_distributional}
	For every multiindex \(\balpha\), the distribution
	\(\partial_W^{\balpha}\ploss\in\cD'(\weightspace)\) is well
	defined and satisfies
	\begin{align}
		\partial_W^{\balpha}\ploss
		&=
		\int_{\dataspace}
		\partial_W^{\balpha}\smloss(\z,\cdot)\,d\dd(\z)
		\quad \text{in }\,\, \mathcal D'(\weightspace).
	\end{align}
\end{proposition}
\begin{proof}[Sketch]
	Test against $\varphi\in C_c^\infty(\weightspace)$ and use Fubini, justified
	by local integrability.
\end{proof}
The preceding identity requires only local integrability. The GN term contains
products of first derivatives and therefore requires additional regularity.

\begin{proposition}[GN decomposition of the pointwise loss]
	\label{prop:pointwise_gn_decomposition}
	Assume that \(\ell\in C^2(\outspace\times\outspace)\) and that  \(N\), as
	a map on \(\universe=\dataspace\times\weightspace\), is
	\(SBH_{\mathrm{loc}}\cap W^{1,\infty}_{\mathrm{loc}}\) in the
	\(W\)-variables.\footnote{For \eqref{gn_dist_dec}, the weaker assumption $BH_{\mathrm{loc}}$
		would suffice. We assume $SBH_{\mathrm{loc}}$ for the stated regularity
		conclusion, which holds for the networks in this work.} Then \(\smloss\) is \(SBH_{\mathrm{loc}}\) in the
	\(W\)-variables, the pointwise GN form in \eqref{gn_def} induces a regular
	\(\operatorname{Sym}^2(T^*\weightspace)\)-valued distribution on
	\(\universe\), denoted (by the usual abuse of notation) by
	\(d_W^{\mathrm{GN}}\smloss\), and, as locally finite
	\(\operatorname{Sym}^2(T^*\weightspace)\)-valued distributions on
	\(\universe\),
	\begin{align}\label{gn_dist_dec}
		d_W^2\smloss
		&=
		d_W^{\mathrm{GN}}\smloss
		+
		\sum_{a=1}^k
		\partial_{1,a}\ell(N,\y)\,d_W^2N_a .
	\end{align}
\end{proposition}

\begin{proof}[Sketch]
	The Sobolev chain rule gives
	$d_W\smloss=\sum_a\partial_{1,a}\ell(N,\y)d_WN_a$ in the $W$-variables.
	Since $d_WN_a\in BV_{\mathrm{loc}}$ and
	$\partial_{1,a}\ell(N,\y)$ is locally Lipschitz, the $BV$ product rule gives
	the stated identity. The corresponding $SBV_{\mathrm{loc}}$ conclusion follows since
	$D_W^2N$ has no Cantor part.
\end{proof}
The preceding decomposition is an identity of generalized functions.
At a fixed parameter, the classical Hessian can be identified with the
 GN form under a separate set of local assumptions.

\begin{proposition}[GN identity at a well-specified point]
	\label{prop:well_specified_gn_identity}
	Let $W^*\in\weightspace$. Assume that
	$\ell\in C^2(\outspace\times\outspace)$ and that $d_1^2\ell$ is
	bounded. Assume further that $\ploss$ is twice Fréchet
	differentiable at $W^*$ and that $N(\x;\cdot)$ is differentiable at
	$W^*$ for $\dd_{\inspace}$-almost every $\x$. Suppose that there are a
	neighborhood $U$ of $W^*$ and a function
	$\lambda\in L^2(\dd_{\inspace})$ such that
	\begin{align}
		\|N(\x;W)-N(\x;W^*)\|
		\le
		\lambda(\x)\|W-W^*\|
	\end{align}
	for all $W\in U$ and for $\dd_{\inspace}$-almost every $\x$. If the
	model is well specified at $W^*$, then
	\begin{align}
		\nabla_W^2\ploss(W^*)
		=
		\EE_{(\x,\y)\sim\dd}
		\bigl[
		\hgn_W(\x,W^*,\y)
		\bigr].
	\end{align}
\end{proposition}

\begin{proof}[Sketch]
	Taylor's formula in the first argument of $\ell$ gives a linear term
	and a quadratic remainder. Well specification eliminates the linear term. The rest follows by standard domination.
\end{proof}

\paragraph{DAG networks.} We now verify that the networks considered in this work satisfy the regularity hypotheses of \pref{prop:pointwise_gn_decomposition}. We begin by specifying the class
of networks under consideration.

\begin{definition}[Definable DAG networks]\label{def:definable_dag}
Fix an o-minimal expansion of $\RR$. Let
	$(V_N,E_N)$ be a directed acyclic graph with sets of input and
	output nodes $V_{\mathrm{in}},V_{\mathrm{out}}\subseteq V_N$.
	Associate with each node $v\in V_N$ a Euclidean
	feature space $\mathcal Z_v$, with
	$\prod_{v\in V_{\mathrm{out}}}\mathcal Z_v=\outspace$.
	Each non-input node $v$ is assigned a locally bounded definable map
	$\Psi_v:
	\left(\prod_{u\,|\,(u,v)\in E_N}\mathcal Z_u\right)
	\times\weightspace
	\longrightarrow\mathcal Z_v$. 	For $\x=(x_v)_{v\in V_{\mathrm{in}}}\in\inspace$ and
	$\theta\in\weightspace$, set
	$z_v(\x;\theta)\defeq x_v$ for $v\in V_{\mathrm{in}}$, and define
	recursively
	\begin{align}
		z_v(\x;\theta)
		&=
		\Psi_v
		\left(
		\bigl(z_u(\x;\theta)\bigr)_{u\,|\,(u,v)\in E_N},
		\theta
		\right),
		&&
		v\in V_N\setminus V_{\mathrm{in}}.
	\end{align}
	The associated realization map is
	$N(\x;\theta)
	=
	\bigl(z_v(\x;\theta)\bigr)_{v\in V_{\mathrm{out}}}$. A DAG network satisfying the preceding conditions is called definable. It is called locally DC if, in addition, every node map is componentwise locally DC, meaning that each component is
	locally the difference of two convex functions. 
\end{definition}

Induction along a topological ordering of the DAG shows that $N$ is definable, in particular Borel measurable, and locally bounded.
If every node map is componentwise locally DC, then so is $N$. The definable class encompasses fully connected, convolutional, and message-passing architectures (including attention layers), together with skip connections, parameter sharing,
 average and max pooling, and locally bounded definable readout maps. It also
 allows locally bounded definable activations and operations, including ReLU,
 Heaviside, GELU, and absolute value. Discontinuous node maps may invalidate the GN decomposition used extensively in  this work. Accordingly, for the remainder of this discussion,  we additionally require every node map to be componentwise locally DC. This excludes Heaviside but retains all architectures considered in this work.

\paragraph{First-order derivatives.}
Since $N$ is componentwise locally DC, it is locally Lipschitz in $(\x,W)$ and
hence $N\in W^{1,\infty}_{\mathrm{loc}}
(\inspace\times\weightspace;\outspace)$. Consequently, every first-order
distributional derivative of $N$, including those with respect to the
$W$-variables, is a regular distribution.

\paragraph{Second-order derivatives.}
Since $N$ is componentwise locally DC,
$\partial_\theta N\in BV_{\mathrm{loc}}$ for every $\theta$. Hence,
$D_\eta D_\theta N\in
\mathcal M_{\mathrm{loc}} (\inspace\times\weightspace;\outspace)$
for every pair $\eta,\theta$. Fix a common finite definable $C^2$ stratification adapted to $N$ and its
first derivatives such that $N$ is $C^2$ on each full-dimensional stratum. Then $\partial_\theta N$ is $C^1$ there, so the singular part of $D(\partial_\theta N)$ is supported on the
lower-dimensional strata. These strata are locally $\mathcal H^{d+p-1}$-finite by definability
\cite{fornasiero2012hausdorff}. Hence, the Cantor part of
$D(\partial_\theta N)$ vanishes \cite{ambrosio2000functions}. The one-sided traces of $\partial_\theta N$ exist
$\mathcal H^{d+p-1}$-almost everywhere on each codimension-one stratum. For each codimension-one stratum $\cI$, choose a unit normal $\nu_{\cI}$
oriented from the minus side to the plus side, and write $[\![\partial_\theta N]\!]_{\cI}=(\partial_\theta N)^+\big|_{\cI}-(\partial_\theta N)^-\big|_{\cI}$. Therefore
\begin{align}
	D_\eta D_\theta N
	&=
	\partial_\eta\partial_\theta N\,d\x\,dW
	+
	\sum_{\cI}
	[\![\partial_\theta N]\!]_{\cI}
	(\nu_{\cI}\cdot\e_\eta)
	\mathcal H^{d+p-1}\!\restriction\cI.
	\label{eq:relu_second_derivative_interface}
\end{align}
Together with the conclusion
above, this yields
$N\in
SBH_{\mathrm{loc}}
(\inspace\times\weightspace;\outspace)
\cap
W^{1,\infty}_{\mathrm{loc}}
(\inspace\times\weightspace;\outspace)$.

\paragraph{Interface delta representation.}
Let $\Omega\subseteq\RR^{n}$ be open, let $g\in C^1(\Omega)$, and assume that
$\nabla g\neq0$ on $\{g=0\}$. We denote by $\delta(g)$ the Radon measure on
$\Omega$ defined by
\begin{align}
	\langle \delta(g),\psi\rangle
	&=
	\int_{\{g=0\}}
	\frac{\psi(\boldsymbol\xi)}
	{|\nabla g(\boldsymbol\xi)|}
	\,d\mathcal H^{n-1}(\boldsymbol\xi),
	\qquad
	\psi\in C_c^\infty(\Omega).
\end{align}
By the coarea formula,
$|\nabla g|\,\delta(g)
=
\mathcal H^{n-1}\!\restriction\{g=0\}$. \pref{table:hessian_quantities} gives the
distributional second-variation formulas for
$N^{3,\circ}(\x;W)=W_3\sigma(W_2\sigma(W_1\x))$.

\begin{table}[!htbp]
	\centering
	\footnotesize
	\renewcommand{\arraystretch}{1.35}
	\begin{tabular}{p{.35\textwidth}|p{.59\textwidth}}
		\hline
		\textbf{quantity} & \textbf{expression} \\
		\hline
		
		&
		$
		\z_1=W_1\x,
		\quad
		\r_1=\sigma(\z_1),
		\quad
		\z_2=W_2\r_1,
		\quad
		\r_2=\sigma(\z_2),
		\quad
		\u=W_3\r_2
		$
		\\[1pt]
		
		ReLU masks
		&
		$
		D_1=\operatorname{diag}(\mathbf 1_{\z_1>0}),
		\qquad
		D_2=\operatorname{diag}(\mathbf 1_{\z_2>0})
		$
		\\[1pt]
		
		preactivation variations
		&
		$
		\delta\z_1
		=
		\delta W_1\x,
		\qquad
		\delta\z_2
		=
		\delta W_2\r_1
		+
		W_2D_1\delta\z_1
		$
		\\[1pt]
		
		first variation of logits
		&
		$
		\delta\u
		=
		\delta W_3\r_2
		+
		W_3D_2\delta\z_2
		$
		\\[1pt]
		
		absolutely continuous second variation
		&
		$
		\begin{aligned}
			\delta^2\u_{\mathrm{ac}}
			&=
			2\delta W_3D_2\delta\z_2
			+
			2W_3D_2\delta W_2D_1\delta\z_1
			\\
			&=
			2\Big(
			\delta W_3D_2\delta W_2\r_1
			+
			\delta W_3D_2W_2D_1\delta\z_1
			+
			W_3D_2\delta W_2D_1\delta\z_1
			\Big)
		\end{aligned}
		$
		\\[1pt]
		
		layer-one singular second variation
		&
		$
		\delta^2\u_{\mathrm{sing},1}
		=
		W_3D_2W_2
		\operatorname{diag}(\delta(\z_1))
		(\delta\z_1\odot\delta\z_1)
		$
		\\[1pt]
		
		layer-two singular second variation
		&
		$
		\delta^2\u_{\mathrm{sing},2}
		=
		W_3
		\operatorname{diag}(\delta(\z_2))
		(\delta\z_2\odot\delta\z_2)
		$
		\\[1pt]
		
		cross-entropy quantities
		&
		$
		\p=\operatorname{softmax}(\u),
		\qquad
		\g=\p-\e_\y,
		$
		\\[1pt]
		
		generalized Hessian of the loss
		&
		$
		\delta^2\smloss
		=
		(\delta\u)^\top (\operatorname{diag}(\p)-\p\p^\top)(\delta\u)
		+
		\g^\top\delta^2\u_{\mathrm{ac}}
		+
		\g^\top
		\bigl(
		\delta^2\u_{\mathrm{sing},1}
		+
		\delta^2\u_{\mathrm{sing},2}
		\bigr)
		$
		\\
		
		\hline
	\end{tabular}
\caption{Distributional second-variation formulas for $N^{3,\circ}$. The singular terms are interpreted through the interface-jump formula
	\eqref{eq:relu_second_derivative_interface}.}
	\label{table:hessian_quantities}
\end{table}

\subsection{Presheaves of symmetry groups}\label{sec:formal_sym}
We first prove \pref{prop:inv_of_eloss} and then formulate its higher-order
analogues.

The zeroth-order statement does not require distributional derivatives.
Let $g\in\struc(N,\ell)$ be $\dd$-compatible. Then
$\smloss\circ(\id_{\dataspace},g_{\weightspace})
=
(\smloss\circ g)\circ
(g_{\dataspace}^{-1},\id_{\weightspace})
=
\smloss\circ
(g_{\dataspace}^{-1},\id_{\weightspace})$.
Therefore,
\begin{align}
	\ploss\circ g_{\weightspace}
	&=
	\EE_{\dd}\!\left[
	\smloss\circ(\id_{\dataspace},g_{\weightspace})
	\right]
	=
	\EE_{\dd}\!\left[
	\smloss\circ
	(g_{\dataspace}^{-1},\id_{\weightspace})
	\right]
	=
	\EE_{\dd}[\smloss]
	=
	\ploss,
	\label{loss_obj}
\end{align}
where the penultimate equality follows because
$g_{\dataspace}^{-1}\in\dist(\dd)$ and hence
$\EE_{\dd}[\phi]
=
\EE_{\dd}[\phi\circ g_{\dataspace}^{-1}]$
for every $\dd$-integrable measurable
$\phi\colon\dataspace\to\RR$.
This proves the invariance assertion in \pref{prop:inv_of_eloss}. The group
assertions follow because the set of $\dd$-compatible structural symmetries
is precisely
$\struc(N,\ell)\cap
\bigl(\dist(\dd)\times\aut(\weightspace)\bigr)$,
which is a subgroup of $\aut(\universe)$. The internal case is analogous.

The extension to higher derivatives is as follows. Let
$\Omega=\Omega_{\dataspace}\times\Omega_{\weightspace}\subseteq\universe$
be open,   write
$\ploss\defeq \ploss_\Omega(W) = \int_{\Omega_{\dataspace}}\smloss(\z,W)\,d\dd(\z)$, and define
$\aut_{\Omega}\defeq
\aut(\Omega_{\dataspace})\times
\aut(T\Omega_{\weightspace})$, where the
second factor is a chosen group of admissible vector bundle automorphisms
of $T\Omega_{\weightspace}$ covering
$\id_{\Omega_{\weightspace}}$. In the Euclidean setting considered,
each admissible affine transformation $g_\Theta$ induces the constant
bundle automorphism $\widehat g_\Theta$ given by
$\widehat g_\Theta(W,\delta W)
=
(W,dg_\Theta\delta W)$. Let $r\geq1$ and let $V$ be a Euclidean space. For a family
$\Phi=(\Phi_{\z})_{\z\in\Omega_{\dataspace}}$ of distributions
\[
\Phi_{\z}\in\mathcal D'\bigl(
\Omega_{\weightspace};
V\otimes\operatorname{Sym}^r(T^*\Omega_{\weightspace})
\bigr)
\]
and $g=(g_{\dataspace},A)\in\aut_{\Omega}$, define
\[
(g\cdot\Phi)_{\z}
\defeq
\bigl(\id_V\otimes(A^{-1})^*\bigr)
\Phi_{g_{\dataspace}^{-1}\z}
\]
(see \pref{ftnote:actionon_distr}).
Here, taking $A=\widehat g_\Theta$ gives, equivalently,
\begin{align}\label{eq:action_on_tensor_distributions}
	(g\cdot\Phi)_{\z}(\varphi)[\delta_1,\dots,\delta_r]
	&=
	\Phi_{g_{\dataspace}^{-1}\z}(\varphi)
	\bigl[
	(dg_\Theta)^{-1}\delta_1,\dots,(dg_\Theta)^{-1}\delta_r
	\bigr],
\end{align}
where $\varphi\in C_c^\infty(\Omega_{\weightspace})$ and
$\delta_1,\dots,\delta_r\in\weightspace$, using the canonical
identification $T_W\weightspace\simeq\weightspace$ for every
$W\in\Omega_{\weightspace}$. For example, taking $V=\outspace$ gives the
action on $d_W^rN$, while taking $V=\RR$ gives the actions on
$d_W^r\smloss$ and, for $r=2$, on $\qgn_W$. The same definition applies to
the other forms considered in this work. The action is defined so that distributional invariance on open sets can
pass to pointwise invariance. We define the corresponding symmetry groups on product open subsets of
$\universe$ as stabilizers under this action, with ambient group
$\aut_{\Omega}$. We denote them by replacing the pointwise subscript $W$
with $\Omega$, as in $\model_W^{[r]}(N)$ and
$\model_{\Omega}^{[r]}(N)$. Fix an open set $\Omega_{\dataspace}\subseteq\dataspace$. As
$\Omega_{\weightspace}\subseteq\weightspace$ varies, these assignments on
$\Omega_{\dataspace}\times\Omega_{\weightspace}$ form presheaves of groups
over $\weightspace$ under restriction in the parameter variable.

\begin{remark}\label{rem:sym_noneuclid}
When $\dataspace$ is a manifold, one may allow $\dist(\dd,\Omega_{\dataspace})$ to vary
	with $\Omega_{\dataspace}$, potentially leading to nontrivial
	local-to-global phenomena \cite{arjevani2027optsym}.
\end{remark}

We now turn to completing the proof of \pref{prop:inv_of_eloss} at higher orders.
\begin{theorem}\label{thm:local_symmetry_inv}
For \(\alpha \in \{0,1,2,\dots\} \cup \{\mathrm{GN}\}\),
	\begin{enumerate}[leftmargin=*, label=(\roman*)]
\item $\Gamma_\Omega^{[\alpha]}$ is a group and
$d_W^\alpha\ploss$ is $\Gamma_\Omega^{[\alpha]}$-invariant on $\Omega$.

\item Whenever the relevant expectation identity holds pointwise at $W$,
$d_W^\alpha\ploss(W)$ is $\Gamma_W^{[\alpha]}$-invariant.

\item If $g_{\weightspace}\in\Gamma_W$, then
$dg_{\weightspace}\in\Gamma_W^{[r]}$ for every $r\in\NN$. If
$g\in\model_W^{[1]}(N)$ is $\dd$-compatible and
$g_{\dataspace}$ stabilizes
$\z\mapsto d_1^2\ell\bigl(N(\x;W),\y\bigr)$, then
$g_{\weightspace}\in\Obj_W^{[\mathrm{GN}]}$.

\item If $d_W^\alpha\ploss$ admits a locally integrable representative,
then every $g_{\weightspace}\in\Gamma_\Omega^{[\alpha]}$ stabilizes
$d_W^\alpha\ploss(W)$  for almost every
$W\in\Omega_{\weightspace}$ (fiberwise when $\alpha\ne0$), and for every $W$ if the representative is~continuous.

	\end{enumerate}
The GN assertions assume that $\EE_{\dd}[d_W^{\mathrm{GN}}\smloss]$
exists in $\mathcal D'(\Omega_{\weightspace})$ and, where used
pointwise, is finite.
\end{theorem}

\begin{proof}
\noindent\emph{Item (i).}
By definition, $\Gamma_\Omega^{[\alpha]}$ is the image under
$\pi_{\Omega_{\weightspace}}$ of the subgroup
\begin{align}
	\struc_\Omega^{[\alpha]}
	\cap
	\bigl(
	\dist(\dd)\times\aut(T\Omega_{\weightspace})
	\bigr)
\end{align}
of $\aut_\Omega$, and is therefore a group. For $\alpha=0$, invariance is
exactly \eqref{loss_obj}. For $\alpha=r\in\NN$, the same argument applies,
with \pref{prop:partial_distributional} providing \eqref{holy_eqn} on
$\Omega$. The case $\alpha=\mathrm{GN}$ is analogous. \emph{Item~(ii).}
The same argument applies pointwise at $W$, with the relevant expectation
identity replacing \pref{prop:partial_distributional}. \emph{Item (iii).}
Let $g_{\weightspace}\in\stab_{\Gamma}(W)$. By definition, there exists
$h\in\dist(\dd)$ such that
$g=(h,g_{\weightspace})\in\struc(N,\ell)$ and
$g_{\weightspace}(W)=W$. 
 Since $g_{\weightspace}$ is affine,
differentiating on this domain gives
\begin{align}
	d_W^r\smloss(\z,W)
	\bigl[
	dg_{\weightspace}\delta_1,\dots,
	dg_{\weightspace}\delta_r
	\bigr]
	&=
	d_W^r\smloss(h^{-1}\z,W)
	\bigl[
	\delta_1,\dots,\delta_r
	\bigr].
\end{align}
Hence $(h,dg_{\weightspace})\in\struc_W^{[r]}$, and therefore
$dg_{\weightspace}\in\Gamma_W^{[r]}$. Now let $g\in\model_W^{[1]}(N)$ be $\dd$-compatible, and suppose that
$g_{\dataspace}$ stabilizes
$\z\mapsto d_1^2\ell\bigl(N(\x;W),\y\bigr)$. 
This gives 
\begin{align}
	\qgn_W(\z)
	\bigl[
	g_{\weightspace}\delta_1,
	g_{\weightspace}\delta_2
	\bigr]
	&=
	\qgn_W(g_{\dataspace}^{-1}\z)
	\bigl[
	\delta_1,\delta_2
	\bigr]
\end{align}
wherever the GN form is defined. Hence
$g\in\struc_W^{[\mathrm{GN}]}$, and $\dd$-compatibility yields
$g_{\weightspace}\in\Obj_W^{[\mathrm{GN}]}$.  \emph{Item~(iv).}
By item (i), every
$g_{\weightspace}\in\Gamma_\Omega^{[\alpha]}$ satisfies
$g_{\weightspace}\cdot d_W^\alpha\ploss=d_W^\alpha\ploss$ on
$\Omega_{\weightspace}$. Their locally integrable representatives 
agree almost everywhere, and everywhere if they are continuous, giving 
the pointwise conclusion.
\end{proof}

By the preceding theorem, \pref{lem:hessian_inv_despite} reduces to verifying
the corresponding model symmetries on the relevant open region, which follows
directly from the formulas in
\pref{table:hessian_quantities}.

\subsection{Quantitative  analysis via the GN form}
\label{sec:GN_scale}
The favorable structure of the GN form provides direct quantitative control
on the spectrum. We derive uniform and $O(k^{-1})$ bounds in
\pref{setting:gaussian_mixture} and corresponding bounds in
\pref{setting:gaussian}. Sharper estimates are possible, but those below
suffice for our purposes.

\begin{proposition}\label{prop:gn_expression_mix_gauss}
In \pref{setting:gaussian_mixture}, take $d=k$, $\tau=1$, and centers
given by the rows of $I_k$.  At $W^{(k)}$  defined in
	\eqref{W_star_mix} for $N^{3,\circ}$,
	\begin{align}
		d_W^{\mathrm{GN}}\ploss(W^{(k)})[\delta W, \delta W]
		\le
		\left(16+\frac{32}{k}\right)\|\delta W\|_F^2 .
	\end{align}
On the subspace where the rows of the block matrices in
\eqref{A_eight_blocks} have zero sum,
	\begin{align}
		d_W^{\mathrm{GN}}\ploss(W^{(k)})[\delta W,\delta W]
		\le
		\frac{32}{k}\|\delta W\|_F^2 .
	\end{align}
\end{proposition}

\begin{proof}
	With $S_+=I_k\otimes\e_1$ and $S_-=I_k\otimes\e_2$, define the eight block maps by
	\begin{align}\label{eight_blocks}
		\Phi_{1,\epsilon}(A)(\x)
		&=
		W_3^*D_2(\x)W_2^*D_1(\x)S_\epsilon A\x,
		\\
		\Phi_{2,\epsilon,\epsilon'}(A)(\x)
		&=
		W_3^*D_2(\x)S_\epsilon A S_{\epsilon'}^\top
		\sigma(W_1^*\x),
		\\
		\Phi_{3,\epsilon}(A)(\x)
		&=
		A S_\epsilon^\top
		\sigma(W_2^*\sigma(W_1^*\x)),
	\end{align}
	where $\epsilon,\epsilon'\in\{\pm1\}$. We index these blocks by
		$\sB
		=
		\{(1,\epsilon):\epsilon\in\{\pm1\}\}
		\cup
		\{(2,\epsilon,\epsilon'):\epsilon,\epsilon'\in\{\pm1\}\}
		\cup
		\{(3,\epsilon):\epsilon\in\{\pm1\}\}$.
	Writing $A_B$ for the perturbation in block $B$, the definitions of $\Phi_B$ give
	\begin{align}
		d_WN^{3,\circ}(\x;W^*)[\delta W]
		=
		\sum_{B\in\sB}\Phi_B(A_B)(\x).
	\end{align}
	Each block map has the form
		$\Phi_B(A)(\x)
		=
		s_B R_B(\x)A\mathbf v_B(\x)$,
	where $s_B\in\{\pm1\}$, $R_B(\x)$ is diagonal with entries in
	$\{0,\pm1\}$, and 
	$\mathbf v_B(\x)\in\{\x,\sigma(\x),\sigma(-\x)\}$. Set
	\begin{align}
		 C_p(\x)
		=
		\operatorname{diag}(p(\x))-p(\x)p(\x)^\top .
	\end{align}
	Since $ C_p(\x)\preceq\operatorname{diag}(p(\x))$, 
		$R_B(\x)^\top C_p(\x)R_B(\x)
		\preceq
		\operatorname{diag}(p(\x))$.
	If $\a_i^\top$ is the $i$-th row of $A$, then 
	\begin{align*}
		\Phi_B(A)(\x)^\top C_p(\x)\Phi_B(A)(\x)
		&=
		\bigl(A\mathbf v_B(\x)\bigr)^\top
		R_B(\x)^\top C_p(\x)R_B(\x)
		\bigl(A\mathbf v_B(\x)\bigr)
		\le
		\sum_{i=1}^k
		p_i(\x)
		\bigl(\a_i^\top\mathbf v_B(\x)\bigr)^2 .
	\end{align*}
	By well specification, $p_i(\x)=\Prob(Y=i\mid \x)$. Hence, for every
	nonnegative measurable $h$,
	$\E_{\x}[p_i(\x)h(\x)]=k^{-1}\E_{\x\mid Y=i}[h(\x)]$. Therefore
	\begin{align}
		\E_{\x}[
		\Phi_B(A)(\x)^\top C_p(\x)\Phi_B(A)(\x)]
		&\le
		\frac1k
		\sum_{i=1}^k
		\E_{\x\mid Y=i}[
		(\a_i^\top\mathbf v_B(\x))^2].
	\end{align}
	Given $\mathbf v_B\in\{\x,\sigma(\x),\sigma(-\x)\}$, set
	$\v=\mathbf v_B(\x)$. Define $m_{01},m_{11},m_{02},m_{12}$ by
	\begin{align}
		\E_{\x\mid Y=i}[v_j]
		&=
		m_{01},
		&
		\E_{\x\mid Y=i}[v_i]
		&=
		m_{11},
		&
		\E_{\x\mid Y=i}[v_j^2]
		&=
		m_{02},
		&
		\E_{\x\mid Y=i}[v_i^2]
		&=
		m_{12}
	\end{align}
	for $j\ne i$. In all cases,
	$m_{02}\le 1$, $m_{12}\le 2$, $|m_{01}|\le 1$, and
	$|m_{11}-m_{01}|\le 1$. By independence of $v_1,\ldots,v_k$
	conditioned on $Y=i$,
	\begin{align}
		\E_{\x\mid Y=i}[(\a_i^\top\v)^2]
		&=
		(m_{02}-m_{01}^2)\sum_{j\ne i}a_{ij}^2
		+
		(m_{12}-m_{11}^2)a_{ii}^2
		+
		\left(
		m_{01}\a_i^\top\bones
		+
		(m_{11}-m_{01})a_{ii}
		\right)^2
		\nonumber\\
		&\le
		4\|\a_i\|^2
		+
		2(\a_i^\top\bones)^2 .
	\end{align}
	Therefore
	\begin{align}\label{two_norm_bounds}
		\E_{\x}[
		\Phi_B(A)(\x)^\top C_p(\x)\Phi_B(A)(\x)]
		&\le
		\frac4k\|A\|_F^2+\frac2k\|A\bones\|^2 .
	\end{align}
Passing from single blocks to the full quadratic form
(see \eqref{hessian_into_contributions}) yields
	\begin{align}
		d_W^{\mathrm{GN}}(W^*)[\delta W,\delta W]
		&=
		\E_{\x}
		\left[
		\left(
		\sum_{B\in\sB}\Phi_B(A_B)(\x)
		\right)^\top
		 C_p(\x)
		\left(
		\sum_{B\in\sB}\Phi_B(A_B)(\x)
		\right)
		\right]\nonumber\\
		&\le
		\left(
		\sum_{B\in\sB}
		\sqrt{
			\E_\x[
			\Phi_B(A_B)(\x)^\top C_p(\x)\Phi_B(A_B)(\x)]
		}
		\right)^2 .
	\end{align}
	On the subspace $A_B\bones=0$, we find by \eqref{two_norm_bounds} that
	{\footnotesize \begin{align}
			d_W^{\mathrm{GN}}(W^*)[\delta W,\delta W]
			&\le
			\left(\sum_{B\in\sB}\sqrt{\frac4k}\|A_B\|_F\right)^2
			\le
			\left(\sum_{B\in\sB}\frac4k\right)
			\left(\sum_{B\in\sB}\|A_B\|_F^2\right)
			=
			\frac{32}{k}\sum_{B\in\sB}\|A_B\|_F^2
			=
			\frac{32}{k}\|\delta W\|_F^2 .
	\end{align}}
	On the full space, using $\|A_B\bones\|^2\le k\|A_B\|_F^2$,
	\eqref{two_norm_bounds} gives
	{\footnotesize \begin{align}
			d_W^{\mathrm{GN}}(W^*)[\delta W,\delta W]
			&\le
			\left(\sum_{B\in\sB}\sqrt{2+\frac4k}\|A_B\|_F\right)^2
			\le
			\left(\sum_{B\in\sB}\left(2+\frac4k\right)\right)
			\left(\sum_{B\in\sB}\|A_B\|_F^2\right)
			=
			\left(16+\frac{32}{k}\right)\|\delta W\|_F^2 .
	\end{align}}

\end{proof}
The same methods apply to \pref{setting:gaussian}.

\begin{proposition}\label{prop:exact_gn_expression_gauss}
	In \pref{setting:gaussian}, for $N^{3,\circ}$, let $W^\star$ and $\q$
	be as in \eqref{second_diff_setting_gauss}. Then
	\eqref{second_diff_setting_gauss} holds. In particular, if
	$\bones^\top\delta W_1=0$, $\operatorname{tr}\delta W_1=0$, and
	$\q=0$, then
	\begin{align}
		d_W^{\mathrm{GN}}(W^\star)[\delta W]
		&\le
		\left(\frac12+\frac1\pi\right)
		\|\delta W_1\|_F^2.
	\end{align}
For arbitrary $\delta W$, one has
$d_W^{\mathrm{GN}}(W^\star)[\delta W]\le 4d^2\|\delta W\|_F^2$, and the order $d^2$ is asymptotically sharp.
\end{proposition}

\begin{proof}
Substituting the expression for $d_WN^{3,\circ}(W^\star)$ from
\eqref{dn3circ} into \eqref{gn_def} and expanding, we find
\begin{align}
d_W^{\mathrm{GN}}(W^\star)[\delta W]
=
2\E[u^2]+4\E[uv]+2\E[v^2],\quad	u=\bones^\top D\sigma(\x)\delta W_1\x,
	\qquad
	v=\q^\top\sigma(\x).
\end{align}
Computing, we find that 
{\small	\begin{align}
		\E[u^2]
		&=
		\frac14\|\delta W_1\|_F^2
		+
		\frac14\|\bones^\top\delta W_1\|_2^2+
		\frac1{2\pi}
		\left(
		(\operatorname{tr}\delta W_1)^2
		+
		\operatorname{tr}((\delta W_1)^2)
		-
		2\|\operatorname{diag}(\delta W_1)\|_2^2
		\right),\nonumber\\
		\E[uv] &= \q^\top\E[\sigma(\x)\bones^\top D\sigma(\x)\delta W_1\x]
	=
	\q^\top\left(\frac14(\delta W_1)^\top\bones
	+
	\left(\frac14-\frac1{2\pi}\right)
	\operatorname{diag}(\delta W_1)
	+
	\frac{\operatorname{tr}(\delta W_1)}{2\pi}\bones\right),\nonumber\\
	\E[v^2]
	&=
	\q^\top\E[\sigma(\x)\sigma(\x)^\top]\q,\quad
	\text{and}
\quad
	\E[\sigma(\x)\sigma(\x)^\top]
	=
	\left(\frac12-\frac1{2\pi}\right)I_d
	+
	\frac1{2\pi}\bones\bones^\top,\nonumber
	\end{align}}
\par\noindent	giving the displayed expression.  The general bound follows from \eqref{second_diff_setting_gauss} by applying
the Cauchy--Schwarz inequality to the terms involving $\q$, together with
$\|\bones^\top\delta W_1\|_2^2\le d\|\delta W_1\|_F^2$,
$(\operatorname{tr}\delta W_1)^2\le d\|\delta W_1\|_F^2$, and
$\|\q\|_2^2\le
2d(\|\delta W_2\|_F^2+\|\delta W_3\|_2^2)$.
The tangent vector $\delta W=(0,d^{-1}\bones\bones^\top,0)$ shows that the order
$d^2$ is asymptotically sharp. On the restricted subspace, substituting the
constraints into \eqref{second_diff_setting_gauss} and using
$\operatorname{tr}((\delta W_1)^2)
-2\|\operatorname{diag}(\delta W_1)\|_2^2
\le\|\delta W_1\|_F^2$
gives the claimed bound.
\end{proof}

\subsection{Characterization of global minimizers for two Gaussians}

We now characterize the global minimizers of $N^{3,\circ}$ under \pref{setting:gaussian_mixture} with $d=h=k=2$. Since the Gaussian input law has full support and the network is
continuous, the Bayes identity may be imposed for every input. 

\begin{proposition} 
	\label{prop:k2_global_minima}
	Suppose that Bayes optimality is equivalent to
	\begin{align}
		\m^\top N^{3,\circ}(\x;W)
		&=
		\lambda\bones^\top\x
		\qquad
		\text{for all }\x\in\mathbb R^2,
		\label{eq:k2_bayes_identity}
	\end{align}
	where $\m=(-1,1)^\top$ and $\lambda>0$. Up to permutation and positive
	rescaling, the global minimizers are precisely
	those satisfying $W_1=-\m\bones^\top$,
	$\sigma(W_2)\in\mathrm{GL}_2(\mathbb R)$, and
	$W_3=-\frac{\lambda}{2}\m\m^\top\sigma(W_2)^{-1}
	+\bones\c^\top, \c\in\mathbb R^2$.
\end{proposition}

\begin{proof}
	Write the rows of $W_1$ as $\w_1^\top,\w_2^\top$ and the columns of
	$W_2$ as $\u_1,\u_2$. Suppose first that $W_1$ is invertible. On the
	nonempty open cone where only the first unit of the first layer is active,
	$N^{3,\circ}(\x;W)=W_3\sigma(\u_1)\w_1^\top\x$. Thus
	\eqref{eq:k2_bayes_identity} implies
	$\bigl(\m^\top W_3\sigma(\u_1)\bigr)\w_1^\top
	=\lambda\bones^\top$, and hence
	$\w_1\in\operatorname{span}(\bones)$. The analogous argument for the
	second unit gives $\w_2\in\operatorname{span}(\bones)$, contradicting
	invertibility. Since $W_1=0$ would make the network identically zero,
	$\operatorname{rank}W_1=1$. If $\h\in\ker W_1$, then
	$N^{3,\circ}(\x+t\h;W)=N^{3,\circ}(\x;W)$, so \eqref{eq:k2_bayes_identity} gives
	$\bones^\top\h=0$. Hence $\ker W_1=\ker\bones^\top$, and therefore
	$\w_i^\top=\alpha_i\bones^\top$ for some $\alpha_i\in\mathbb R$.
	The coefficients $\alpha_1$ and $\alpha_2$ are nonzero and have opposite
	signs, since otherwise the network vanishes on one of the half spaces
	$\bones^\top\x>0$ and $\bones^\top\x<0$. After permuting and positively rescaling the units, we may thus assume that
	$W_1=-\m\bones^\top$.
	Set $s=\bones^\top\x$. Evaluating \eqref{eq:k2_bayes_identity} on the
	half spaces $s>0$ and $s<0$ gives
	$\m^\top W_3\sigma(W_2)=\lambda(1,-1)$. In particular, both columns of
	$\sigma(W_2)$ are nonzero. If $\sigma(W_2)$ were singular, its
	nonnegative columns would be positive scalar multiples of one another,
	contradicting this identity. Hence
	$\sigma(W_2)\in\mathrm{GL}_2(\mathbb R)$. Multiplying the identity on the
	right by $\sigma(W_2)^{-1}$ and using
	$(1,-1)=-\m^\top$ and $\m^\top\m=2$ gives
	$\m^\top\bigl(W_3+\frac{\lambda}{2}\m\m^\top
	\sigma(W_2)^{-1}\bigr)=0$. Since
	$\ker\m^\top=\operatorname{span}(\bones)$, this is equivalent to
	$W_3=-\frac{\lambda}{2}\m\m^\top\sigma(W_2)^{-1}
	+\bones\c^\top$ for some $\c\in\mathbb R^2$. The converse follows by
	direct substitution.
\end{proof}
The argument also extends to other settings considered in this work, with
some additional technical details. We defer these extensions to
\cite{arjevani2026dsb_tf}.

\section{Concluding Remarks}\label{sec:concluding}
In this work, we proposed symmetry and its breaking as an organizing principle
for the spectral structure of the Hessian. The theoretical framework was developed in some generality, while our concrete applications focused on the empirical observations described in the introduction, largely following the settings studied in \cite{sagun2016eigenvalues,sagun2017empirical}. The formulations and proofs of the main results require constructing nearby symmetric reference configurations, identifying the invariance groups of the Hessian and the GN form produced by structural and distributional symmetries, handling nonsmoothness through generalized functions, decomposing the parameter space into isotypic components, determining how the resulting structure constrains the multiplicity and scale of the eigenvalues, and finally analyzing how SB transfers this structure back to the original models. In \pref{table:mixture_gaussians_resolution} and \pref{table:gaussian_regression_resolution},  we summarize how the SB framework explains some of these empirical observations and applies to several related settings. The concepts, methods, and results point toward further extensions,
briefly previewed in \pref{sec:dsb}. These extensions are being developed
in separate publications, to be incorporated later into broader treatments
\cite{arjevani2026dsb_tf,arjevani2027optsym}. We close with remarks on
optimization, invariant tensors, and equivariance, also
covered by these works.

\begin{table}[!htbp]
	\centering
	\small
	\renewcommand{\arraystretch}{1.}
	\setlength{\tabcolsep}{5pt}
	\begin{tabular}{@{}
			>{\raggedright\arraybackslash}p{0.23\textwidth}
			>{\raggedright\arraybackslash}p{0.74\textwidth}
			@{}}
		\toprule
		\multicolumn{2}{c}{\emph{\pref{setting:gaussian_mixture}}}\\
		\midrule
	Phenomenon 
		& Explanation within the framework\\
		\midrule
		Large Hessian kernels at convergence.
		&
		\begin{minipage}[t]{\linewidth}
			\begin{itemize}[itemsep=1pt,topsep=2pt,leftmargin=*]
				\item orbit directions account for
				\(O(h)\)
				\item transverse directions (for ReLU, but not GELU, see \pref{fig:lb_zeros})  account for~\(\Omega(h^2)\)  
			\end{itemize}
		\end{minipage}
\\
		\addlinespace
Clusters and outliers.
		&
The rich Hessian invariance group forces
eigenvalues with multiplicities of order $1$, $k$, and $k^2$,
unfolding into clusters and outliers under symmetry~breaking.  \\
		\addlinespace
Outlier count roughly matches the number of classes.
		&
A leading eigenvalue of multiplicity $k-1$ at the symmetric reference
splits under distributional SB into $k-1$ simple outliers
(see \pref{fig:original_vs_symmetric_mix}).\\
		\addlinespace
Structure of the Hessian ($k=2$).
&
Symmetry identifies local degrees of freedom that account for the large
kernel. On a codimension-one subspace, the Hessian lies in the associated commutant~algebra.\\
\addlinespace
Same qualitative picture across different variances ($k=2$).
&
The composite Hessian invariance group and the kernel dimension
do not depend on the variance. \\
		\addlinespace
Same qualitative picture for squared loss ($k=2$).
		&		
The composite Hessian invariances and local degrees of freedom mainly
arise from the architecture, not the choice of the loss. \\
\addlinespace
Near-zero bulk at random initialization ($k=2$).
&
The averaged Hessian has a rich invariance that forces an eigenvalue of
multiplicity $(h-1)^2$, whose scale is $O(h^{-1})$. This yields, with high
probability as $h\to\infty$, a near-zero bulk containing a $1-o(1)$ fraction
of the Hessian spectrum at a random  initialization.\\
\bottomrule
	\end{tabular}
\par\vspace{\abovecaptionskip}
\refstepcounter{table}
\label{table:mixture_gaussians_resolution}
{\centering\normalsize Table~\thetable\par}
\end{table}

\begin{table}[!htbp]
	\centering
	\small
	\renewcommand{\arraystretch}{1.}
	\setlength{\tabcolsep}{5pt}
	\begin{tabular}{@{}
			>{\raggedright\arraybackslash}p{0.23\textwidth}
			>{\raggedright\arraybackslash}p{0.74\textwidth}
			@{}}
		\toprule
	Setting
		& Explanation within the framework\\
		\midrule
Trivial  weight symmetries (\pref{fig:sq_evs}).
& 
The point-stabilizer approach generally fails to detect the relevant
Hessian invariances. The composite invariance mechanism makes them
explicit by combining structural symmetries of the network differential
with distributional symmetries. \\
\addlinespace
$L$-layer fully connected networks.
&
The general $N^L$ case was considered in \cite{arjevani2019principle}.
There, the lack of weight symmetry was also identified as a major
obstruction. The composite invariance mechanism overcomes it, detecting Hessian invariances that force multiplicities of orders $\Theta(1)$, $\Theta(d)$, and $\Theta(d^2)$. \\
\addlinespace
Residual connections
&
Residual connections can remove the transverse flat
directions, allowing weight symmetries to capture the relevant Hessian
invariances \cite{arjevani2025three}. \\
\addlinespace
Diagonal networks
&
For the diagonal linear networks (see \cite{pesme2023saddle}),
$N^{3,\odot}(\x;\a,\b,\c)=\bones^\top (\a \odot \b \odot \c \odot \x)$ with
\(y(\x)=\bones^\top\x\), the global
minima are precisely \(\a\odot \b\odot \c=\bones\), where
$\dim\ker\nabla_W^2\mathcal L(W^\star)=2d$ and 
$T_{W^\star}(\mathbb R^d)^3/\ker\nabla_W^2\mathcal L(W^\star)
\cong
\ft\oplus\fs_{d-1,1}$
(and see algorithmic symmetries in~\pref{sec:dynamics})
\\
\bottomrule
	\end{tabular}
\par\vspace{\abovecaptionskip}
\refstepcounter{table}
\label{table:gaussian_regression_resolution}
{\centering\normalsize Table~\thetable\par}
\end{table}

\subsection{Optimization}
\label{sec:dynamics}
The analysis in this work mainly concerns static properties of the
loss landscape. SB also plays a role in optimization  and
generalization. We briefly review
several such connections here.

\paragraph{Conservation laws}
Lie group actions can give vectors in the Hessian kernel, which we
referred to as orbit flat directions. Beyond this local effect, they
also have global consequences, including conservation laws for
\emph{gradient flow}, in the spirit of Noether's theorem. These laws
give rise, for example, to the auto-balancing property
\cite{du2018algorithmic}, as established in
\cite[Section 3]{arjevani2019principle}. More general conservation laws
are studied in \cite{kunin2020neural,marcotte2024keep}. While conservation
laws clarify some aspects of optimization, they do not explain
quantities that \emph{do} change substantially along the~trajectory.

\paragraph{Gradient descent takes place in a tiny subspace.}
It has been observed \cite{gur2018gradient,ghorbani2019investigation} that
the gradient concentrates in low-dimensional subspaces early in training.
These subspaces have been empirically identified with the leading eigenspaces
of the Hessian. Attempts to account for this concentration
\cite{gur2018gradient,ghorbani2019investigation} have remained limited. In the
three-layer networks with residual connections \cite{arjevani2025three},
the SB framework gives this observation a natural interpretation. There, gradient
dynamics converge to fixed-point subspaces and evolve predominantly tangentially
to them. The isotypic decomposition organizes the remaining~directions.

\paragraph{Algorithmic symmetries.}
The above phenomenon may be examined more finely through gradient
equivariance: modulo finite-precision effects, the symmetry 
of successive gradient descent iterates can only increase
\cite{arjevani2019principle}, suggesting a form of \emph{incremental
learning} (see \cite{kalimeris2019sgd,pesme2023saddle}).  Other gradient-based algorithms, such as Adam, may exhibit different
equivariances and biases depending on the structure of their update
rules, as can already be observed in
the settings considered in this work. Moreover, during optimization, a trajectory may pass near several
symmetric reference critical points, leaving traces of their properties
in the spectrum, with effects of
structural symmetries generally more visible than those of distributional~symmetries.

\subsection{Symmetry and spectral structure of invariant tensors}\label{sec:inv_tensors}

\newcommand{\mtpoly}{p}

Tensor eigenpairs are computable in principle, but existing methods rapidly become impractical as the ambient dimension and tensor order grow. In the same way that symmetry imposes strong structure on Hessian spectra, it also constrains tensor eigenspectra. Consider the homogenized Motzkin polynomial
\begin{align}
	\mtpoly(x,y,z)
	=
	z^6 + x^4y^2 + x^2y^4 - 3x^2y^2z^2 .
\end{align}
The polynomial is invariant under coordinate sign changes and the swap \(x \leftrightarrow y\), and these symmetries already determine much of the structure of its tensor $Z$-eigenpairs (see \cite{qi2005eigenvalues}) as illustrated~below.

{\footnotesize\begin{align*}
	\begin{array}{c|ccccc}
		\text{Eigenvalue}
		&
		\mathbf 0
		&
		\mathbf 0
		&
		\mathbf 1
		&
		\mathbf{1/64}
		&
		\mathbf{1/4}
		\\ \hline
		(x,y,z)
		&
		\left(\frac1{\sqrt3},\frac1{\sqrt3},\frac1{\sqrt3}\right)
		&
		(1,0,0)
		&
		(0,0,1)
		&
		\left(\sqrt{(3+\sqrt6)/8},\sqrt{(3-\sqrt6)/8},1/2\right)
		&
		\left(\frac1{\sqrt2},\frac1{\sqrt2},0\right)
		\\
		\text{Multiplicity}
		&1&5&1&1&1\\
		\text{Isotropy group size}
		&2&4&8&1&4\\
		\text{Orbit size}
		&8&4&2&16&4
	\end{array}.
\end{align*}}
\par\noindent Tensor eigenvalue equations may be
interpreted as tangency equations of the associated homogeneous polynomial
at the origin, and this viewpoint extends to more general notions of tensor
spectra. Symmetric eigenvectors can be computed from polynomial systems of much
smaller dimension, sometimes independently of the ambient tensor dimension. Moreover, once a single generic tensor in an invariant space is solved, linear {homotopy} continuation transfers the computation to other generic tensors in the same class while tracing only paths compatible with the symmetry.  Relaxation methods provide tractable bounds on extremal eigenvalues when exact computation is infeasible. This framework applies broadly to tensors invariant under actions of the symmetric and hyperoctahedral groups or affine and projective groups over finite fields, as well as to tensors associated with hypergraphs,
including the examples in \cite[Example 6.4]{li2013z} and
\cite[Example 2]{xie2013z}, to mention only a few.

More formally, let \(V\) be a finite-dimensional complex vector space, let $m,m'$ be positive integers, 
and let $\cA\in(V^*)^{\otimes(m-1)}\otimes V$ and
$\cB\in(V^*)^{\otimes(m'-1)}\otimes V$. Consider $\v\neq0$ such that
\begin{align}
	\cA \v^{m-1}
	=
	\lambda \cB \v^{m'-1}.
\end{align}
If the solution scheme on $\v\ne0$ is of pure dimension one, its
irreducible components correspond to equivalence classes of eigenpairs
under $(\v,\lambda)\sim(t\v,t^{m-m'}\lambda)$ for $t\ne0$. The eigencount is the
number of these classes counted with their scheme-theoretic multiplicities,
equivalently, the number of isolated solutions counted with multiplicity
after imposing a generic affine normalization on $\v$. If higher-dimensional families of solutions occur, the eigencount is declared infinite. More generally, if \(E\subseteq (V^*)^{\otimes m-1}\otimes V\) is a linear subspace and there exists a nonempty Zariski-open subset of \(E\) on which the eigencount is finite and constant, then this value is called the generic eigencount of \(E\). The scheme-theoretic formulation is required here since multiplicities are encoded intrinsically by the nonreduced structure of the solution scheme, while the equivalence with the augmented formulation follows from standard intersection-theoretic arguments, see \cite{eisenbud20163264}.

\newcommand{\pN}{\overline{N}}
\subsection{Equivariance}\label{sec:equivariance}
At the global minima considered in both \pref{setting:gaussian_mixture} and \pref{setting:gaussian}, the predictor is equivariant under the
symmetries of the data distribution in the following sense. Write $\pN$ for the predictor, equal to $N$ under squared loss and to $\operatorname{softmax}\circ N$ under cross-entropy. If
$\mathsf t\in\dist(\dd)$ acts by
$\mathsf t(\x,\y)=(\mathsf t_\x\x,\mathsf t_\y\y)$,~then 
\begin{align}
	\pN(\mathsf t_\x\x;W)=\mathsf t_\y \pN(\x;W),
	\qquad\text{for all }\x.
\end{align}
The broader interest in equivariance is motivated
by the fact that encoding known symmetries can, in particular, improve generalization while retaining
expressivity within the relevant equivariant function class (see, e.g.,
\cite{elesedy2021provably,maron2019universality}). A large literature is devoted to architectures that enforce equivariance 
\cite{cohen2016group,zaheer2017deep,navon2023equivariant,satorras2021n,agrawal2022classification}. The framework introduced here applies unchanged to equivariant networks,
as illustrated by the GNN example in \pref{sec:dsb}, and accommodates
restrictions of the parameter space as in \pref{sec:layerwise}.

\appendix

\section*{Selected background references}\label{sec:reading}

The following guide lists some references on the mathematical background
relevant to this work and is not intended to be exhaustive.

\section{Algebra}

\subsection{Bilinear forms}
The second differential of a loss is a symmetric bilinear form, while
a choice of inner product on parameter space represents it as a
self-adjoint operator. The spectrum and its eigenvalue multiplicities
may vary with this choice, whereas the form and its invariance group
remain fixed (\pref{rem:spectral_refraction}). The radical of the form, which is
the kernel of every such operator, is likewise independent of the inner
product. The null cone consists of directions on which the quadratic form
vanishes and is strictly larger than the radical when the form is indefinite,
a distinction used when comparing flat directions with directions in the
Hessian kernel. For background on inner-product spaces and symmetric
bilinear, sesquilinear, and quadratic forms, see
\citet{greub1975linear,bourbaki2007algebre}.

\subsection{Group theory and permutation groups}
The transformations of networks and data laws considered in this paper lead to
group actions, stabilizers, direct and semidirect products, and permutation
groups.  Subgroups of direct products help organize the possible stabilizers.
 \citet{dummit2004abstract} provide the general algebraic background.
\citet{cameron1999permutation} and \citet{DixonMortimer1996} treat permutation groups and their subgroup
structure. For Goursat's lemma and subgroups of direct products, see
\citet[Section~1.6]{schmidt1994subgroup}. 

\subsection{Representation theory}
When a Hessian invariance group acts orthogonally on parameter perturbations,
isotypic components organize invariant Hessian blocks and constrain
eigenvalue multiplicities.
Some of the representation theory needed here is developed in the main text.
\citet{serre1977linear} gives a concise introduction to finite
group representations, while \citet{fulton1991representation} provide a
broader treatment, including irreducible representations of symmetric
groups. 

\section{Geometry}

\subsection{Differential geometry and topology}
Smooth group orbits supply manifold directions along which an invariant
objective is constant.  We use their tangent spaces to identify some flat
directions and use constrained Hessians when parameters lie on a sphere.
\citet{warner1983foundations} covers differentiable manifolds and Lie
groups.  \citet{lee2018introduction,boumal2023introduction} provide background
on Riemannian metrics, submanifold geometry, and Hessians of functions
restricted to smooth manifolds.  For the
broader relation between group actions and symmetry breaking, see
\citet{golubitsky1988singularities}. Differential topology also enters
the analysis of tangency arcs through genericity and stability results
\cite{arjevani2024sc}.

\subsection{Real algebraic and tame geometry}
ReLU networks and their activation regions admit semialgebraic descriptions.
The more general setting of definable networks uses an o-minimal structure to
obtain finite stratifications on which the relevant maps are $C^2$.  This
geometry supports the treatment of activation interfaces.  \citet{benedetti1991real} introduce real algebraic and semialgebraic sets.
\citet{van1998tame} covers o-minimal~geometry.  

\subsection{Variational geometry}
At a point of a level set, its tangent cone may contain directions not
generated by a group orbit and need not be a linear space.  This is the
geometric language behind the paper's distinction between orbit and
transverse flat directions.  For tangent and normal constructions in
variational and nonsmooth analysis, see
\citet{rockafellar1998variational,clarke1998nonsmooth}.

\subsection{Geometric measure theory}
For the locally DC definable networks considered in this paper,
the singular part of a network's second distributional derivative
is carried by lower-dimensional
activation interfaces.  Hausdorff measure, the coarea formula, and the jump
part of the derivative of a function of bounded variation explain the
interface measure in the resulting Hessian formulas.  \citet{ambrosio2000functions} provide the principal textbook treatment of
these tools.  For locally finite Hausdorff measure in the definable setting,
see \cite{fornasiero2012hausdorff}.

\subsection{Algebraic geometry, schemes, and intersection theory}
The concluding discussion of invariant tensors considers polynomial
eigenpair equations up to scaling and counts their solutions with
scheme-theoretic multiplicity. Algebraic geometry also enters the
analysis of tangency arcs. After reduction to finite jets and
complexification, questions of
genericity are formulated in terms of polynomial systems.
\citet{mumford1976algebraic}
provides background on projective varieties, and \citet{eisenbud20163264} develop the intersection-theoretic viewpoint needed
for multiplicity counts.

\section{Analysis}

\subsection{Measure theory}
The data law is treated as a Radon probability measure, and the expected loss
is obtained by integrating the pointwise loss over data.  Product measures,
pushforwards, and Fubini's theorem give the basic language for this passage
and for averaging over initialization.  \citet{folland1999real}
develops these measure-theoretic foundations. 

\subsection{Functional analysis and distribution theory}
Treating the pointwise loss as a distribution allows differentiation
in the parameter variables where its classical Hessian does not exist,
even when the expected loss has a classical Hessian. This formulation
permits differentiation under expectation and supports our direct
analysis of invariances.  For distributions and their operations, see \citet{Hormander1983}. \citet{evans2015measure} cover the Sobolev and BV
regularity used in the analysis of network derivatives.

\subsection{Matrix spectral theory}
After a symmetric reference has identified invariant subspaces, the min--max
principle and Cauchy interlacing compare its Hessian spectrum with spectra
obtained by restriction or by adding parameter blocks.  Inertia bounds and
eigenvalue perturbation then help describe how large multiplicities may
split into clusters.  \citet{horn2012matrix} give the
underlying self-adjoint matrix theory.

\section{Statistical Learning}

The objectives studied here are expected losses under specified
data laws, including empirical laws.
For squared loss and cross-entropy, their Bayes predictors are determined by
conditional means and conditional class probabilities, respectively.  This
also explains the well-specification condition used to identify a Hessian
with its GN form.  For background on statistical risk and Bayes prediction,
see \citet{hastie2009elements}.

\vskip 0.2in
\bibliography{bib}

\end{document}